%% file: thesis.tex
\documentclass[12pt,times,mathptm,letterpaper]{report}          
\usepackage[T1]{fontenc}

\usepackage[intlimits]{amsmath}
\usepackage{amsfonts,amssymb,amsthm}
\usepackage{algorithm,algorithmic}
\usepackage[algo2e,ruled,vlined]{algorithm2e}
\DeclareSymbolFontAlphabet{\mathbb}{AMSb}
\usepackage{natbib}
\usepackage{apalike}
\usepackage{float}
\usepackage[bf]{caption}
\usepackage{fancyhdr}
\usepackage{fancybox}
\usepackage{fontawesome}
\usepackage{ifthen}
\usepackage{bu_ece_thesis}
\usepackage{url}
\usepackage{lscape,afterpage}
\usepackage{xspace}
\usepackage{epstopdf}
\usepackage{xcolor}         
\usepackage[hidelinks]{hyperref}
\usepackage{wrapfig}
\usepackage{subcaption}
\usepackage{enumitem}
\usepackage{multirow}
\usepackage{makecell}
\usepackage{mathtools}
\mathtoolsset{showonlyrefs}

\allowdisplaybreaks

\usepackage{graphicx}
\usepackage{appendix}

\newcommand{\PartialDerivative}[2]{\frac{\partial F}{\partial x_{#2}}(\bx_{#1})}
\newcommand{\PartialDerivativeGeneral}[2]{\frac{\partial F}{\partial x_{#2}}(#1)}

\newcommand{\StocGradient}[2]{g_{#1, #2}}

\newcommand{\PNorm}[1]{\|#1\|_2}
\newcommand{\PNormDimension}{\sqrt{d}}

\counterwithin{algorithm}{chapter}

\newtheorem{thm}{Theorem}[chapter]
\newtheorem{lemma}[thm]{Lemma}

\newtheorem{assumption}{Assumption}[chapter]
\newtheorem{corollary}[thm]{Corollary}

\newlength\myboxwidth
\input{symbols}

\begin{document}

\include{0_Prelim/prelim}        
\cleardoublepage

\doublespacing

\include{1_Intro/intro}
\cleardoublepage


\include{3_Adapt_Noise/adapt_noise}
\cleardoublepage

\include{4_Adapt_Scale/adapt_scale}
\cleardoublepage

\include{5_Adapt_Smoothness/adapt_smoothness}
\cleardoublepage

\include{6_Conclusions/conclusions}
\cleardoublepage

\begin{appendices}
\include{Appendix/Appendix}
\end{appendices}
\newpage
\singlespace
\bibliographystyle{apalike}

\bibliography{thesis}
\cleardoublepage

\include{0_Prelim/cv} 

\end{document}

%% file: symbols.tex
\renewcommand{\Pr}{\field{P}}

\newcommand{\ba}{\boldsymbol{a}}
\newcommand{\bb}{\boldsymbol{b}}

\newcommand{\bd}{\boldsymbol{d}}

\newcommand{\bm}{\boldsymbol{m}}
\newcommand{\bp}{\boldsymbol{p}}

\newcommand{\bg}{\boldsymbol{g}}
\newcommand{\bx}{\boldsymbol{x}}
\newcommand{\bu}{\boldsymbol{u}}
\newcommand{\by}{\boldsymbol{y}}

\newcommand{\bv}{\boldsymbol{v}}

\newcommand{\boldeta}{\boldsymbol{\eta}}

\newcommand{\boldepsilon}{\boldsymbol{\epsilon}}

\newcommand{\argmin}{\mathop{\mathrm{argmin}}}

\newcommand{\field}[1]{\mathbb{#1}}

\newcommand{\R}{\field{R}}

\newcommand{\E}{\field{E}}



%% file: 0_Prelim/prelim.tex

\title{Adaptive Strategies in Non-convex Optimization}

\author{Zhenxun Zhuang}

\degree=2

\prevdegrees{B.Eng., University of Science and Technology of China, 2016}

\department{Department of Computer Science}

\defenseyear{2022}
\degreeyear{2022}

\reader{First Reader}{Francesco Orabona, Ph.D.}{Associate Professor of Electrical and Computer Engineering\\Associate Professor of Computer Science\\Associate Professor of Systems Engineering\\Associate Professor of Computing and Data Sciences}
\reader{Second Reader}{Bryan A. Plummer, Ph.D.}{Assistant Professor of Computer Science}
\reader{Third Reader}{Ioannis Ch. Paschalidis, Ph.D.}{Distinguished Professor of Engineering\\Professor of Electrical and Computer Engineering\\Professor of Systems Engineering\\Professor of Biomedical Engineering\\Professor of Computing and Data Sciences}

\numadvisors=1
\majorprof{Francesco Orabona}{\\Associate Professor of Electrical and Computer Engineering\\Associate Professor of Computer Science\\Associate Professor of Systems Engineering\\Associate Professor of Computing and Data Sciences}




\maketitle
\cleardoublepage

\copyrightpage
\cleardoublepage

\approvalpagewithcomment
\cleardoublepage

\newpage
\input{0_Prelim/quote}
\cleardoublepage

\newpage
\section*{\centerline{Acknowledgments}}
\input{0_Prelim/ack}
\cleardoublepage


\begin{abstractpage}
\input{0_Prelim/abs}
\end{abstractpage}
\cleardoublepage


\tableofcontents
\cleardoublepage

\newpage
\listoftables
\cleardoublepage
%
\newpage
\listoffigures
\cleardoublepage
%
\chapter*{List of Abbreviations}
%

\begin{center}
 \begin{tabular}{lll}
    \hspace*{2em} & \hspace*{0.38in} & \hspace*{4.5in} \\
    ACM & \dotfill & Association for Computing Machinery\\
    BN & \dotfill & Batch Normalization\\
    CNN & \dotfill & Convolutional Neural Network\\
    CV & \dotfill & Computer Vision\\
    DNN & \dotfill & Deep Neural Network\\
    $\E[\cdot]$ & \dotfill &  Expectation\\
    FTRL & \dotfill & Follow The Regularized Leader\\
    GD & \dotfill & Gradient Descent\\
    ICLR & \dotfill &
    International Conference on Learning Representations\\
    ICML & \dotfill & International Conference on Machine Learning\\
    IEEE & \dotfill & Institute of Electrical and Electronics Engineers\\
    JMLR & \dotfill & Journal of Machine Learning Research\\
    LSTM & \dotfill & Long Short-Term Memory\\
    ML & \dotfill & Machine Learning\\
    NeurIPS & \dotfill & Neural Information Processing Systems\\
    NLP & \dotfill & Natural Language Processing\\
    PL & \dotfill & Polyak-\L{}ojasiewicz\\
    PMLR & \dotfill & Proceedings of Machine Learning Research\\
    $\Pr[\cdot]$ & \dotfill &  Probability\\
    $\mathbb{R}^{d}$  & \dotfill &  Real coordinate space of dimension $d$ \\
    r.h.s. & \dotfill & right hand side\\
    SGD  & \dotfill & Stochastic Gradient Descent \\
    SOTA & \dotfill & State-Of-The-Art\\
    w.r.t. & \dotfill & with respect to\\
 \end{tabular}
\end{center}
\cleardoublepage


\newpage
\endofprelim

%% file: 0_Prelim/quote.tex
\phantom{.}

\begin{singlespace}
\begin{quote}
    \textit{God, give me grace to accept with serenity\\
    the things that cannot be changed,\\
    Courage to change the things\\
    which should be changed,\\
    and the wisdom to distinguish\\
    the one from the other.\\
    \\
    Living one day at a time,\\
    Enjoying one moment at a time,\\
    Accepting hardship as a pathway to peace,\\
    Taking, as Jesus did,\\
    This sinful world as it is,\\
    Not as I would have it.}\\
    \phantom{.}\hfill{Reinhold Niebuhr}
\end{quote}
\vspace{3em}
\begin{quote}
    \textit{She was still too young to know that life never gives anything for nothing, and that a price is always exacted for what fate bestows.}\\
    \phantom{.}\hfill{Stefan Zweig}
\end{quote}
\end{singlespace}

%

%% file: 0_Prelim/ack.tex
The journey toward my Ph.D. has eventually come to an end. It is full of obstacles and setbacks, yet is also filled with joy and achievements. So many people have helped me along the way and I am forever indebted to them.

First and foremost, I would like to express my deepest gratitude to my adviser Francesco Orabona without whom this adventure would be impossible. I knew literally nothing about research in machine learning upon entering my Ph.D. and it is him who tirelessly and patiently taught me right from wrong and trained me to build a full skill-set on being a researcher. He will always be my role model for his passion for life and his rigor towards work.

I also thank all my collaborators: Ashok Cutkosky, Xiaoyu Li, Mingrui Liu,  Songtao Lu, Yunlong Wang, Kezi Yu, and a lot more. I will always remember those inspiring discussions on new problems and those sleepless nights catching deadlines.

Special thanks go to my committee members Alina Ene, Yannis Paschalidis, and Bryan Plummer for their service and their attention to my work. I also thank all the people in the department and the university, some of whom I became friends with. Altogether you created a really enjoyable atmosphere to be working in.

Last but not least, I thank my parents Shaobing and Xiuyue for their unconditional support and my love Xiaoqing for fighting with me till the end.

%% file: 0_Prelim/abs.tex
Modern applications in machine learning have seen more and more usage of non-convex formulations in that they can often better capture the problem structure. One prominent example is the Deep Neural Networks which have achieved innumerable successes in various fields including computer vision and natural language processing. However, optimizing a non-convex problem presents much greater difficulties compared with convex ones. A vastly popular optimizer used for such scenarios is Stochastic Gradient Descent (SGD), but its performance depends crucially on the choice of its step sizes. Tuning of step sizes is notoriously laborious and the optimal choice can vary drastically across different problems. To save the labor of tuning, adaptive algorithms come to the rescue: An algorithm is said to be adaptive to a certain parameter (of the problem) if it does not need a priori knowledge of such parameter but performs competitively to those that know it.

This dissertation presents our work on adaptive algorithms in following scenarios:
\begin{enumerate}
\item In the stochastic optimization setting, we only receive stochastic gradients and the level of noise in evaluating them greatly affects the convergence rate. Tuning is typically required when without prior knowledge of the noise scale in order to achieve the optimal rate. Considering this, we designed and analyzed noise-adaptive algorithms that can automatically ensure (near)-optimal rates under different noise scales without knowing it.
\item In training deep neural networks, the scales of gradient magnitudes in each coordinate can scatter across a very wide range unless normalization techniques, like BatchNorm, are employed. In such situations, algorithms not addressing this problem of gradient scales can behave very poorly. To mitigate this, we formally established the advantage of scale-free algorithms that adapt to the gradient scales and presented its real benefits in empirical experiments.
\item Traditional analyses in non-convex optimization typically rely on the smoothness assumption. Yet, this condition does not capture the properties of some deep learning objective functions, including the ones involving Long Short-Term Memory (LSTM) networks and Transformers. Instead, they satisfy a much more relaxed condition, with potentially unbounded smoothness. Under this condition, we show that a generalized SignSGD (update using only the sign of each coordinate of the stochastic gradient vector when running SGD) algorithm can theoretically match the best-known convergence rates obtained by SGD with gradient clipping but does not need explicit clipping at all, and it can empirically match the performance of Adam and beat others. Moreover, it can also be made to automatically adapt to the unknown relaxed smoothness.
\end{enumerate}

%% file: 1_Intro/intro.tex
\chapter{Introduction}
\label{chapter:Introduction}
\thispagestyle{myheadings}

Recent decades have witnessed a surge of interest in the field of \emph{machine learning}~\citep{jordan2015machine}. As the name suggests, this discipline strives to build machines that can improve automatically through experience obtained by learning from data. Apart from the efforts of researchers in developing novel theories and algorithms, its rapid advancement is also partly attributed to the accumulation of vast datasets and the explosive growth of the semiconductor industry which resulted in efficient and low-cost computations.

Most machine learning tasks can be modeled as the mathematical optimization problem
\begin{equation}
\label{eq:opt}
\min_{\bx\in\mathcal{X}} F(\bx),
\end{equation}
where we assume the domain $\mathcal{X}\subseteq\R^d$ and call $F: \R^d\rightarrow \R$ the objective function. We also make the following assumption on $F$.
\begin{assumption}
\label{asp:objective_func}
$F$ is differentiable and bounded from below by $F^*$.
\end{assumption}

\section{Convex and Non-convex Optimization}
\label{sec:intro_convex_nonconvex}
By assuming different conditions on $F$ and $\mathcal{X}$, we restrict our focus to different families or classes of optimization problems.

A classic and important class of optimization problems is convex optimization problems~\citep{Rockafellar76, BoydV04} in which the domain $\mathcal{X}$ is a convex set and the objective function $F$ is a convex function.

We say a set $\mathcal{X}$ is \emph{convex} if the line segment between any two points in $\mathcal{X}$ lies entirely in $\mathcal{X}$, i.e., for any $\bx_1, \bx_2\in \mathcal{X}$ and any $\theta$ with $0\le\theta\le1$, we have
\begin{equation}
\label{eq:convex_set}
    \theta\bx_1 + (1-\theta)\bx_2 \in \mathcal{X}~.
\end{equation}

We say a function $F:\mathcal{X}\rightarrow\R$ is \emph{convex} if its domain $\mathcal{X}$ is a convex set and for any $\bx_1, \bx_2\in\mathcal{X}$ and any $\theta$ with $0\le\theta\le1$, we have
\begin{equation}
\label{eq:convex}
    F(\theta\bx_1 + (1-\theta)\bx_2) \le \theta F(\bx_1) + (1-\theta)F(\bx_2)~.
\end{equation}
Moreover, when $F$ is differentiable,~\eqref{eq:convex} is equivalent to that for any $\bx_1, \bx_2\in\mathcal{X}$ it holds that
\begin{equation}
\label{eq:convex_grad}
    F(\bx_1) + \langle\nabla F(\bx_1), \bx_2 - \bx_1\rangle \le F(\bx_2),
\end{equation}
where $\nabla F(\bx)$ denotes the gradient of $F$ at $\bx$ and $\langle\cdot,\cdot\rangle$ denotes the inner product of two vectors.

Additionally, a function $F:\mathcal{X}\rightarrow\R$ is said to be $\mu$-\emph{strongly convex}~\citep[Theorem 2.1.9]{Nesterov04} if for any $\bx_1, \bx_2\in\mathcal{X}$ we have
\begin{equation}
\label{eq:strongly_convex}
    F(\theta\bx_1 + (1-\theta)\bx_2) \le \theta F(\bx_1) + (1-\theta)F(\bx_2) - \frac{\mu}{2}\theta(1-\theta)\|\bx_2 - \bx_1\|_2^2~.
\end{equation}
where $\|\cdot\|_2$ denotes the $\ell_2$ norm.

Many classical machine learning methods are inspired by or can be reduced to a convex optimization problem. Notable examples include least squeares~\citep{gauss1995theory}, logistic regression~\citep{verhulst1845resherches,cramer2002origin}, and support vector machines~\citep{cortes1995support}.

One nice and important property of convexity is that any stationary point with zero gradient is bound to be a global minimum point which is evident from the definition~\eqref{eq:convex_grad}. Here we say $\bx$ is a \emph{stationary point} if $\nabla F(\bx)=\boldsymbol{0}$ and say $\bx$ is a \emph{global minimum point} if $F(\bx)\le F(\by)$ for any $\by\in\mathcal{X}$. We will also call $\bx$ a \emph{local minimum point} if for some $\epsilon>0$ and some distance measure $dist(\cdot,\cdot)$ that $F(\bx) \le F(\by)$ for any $\by\in\mathcal{X}$ with $dist(\bx, \by)\le\epsilon$, and similarly for a \emph{local maximum point}. We can also show that (proof in Appendix~\ref{sec:omit_proofs}):
\begin{lemma}
\label{lem:no_local_minimum_convex}
Let $F:\mathcal{X}\rightarrow\R$ be a convex function, then any local minimum point of $F$ in $\mathcal{X}$ is also a global minimum point.
\end{lemma}

\begin{figure}[t]
    \centering
    \begin{subfigure}[b]{0.48\textwidth}
        \centering
        \includegraphics[width=\textwidth]{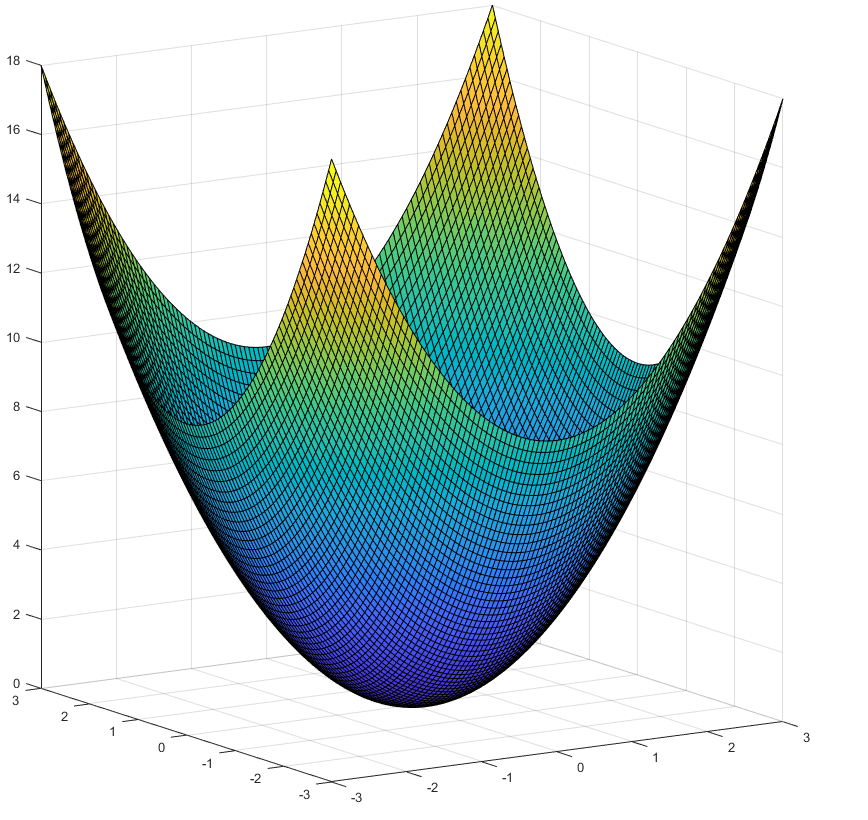}
        \caption{A convex function.}
        \label{fig:convex_function}
    \end{subfigure}
    \hfill
    \begin{subfigure}[b]{0.48\textwidth}
        \centering
        \includegraphics[width=\textwidth]{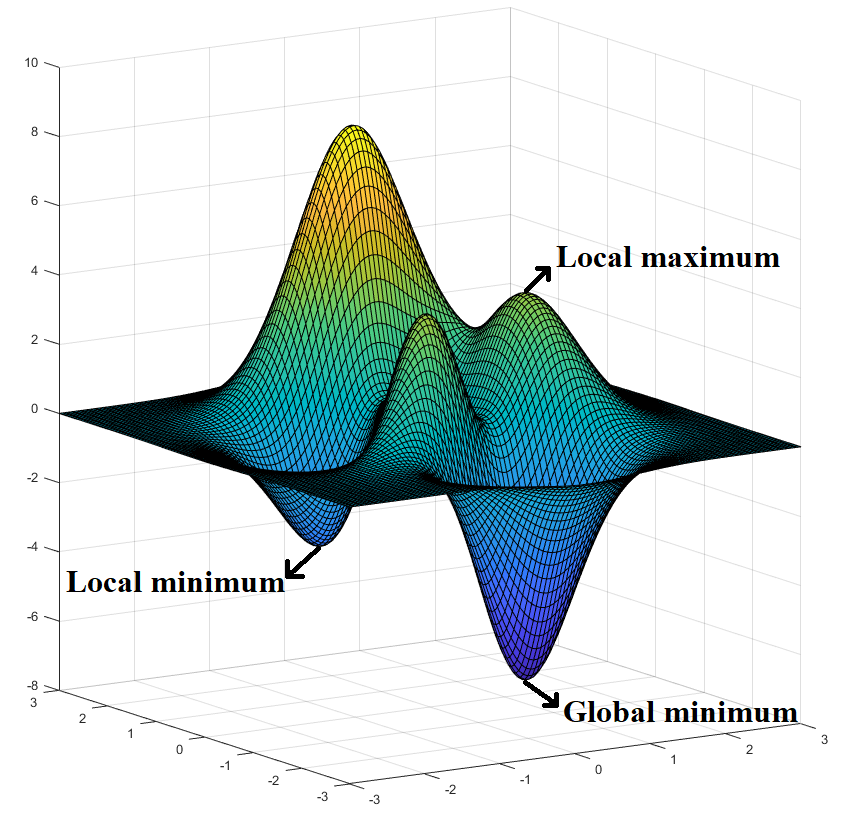}
        \caption{A non-convex function.}
        \label{fig:non_convex_function}
    \end{subfigure}
    \caption{A convex function vs.~a non-convex function.}
    \label{fig:convex_nonconvex_function}
\end{figure}

However, when the convexity condition no longer holds, the optimization problem becomes fundamentally harder and we enter into the wild world of \emph{non-convex optimization}. The above nice property of convex functions is not true anymore, and the global minimum points could be hidden among an infinite number of local minimum points. As~\citet{Rockafellar93Lagrange} puts it: ``the great watershed in optimization isn’t between linearity and nonlinearity, but convexity and nonconvexity." As an example, we plot in Figure~\ref{fig:convex_nonconvex_function} a convex function and a non-convex function, from which it can be immediately seen that the non-convex function is much more complex than the convex one.

Indeed,~\citet{NemirovskyY83} showed that the information-based complexity of convex optimization problems is far lower than that of general nonconvex optimization problems. In fact, they showed in Section 1.6 that finding a global optimum for a general non-convex problem is NP-hard. The situation is made worse that even approximately solving a range of non-convex problems is NP-hard~\citep{MekaJCD08}. Thus, in this dissertation, assuming $F$ to be differentiable (Assumption~\ref{asp:objective_func}), following an established line of research, we settle for finding (first-order) $\epsilon$-stationary points, a.k.a.~critical points, where the gradient goes to zero, namely some $\bx\in\mathcal{X}$ with $\|\nabla F(\bx)\| \le \epsilon$.

Despite the added difficulty in optimizing, modern applications in machine learning have seen more and more usage of non-convex formulations due to their better capability to capture the problem structure. One of the most prominent examples is the Deep Neural Networks, which have scored enormous successes in a vast range of tasks including but not limited to computer vision~\citep{KrizhevskySH12}, language translation~\citep{VaswaniMLSGLJ19}, speech recognition~\citep{ZhangCJ17Speech}, and recommendation systems~\citep{ZhengNY17Recommend}. Consequently, progresses in this field is much awaited.

It is worth stressing that non-convex functions are not characterized by a particular property, but rather by the lack of a specific property: convexity. In this sense, trying to carry out any meaningful analyses on the entire class of non-convex functions is hopeless. Therefore, we typically focus on specific classes of non-convex problems by adding additional assumptions. Ideally, the assumptions we use shall balance the trade-off of \emph{approximately} model many interesting machine learning problems while allowing us to restrict the class of non-convex functions to particular subsets where we can unearth interesting behaviors. One common assumption people use is the smoothness one presented below.

\begin{assumption}
\label{asp:smooth}
A differentiable function $F:\mathcal{X}\rightarrow\R$ is called $L$-smooth, if for all $\bx_1,\bx_2\in\mathcal{X}$ we have $\|\nabla F(\bx_1)-\nabla F(\bx_2)\|_2 \leq L \|\bx_1-\bx_2\|_2$ w.r.t.~the $\ell_2$ norm.
Note that this directly implies that~\citep[Lemma 1.2.3]{Nesterov04}, for all $\bx,\by \in \R^d$ we have
\begin{equation}
\label{eq:smooth}
\left|F(\bx_2)-F(\bx_1)-\langle \nabla F(\bx_1), \bx_2-\bx_1\rangle\right|
\leq \frac{L}{2}\|\bx_2-\bx_1\|_2^2~.
\end{equation}
\end{assumption}

More in detail, the above smoothness assumption is considered ``weak'' and is ubiquitous in analyses of optimization algorithms in the non-convex setting. Admittedly, in many neural networks, it is only approximately true because ReLUs activation functions are non-smooth. However, if the number of training points is large enough, it is a good approximation of the loss landscape.

\section{Black-box Oracles and Convergence Rates}
To solve an optimization problem, we employ an \emph{algorithm}, a finite sequence of rigorous instructions which, given information about the problem, will eventually output a solution to the problem. We are mainly interested in \emph{iterative algorithms} which proceeds in discrete steps $t=1,2,\ldots$ and in each step operates on the results of previous steps resulting a sequence of updates $\bx_1, \bx_2,\ldots$. Of course, the more information we are given, the more efficient we can solve the problem, with one extreme of revealing everything and the other extreme of giving nothing. Thus, it makes sense to restrict our attention to certain classes of algorithms by specifying which information is accessible, as only then can we compare one with another and pick the best candidate.

The algorithms we focus on in this dissertation are those with access to a deterministic first-order black-box oracle (FO)~\citep{Nesterov04}:
\begin{equation*}
Given\ \bx,\ an\ FO\ returns\ [F(\bx), \nabla F(\bx)],
\end{equation*}
which we call the \emph{deterministic} setting, or a stochastic first-order black-box oracle (SFO):
\begin{align*}
&Given\ \bx,\ an\ SFO\ returns\ [f(\bx, \xi), \nabla f(\bx, \xi)]\\
&with\ \E_{\xi} f(\bx, \xi) = F(\bx),\ \E_{\xi} \nabla f(\bx, \xi) = \nabla F(\bx),
\end{align*}
where $\xi$ is a vector drawn by the oracle from an arbitrary set and we call it the \emph{stochastic} setting.

In the stochastic setting, we will focus on the optimization problem
\begin{equation}
\label{eq:stoc_opt_obj}
    \min_{\bx\in\R^d} F(x) := \E_{\xi\sim\mathcal{D}}[f(\bx, \xi)],
\end{equation}
where $\xi$ is a random variable representing a randomly selected data sample or random noise following an unknown distribution $\mathcal{D}$.

When comparing one algorithm to another, we would first check if they will \emph{converge}, namely if they will approach the desired solution, and if so, how fast they converge, namely the \emph{convergence rate}. There are two interchangeable forms that are widely used to describe the convergence rate:
\begin{itemize}
    \item fix the precision $\epsilon > 0$ quantifying how close we want the output of the algorithm to be w.r.t. the desired solution, e.g., $F(\bx_{output}) - F^* \le \epsilon$ and compute the number of oracle calls $N$ to achieve such precision in terms of $\epsilon$ like $N=\mathcal{O}(\epsilon^{-2})$.
    \item  fix the number of oracle calls an algorithm can make say $N$ and computes the precision of the output in terms of $N$ like $F(\bx_{output}) - F^* = \mathcal{O}\left(\frac{1}{N}\right)$.
\end{itemize}
We will mainly use the second form in this dissertation.

\section{Adaptive Optimization Algorithms}
For the scenario of a deterministic first-order black-box oracle, a vastly popular optimizer is the \textbf{Gradient Descent} (GD) which can be traced back to~\citet{cauchy1847methode}. GD proceeds along the negative direction of the gradient based on the intuition that the gradient represents the direction of the fastest increase. Mathematically, given an initial point $\bx_1\in\mathcal{X}$, GD iteratively updates through
\begin{equation}
    \bx_{t+1} = \bx_{t} - \eta_t\nabla F(\bx_t),
\end{equation}
where $\eta_t$ is called the \emph{step size} at time $t$ and controls how far the algorithm moves. Examples of step size sequences including a constant schedule $\eta_t = \eta$, a polynomial schedule $\eta_t = \frac{1}{t}$, and an exponential schedule $\eta_t = \exp(-t)$.

Meanwhile, in the stochastic setting, when the true gradient is unavailable, we only receive a stochastic gradient $\nabla f(\bx,\xi)$ where $\xi$ is a random variable denoting the stochasticity. Note that, when convenient, we will refer to $\nabla f(\bx_t,\xi_t)$ as $\bg_t$. Then, a counterpart of GD that is widely used is \textbf{Stochastic Gradient Descent}~\citep{robbins1951stochastic} which updates through
\begin{equation}
\label{eq:sgd}
\bx_{t+1} = \bx_t - \eta_t \nabla f(\bx_t,\xi_t),
\end{equation}

Despite their wide usage, the performance of GD/SGD depends crucially on the choice of its step size. The tuning of the step size is notoriously laborious and the optimal choice of it can vary drastically across different problems.

To save the labor of tuning, adaptive algorithms come to the rescue. \emph{An algorithm is said to be adaptive to a certain parameter (of the optimization problem) if it does not need a priori knowledge of such parameter but performs competitively to those that know it (up to some additional cost).}

Adaptation is a general concept and an algorithm can be adaptive to any characteristic of the optimization problem. The idea is formalized in \citep{Nesterov15b} with the equivalent name of \emph{universality}, but it goes back at least to the ``self-confident'' strategies in online convex optimization~\citep{AuerCG02}. Indeed, the famous AdaGrad algorithm~\citep{McMahanS10,DuchiHS10} uses exactly this concept of adaptation to design an algorithm \emph{adaptive to the gradients}. Nowadays, ``adaptive step size'' tend to denote coordinate-wise ones, with no guarantee of adaptation to any particular property. There is an abundance of adaptive optimization algorithm in the convex setting~\citep[e.g.,][]{McMahanS10,DuchiHS10,KingmaB15,ReddiKK18}, while only a few in the more challenging non-convex setting~\citep[e.g.,][]{ChenZTYG18}. The first analysis to show adaptivity to the noise of non-convex SGD with appropriate step sizes is in \citet{LiO19} and later in \citet{WardWB19} under stronger assumptions. Then, \citet{LiO20} studied the adaptivity to the noise of AdaGrad plus momentum, with a high probability analysis.

\begin{algorithm}[t]
\begin{algorithmic}[1]
\STATE{\textbf{Input} $\bx_1\in\mathcal{X}$, $D$ (diameter of $\mathcal{X}$)}
\STATE{\textbf{Set} $Q_0 = 0$}
\FOR{$t = 1,2,\ldots,T$}
\STATE{\textbf{Update} $Q_{t}=Q_{t-1}+\|\nabla F(\bx_t)\|^2$}
\STATE{\textbf{Set} $\eta_t = \frac{D}{\sqrt{2Q_t}}$}
\STATE{\textbf{Update} $\bx_{t+1}=\prod_{\mathcal{X}}(\bx_{t}-\eta_t \nabla F(\bx_t))$}
\ENDFOR
\STATE{\textbf{Output}: $\bar{\bx}\triangleq\frac1T\sum^T_{t=1}\bx_t$}
\end{algorithmic}
\caption{AdaGrad with a global step size~\citep{LevyYC18}}
\label{algo:adagrad}
\end{algorithm}

As an example describing what adaptivity is, we show a variant of AdaGrad in Algorithm~\ref{algo:adagrad} where $\prod_{\mathcal{X}}(\cdot)$ denotes the projection onto $\mathcal{X}$ namely $\prod_{\mathcal{X}}(\bx) = \arg\min_{\by\in\mathcal{X}}\|\by - \bx\|_2^2$. It has the following guarantee~\citep{LevyYC18} (proof in Appendix~\ref{sec:omit_proofs}):
\begin{thm}
Assume $F$ to be differentiable, convex, and has $D$-bounded-domain namely $D :=\max_{\bx, \by\in\mathcal{X}}\|\bx - \by\|_2$ which we call the diameter of $\mathcal{X}$, Algorithm~\ref{algo:adagrad} guarantees
\begin{equation}
\label{eq:conv_bound_adagrad}
F(\bar{\bx}) - F(\bx^*) \le \frac{\sqrt{2}D}{T}\sqrt{\sum^T_{t=1}\|\nabla F(\bx_t)\|_2^2}~.
\end{equation}
\end{thm}

It can \emph{adapt} to the sum of the gradients norm squared! To see why this is good, consider the projected gradient descent algorithm with a constant step size which updates in the form of Line 6 of Algorithm~\ref{algo:adagrad} but with a fixed step size $\eta_t = \eta$. We can show the following guarantee for this algorithm.
\begin{thm}
\label{thm:pgd_convex}
Assume $F$ to be differentiable, convex, and has $D$-bounded-domain, the projected GD algorithm with a fixed step size $\eta$ guarantees for $\bar{\bx} = \frac1T\sum^T_{t=1}\bx_t$ that
\begin{equation}
\label{eq:conv_bound_pgd}
F(\bar{\bx}) - F(\bx^*) \le \frac{D^2}{2\eta T} + \frac{\eta}{2T}\sum^T_{t=1}\|\nabla F(\bx_t)\|_2^2~.
\end{equation}
\end{thm}

Obviously, to get a guarantee like~\eqref{eq:conv_bound_adagrad}, we would need to set $\eta = \frac{D}{\sqrt{\sum^T_{t=1}\|\nabla F(\bx_t)\|_2^2}}$. Yet, this step size requires knowledge of all updates which is clearly impossible amid running the algorithm. This immediately shows the advantage of adaptive algorithms.

\section{Structure of the Dissertation}


Chapter~\ref{chapter:adapt_noise} is devoted to the adaptation to the level of noise in evaluating stochastic gradients under the stochastic optimization setting. There, we will present our work on designing/analyzing noise adaptive algorithms in both the general smooth non-convex setting and the setting with the additional PL condition.

Chapter~\ref{chapter:adapt_scale} discusses the problem that the scales of gradient magnitudes can vary significantly across layers in training deep neural networks. We will identify scenarios where the renowned Adam optimizer~\citep{KingmaB15} is inferior to its variant AdamW~\citep{LoshchilovH18}, and then correlate this observation to the scale-freeness property AdamW enjoys while Adam does not. A connection between AdamW and proximal updates will then be revealed providing a potential explanation for where the scale-freeness comes from.

Chapter~\ref{chapter:adapt_smoothness} focuses on the setting of relaxed smoothness where the gradients can change drastically. We will start by reporting empirical evidence showing that this condition captures the training of Transformers and propose to further refine it to a coordinate-wise level. A generalized SignSGD algorithm is then proposed which on one end recovers the SignSGD algorithm with matching theoretical convergence guarantees as SGD with gradient clipping, while on the other end closely mimics Adam with matching empirical performance. This algorithm can be made to adapt to the unknown parameter characterizing the relaxed smoothness condition.

Finally, Chapter~\ref{chapter:Conclusions} concludes this dissertation with major contributions.

\section{Notations}
We use bold lower-case letters to denote vectors and upper-case letters for matrices, e.g., $\bu\in\R^d, \boldsymbol{A}\in\R^{m\times n}$. The i$^{th}$ coordinate of a vector $\bu$ is $u_i$.
Unless otherwise noted, we study the Euclidean space $\R^d$ with the inner product $\langle\cdot, \cdot\rangle$, and all the norms are the Euclidean norms.
The dual norm $\|\cdot\|_*$ is the norm defined by $\|\bv\|_*=\sup_{\bu}\{\langle \bu, \bv\rangle:\|\bu\|\le1\}$.
$\E[\bu]$ means the expectation with respect to the underlying probability distribution of a random variable $\bu$, and $\E_t[\bu]$ is the conditional expectation of $\bu$ with respect to the past.
The gradient of $F$ at $\bx$ is denoted by $\nabla F(\bx)$.
$\partial F$ denotes the set of subgradients. $[d]$ denotes the sequence $\{1,\ldots,d\}$.

%% file: 3_Adapt_Noise/adapt_noise.tex
\chapter{Adaptation to Noise}
\label{chapter:adapt_noise}
\thispagestyle{myheadings}

[The results in Section~\ref{sec:noise_adapt_smooth} appeared in~\citet{ZhuangCO19} and the results in Section~\ref{sec:noise_adapt_pl} appeared in~\citet{LiZO21}.]

Gradient Descent is an intuitive yet effective algorithm that enjoys vast popularity in the machine learning community and beyond. In practice, however, the true gradient is not always available, either because it is impossible to obtain at all, or because evaluating it would be too expensive. Under such scenarios, we do stochastic optimization in which we only access a stochastic gradient and run Stochastic Gradient Descent instead. Yet, the noise in evaluating the stochastic gradients slows down the convergence or even leads to divergence if we do not tune the step size carefully.

Formally, in this chapter, we focus on optimizing problem~\eqref{eq:opt} using GD or~\eqref{eq:stoc_opt_obj} using SGD. Further, we will use following assumptions:

\begin{assumption}
\label{asp:noise_unbiased}
The stochastic gradient at each step $t$ is unbiased given the past, that is,
\begin{align*}
\E_{t}\left[\nabla f(\bx_t,\xi_t)\right]
=
\nabla F(\bx_t)~.
\end{align*}
\end{assumption}

\begin{assumption}
\label{asp:noise_var_bounded}
The stochastic gradient at each step $t$ has finite variance with respect to the $\ell_2$ norm given the past, that is,
\begin{align}
\E_{t}\left[\left\|\nabla f(\bx_t,\xi_t)-\nabla F(\bx_t)\right\|^2\right]
&=\sigma^2~.
\end{align}
\end{assumption}

The structure of this chapter is the following: we will first discuss how these two settings (deterministic vs.~stochastic) are different and why adapting to noise is desirable in Section~\ref{sec:noise_harm}. Next, we will introduce our work~\citep{ZhuangCO19} on designing an algorithm that uses no-regret online algorithms to compute optimal step sizes on the fly and guarantees convergence rates that are automatically adaptive to the level of noise in Section~\ref{sec:noise_adapt_smooth}. Then, in Section~\ref{sec:noise_adapt_pl}, we will present our work~\citep{LiZO21} showing that, under the added PL condition, SGD employing two vastly popular empirical step size schedules enjoys a faster convergence rate while still being adaptive to the level of noise.

\section{Why Adaptation to Noise is Desirable}
\label{sec:noise_harm}
Intuitively, when using SGD, the noisy gradients pointing to a random direction would slow down the convergence speed. Indeed, there are already established lower bounds showing that stochastic optimization is fundamentally more difficult than the deterministic one. For example, in the general $L$-smooth non-convex scenario, when we can access the true gradient (the deterministic setting), the best possible worst-case rate of convergence to a stationary point for any first-order optimization algorithm is $O(\frac1T)$~\citep{CarmonDHS21}, which can be obtained by Gradient Descent with a constant step size $\eta=\frac{1}{L}$. In contrast, when there is noise (assuming zero mean and bounded variance), namely the stochastic setting, no first-order algorithm can do better than the $O(\frac{\sigma}{\sqrt{T}})$ rate~\citep{ArjevaniCDFSW19} and this rate can be achieved by Stochastic Gradient Descent with a constant step size of $\eta=\frac{1}{\sigma\sqrt{T}}$ or a time-varying one $\eta_t = \frac{1}{\sigma\sqrt{t}}$.

Clearly, the deterministic and the stochastic settings are intrinsically different. Without knowing the noise level $\sigma$, especially if $\sigma = 0$ or not, there is no single step size schedule that will make GD/SGD converge at the optimal rate in both settings. We would have to spend a lot of time tuning the algorithm very carefully. Indeed, \citet{GhadimiL13} proved the following result for SGD with a constant step size in the smooth setting:
\begin{thm}
For a $L$-smooth function $F$ (Assumption~\ref{asp:smooth}), under Assumption~\ref{asp:noise_unbiased} and~\ref{asp:noise_var_bounded}, SGD with a constant step size $\eta\leq\frac{1}{L}$ guarantees
\begin{equation}
\E [\|\nabla F(\bx_i)\|^2]
\leq O\left( \frac{F(\bx_1)-F^\star}{\eta T} + \eta \sigma^2 \right),
\end{equation}
where $\bx_i$ is uniformly randomly picked in $\bx_1,\ldots,\bx_T$.
\end{thm}

From the above, it is immediate to see that we need a step size of the form $O(\min(\tfrac{\sqrt{F(\bx_1)-F^\star}}{\sigma \sqrt{T}},\tfrac{1}{L}))$ to have the best worst case convergence of $O(\frac{1}{T}+\frac{\sigma}{\sqrt{T}})$. In words, this means that we get a faster rate, $O(\tfrac{1}{T})$, when there is no noise, and a slower rate, $O(\tfrac{\sigma}{\sqrt{T}})$, in the presence of noise.

In practice, however, we usually do not know the variance of the noise $\sigma$, which makes the above optimal tuning of the step size difficult to achieve. Even worse, the variance can change over time. For example, it may decrease over time if $F(\bx) = \E_j[F_j(\bx)]$ and each $F_j$ has zero gradient at the local optimum we are converging to. Moreover, even assuming the knowledge of the variance of the noise, the step sizes proposed in \citet{GhadimiL13} assume the knowledge of the unknown quantity $F(\bx_1)-F^\star$.

\begin{figure}
    \centering
    \includegraphics[width=\textwidth]{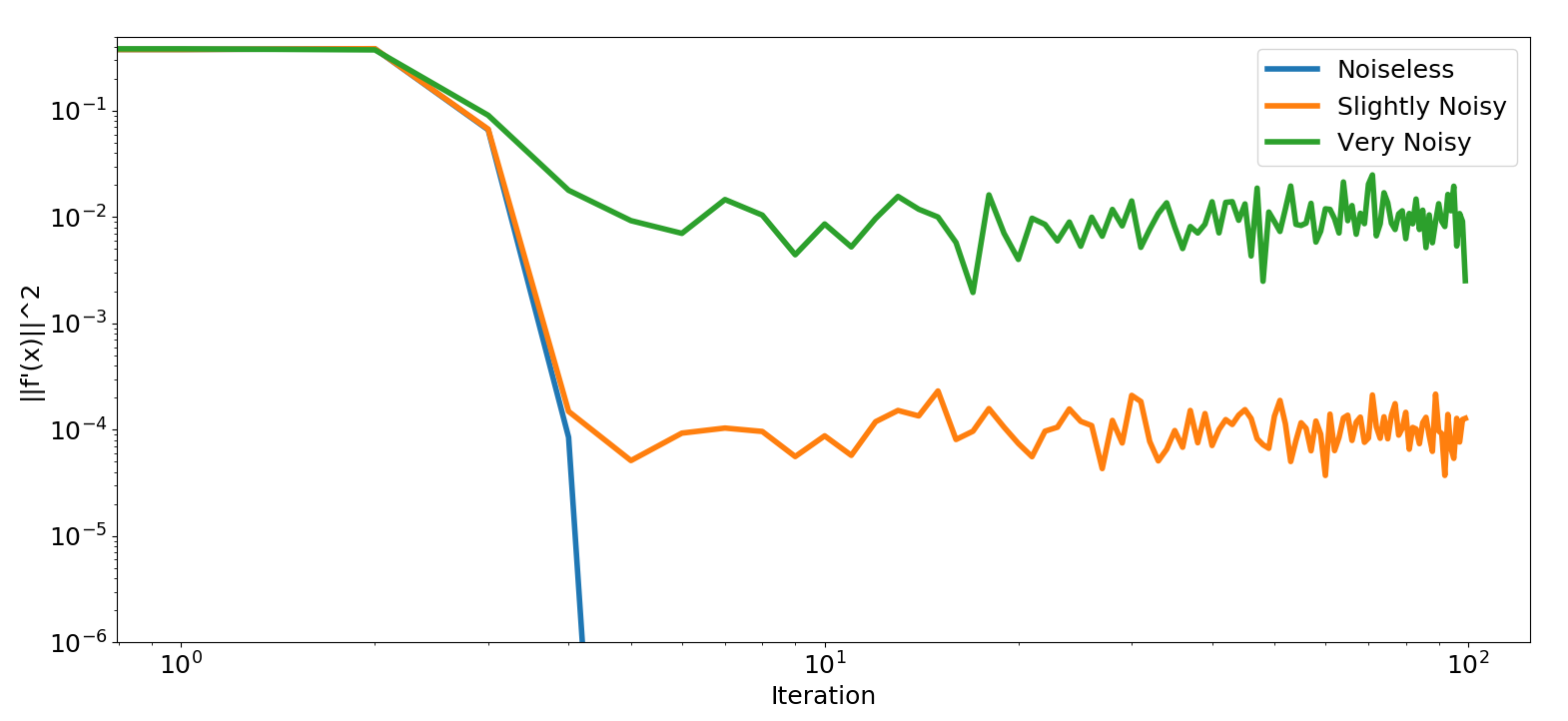}
    \caption{Running SGD with a fixed step size on a smooth non-convex function for different noise levels.}
    \label{fig:sgd_constant_lr_smooth}
\end{figure}

We use Figure~\ref{fig:sgd_constant_lr_smooth} to depict the above result, in which we run SGD with a fixed constant step size $\eta=0.5$ on a non-convex function $F(x) = \frac{x^2}{1+x^2}$ which is $2$-smooth. When accessing the gradient, we add additive white Gaussian noise with different variances $\sigma$ ($0, 0.01, 0.1$, respectively). The figure clearly shows that one constant step size does not work for different noise settings and that tuning is necessary.

Another solution would be to obtain an explicit estimate of the variances of the noise, for example by applying some concentration inequality to the sample variance, and using it to set the step sizes. This approach is suboptimal because it does not directly optimize the convergence rates, relying instead on a loose worst-case analysis.

In light of this problem, in the next section, we will present an algorithm we designed that guarantees the optimal rates in both deterministic and stochastic settings automatically without knowing $\sigma$, namely adapting to noise.

\section{Adaptation to Noise in the Smooth Non-convex Setting}
\label{sec:noise_adapt_smooth}
\subsection{Surrogate Losses}
\label{sec:surrogate}
Before presenting the algorithm, we first introduce the notion of surrogate losses which motivates the key idea behind the design of our algorithm: we use the smoothness of the objective function to transform the problem of optimizing a non-convex objective function into the problem of optimizing a series of convex loss functions, which we solve by an online learning algorithm.

Specifically, consider using SGD with non-convex $L$-smooth losses starting from an initial point $\bx_1$. At each time $t$, we define the \emph{surrogate loss} $\ell_t:\R^d \rightarrow\R$ as
\begin{equation}
\ell_t(\eta) 
=-\eta \langle \nabla f(\bx_t,\xi_t), \nabla f(\bx_t,\xi'_t) \rangle 
+
\frac{L \eta^2}{2} \|\nabla f(\bx_t,\xi_t)\|^2, \label{eq:surrogate}
\end{equation}
where $\nabla f(\bx_t,\xi_t)$ and $\nabla f(\bx_t,\xi'_t)$ are the noisy stochastic gradients received from the black-box oracle at time $t$. Note that, when convenient, we will refer to $\nabla f(\bx_t,\xi_t)$ and $\nabla f(\bx_t,\xi'_t)$ as $\bg_t$ and $\bg_t'$ respectively. We also assume that:

\begin{assumption}
\label{asp:two_grad_independent}
The two stochastic gradients at step $t$ are independent given the past, i.e.,
\begin{equation}
\E_t\left[\langle \nabla f(\bx_t,\xi_t), \nabla f(\bx_t,\xi'_t)\rangle\right]
=\|\nabla F(\bx_t)\|^2~.
\end{equation}
\end{assumption} 
It is clear that $\ell_t$ is convex. Moreover, the following key result shows that these surrogate losses upper bound the expected decrease of the function value $F$.
\begin{thm}
Assume Assumption~\ref{asp:noise_unbiased} and~\ref{asp:two_grad_independent} hold and $\eta_t$ is independent from $\xi_j$ and $\xi'_j$ for $j\geq t$. Then, for an $L$-smooth function, the SGD update in \eqref{eq:sgd} gives us
\label{thm:surrogate}
\[
\E\left[F(\bx_{t+1})-F(\bx_t)\right]
\leq \E\left[\ell_t(\eta_t)\right]~.
\]
\end{thm}
\begin{proof}[Proof of Theorem~\ref{thm:surrogate}]
The $L$-smoothness of $F$ gives us:
\begin{align}
\E\left[F(\bx_{t+1})-F(\bx_t)\right]
&\le \E\left[\langle\nabla F(\bx_t),\bx_{t+1}-\bx_t \rangle + \frac L2\|\bx_{t+1}-\bx_t\|^2\right]\\
&= \E\left[\langle\nabla F(\bx_t),-\eta_t \nabla f(\bx_t,\xi_t) \rangle+\frac L2\eta_t^2\|\nabla f(\bx_t,\xi_t)\|^2\right]\\
&= \E\left[\langle\nabla F(\bx_t),\E_t\left[-\eta_t \nabla f(\bx_t,\xi_t)\right] \rangle+\frac L2\eta_t^2\|\nabla f(\bx_t,\xi_t)\|^2\right].
\end{align}
Now observe that $\nabla F(\bx_t)=\E_t\left[\nabla f(\bx_t,\xi'_t)\right]$, so that
\begin{align*}
\E\left[\langle \E_t\left[\nabla f(\bx_t,\xi'_t)\right],\E_t\left[-\eta_t \nabla f(\bx_t,\xi_t)\right]\rangle\right]
&= \E\left[ \E_t\left[\langle \nabla f(\bx_t,\xi'_t),-\eta_t \nabla f(\bx_t,\xi_t) \rangle\right]\right]\\
&= \E\left[\langle \nabla f(\bx_t,\xi'_t),-\eta_t \nabla f(\bx_t,\xi_t) \rangle\right]~.
\end{align*}
Putting it all together, we have the stated inequality.
\end{proof}

The above theorem tells us that if we want to decrease the function $f$, we might instead try to minimize the convex surrogate losses $\ell_t$. In the following, we build upon this intuition to design an online learning procedure that adapts the step sizes of SGD to achieve the optimal convergence rate.

\subsection{SGD with Online Learning}
\label{ssec:sgdol}
The surrogate losses allow us to design an online convex optimization procedure to learn the optimal step sizes. In each round, the step sizes are chosen by an online learning algorithm $\mathcal{A}$ fed with the surrogate losses $\ell_t$.
The online learning algorithm will minimize the regret: the difference between the cumulative sum of the losses of the algorithm, $\ell(\eta_t)$, and the cumulative losses of any fixed point $\eta$. In formulas, for a 1-dimensional online convex optimization problem, the regret is defined as
\[
\mathrm{Regret}_T(\eta) =\sum_{t=1}^T (\ell_t(\eta_t)- \ell_t(\eta))~.
\]
If the regret is small, we will have that the losses of the algorithm are not too big compared to the best losses, which implies that the step sizes chosen by the online algorithm are not too far from the optimal (unknown) step size.

\begin{algorithm}[tb]
\caption{Stochastic Gradient Descent with Online Learning (SGDOL)}
\label{algo:sgdol}
\begin{algorithmic}[1]
\STATE {\bfseries Input:} $\bx_1\in\mathcal{X},\ L$, an online learning algorithm $\mathcal{A}$
\FOR{$t = 1,2,\ldots,T$}
\STATE{\textbf{Compute} $\eta_t$ by running $\mathcal{A}$ on $\ell_{i}, i=1,\ldots,t-1$, as defined in \eqref{eq:surrogate}} \STATE{\textbf{Receive} two independent unbiased estimates of $\nabla F(\bx_{t})$: $\bg_t$, $\bg'_t$}
\STATE{\textbf{Update} $\bx_{t+1}=\bx_{t}-\eta_t \bg_t$}
\ENDFOR
\STATE{\textbf{Output}: uniformly randomly choose a $\bx_k$ from $\bx_1,\ldots,\bx_T$.}
\end{algorithmic}
\end{algorithm}

We call this procedure Stochastic Gradient Descent with Online Learning (SGDOL) and the pseudocode is in Algorithm~\ref{algo:sgdol}.

To prove its convergence rate, we need the following assumption:
\begin{assumption}
\label{asp:noise_norm_bounded}
The stochastic gradients at each step $t$ have bounded norms:
\[
\|\nabla f(\bx_t,\xi_t)\|\le G,\quad
\|\nabla f(\bx_t,\xi'_t)\|\le G~.
\]
\end{assumption}

Then, we can prove the following Theorem.
\begin{thm}
\label{thm:olsmooth}
For an $L$-smooth function $F$, under Assumption~\ref{asp:noise_unbiased} and~\ref{asp:noise_var_bounded}, for any $\eta>0$, SGDOL in Algorithm~\ref{algo:sgdol} satisfies
\begin{equation*}
\E\left[\left(\eta-\frac{L}{2}\eta^2\right)\sum_{t=1}^T \|\nabla F(\bx_t)\|^2\right]
\leq F(\bx_1) - F^\star + \E\left[\mathrm{Regret}_T(\eta)\right] + \frac{L \eta^2\sigma^2 T}{2}~.
\end{equation*}
\end{thm}
\begin{proof}[Proof of Theorem~\ref{thm:olsmooth}]
Summing the inequality in Theorem~\ref{thm:surrogate} from 1 to $T$:
\begin{align}
F^\star-F(\bx_1)
\le&\ \E[F(\bx_{T+1})]-F(\bx_1)\\
=&\sum^T_{t=1}\E\left[F(\bx_{t+1})-F(\bx_t)\right]\\
\le&
\sum^T_{t=1}\E\left[\ell_t(\eta_t)\right]\\
=&
\sum^T_{t=1}\E\left[\ell_t(\eta_t)-\ell_t(\eta)\right]
+
\sum^T_{t=1}\E\left[\ell_t(\eta)\right]\\
\le&\ 
\E\left[\mathrm{Regret}_T(\eta)\right]
+
\sum^T_{t=1}\E\left[\ell_t(\eta)\right]~.
\end{align}

Using the fact that
\begin{align*}
\E_t[\ell_t(\eta)] 
&= \left(-\eta+\frac{L}{2}\eta^2\right)\|\nabla F(\bx_t)\|^2 + \frac{L}{2}\eta^2\sigma^2,
\end{align*}
we have the stated bound.
\end{proof}

The only remaining ingredient for SGDOL is to decide on an online learning procedure. Given that the surrogate losses are strongly convex, we can use a Follow The Regularized Leader (FTRL) algorithm~\citep{Shalev-Shwartz07,AbernethyHR08,AbernethyHR12,McMahan17}. Note that this is not the only possibility, for example, we could use an optimistic FTRL algorithm which achieves even smaller regret~\citep{mohri2016accelerating}. Yet, FTRL is enough to show the potential of our surrogate losses. As shown in Algorithm~\ref{algo:aftrl}, in an online learning game in which we receive the convex losses $\ell_t$, FTRL constructs the predictions $\bv_t$ by solving the optimization problem 
\[
\bv_{t+1} = \argmin_{\bv\in\R^d} \ r(\bv) + \sum^t_{s=1} \ell_s(\bv),
\]
where $r:\R^d \rightarrow \R$ is a regularization function.
We can upper bound the regret of FTRL with the following theorem.

\begin{algorithm}[tb]
\caption{Follow the Regularized Leader (FTRL)}
\label{algo:aftrl}
\begin{algorithmic}[1]
\STATE{\textbf{Parameters}: $r(\bv)\geq 0$}
\STATE{$\bv_1\leftarrow\argmin_{\bv\in\R^d} \ r(\bv)$}
\FOR{$t = 1,2,\ldots$}
\STATE{Observe convex loss function $\ell_t:\R^d\rightarrow \R\cup\{\infty\}$}
\STATE{Incur loss $\ell_t(\bv_t)$}
\STATE{Update $\bv_{t+1}\leftarrow \argmin_{\bv\in\R^d} \ r(\bv) + \sum^t_{s=1} \ell_s(\bv)$}
\ENDFOR
\end{algorithmic}
\end{algorithm}

\begin{thm}{\citep{McMahan17}}
For the FTRL Algorithm~\ref{algo:aftrl}, suppose each $\ell_t$ is convex and differentiable and $r$ is chosen such that $h_{t} = r + \sum_{i=1}^t \ell_{i}$ is 1-strongly-convex w.r.t. some norm $\|\cdot\|_{(t)}$.
Then, for any $\bx^\star\in\R^d$ and for any $T>0$,
\begin{equation}
\label{thm:prox}
\mathrm{Regret}_T(\bx^\star)
\le r(\bx^\star) + \frac12\sum^T_{t=1} \|\nabla\ell_t(\bx_t)\|^2_{(t),\star},
\end{equation}
where $\|\cdot\|_{(t),\star}$ is the dual norm of $\|\cdot\|_{(t)}$ namely $\|\bx\|_{(t),\star}:=\sup_{\by:\|\by\|_{(t)}\le1}\langle\bx, \by\rangle$.
\end{thm}

We can now put it all together and get a convergence rate guarantee for SGDOL.
\begin{thm}
\label{thm:ftrlp}
For an $L$-smooth function $F$, under Assumption~\ref{asp:noise_unbiased},~\ref{asp:noise_var_bounded},~\ref{asp:two_grad_independent}, and~\ref{asp:noise_norm_bounded} and using FTRL (Algorithm~\ref{algo:aftrl}) by choosing $r(\eta)=\frac{L\alpha}2\left(\eta-\frac1L\right)^2+\mathcal{I}\left(\eta\in\left[0,\frac2L\right]\right)$ with $\alpha>0$ in Algorithm~\ref{algo:sgdol}, for an uniformly randomly picked $\bx_k$ from $\bx_1,\ldots,\bx_t$, we have:
\begin{align}
\E_k\left[\|\nabla F(\bx_k)\|^2\right]
&\le \frac{2L}{T}\left(F(\bx_1)-F^\star + \frac{5G^2}{L}\ln\left(1+\frac{G^2T}{\alpha}\right)\right)\\
&\quad + \frac1{\sqrt{T}}\sqrt{2L\sigma^2\left(f(\bx_1)-f^\star + \frac{\alpha}{2L}\right)}\\
&\quad + \frac1{\sqrt{T}}\sqrt{10G^2\sigma^2\ln\left(1+\frac{G^2T}{\alpha}\right)}~.
\end{align}
\end{thm}

Clearly, when $\sigma = 0$ we obtain the convergence rate of $\mathcal{O}\left(\frac{1}{T}\right)$, and when $\sigma > 0$ we can still converge in the rate of $\mathcal{O}\left(\frac{\sigma}{\sqrt{T}}\right)$. Note that we do not need to know $\sigma$ to achieve this result, thus our algorithm is adaptive to $\sigma$.

Before proving this theorem, we make some observations.

The FTRL update gives us a very simple strategy to calculate the step sizes $\eta_t$. In particular, the FTRL update has a closed-form:
\begin{equation}
\label{eq:sgdol_ftrl_update}
\eta_{t} = \max\left\{0,\min\left\{\frac{\alpha+\sum^{t-1}_{j=1}\langle \bg_j, \bg'_j\rangle}{L\left(\alpha+\sum^{t-1}_{j=1}\|\bg_j\|^2\right)},\frac2L\right\}\right\}~.
\end{equation}
Note that this can be efficiently computed by keeping track of the quantities\\ $\sum^{t-1}_{j=1}\langle \bg_j, \bg'_j\rangle$ and $\sum^{t-1}_{j=1}\|\bg_j\|^2$.

While the computational complexity of calculating $\eta_t$ by FTRL is negligible, SGDOL requires two unbiased gradients per step. This increases the computational complexity with respect to a plain SGD procedure by a factor of two.

The theorem also shows that the parameter $\alpha$ has a minor influence on the convergence rate: although it should optimally be set to any constant value on the order of $G^2$, it is safe to set it reasonably small without blowing up the log factor.

We can now prove the convergence rate in Theorem~\ref{thm:ftrlp}.
For the proof, we need the following lemma.
\begin{lemma}
\label{lm:integ}
Let $a_i\geq0$ for $i=0, \cdots, T$ and $h:[0,+\infty)\rightarrow [0, +\infty)$ be a nonincreasing function.
Then
\begin{align*}
\sum_{t=1}^T a_t h\left(a_0+\sum_{i=1}^{t} a_i\right) 
&\leq \int_{a_0}^{\sum_{t=0}^T a_t} h(x) dx~.
\end{align*}
\end{lemma}
\begin{proof}[Proof of Lemma~\ref{lm:integ}]
Denote by $s_t=\sum_{i=0}^{t} a_i$.
\begin{align*}
a_i h(s_i) 
=  \int_{s_{i-1}}^{s_i} h(s_i) d x 
\leq \int_{s_{i-1}}^{s_i} h(x) d x~.
\end{align*}
Summing over $i=1, \cdots, T$, we obtain the stated bound.
\end{proof}

\begin{proof}[Proof of Theorem~\ref{thm:ftrlp}]
As $\ell_t^{\prime\prime}(\eta)=L\|\bg_t\|^2$, we have that $h_{t} = r + \sum_{i=1}^t \ell_{i}$ is 1-strongly-convex with respect to the norm $\sqrt{L\left(\alpha + \sum^t_{s=1}\|\bg_s\|^2\right)}\|\cdot\|$.

Applying Theorem~\ref{thm:prox}, we get that, for any $\eta\in\left[0,\frac2L\right]$,
\begin{equation}
\mathrm{Regret}_T(\eta)
\le \frac{L\alpha}2\left(\eta-\frac1L\right)^2 + \frac{1}{2L}\sum^T_{t=1}\frac{\left(\ell_t^{\prime}(\eta_t)\right)^2}{\alpha + \sum^t_{s=1}\|\bg_s\|^2}~. \label{eq:regret}
\end{equation}

Now observe that 
\begin{align*}
\left(\ell^{\prime}_t(\eta_t)\right)^2
&= \left(-\langle\bg_{t},\bg^{\prime}_{t} \rangle + L\eta_t\|\bg_{t}\|^2\right)^2\\
&\le 2\langle\bg_{t},\bg^{\prime}_{t} \rangle^2 + 2L^2\eta_t^2\|\bg_{t}\|^4\\
&\le 2\|\bg_{t}\|^2\|\bg_{t}^{\prime}\|^2 + 8\|\bg_{t}\|^4
\le 10G^2\|\bg_{t}\|^2,
\end{align*}
where in the third line we used the Cauchy-Schwarz inequality and $\eta_t\le\frac2L$.
Hence, the last term in \eqref{eq:regret} can be upper bounded as
\begin{align*}
\frac{1}{2}\sum^T_{t=1}\frac{\left(\ell_t^{\prime}(\eta_t)\right)^2}{L\left(\alpha + \sum^t_{s=1}\|\bg_s\|^2\right)}
&\leq \frac{5G^2}{L}\sum^T_{t=1}\frac{\|\bg_t\|^2}{\alpha + \sum^t_{s=1}\|\bg_s\|^2}\\
&\le \frac{5G^2}{L}\ln\left(\frac{\alpha + \sum^T_{t=1}\|\bg_t\|^2}{\alpha}\right)\\
&\le \frac{5G^2}{L}\ln\left(1+\frac{G^2T}{\alpha}\right),
\end{align*}
where in the first inequality we used Lemma~\ref{lm:integ}.

Now put the last inequality above back into Theorem~\ref{thm:olsmooth}, to obtain
\begin{align}
\E\left[\left(\eta-\frac{L}{2}\eta^2\right)\sum_{t=1}^T \|\nabla F(\bx_t)\|^2\right]
&\leq F(\bx_1) - F^\star + \frac{L\alpha}2\left(\eta-\frac1L\right)^2 \\
&\quad + \frac{5G^2}{L}\ln\left(1+\frac{G^2T}{\alpha}\right) + \frac{L \eta^2\sigma^2 T}{2}~.
\end{align}

Denote $A\triangleq \sum^T_{t=1}\E\left[\|\nabla F(\bx_t)\|^2\right]$, we can transform the above into a quadratic inequality of $\eta$:
\begin{align*}
0\ \le&\phantom{+}\frac L2\left(A+\alpha+\sigma^2 T\right)\eta^2 - (A+\alpha)\eta + F(\bx_1) - F^\star + \frac{5G^2}{L}\ln\left(1+\frac{G^2T}{\alpha}\right) + \frac{\alpha}{2L}~.
\end{align*}
Choosing $\eta$ as the minimizer of the right hand side: $\eta^*=\frac{\alpha+A}{L(\alpha+A+\sigma^2 T)}$ (which satisfies $\eta^*\le\frac2L$) gives us
\begin{align}
\frac{(\alpha+A)^2}{2L\left(\alpha+A + \sigma^2 T\right)}
\le\ F(\bx_1)-F^\star + \frac{\alpha}{2L} + \frac{5G^2}{L}\ln\left(1+\frac{G^2T}{\alpha}\right)~.
\end{align}

Solving this quadratic inequality of $A$ yields
\begin{align}
A 
&\le 2L\left(F(\bx_1)-F^\star + \frac{5G^2}{L}\ln\left(1+\frac{G^2T}{\alpha}\right)\right)\\
&\quad + \sqrt{2L \sigma^2 T \left(F(\bx_1)-F^\star + \frac{\alpha}{2L}\right)}\\
&\quad + \sqrt{10G^2 \sigma^2 T \ln\left(1+\frac{G^2T}{\alpha}\right)}~.
\end{align}

By taking an $\bx_k$ from $\bx_1,\ldots,\bx_t$ randomly, we get:
\begin{align*}
\E_k\left[\|\nabla F(\bx_k)\|^2\right] 
= \E_k\left[\E\left[\|\nabla F(\bx_k)\|^2\middle\vert k\right]\right]
= \frac1T\sum^T_{t=1}\E\left[\|\nabla F(\bx_t)\|^2\right],
\end{align*}
which completes the proof.
\end{proof}

Our step size schedule~\eqref{eq:sgdol_ftrl_update} is very unique; thus, to illustrate its behavior in practice, we experiment on fitting a classification model on the adult (a9a) dataset from the LibSVM website~\citep{ChangL01}. The objective function is
\begin{equation*}
F(\bx):=\frac1m \sum^m_{i=1}\phi(\boldsymbol{a}_i^\top\bx-y_i),
\end{equation*}
where $\phi(\theta)=\frac{\theta^2}{1+\theta^2}$, and $(\boldsymbol{a}_i,y_i)$ are the couples feature vector/label. The loss function $\phi$ is non-convex, 1-Lipschitz and 2-smooth w.r.t. the $\ell_2$ norm.

We consider the minimization problem with respect to all training samples. Also, as the dataset is imbalanced towards the group with an annual income of less than 50K, we subsample that group to balance the dataset, which results in 15682 samples with 123 features each. In addition, we append a constant element to each sample feature vector to introduce a constant bias. $\bx_1$ is initialized to be all zeros. For each setting, we repeat the experiment with different random seeds but with the same initialization 5 times and plot the average of the relevant quantities.

In this experiment, the noise on the gradient is generated by the use of mini-batches. Specifically, we compare SGDOL with SGD on three different mini-batch sizes, namely different noise scales: using all samples (Noiseless), 50 i.i.d.~samples (Moderate Noise), or 1 random sample (Heavy Noise) for evaluating the gradient at a point. The step size of SGD is selected as the one giving the best convergence rate when the full batch scheme, namely zero noise, is employed which turns out to be $0.1$. We take the reciprocal of SGD's best step size as the parameter $L$ for SGDOL, and we set $\alpha=10$ without any tuning based on our discussion on the influence of $\alpha$ above. These parameters are then employed in the other two noisy settings.

\begin{figure}[t]
    \centering
    \includegraphics[width=\textwidth]{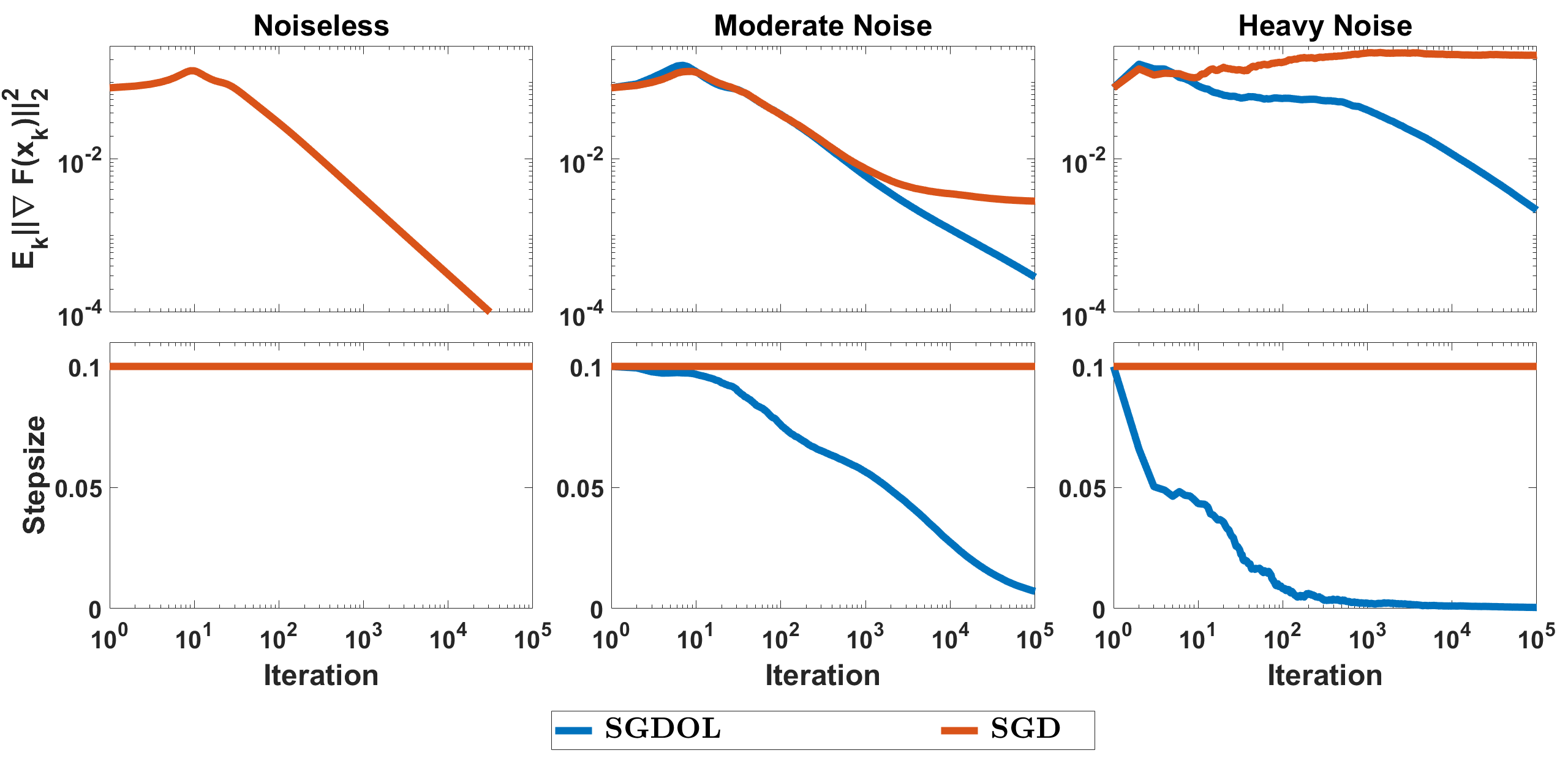}
    \caption{Comparison of SGDOL vs.~SGD on optimizing a smooth non-convex function with various noise scales.}
    \label{fig:sgdol_stepsizes}
\end{figure}

We report the results in Figure~\ref{fig:sgdol_stepsizes}. Figures in the top row show $\E[\|\nabla F(\bx_k)\|^2]$ vs.~number of iterations, whereas those in the bottom row are per-round step sizes. The x-axis in all figures and the y-axis in the top three are logarithmic. 

As can be seen, the step size of SGDOL is the same as SGD at first, but gradually decreases automatically. Also, the larger the noise, the sooner the decreasing phase starts. The decrease of the step size makes the convergence of SGDOL possible. In particular, SGDOL recovers the performance of SGD in the noiseless case, while it allows convergence in the noisy cases through an automatic decrease of the step sizes. In contrast, when noise exists, after reaching the proximity of a stationary point, SGD oscillates thereafter without converging, and the value it oscillates around depends on the variance of the noise. This underlines the superiority of the surrogate losses, rather than choosing a step size based on a worst-case convergence rate bound.

\subsection{Adapting Per-coordinate Step Sizes}
\label{ssec:coordinate}

In the previous subsection, we have shown how to use the surrogate loss functions to adapt a step size. Another common strategy in practice is to use a \emph{per-coordinate} step size. This kind of scheme is easily incorporated into our framework and we show that it can provide improved adaptivity to per-coordinate variances.

Specifically, we consider $\boldeta_t$ now to be a vector in $\R^d$, $\boldeta_t=(\eta_{t,1},\dots,\eta_{t,d})$, and use the update $\bx_{t+1} = \bx_t - \boldeta_t\bg_t$
where $\boldeta_t\bg_t$ now indicates coordinate-wise product $(\eta_{t,1}g_{t,1},\dots,\eta_{t,d}g_{t,d})$. Then we define the surrogate losses to be
\begin{align*}
\ell_t(\boldeta) &= -\langle \boldeta \nabla f(\bx_t,\xi_t), \nabla f(\bx_t,\xi_t')\rangle + \frac{L}{2}\|\boldeta \nabla f(\bx_t,\xi_t)\|^2
= \sum_{i=1}^d -\eta_{i}g_{t,i}g^{\prime}_{t,i} + \frac{L}{2}\eta_{i}^2g_{t,i}^2~.
\end{align*}

To take advantage of this scenario, we need more detail about the variance, which we encapsulate in the following assumption:

\begin{assumption}
\label{asp:noise_var_bounded_coor}
The noisy gradients have finite variance in each coordinate:
\begin{align}
\E_{t}\left[\left(g_{t,i}-\PartialDerivativeGeneral{\bx_t}{i}\right)^2\right]
&=\sigma_{i}^2~.
\end{align}
\end{assumption}
Note that this assumption is not actually stronger than Assumption~\ref{asp:noise_var_bounded} because we can define $\sigma^2 = \sum_{i=1}^d \sigma_{i}^2$. This merely provides finer-grained variable names.

Also, we make the following assumption:

\begin{assumption}
\label{asp:noise_norm_bounded_coor}
The noisy gradients have bounded coordinate values: 
\begin{align}
|g_{t,i}|\le G_i,\ |g^{\prime}_{t,i}|\le G_i~.
\end{align}
\end{assumption}

Now the exact same argument as for Theorem \ref{thm:olsmooth} yields:
\begin{thm}
\label{thm:olsmoothdiag}
For an $L$-smooth function $F$, assume the two noisy gradients in each round $t$ to satisfy Assumption~\ref{asp:noise_unbiased},~\ref{asp:two_grad_independent}, and~\ref{asp:noise_var_bounded_coor}. Then, for any $\boldeta\in \R^d$ with $\eta_i>0$ for all $i$, the per-coordinate variant of Algorithm \ref{algo:sgdol} gives
\begin{equation*}
\E\left[\sum_{t=1}^T \sum_{i=1}^d\left(\eta_i-\frac{L}{2}\eta_i^2\right)\left(\PartialDerivativeGeneral{\bx_t}{i}\right)^2\right]
\leq
F(\bx_1) - F^\star + \E\left[\mathrm{Regret}_T(\boldeta)\right] + \frac{LT}{2}\sum_{i=1}^d \eta_i \sigma_{i}^2~.
\end{equation*}
\end{thm}

With this Theorem in hand, once again all that remains is to choose the online learning algorithm. To this end, observe that we can write $\ell_t(\eta) = \sum_{i=1}^d \ell_{t,i}(\eta_i)$ where
\begin{align*}
\ell_{t,i}(\eta_i) = -\eta_{i}g_{t,i}g^{\prime}_{t,j} + \frac{L}{2}\eta_{i}^2 g_{t,i}^2~.
\end{align*}
Thus, we can take our online learning algorithm to be a per-coordinate instantiation of Algorithm \ref{algo:aftrl}, and the total regret is simply the sum of the per-coordinate regrets. Each per-coordinate regret can be analyzed in exactly the same way as Algorithm \ref{algo:aftrl}, leading to
\begin{align*}
\mathrm{Regret}_T(\eta)&=\sum_{i=1}^d \mathrm{Regret}_{T,i}(\eta_i),\\
\mathrm{Regret}_{T,i}(\eta_i)&\le \frac{L\alpha}{2}\left(\eta_i-\frac{1}{L}\right)^2 + \frac{5G_i^2}{L}\ln\left(1+\frac{G_i^2T}{\alpha}\right)~.
\end{align*}
From these inequalities, we can make a per-coordinate bound on the gradient magnitudes. In words, the coordinates which have smaller variances $\sigma^2_{t,i}$ achieve smaller gradient values faster than coordinates with larger variances. Further, we preserve adaptivity to the full variance $\sigma^2$ in the rate of decrease of $\|\nabla F(x)\|$.
\begin{thm}\label{thm:pcftrlp}
For an $L$-smooth function $F$, assume Assumption~\ref{asp:noise_unbiased},~\ref{asp:noise_var_bounded_coor}, and~\ref{asp:noise_norm_bounded_coor}, suppose we run a per-coordinate variant of Algorithm~\ref{algo:sgdol}, with regularizer $r(\eta)=\frac{L\alpha}2\left(\eta_i-\frac1L\right)^2+\mathcal{I}\left(\eta_i\in\left[0,\frac2L\right]\right)$ in each coordinate $i$ with $\alpha>0$. Then, for each $i\in\{1,\dots,d\}$, we have
\begin{align*}
\E\left[\sum_{t=1}^T \left(\PartialDerivativeGeneral{\bx_t}{i}\right)^2\right]
&\le 2L\left(F(\bx_1)-F^\star + \sum_{i=1}^d \frac{5G_i^2}{L}\ln\left(1+\frac{G_i^2T}{\alpha}\right)\right)\\
&\quad + \sqrt{2L\sigma_{i}^2T\left(F(\bx_1)-F^\star + \frac{d\alpha}{2L}\right)}\\
&\quad + \sqrt{10\sigma_{i}^2T\sum_{i=1}^d G_i^2\ln\left(1+\frac{G_i^2T}{\alpha}\right)}\\
&\quad+(d-1)\alpha~.
\end{align*}
Further, with $\sigma^2=\sum_{t=1}^T \sigma_{i}^2$ it also holds
\begin{align*}
\E\left[\sum_{t=1}^T\|\nabla F(\bx_t)\|^2\right]
&\le 2L\left(F(\bx_1)-F^\star + \frac{5}{L}\sum_{i=1}^d G_i^2\ln\left(1+\frac{G_i^2T}{\alpha}\right)\right)\\
&\quad + \sqrt{2L\sigma^2T\left(F(\bx_1)-F^\star + \frac{d\alpha}{2L}\right)}\\
&\quad + \sqrt{10\sigma^2T\sum_{i=1}^d G_i^2\ln\left(1+\frac{G_i^2T}{\alpha}\right)}~.
\end{align*}
\end{thm}

\begin{proof}[Proof of Theorem~\ref{thm:pcftrlp}]
The proof is nearly identical to that of Theorem \ref{thm:ftrlp}. We have
\begin{align*}
\E\left[\sum_{i=1}^d \left(\eta_i - \frac{L}{2}\eta_i^2\right)\sum_{t=1}^T \left(\PartialDerivativeGeneral{\bx_t}{i}\right)^2\right]
\le
&\ F(\bx_1)-F^\star + \frac{L\alpha }{2} \sum_{i=1}^d \left(\eta_i - \frac{1}{L}\right)^2\\
&\ +\sum_{i=1}^d \frac{5G_i^2}{L}\ln\left(1+\frac{G_i^2T}{\alpha}\right) +\sum_{i=1}^d \frac{L\eta_i^2\sigma_{i}^2T}{2}~.
\end{align*}
Define $A_i=\E\left[\sum_{t=1}^T \left(\PartialDerivativeGeneral{\bx_t}{i}\right)^2\right]$ and set $\eta_i=\frac{\alpha+A_i}{L(\alpha+A_i+\sigma_{i}^2T)}$ to obtain
\begin{equation*}
\sum_{i=1}^d \frac{(\alpha+A_i)^2}{2L\left(\alpha+A_i + \sigma_{i}^2T\right)}
\le F(\bx_1)-F^\star + \frac{d\alpha}{2L}+\sum_{i=1}^d \frac{5G_i^2}{L}\ln\left(1+\frac{G_i^2T}{\alpha}\right)~.
\end{equation*}
Now, the first statement of the Theorem follows by observing that each term on the LHS is non-negative so that the sum can be lower-bounded by any individual term. For the second statement, define
\begin{align*}
Q_i&=\frac{(\alpha+A_i)^2}{2L\left(\alpha+A_i + \sigma_{i}^2T\right)},\\
Q&=F(\bx_1)-F^\star + \frac{d\alpha}{2L}+\sum_{i=1}^d \frac{5G_i^2}{L}\ln\left(1+\frac{G_i^2T}{\alpha}\right),
\end{align*}
so that $\sum_{i=1}^d Q_i \le Q$. By the quadratic formula and definition of $Q_i$, we have
\begin{align*}
A_i &\le 2LQ_i + \sqrt{2LQ_i\sigma_{i}^2T} - \alpha~.
\end{align*}
Thus,
\begin{align*}
\sum_{i=1}^d A_i&\le 2LQ -d\alpha + \sum_{i=1}^d \sqrt{2LQ_i\sigma_{i}^2T}\\
&\le 2LQ -d\alpha + \sqrt{2L}\sqrt{\sum_{i=1}^d  Q_i}\sqrt{\sum_{i=1}^d\sigma_{i}^2T}\\
&=2LQ -d\alpha +  \sqrt{2LQ\sigma^2T}~.
\end{align*}
From which the second statement follows.
\end{proof}

\subsection{Summary}
In summary, we have presented a novel way to reduce the adaptation of step sizes for the stochastic optimization of smooth non-convex functions to an online convex optimization problem. The reduction goes through the use of novel surrogate convex losses. This framework allows us to use no-regret online algorithms to learn step sizes on the fly. The resulting algorithm has an optimal convergence guarantee for any level of noise, without the need to estimate the noise or tune the step sizes. The overall price to pay is a factor of 2 in the computation of the gradients. We also have presented a per-coordinate version of our algorithm that achieves faster convergence on the coordinates with less noise.

As a side note, the optimal convergence rate was also obtained by \citet{WardWB19} using AdaGrad global step sizes, without the need to tune parameters. \citet{LiO19} improves over the results of \citet{WardWB19} by removing the assumption of bounded gradients. However, both analyses focus on the adaptivity of non-per-coordinate updates and are somewhat complicated in order to deal with unbounded gradients or non-independence of the current step size from the current step gradient. In comparison, our technique is relatively simple, allowing us to easily show a nontrivial guarantee for per-coordinate updates.

The idea of tuning step sizes with online learning has been explored in the online convex optimization literature~\citep{KoolenvEG14,vanErvenK16}. There, the possible step sizes are discretized and an expert algorithm is used to select the step size to use online. Instead, in our work the use of convex surrogate loss functions allows us to directly learn the optimal step size, without needing to discretize the range of step sizes.

\section{Adaptation to Noise under the PL condition}
\label{sec:noise_adapt_pl}
\begin{figure}[t]
    \centering
    \includegraphics[width=\textwidth]{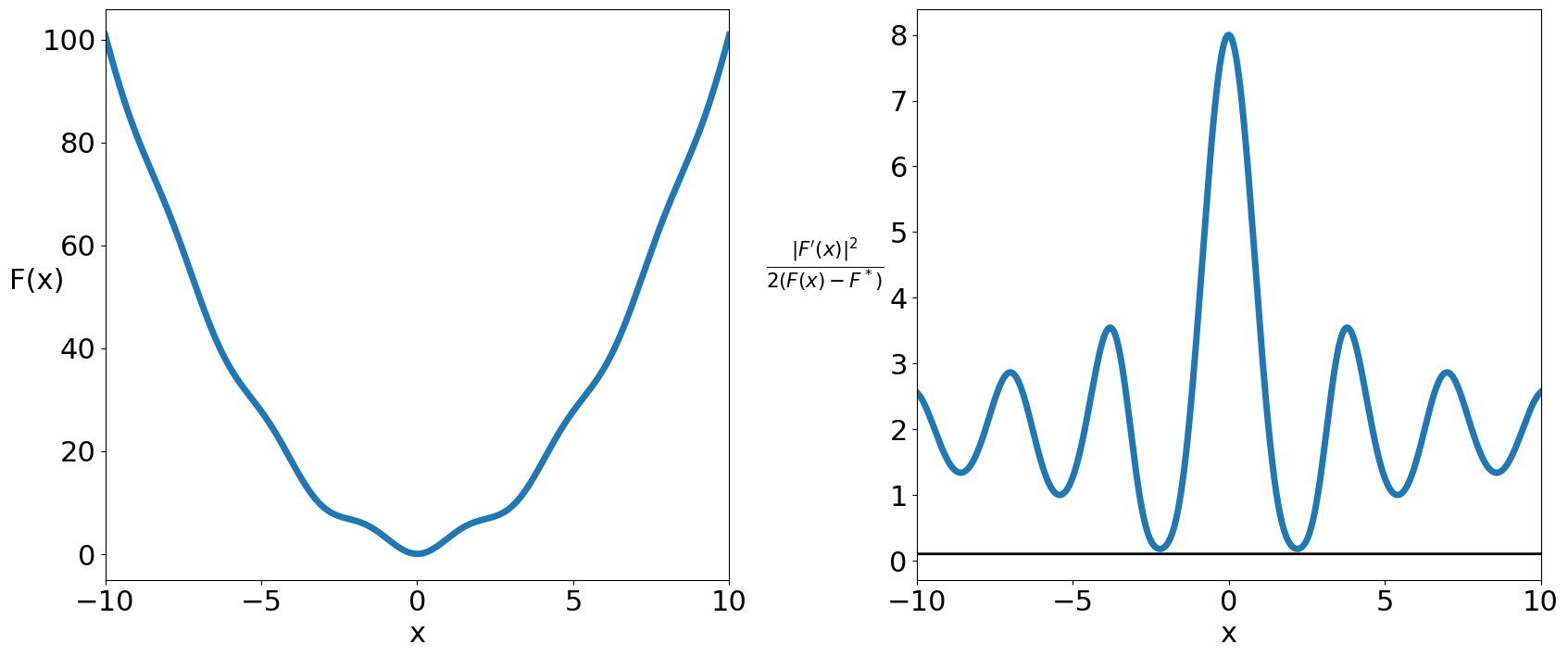}
    \caption[A function satisfying the PL condition]{The left figure plots the function $F(x) = x^2 + 3\sin^2(x)$ and the right one shows that it satisfies the PL condition where the black line is $y=\frac{1}{10}$.}
    \label{fig:pl_example}
\end{figure}

In the previous section, we showed how to adapt to the noise in the general smooth non-convex scenario. In practice, we might learn additional properties of the problem which we can use to choose/set algorithms to make it converge faster under these conditions. For example, if $F$ is a convex function, then GD can obtain a convergence rate of $O(\frac{1}{\sqrt{T}})$ with a step size proportional to $\frac1{\sqrt{T}}$; while if $F$ is in addition $L$-smooth and $\mu$-strongly-convex, GD with a constant step size $\frac1L$ guarantees a linear convergence rate of $O(\exp(-\frac{\mu}{L}T))$. Similarly, for the non-convex setting, there are conditions under which we can also get a linear rate~\citep{KarimiNS16}. One popular option is the Polyak-\L{}ojasiewicz (PL) condition~\citep{Polyak63,Lojasiewicz63}:
\begin{assumption}
A differentiable function $F:\R^d\rightarrow\R$ is said to satisfy the PL condition if for some $\mu > 0$ we have for any $\bx\in\R^d$ that
\[
\frac{1}{2} \| \nabla F(\bx) \|^2 \geq \mu \left(F(\bx) - F^{\star} \right)~.
\]
\end{assumption} 
In words, the gradient grows as at least as a quadratic function of the sub-optimality. As an example, we show in Figure~\ref{fig:pl_example} a function $F(x) = x^2 + 3\sin^2(x)$ which satisfies the PL condition with $\mu = \frac{1}{10}$~\citep{KarimiNS16}.

We want to stress that a function satisfying the PL condition is not necessarily convex as defined in~\eqref{eq:convex},
which can be clearly seen from Figure~\ref{fig:pl_example}. Yet, the PL condition does imply the weaker condition of \emph{invexity}~\citep{hanson1981sufficiency}. Recall that a function $F$ is invex if it is differentiable and there exists a vector valued function $\bv$ such that for any $\bx, \by \in \R^d$
, it holds that:
\begin{equation}
    \label{eq:invex}
    F(\by) \ge F(\bx) + \langle\nabla F(\bx), \bv(\by, \bx)\rangle~.
\end{equation}
Obviously, invexity admits convexity as a special case of $\bv(\by, \bx)=\by-\bx$. As a smooth function $F$ is \emph{invex} if and only if every
stationary point (namely a point where the gradient is zero) of $F$ is a global minimum~\citep{craven1985invex}, any smooth function that satisfies the PL condition must be invex~\citep{KarimiNS16}.

Though the PL condition is often considered a ``strong'' condition, it was formally proved to hold locally in a sufficiently large neighborhood of the random initialization when training deep neural networks in \citet{Allen-ZhuLS19}. Furthermore, \citet{KleinbergLY18} empirically observed that the loss surface of neural networks has good one-point convexity properties, and thus locally satisfies the PL condition, at least for the whole neighborhood along the SGD trajectory. For us, we only need it to hold along the optimization path and not over the entire space, as also pointed out in \citet{KarimiNS16}. So, while being strong, it actually models the cases we are interested in.
Moreover, dictionary learning~\citep{AroraGMM15}, phase retrieval~\citep{ChenC15}, and matrix completion~\citep{SunL16}, all satisfy the one-point convexity locally~\citep{Zhu18b}, and in turn they all satisfy the PL condition locally.

Apart from the linear rate GD obtains in the deterministic setting under the PL condition by using a constant step size, for the stochastic setting, the best rate we know of is $O(\frac{\sigma^2}{\mu^2T})$ obtainable using SGD with the decaying step size $O(\frac{1}{\mu t})$~\citep{KarimiNS16, LiZO21}. Clearly, to achieve the best performance in each setting, we need two completely different step size decay schedules.

Below, we prove that using two popular empirical step size schedules, the exponential and the cosine step sizes, SGD's convergence rate adapts to the noise.

\subsection{Exponential and Cosine Step Sizes}
\label{ssec:exp_cos_step_sizes}
Specifically, we use the following definition for the exponential step size
\begin{equation}
\label{eq: step_size}
\eta_t = \eta_0 \cdot \alpha^t,
\end{equation}
and for the cosine step size~\citep{LoshchilovH17}
\begin{equation}
\label{eq:cosine_step}
\eta_t = \frac{\eta_0 }{2}\left(1+ \cos \frac{t\pi}{T}\right)~.
\end{equation}

The exponential step size is simply an exponential decaying step size. It is less discussed in the optimization literature and it is also unclear who proposed it first, even if it has been known to practitioners for a long time and already included in many deep learning software libraries including TensorFlow~\citep{Tensorflow15} and PyTorch~\citep{Pytorch19}. Yet, no convergence guarantee has ever been proved for it. The closest strategy is the \emph{stagewise step decay}, which corresponds to the discrete version of the exponential step size we analyze.
The stagewise step decay uses a piece-wise constant step size strategy, where the step size is cut by a factor in each ``stage''.
This strategy is known with many different names: ``stagewise step size''~\citep{YuanYJY19}, ``step decay schedule''~\citep{GeKKN19}, ``geometrically decaying schedule''~\citep{DavisDXZ19}, and ``geometric step decay''~\citep{DavisDC19}. In this paper, we will call it stagewise step decay. The stagewise step decay approach was first introduced in \citep{Goffin77} and used in many \emph{convex} optimization problem~\citep[e.g.,][]{HazanK11,AybatFGO19,KulunchakovM19,GeKKN19}. Interestingly, \citet{GeKKN19} also shows promising empirical results on non-convex functions, but instead of using their proposed decay strategy, they use an exponentially decaying schedule, like the one we analyze here.
The only use of the stagewise step decay for non-convex functions we know are for sharp functions~\citep{DavisDC19} and weakly-quasi-convex functions~\citep{YuanYJY19}. However, they do not show any adaptation property and they still do not consider the exponential step size but rather its discrete version.

The cosine step size, which anneals the step size following a cosine function, has exhibited great power in practice but it does not have any theoretical justification. The cosine step size was originally presented in~\citet{LoshchilovH17} with two tunable parameters. Later, \citet{HeZZZXL19} proposed a simplified version of it with one parameter. However, there is no theory for this strategy though it is popularly used in the practical world \citep{LiuSY18,ZhangWZZ19, CubukZMVL19,ZhaoJK20,YouLHX20,ChenKNH20,GrillSATRBDABGGPKMV20}.

\subsection{Theoretical Analyses of Exponential and Cosine Step Sizes}
\label{ssec:exp_cos_convergence}
We now prove the convergence guarantees for these two step sizes showing that they can adapt to noise. We would use the following assumption on noises:

\begin{assumption}
\label{asp:noise_var_bounded_relax}
For $t = 1, 2, \dots, T$, we assume $\E_t [\| \nabla f(\bx_t, \xi_t) - \nabla F(\bx_t) \|^2 ] \leq a \| \nabla F(\bx_t) \|^2 + b $, where $a, b \geq 0$. 
\end{assumption}

This assumption on the noise is strictly weaker than the common assumption of assuming a bounded variance (Assumption~\ref{asp:noise_var_bounded}). Indeed, this assumption recovers the bounded variance case with $a=0$ while also allowing for the variance to grow unboundedly far from the optimum when $a>0$. This is indeed the case when the optimal solution has low training error and the stochastic gradients are generated by mini-batches. This relaxed assumption on the noise was first used by \citet{BertsekasT96} in the analysis of the asymptotic convergence of SGD.

\begin{thm}[SGD with exponential step size]
\label{thm: pl_smooth_cst_noise}
Assume $F$ to be $L$-smooth and $\mu$-PL and Assumption~\ref{asp:noise_unbiased} and~\ref{asp:noise_var_bounded_relax} to hold. For a given $T \geq \max\{3, \beta \}$ and $\eta_0 = (L(1+a))^{-1}$, with step size \eqref{eq: step_size}, SGD guarantees 
\begin{align*}
& \E F(\bx_{T+1}) - F^{\star} \leq \frac{5LC(\beta)}{e^2 \mu^2 } \frac{\ln^2 \frac{T}{\beta}}{T} b
+ C(\beta) \exp\left(-\frac{0.69\mu }{L+a} \left(\frac{T}{\ln \frac{ T}{\beta}}\right)\right)\cdot (F(\bx_1) - F^{\star}), 
\end{align*}
where $C(\beta)\triangleq \exp \left((2\mu\beta)/(L (1+a)\ln T/\beta)\right)$.
\end{thm}

\textbf{Choice of $\beta$} 
Note that if $\beta = L(1+a)/\mu$, we get 
\begin{align*}
\E F(\bx_{T+1}) - F^{\star}
\leq O\left(\exp\left(-\frac{\mu }{L+a} \left(\frac{T}{\ln \frac{\mu T}{L}}\right)\right)+ \frac{b \ln^2 \frac{\mu T}{L}}{\mu^2  T} \right)~.
\end{align*}
In words, this means that we are basically free to choose $\beta$, but will pay an exponential factor in the mismatch between $\beta$ and $\frac{L}{\mu}$, which is basically the condition number for PL functions. This has to be expected because it also happens in the easier case of stochastic optimization of strongly convex functions~\citep{MoulinesB11}.

\begin{thm}[SGD with cosine step size]
\label{thm:PL_cosine}
Assume $F$ to be $L$-smooth and $\mu$-PL and Assumption~\ref{asp:noise_unbiased} and~\ref{asp:noise_var_bounded_relax} to hold. For a given $T$ and $\eta_0 = (L(1+a))^{-1}$, with step size \eqref{eq:cosine_step}, SGD guarantees 
\begin{align*}
\E F(\bx_{t+1}) - F^{\star} 
\leq
&\exp \left(- \frac{\mu (T-1)}{2L(1+a)}\right) (F(x_1) - F^{\star})\\
& \quad + \frac{ \pi^4 b}{32 (1+a)T^4} \left( \left(\frac{8T^2}{\mu}\right)^{4/3} + \left(\frac{6T^2}{\mu}\right)^{\frac{5}{3}}\right) ~.
\end{align*}
\end{thm}

\textbf{Adaptivity to Noise} From the above theorems, we can see that both the exponential step size and the cosine step size have a provable advantage over polynomial ones: \emph{adaptivity to the noise}. Indeed, when $b=0$, namely there is only noise relative to the distance from the optimum, they both guarantee a linear rate. Meanwhile, if there is noise, using the \emph{same step size without any tuning}, the exponential step size recovers the rate of $O\left(1/(\mu^2 T)\right)$ while the cosine step size achieves the rate of $O(1/(\mu^{\frac{5}{3}}T^{\frac{2}{3}}))$ (up to poly-logarithmic terms). In contrast, polynomial step sizes would require two different settings---decaying vs constant---in the noisy vs no-noise situation~\citep{KarimiNS16}.
It is worth stressing that the rate in Theorem~\ref{thm: pl_smooth_cst_noise} is one of the first results in the literature on the stochastic optimization of smooth PL functions \citep{khaled2020better}.

It is worth reminding the reader that \emph{any} polynomial decay of the step size does not give us this adaptation. So, let's gain some intuition on why this should happen with these two step sizes. In the early stage of the optimization process, we can expect that the disturbance due to the noise is relatively small compared to how far we are from the optimal solution. Accordingly, at this phase, a near-constant step size should be used. This is exactly what happens with \eqref{eq: step_size} and \eqref{eq:cosine_step}. On the other hand, when the iterate is close to the optimal solution, we have to decrease the step size to fight the effects of the noise. In this stage, the exponential step size goes to 0 as $O \left(1/T \right)$, which is the optimal step size used in the noisy case. Meanwhile, the last $i$th cosine step size is $\eta_{T-i} = \frac{\eta_0}{2}(1- \cos\frac{i \pi }{T})= \eta_0 \sin^2 \frac{i\pi}{2T}$, which amounts $O (1/T^2)$ when $i$ is much smaller than $T$.

\textbf{Optimality of the bounds} As far as we know, it is unknown if the rate we obtain for the optimization of non-convex smooth functions under the PL condition is optimal or not. However, up to poly-logarithmic terms, Theorem~\ref{thm: pl_smooth_cst_noise} matches at the same time the best-known rates for the noisy and deterministic cases~\citep{KarimiNS16}. We would remind the reader that this rate is not comparable with the one for strongly convex functions which is $O(1/(\mu T))$.
Meanwhile, cosine step size achieves a rate slightly worse in $T$ (but better in $\mu$) under the same assumptions.

Before proving the above theorems, we first introduce some technical lemmas.

\begin{lemma}
\label{lemma:start}
Assume $F$ to be $L$-smooth, Assumption~\ref{asp:noise_unbiased} and~\ref{asp:noise_var_bounded_relax} to hold, and $\eta_t \leq \frac{1}{L(1+a)}$, then SGD guarantees 
\begin{equation}
\label{eq:thm2_eq1}
\begin{split}
\E F(\bx_{t+1})- \E  F(\bx_t)
\leq -  \frac{\eta_t}{2} \E \| \nabla F(\bx_t) \|^2 + \frac{L \eta_t^2 b}{2}~. 
\end{split}
\end{equation}
\end{lemma}
\begin{proof}[Proof of Lemma~\ref{lemma:start}]
By \eqref{eq:smooth}, we have
\begin{equation}
\label{eq:smooth_one_step}
F(\bx_{t+1}) \leq F(\bx_t) - \langle \nabla F(\bx_t), \eta_t\bg_t \rangle + \frac{L}{2} \eta_t^2 \| \bg_t \|^2~.
\end{equation}
Taking expectation on both sides, we get
\begin{align*}
\E F(\bx_{t+1})- \E F(\bx_t)
&\leq
- \left(\eta_t - \frac{L(a+1)}{2} \eta_t^2 \right)\E \| \nabla F(\bx_t) \|^2 + \frac{L}{2}\eta_t^2 b\\
&\leq
- \frac{1}{2}\eta_t \E \| \nabla F(\bx_t) \|^2 + \frac{L}{2}\eta_t^2 b,
\end{align*}
where in the last inequality we used the fact that $\eta_t \leq \frac{1}{L(1+a)}$.
\end{proof}

\begin{lemma}
\label{lemma: ratio_bound}
Assume $X_k, A_k, B_k \geq 0, k = 1 ,...$, and $X_{k+1} \leq A_k X_k + B_k$, we have 
\[
X_{k+1} \leq \prod_{i=1}^k A_i X_1 + \sum_{i=1}^{k} \prod_{j=i+1}^k A_j B_i~. 
\]
\end{lemma}
\begin{proof}[Proof of Lemma~\ref{lemma: ratio_bound}]
When $k=1$, $X_2 \leq A_1 X_1 + B_1$ satisfies. By induction, assume $X_{k} \leq \prod_{i=1}^{k-1} A_i X_1 + \sum_{i=1}^{k-1}\prod_{j=i+1}^{k-1} A_j B_i$, and we have
\begin{align}
X_{k+1}
&\leq A_k \left( \prod_{i=1}^{k-1} A_i X_1 + \sum_{i=1}^{k-1} \prod_{j=i+1}^{k-1} A_j B_i \right) + B_k\\
&= \prod_{i=1}^k A_i X_1 + \sum_{i=1}^{k-1} \prod_{j=i+1}^{k} A_j B_i + A_k B_k \\
&= \prod_{i=1}^k A_i X_1 + \sum_{i=1}^k \prod_{j=i+1}^k A_j B_i~. \qedhere
\end{align}
\end{proof}

\begin{lemma}
\label{lemma:sum_cosine}
For any $T \geq 1$, we have $\sum_{t=1}^{T} \cos\frac{t \pi}{T} = -1$.
\end{lemma}
\begin{proof}[Proof of Lemma~\ref{lemma:sum_cosine}]
If $T$ is odd, we have
\begin{align*}
\sum_{t=1}^{T} \cos\frac{t \pi}{T}
= \cos \frac{T\pi }{T} + \sum_{t=1}^{(T-1)/2} \cos \frac{t \pi}{T} + \cos \frac{(T-t)\pi}{T} 
 = \cos \pi = -1,
\end{align*}
where in the second inequality we used the fact that $\cos (\pi - x) = - \cos (x)$ for any $x$.
If $T$ is even, we have
\begin{equation}
\sum_{t=1}^{T} \cos\frac{t \pi}{T}
= \cos \frac{T\pi }{T} + \cos \frac{T\pi }{2T} + \sum_{t=1}^{T/2 - 1} \cos \frac{t \pi}{T} + \cos \frac{(T-t)\pi}{T} 
= \cos \pi = -1~.
\end{equation}
\end{proof}

\begin{lemma}
\label{lemma: ineq_constant}
For $T \geq 3$, $\alpha \geq 0.69$ and $\frac{ \alpha^{T+1}}{(1-\alpha)} \leq \frac{2\beta}{\ln \frac{T}{\beta}}$. 
\end{lemma}
\begin{proof}[Proof of Lemma~\ref{lemma: ineq_constant}]
We have
\begin{align*}
\frac{\alpha^{T+1}}{(1-\alpha)}
= \frac{\alpha \beta }{T (1-\alpha)}
= \frac{\beta }{T\left(1 - \exp\left(-\frac{1}{T} \ln \frac{T}{\beta}\right)\right)}
\leq \frac{2\beta }{\ln \frac{T}{\beta}},
\end{align*}
where in the last inequality we used $\exp(-x) \leq 1- \frac{x}{2}$ for $0 < x < \frac{1}{e}$ and the fact that $\frac{1}{T} \ln\left(\frac{T}{\beta}\right) \leq \frac{\ln T}{T} \leq \frac{1}{e}$.
\end{proof}

\begin{lemma}
\label{lemma: ineq_alpha}
$
1-x \leq \ln \left(\frac{1}{x}\right), \forall x > 0. 
$
\end{lemma}
\begin{proof}[Proof of Lemma~\ref{lemma: ineq_alpha}]
It is enough to prove that $f(x) := x - 1- \ln x \geq 0$. Observe that $f'(x)$ is increasing and $f'(1) = 0$, hence, we have $f(x) \geq f(1) = 0$.
\end{proof}

\begin{lemma}
\label{lemma: integral_bound}
Let $a,b\geq0$. Then 
\[
\sum_{t=0}^T \exp(-b t) t^a \leq 2\exp(-a)\left(\frac{a}{b}\right)^a+ \frac{\Gamma(a+1)}{b^{a+1}}~.
\]
\end{lemma}
\begin{proof}[Proof of Lemma~\ref{lemma: integral_bound}]
Note that $f(t)=\exp(-b t) t^a$ is increasing for $t\in [0,a/b]$ and decreasing for $t\geq a/b$. Hence, we have
\begin{align}
\sum_{t=0}^T \exp(-b t) t^a
&\leq
\sum_{t=0}^{\lfloor a/b \rfloor-1} \exp(-b t) t^a + \exp(-b \lfloor a/b \rfloor) \lfloor a/b \rfloor^a + \exp(-b \lceil a/b \rceil) \lceil a/b \rceil^a\\
&\quad+\sum_{\lceil a/b \rceil+1}^T \exp(-b t) t^a \\
&\leq 2\exp(-a)(a/b)^a+\int_{0}^{\lfloor a/b \rfloor} \exp(-b t) t^a dt + \int_{\lceil a/b \rceil}^T \exp(-b t) t^a dt\\
&\leq 2\exp(-a)(a/b)^a+\int_{0}^{T} \exp(-b t) t^a dt \\
&\leq 2\exp(-a)(a/b)^a+\int_{0}^{\infty} \exp(-b t) t^a dt \\
&=2\exp(-a)(a/b)^a+ \frac{1}{b^{a+1}} \Gamma(a+1)~. \qedhere
\end{align}
\end{proof}

We can now prove both Theorem~\ref{thm: pl_smooth_cst_noise} and Theorem~\ref{thm:PL_cosine}.
\begin{proof}[Proof of Theorem~\ref{thm: pl_smooth_cst_noise} and Theorem~\ref{thm:PL_cosine}.]
Denote $\E[F(\bx_t)] - F^{\star}$ by $\Delta_t$. From\\ Lemma~\ref{lemma:start} and the PL condition, we get
\begin{equation}
\Delta_{t+1} 
\leq  (1 - \mu \eta_t ) \Delta_t + \frac{L}{2} \eta_t^2 b^2~. 
\end{equation} 
By Lemma~\ref{lemma: ratio_bound} and $1- x \leq \exp (-x)$, we have 
\begin{align}
\Delta_{T+1} 
&  \leq \prod_{t=1}^{T} (1- \mu \eta_t) \Delta_1 +  \frac{L}{2} \sum_{t=1}^{T} \prod_{i=t+1}^{T} (1- \mu \eta_i) \eta_t^2 b\\
&  \leq \exp \left(- \mu \sum_{t=1}^{T} \eta_t \right) \Delta_1  +  \frac{Lb}{2} \sum_{t=1}^{T} \exp \left( - \mu \sum_{i=t+1}^{T} \eta_i \right) \eta_t^2~. 
\end{align}
We then show that both the exponential step size and the cosine step size satisfy $\sum_{t=1}^{T} \eta_t = \Omega (T)$, which guarantees a linear rate in the noiseless case.

For the cosine step size \eqref{eq:cosine_step}, we observe that 
\begin{align*}
\sum_{t=1}^{T} \eta_t
 = \frac{\eta_0 T}{2} + \frac{\eta_0}{2} \sum_{t=1}^{T} \cos \frac{t\pi}{T} = \frac{\eta_0 (T-1)}{2},
\end{align*}
where in the last equality we used Lemma~\ref{lemma:sum_cosine}. 

Also, for the exponential step size \eqref{eq: step_size}, we can show that
\begin{align*}
 \sum_{t=1}^{T} \eta_t = \eta_0 \frac{\alpha - \alpha^{T+1}}{1- \alpha}
 \geq  \frac{\eta_0 \alpha}{1- \alpha} - \frac{2\eta_0\beta }{\ln \frac{T}{\beta}}
 \geq  T \cdot  \frac{0.69\eta_0}{\ln \frac{T}{\beta}} -  \frac{2\eta_0\beta }{\ln \frac{T}{\beta}},
\end{align*}
where we used Lemma~\ref{lemma: ineq_constant} in the first inequality and Lemma~\ref{lemma: ineq_alpha} in the second. 

Next, we upper bound $\sum_{t=1}^{T} \exp \left( - \mu \sum_{i=t+1}^{T} \eta_i \right) \eta_t^2$ for the two step sizes.

For the exponential step size, by Lemma~\ref{lemma: ineq_constant}, we obtain
\begin{align*}
\sum_{t=1}^{T} \exp \left( - \mu \sum_{i=t+1}^{T} \eta_i \right) \eta_t^2
& = \eta_0^2 \sum_{t=1}^{T} \exp \left( - \mu \eta_0 \frac{\alpha^{t+1} - \alpha^{T+1}}{1- \alpha}\right) \alpha^{2t}\\
& \leq \eta_0^2 C(\beta) \sum_{t=1}^{T} \exp \left( - \frac{ \mu \eta_0\alpha^{t+1}}{1- \alpha}\right) \alpha^{2t}\\
& \leq \eta_0^2 C(\beta)  \sum_{t=1}^{T} \left(\frac{e}{2} \frac{\mu \alpha^{t+1}}{L(1+a)(1-\alpha)}\right)^{-2} \alpha^{2t}\\
& \leq \frac{4L^2(1+a)^2}{e^2\mu^2} \sum_{t=1}^T\frac{1}{\alpha^2} \ln^2 \left(\frac{1}{\alpha}\right)\\
& \leq  \frac{10 L^2(1+a)^2\ln^2 \frac{T}{\beta}}{e^2 \mu^2 T}, 
\end{align*}
where in the second inequality we used $\exp(-x) \leq \left(\frac{\gamma}{e x}\right)^\gamma, \forall x >0, \gamma>0$. 

For the cosine step size, using the fact that $\sin x \geq \frac{2}{\pi}x$ for $ 0 \leq x \leq \frac{\pi}{2}$, we can lower bound $\sum_{i=t+1}^{T} \eta_i $ by
\begin{align*}
\sum_{i=t+1}^{T} \eta_i 
= \frac{\eta_0 }{2}\sum_{i=t+1}^{T}  \left(1+ \cos \frac{i \pi}{T}\right)
= \frac{\eta_0 }{2}\sum_{i=0}^{T-t-1}  \sin^2 \frac{i \pi}{2T} 
\geq \frac{\eta_0}{2T^2}\sum_{i=0}^{T-t-1}  i^2
\geq \frac{\eta_0(T-t-1)^3}{6T^2}~. 
\end{align*}
Then, we proceed to get
\begin{align*}
\sum_{t=1}^{T} \exp \left(- \mu \sum_{i=t+1}^{T} \eta_i\right) \eta_t^2
& \leq \frac{\eta_0^2}{4} \sum_{t=1}^{T} \left(1+ \cos \frac{t\pi}{T}\right)^2 \exp \left(- \frac{\mu \eta_0(T-t-1)^3}{6T^2}\right)\\
& = \frac{ \eta_0^2}{4} \sum_{t=1}^{T-1} \left(1-  \cos \frac{t\pi}{T}\right)^2 \exp \left(- \frac{\eta_0 \mu (t-1)^3}{6T^2}\right) \\
& =\eta_0^2\sum_{t=1}^{T-1} \sin^4 \frac{t\pi}{2T} \exp \left(- \frac{\eta_0 \mu (t-1)^3}{6T^2}\right) \\
& \leq \frac{\eta_0^2 \pi^4}{16T^4}\sum_{t=0}^{T-1}t^4  \exp \left(- \frac{\eta_0 \mu t^3}{6T^2}\right)\\
& \leq \frac{\eta_0 \pi^4}{16 T^4} \left(2 \exp\left(-\frac{4}{3}\right) \left(\frac{8T^2}{\mu}\right)^{4/3} + \left(\frac{6T^2}{\mu}\right)^{\frac{5}{3}}\right), 
\end{align*}
where in the third line we used $\cos(\pi-x) = - \cos(x)$, in the forth line we used $1- \cos(2x) = 2\sin^2(x)$, and in the last inequality we applied Lemma~\ref{lemma: integral_bound}. 

Putting things together, we get the stated bounds. 
\end{proof}

\subsection{Experiments Comparing Exponential and Cosine Step Sizes with Other Optimizers}
\label{ssec:exp_cos_exps}

The empirical performance of the exponential and the cosine step sizes is already well-known in the applied world and does not require additional validation. However, both step sizes are often missing as baselines in recent empirical evaluations. Hence, the main aim of this section is to provide a comparison of the exponential and the cosine step sizes to other popular state-of-the-art step size schedules. All experiments are done in PyTorch~\citep{Pytorch19} and the codes can be found at \url{https://github.com/zhenxun-zhuang/SGD-Exponential-Cosine-Stepsize}.

We conducted experiments using deep neural networks to do image classification tasks on various datasets with different network architectures. 

\textbf{Datasets}
We consider the image classification task on CIFAR-10/100 and FashionMNIST. For all datasets, we randomly select 10\% training images for validation.

\textbf{Data Normalization and Augmentation}
Images are normalized per channel using the means and standard deviations computed from all training images. For CIFAR-10/100, we adopt the data augmentation technique following~\citet{LeeXGZT15} (for training only): 4 pixels are padded on each side of an image and a $32\times32$ crop is randomly sampled from the padded image or its horizontal flip.

\textbf{Models}
For FashionMNIST, we use a CNN model consisting of two alternating stages of $5\times5$ convolutional filters and $2\times2$ max-pooling followed by one fully connected layer of 1024 units. To reduce overfitting, 50\% dropout noise is used during training. For CIFAR-10, we employ the 20-layer Residual Network model~\citep{HeZRS16}. For CIFAR-100, we utilize the DenseNet-BC model~\citep{HuangLVW17} with 100 layers and a growth rate of 12. The loss is cross-entropy. 

\textbf{Training}
During the validation stage, we tune each method using grid-search to select the hyperparameters that work best according to their respective performance on the validation set. At the testing stage, the best performing hyperparameters from the validation stage are employed to train the model over all training images. The testing stage is repeated with random seeds 5 times to eliminate the influence of stochasticity.

We use Nesterov momentum~\citep{Nesterov83} of 0.9 without dampening (if having this option), weight-decay of 0.0001 (FashionMNIST and CIFAR-10), and 0.0005 (CIFAR100), and use a batch size of 128. Regarding the employment of Nesterov momentum, we follow the setting of~\citet{GeKKN19}. The use of momentum is essential to have a fair and realistic comparison in that the majority of practitioners would use it when using SGD. 

\textbf{Optimization methods} We consider SGD with the following step size decay schedules:
\begin{equation}\label{eq:decays}
\begin{split}
&\eta_t = \eta_0\cdot\alpha^t; \quad
\eta_t = \eta_0(1+\alpha\sqrt{t})^{-1}; \\
&\eta_t = \eta_0(1+\alpha t)^{-1}; \quad
\eta_t = \eta_0/2 \left(1+\cos\left(t\pi/T\right)\right),
\end{split}
\end{equation}
where $t$ is the iteration number (instead of the number of epochs). We also compare with Adam~\citep{KingmaB15}, SGD+Armijo~\citep{VaswaniMLSGLJ19}, PyTorch's ReduceLROnPlateau scheduler\footnote{\url{https://pytorch.org/docs/stable/optim.html}} and stagewise step decay. We will call the place of decreasing the step size in stagewise step decay a \textbf{milestone}. (As a side note, since we use Nesterov momentum in all SGD variants, the stagewise step decay basically covers the performance of multistage accelerated algorithms \citep[e.g.,][]{AybatFGO19}.)

\paragraph{Hyperparameter tuning} We tune the hyperparameters on the validation set using the following two-stage grid searching strategy. First, search over a coarse grid and select the one yielding the best validation results. Next, continue searching in a fine grid centering at the best-performing hyperparameters found in the coarse stage, and in turn, take the best one as the final choice.

For the starting step size $\eta_0$, the coarse searching grid is \{0.00001, 0.0001, 0.001, 0.01, 0.1, 1\}, and the fine grid is like \{0.006, 0.008, 0.01, 0.02, 0.04\} if the best one in the coarse stage is 0.01.

For the $\alpha$ value, we set its searching grid so that the ratio $\eta_T/\eta_0$, where $\eta_T$ is the step size in the last iteration, is first searched over the coarse grid of \{0.00001, 0.0001, 0.001, 0.01, 0.1, 1\}, and then over a fine grid centered at the best one of the coarse stage. Note that we try all pairs of $(\eta_0, \alpha)$ from their respective searching grids.

For the stagewise step decay, to make the tuning process more thorough, we modify as follows the one employed in Section 6.1 (specifically on tuning SGD V1) of \citet{YuanYJY19}, where they first set two milestones and then tune the starting step size. Put it explicitly and take the experiment on CIFAR-10 as an example, we first run vanilla SGD with a constant step size to search for a good range of starting step size on the grid \{0.00001, 0.0001, 0.001, 0.01, 0.1, 1\}, and find 0.01 and 0.1 work well. Based on this, we set the fine searching grid of starting step sizes as \{0.007, 0.01, 0.04, 0.07, 0.1, 0.4\}. For each of them, we run three settings with an increasing number of milestones: vanilla SGD (with no milestone), SGD with 1 milestone, and SGD with 2 milestones. The searching grid for milestones is \{16k, 24k, 32k, 40k, 48k, 56k\} (number of iterations). For the 1 milestone setting, the milestone can be any of them. For the 2 milestones, they can be any combination of two different elements from the searching grid, like (16k, 32k) or (32k, 48k). The grid search strategy for FashionMNIST and CIFAR-100 is similar but with the searching grid for milestones over \{3k, 6k, 9k, 12k, 15k, 18k\}.

\begin{figure*}[t]
\centering
\includegraphics[width=\textwidth]{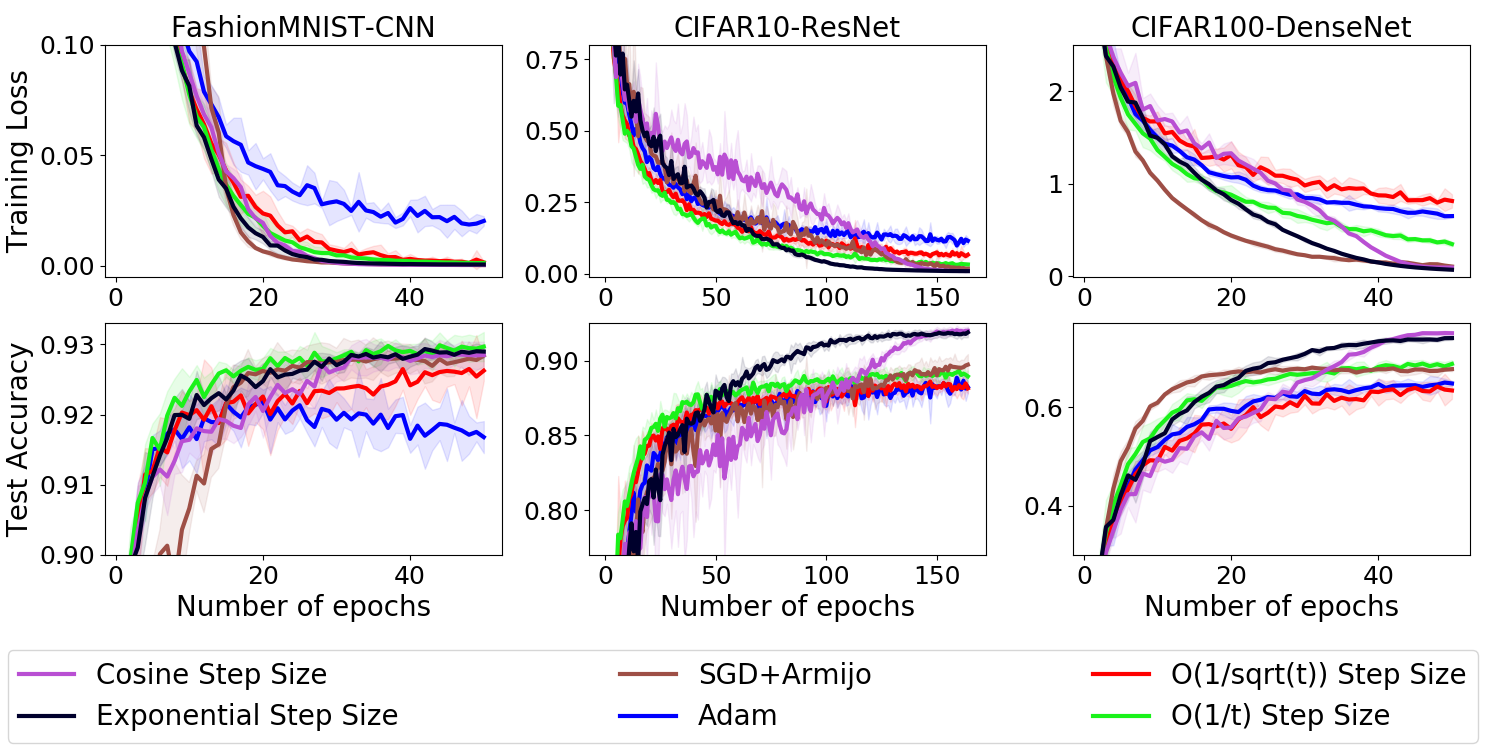}
\vspace{-2em}
\caption[Performance of Exponential and Cosine Step sizes vs.~others (1)]{Training loss (top plots) and test accuracy (bottom plots) curves on employing different step size schedules to do image classification using a simple CNN for FashionMNIST (left), a 20-layer ResNet for CIFAR-10 (middle), and a 100-layer DenseNet on CIFAR-100 (right). \emph{(The shading of each curve represents the 95\% confidence interval computed
across five independent runs from random initial starting points.)}}
\label{fig:stepsize}
\end{figure*}

\begin{figure*}[t]
\centering
\includegraphics[width=\textwidth]{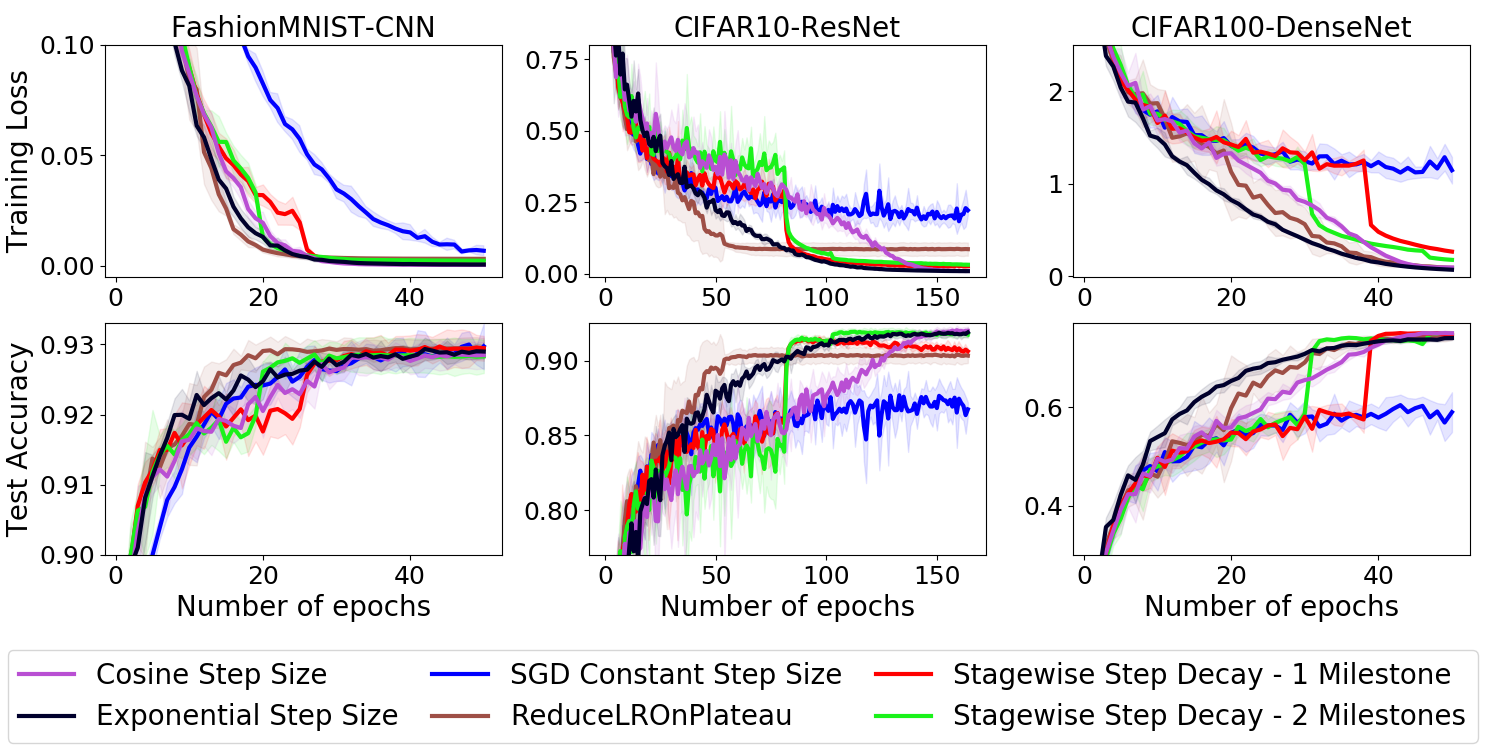}
\vspace{-2em}
\caption[Performance of Exponential and Cosine Step sizes vs.~others (2)]{Training loss (top plots) and test accuracy (bottom plots) curves comparing the exponential and cosine step sizes with stagewise step decay for image classification using a simple CNN for FashionMNIST (left), a 20-layer ResNet for CIFAR-10 (middle), and a 100-layer DenseNet on CIFAR-100 (right). \emph{(The shading of each curve represents the 95\% confidence interval computed
across five independent runs from random initial starting points.)}}
\label{fig:stagewise}
\end{figure*}

The PyTorch ReduceLROnPlateau scheduler takes multiple arguments, among which we tune the starting step size, the factor argument which decides by which the step size will be reduced, the patience argument which controls the number of epochs with no improvement after which the step size will be reduced, and the threshold argument which measures the new optimum to only focus on significant changes. We choose the searching grid for the starting step size using the same strategy for stagewise step decay above: run SGD with a constant step size to search for a good starting step size, then search over a grid centering on the found value, which results in the grid \{0.004, 0.007, 0.01, 0.04, 0.07\} (FashionMNIST) and \{0.01, 0.04, 0.07, 0.1, 0.4\} (CIFAR10/100). We also explore the searching grid of the factor argument over \{0.1, 0.5\}, the patience argument over \{5, 10\} (CIFAR10) or \{3, 6\} (FashionMNIST/CIFAR100), and the threshold argument over \{0.0001, 0.001, 0.01, 0.1\}.

For each setting, we choose the combination of hyperparameters that gives the best final validation loss to be used in testing. Also, whenever the best-performing hyperparameters lie in the boundary of the searching grid, we always extend the grid to make the final best-performing hyperparameters fall into the interior of the grid.

\begin{table*}[!ht]
\caption[Final results achieved by exponential and cosine step sizes vs.~other optimizers on various deep learning tasks]{Average final training loss and test accuracy achieved by each method when optimizing respective models on each dataset. The $\pm$ shows $95\%$ confidence intervals of the mean loss/accuracy value over 5 runs starting from different random seeds.}
\label{tab:exp_cos_results}
\begin{subtable}[t]{\textwidth}
\caption{FashionMNIST -- $3$ layer CNN}
\label{tab:exp_cos_results_FMnist}
\centering
{\small{\begin{tabular}{|c|c|c|}
\hline
Methods & Training loss & Test accuracy \\
\hline
SGD Constant Step Size & $0.0068 \pm 0.0023$ & $0.9297 \pm 0.0033$ \\
\hline
$O(1/t)$ Step Size & $0.0013 \pm 0.0004$ & $\mathbf{0.9297 \pm 0.0021}$ \\
\hline
$O(1/\sqrt{t})$ Step Size & $0.0016 \pm 0.0005$ & $0.9262 \pm 0.0014$\\
\hline
Adam & $0.0203 \pm 0.0021$ & $0.9168 \pm 0.0023$\\
\hline
SGD+Armijo & $\mathbf{0.0003 \pm 0.0000}$ & $0.9284 \pm 0.0016$\\
\hline
ReduceLROnPlateau & $0.0031 \pm 0.0009$ & $0.9294 \pm 0.0015$\\
\hline
Stagewise - 1 Milestone & $0.0007 \pm 0.0002$ & $0.9294 \pm 0.0018$\\
\hline
Stagewise - 2 Milestones & $0.0023 \pm 0.0005$ & $0.9283 \pm 0.0024$\\
\hline
Exponential Step Size & $0.0006 \pm 0.0001$ & $0.9290 \pm 0.0009$\\
\hline
Cosine Step Size & $0.0004 \pm 0.0000$ & $0.9285 \pm 0.0019$\\
\hline
\end{tabular}}}
\end{subtable}

\vspace{1em}
\begin{subtable}[t]{\textwidth}
\caption{CIFAR10 -- $20$ layer Resnet}
\label{tab:exp_cos_results_c10}
\centering
{\small{\begin{tabular}{|c|c|c|}
\hline
Methods & Training loss & Test accuracy \\
\hline
SGD Constant Step Size & $0.2226 \pm 0.0169$ & $0.8674 \pm 0.0048$\\
\hline
$O(1/t)$ Step Size & $0.0331 \pm 0.0028$ & $0.8894 \pm 0.0040$\\
\hline
$O(1/\sqrt{t})$ Step Size & $0.0672 \pm 0.0086$ & $0.8814 \pm 0.0034$\\
\hline
Adam & $0.1161 \pm 0.0111$ & $0.8823 \pm 0.0041$\\
\hline
SGD+Armijo & $0.0185 \pm 0.0043$ & $0.8973 \pm 0.0071$\\
\hline
ReduceLROnPlateau & $0.0867 \pm 0.0230$ & $0.9033 \pm 0.0049$\\
\hline
Stagewise - 1 Milestone & $0.0269 \pm 0.0017$ & $0.9062 \pm 0.0020$\\
\hline
Stagewise - 2 Milestones & $0.0322 \pm 0.0008$ & $0.9174 \pm 0.0020$\\
\hline
Exponential Step Size & $\mathbf{0.0098 \pm 0.0010}$ & $0.9188 \pm 0.0033$\\
\hline
Cosine Step Size & $0.0106 \pm 0.0008$ & $\mathbf{0.9199 \pm 0.0029}$\\
\hline
\end{tabular}}}
\end{subtable}

\vspace{1em}
\begin{subtable}[t]{\textwidth}
\caption{CIFAR100 -- $100$ layer DenseNet-BC}
\label{tab:exp_cos_results_c100}
\centering
{\small{\begin{tabular}{|c|c|c|}
\hline
Methods & Training loss & Test accuracy \\
\hline
SGD Constant Step Size & $1.1467 \pm 0.1437$ & $0.5896 \pm 0.0404$ \\
\hline
$O(1/t)$ Step Size & $0.3489 \pm 0.0263$ & $0.6874 \pm 0.0076$ \\
\hline
$O(1/\sqrt{t})$ Step Size & $0.8147 \pm 0.0717$ & $0.6336 \pm 0.0169$ \\
\hline
Adam & $0.6513 \pm 0.0154$ & $0.6478 \pm 0.0054$ \\
\hline
SGD+Armijo & $0.1063 \pm 0.0153$ & $0.6768 \pm 0.0044$ \\
\hline
ReduceLROnPlateau & $0.0927 \pm 0.0085$ & $0.7435 \pm 0.0076$ \\
\hline
Stagewise - 1 Milestone & $0.2673 \pm 0.0084$ & $0.7459 \pm 0.0030$ \\
\hline
Stagewise - 2 Milestones & $0.1783 \pm 0.0030$ & $0.7487 \pm 0.0025$ \\
\hline
Exponential Step Size & $\mathbf{0.0714 \pm 0.0041}$ & $0.7398 \pm 0.0037$ \\
\hline
Cosine Step Size & $0.0949 \pm 0.0053$ & $\mathbf{0.7497 \pm 0.0044}$ \\
\hline
\end{tabular}}}
\end{subtable}
\end{table*}

\textbf{Results and discussions}
The exact loss and accuracy values are reported in Table~\ref{tab:exp_cos_results}. To avoid overcrowding the figures, we compare the algorithms in groups of baselines. The comparison of performance between step size schemes listed in~\eqref{eq:decays}, Adam, and SGD+Armijo are shown in Figure~\ref{fig:stepsize}. As can be seen, the \emph{only} two methods that perform well on \emph{all} 3 datasets are cosine and exponential step sizes. In particular, the cosine step size performs the best across datasets both in training loss and test accuracy, with the exponential step size following closely.

On the other hand, as we noted above, stagewise step decay is a very popular decay schedule in deep learning. Thus, our second group of baselines in Figure~\ref{fig:stagewise} is composed of the stagewise step decay, ReduceLROnPlateau, and SGD with constant step size. The results show that exponential and cosine step sizes can still match or exceed the best of them with a fraction of their needed time to find the best hyperparameters. Indeed, we need 4 hyperparameters for two milestones, 3 for one milestone, and at least 4 for ReduceLROnPlateau. In contrast, the cosine step size requires only 1 hyperparameter and the exponential one needs 2.

Note that we do not pretend that our benchmark of the stagewise step decay is exhaustive. Indeed, there are many unexplored (potentially infinite!) possible hyperparameter settings. For example, it is reasonable to expect that adding even more milestones at the appropriate times could lead to better performance.
However, this would result in a linear growth of the number of hyperparameters leading to an exponential increase in the number of possible location combinations.
\begin{wrapfigure}{c}{0.6\linewidth}
\begin{center}
\includegraphics[width=\linewidth]{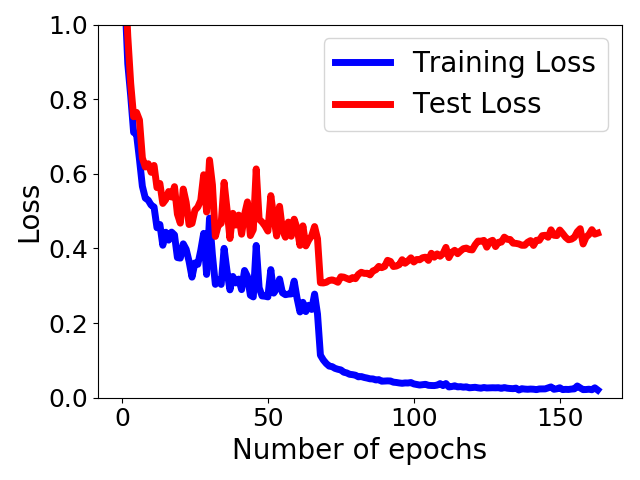}
\end{center}
\caption[Decreasing the step size too soon leads to overfitting]{Plot showing that decreasing the step size too soon would lead to overfitting (ResNet20 on CIFAR10).}
\label{fig:overfitting}
\end{wrapfigure}
This in turn causes the rapid growth of tuning time in selecting a good set of milestones in practice. Worse still, even the intuition that one should decrease the step size once the test loss curve stops decreasing is not always correct. Indeed, we observed in experiments (see Figure~\ref{fig:overfitting}) that doing this will, after the initial drop of the curve in response to the step size decrease, make the test loss curve gradually go up again.

\subsection{Summary}
We have analyzed theoretically and empirically the exponential and the cosine step sizes, two successful step size decay schedules for the stochastic optimization of non-convex functions. We have shown that, in the case of functions satisfying the PL condition, they are both adaptive to the level of noise. Furthermore, we have validated our theoretical findings on real-world tasks, showing that these two step sizes consistently match or outperform other strategies, while at the same time requiring only 1 (cosine) / 2 (exponential) hyperparameters to tune.

\section{Conclusion}
In this chapter, we introduced the notion of adaptation to noise as automatically guaranteeing (near) optimal convergence rates for different levels of noise without knowing it nor needing to tune any parameter. We then presented our works on designing/analyzing algorithms that can adapt to the noise in both the general smooth non-convex setting and the setting with the additional PL condition.

%% file: 4_Adapt_Scale/adapt_scale.tex
\chapter{Adaptation to Gradient Scales}
\label{chapter:adapt_scale}
\thispagestyle{myheadings}

[The results in this chapter appeared in~\citet{ZhuangLCO21}.]

This chapter studies the phenomenon that the scales of gradient magnitudes in each layer can scatter across a very wide range in training deep neural networks. We will first discuss when this variation becomes too severe and why it will be a problem in Section~\ref{sec:motivation}. In Section~\ref{sec:adam_fail}, we will report our observation that the popular Adam optimizer performs worse than its variant AdamW. Then, in Section~\ref{sec:scale_free}, we will show evidence suggesting understanding AdamW's advantage through its scale-freeness property which ensures that its updates are invariant to component-wise rescaling of the gradients thus adapting to gradient scales. Section~\ref{sec:adamw_proximal} presents the surprising connection between AdamW and proximal updates, providing a potential explanation of where its scale-freeness property comes from. Finally, we will show another merit of scale-free algorithms in Section~\ref{sec:scale_free_cond_num}: they can ``adapt'' to the condition number in certain scenarios. 

\section{When Varying Gradient Scales Become a Problem}
\label{sec:motivation}
\begin{figure}
     \centering
     \begin{subfigure}[b]{0.48\textwidth}
         \centering
         \includegraphics[width=\textwidth]{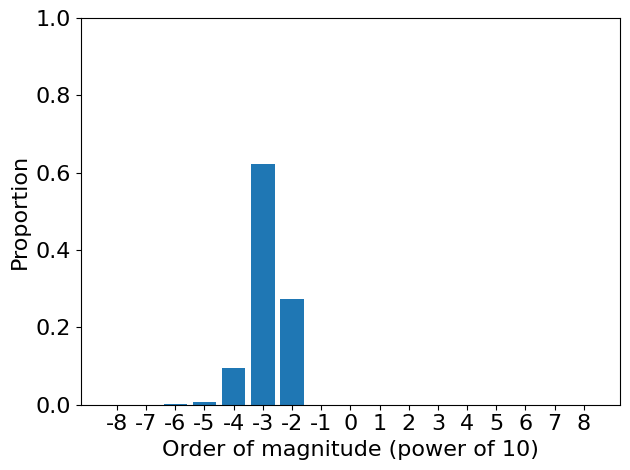}
         \caption{First Convolution Layer}
         \label{fig:grad_scale_layer_first}
     \end{subfigure}
     \hfill
     \begin{subfigure}[b]{0.48\textwidth}
         \centering
         \includegraphics[width=\textwidth]{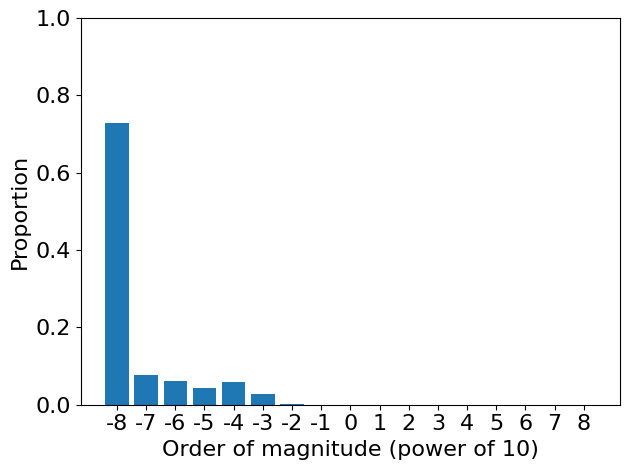}
         \caption{Last Convolution Layer}
         \label{fig:grad_scale_layer_last}
     \end{subfigure}
     \hfill
    \caption[Gradient scales vary significantly across layers on training deep resnets with Batch Normalization disabled]{The histograms of the magnitudes of gradients of the first convolution layer vs.~the last convolution layer in a time step during training a 110-layer Resnet with batch normalization disabled.}
    \label{fig:grad_scale_resnet_nobn}
\end{figure}

Neural networks are known to suffer from vanishing/exploding gradients~\citep{BengioSF94, GlorotB10, PascanuMB13}. This leads to the scales of gradient magnitudes being very different across layers, especially between the first and the last layers. This problem is particularly severe when the model is not equipped with normalization mechanisms like Batch Normalization (BN)~\citep{IoffeS15}. As an example, Figure~\ref{fig:grad_scale_resnet_nobn} shows the huge difference in gradient magnitude scales between the first and the last convolution layers in training a 110-layer Resnet~\citep{HeZRS16} with BN disabled.

BN works by normalizing the input to each layer across the mini-batch to make each coordinate have zero-mean and unit-variance. While BN can greatly help reduce the variation in gradient scales between different layers, it comes with a price. For example, it introduces added memory overhead~\citep{BuloPK18} and training time~\citep{GitmanG17} as well as a discrepancy between training and inferencing~\citep{SinghS19}. BN has also been found to be not suitable for many cases including distributed computing with a small minibatch per GPU~\citep{WuH18, GoyalDGNWKTJH17}, sequential modeling tasks~\citep{BaKH16}, and contrastive learning algorithms~\citep{ChenKNH20}. Actually, there is already active research in the setting of removing BN~\citep{DeS20, ZhangDM19}. Also, there are already SOTA architectures that do not use BN including the BERT model~\citep{DevlinCLT19} and the Vision Transformer~\citep{DosovitskiyB0WZ21}.

In the scenario when the gradient scales vary significantly from layer to layer, the widely used SGD algorithm is not able to handle such a situation nicely. The reason is that SGD adopts a single step size value for all layers, and this value could be too large for some layers with large gradients leading to divergence while at the same time being too small for other layers with small gradients resulting in slow progress.

A natural idea is to use individual ``step sizes'' for each layer or even each coordinate and to make these layer-wise/coordinate-wise step sizes to take into account the gradient scales of corresponding layers or coordinates. A popular optimizer operating in such fashion is Adam~\citep{KingmaB15}, a method that operates coordinate-wisely and utilizes first- and second-order moments of gradients to compute the step size. It has been empirically shown to achieve remarkable results across a variety of problems even by simply adopting the default hyperparameter setting.
However, we observed that it does not entirely solve the problem; in the next section, we show that it performs worse than AdamW~\citep{LoshchilovH18}, a variant of Adam.

\section{When Adam Performs Worse than AdamW}
\label{sec:adam_fail}
Since its debut, Adam has gained tremendous popularity due to less hyperparameter tuning and great performance. In practice, to improve the generalization ability, Adam is typically combined with a $\ell_2$ regularization which adds the squared $\ell_2$ norm of the model weights on top of the loss function (which we will call \emph{Adam-$\ell_2$} hereafter). This technique is usually referred to as weight decay because when using SGD, the $\ell_2$ regularization works by first shrinking the model weights by a constant factor in addition to moving along the negative gradient direction in each step. By biasing the optimization towards solutions with small norms, weight decay has long been a standard technique to improve the generalization ability in machine learning~\citep{KroghH91, BosC96} and is still widely employed in training modern deep neural networks~\citep{DevlinCLT19, TanL19}.

\begin{algorithm}[tb]
\caption{\colorbox{red}{Adam with $\ell_2$ regularization (Adam-$\ell_2$)} and \colorbox{green}{Adam with\phantom{p}}\\\colorbox{green}{decoupled weight decay (AdamW)} \emph{(All operations on vectors are element-wise.)}}
\label{algo:adamw}
\begin{algorithmic}[1]
\STATE{\textbf{Given} $\alpha$, $\beta_1$, $\beta_2$, $\epsilon$, $\lambda \in \R$, $\{\eta_t\}_{t\ge0}$. }
\STATE{\textbf{Initialize:} $\bx_0\in\R^d$, first moment vector $\bm_0 = 0$, second moment vector $\bv_0 = 0$}
\FOR{$t = 1,2,\ldots,T$}
\STATE{Compute a  stochastic evaluation of the true gradient $\nabla f(\bx_{t-1})$ and denote it as $\nabla f_t(\bx_{t-1})$}
\STATE{$\bg_t \leftarrow \nabla f_t(\bx_{t-1})$\colorbox{red}{$+\lambda\bx_{t-1}$}}
\STATE{$\bm_t \leftarrow \beta_1 \bm_{t-1} + (1-\beta_1) \bg_t$, \quad$\bv_t \leftarrow \beta_2 \bv_{t-1} + (1-\beta_2) \bg_t^2$}
\STATE{$\hat{\bm}_t \leftarrow \bm_t/(1-\beta_1^t)$, \quad$\hat{\bv}_t \leftarrow \bv_t/(1-\beta_2^t)$}
\STATE{$\bx_{t} \leftarrow \bx_{t-1}$ \colorbox{green}{$-\eta_t\lambda\bx_{t-1}$} $- \eta_t\alpha\hat{\bm}_t/(\sqrt{\hat{\bv}_t} + \epsilon)$}
\ENDFOR
\end{algorithmic}
\end{algorithm}

However, as pointed out in~\citet{LoshchilovH18}, for Adam, there is no fixed regularization that achieves the same effect $\ell_2$ regulation has on SGD. To address this, they provide a method called AdamW that decouples the gradient of the $\ell_2$ regularization from the update of Adam and directly decays the weights. These two algorithms are presented in Algorithm~\ref{algo:adamw}. Although AdamW is very popular~\citep{KuenPLZT19, LifchitzAPB19, CarionMSUKZ20} and it frequently outperforms Adam-$\ell_2$, it is currently unclear why it works so well. Worse still, recently, \citet{BjorckWG20} applied AdamW in Natural Language Processing and Reinforcement Learning problems and found no significant improvement over sufficiently tuned Adam-$\ell_2$.

Consequently, we conducted deep learning experiments to identify a scenario when AdamW exhibits concrete advantages over Adam-$\ell_2$ for isolating AdamW's unique merits. The search leads to the setting of training very deep neural networks with Batch Normalization disabled on image classification tasks as we reported below.

\textbf{Data Normalization and Augmentation:}
We consider the image classification task on CIFAR-10/100 datasets. Images are normalized per channel using the means and standard deviations computed from all training images. We adopt the data augmentation technique following~\citet{LeeXGZT15} (for training only): 4 pixels are padded on each side of an image and a $32\times32$ crop is randomly sampled from the padded image or its horizontal flip.

\textbf{Models:}
For both the CIFAR-10 and CIFAR-100 datasets, we employ the Residual Network model~\citep{HeZRS16} of 20/44/56/110/218 layers; and for CIFAR-100, we additionally utilize the DenseNet-BC model~\citep{HuangLVW17} with 100 layers and a growth-rate of 12. The loss is the cross-entropy loss.

\textbf{Hyperparameter tuning:} For both Adam-$\ell_2$ and AdamW, we set $\beta_1=0.9$, $\beta_2=0.999$, $\epsilon=10^{-8}$ as suggested in the Adam paper~\citep{KingmaB15}. To set the initial step size $\alpha$ and the weight decay parameter $\lambda$, we grid search over $\{0.00005, 0.0001, 0.0005, 0.001, 0.005\}$ for $\alpha$ and $\{0, 0.00001, 0.00005, 0.0001, 0.0005,$ $0.001\}$ for $\lambda$. Whenever the best performing hyperparameters lie in the boundary of the searching grid, we always extend the grid to ensure that the final best-performing hyperparameters fall into the interior of the grid.

\textbf{Training:} For each experiment configuration (e.g., 110-layer Resnet without BN on CIFAR-10), we randomly select an initialization of the model to use as a fixed starting point for all optimizers and hyperparameter settings. We use a mini-batch of 128, and train 300 epochs unless otherwise specified.

\begin{figure}[t]
    \centering
    \begin{subfigure}[b]{0.48\textwidth}
       \includegraphics[width=\textwidth]{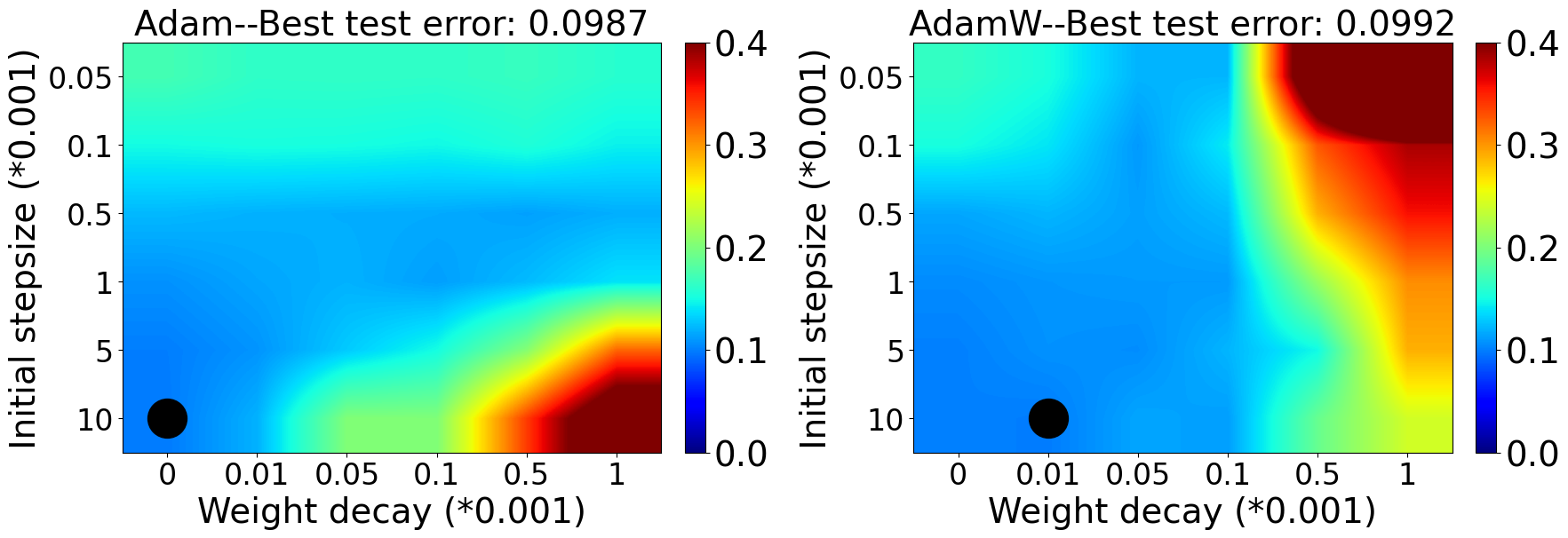}
       \caption{20 Layer Resnet on CIFAR10}
       \label{fig:resnet20} 
    \end{subfigure}
    \hspace{0.02\textwidth}
    \begin{subfigure}[b]{0.48\textwidth}
       \includegraphics[width=\textwidth]{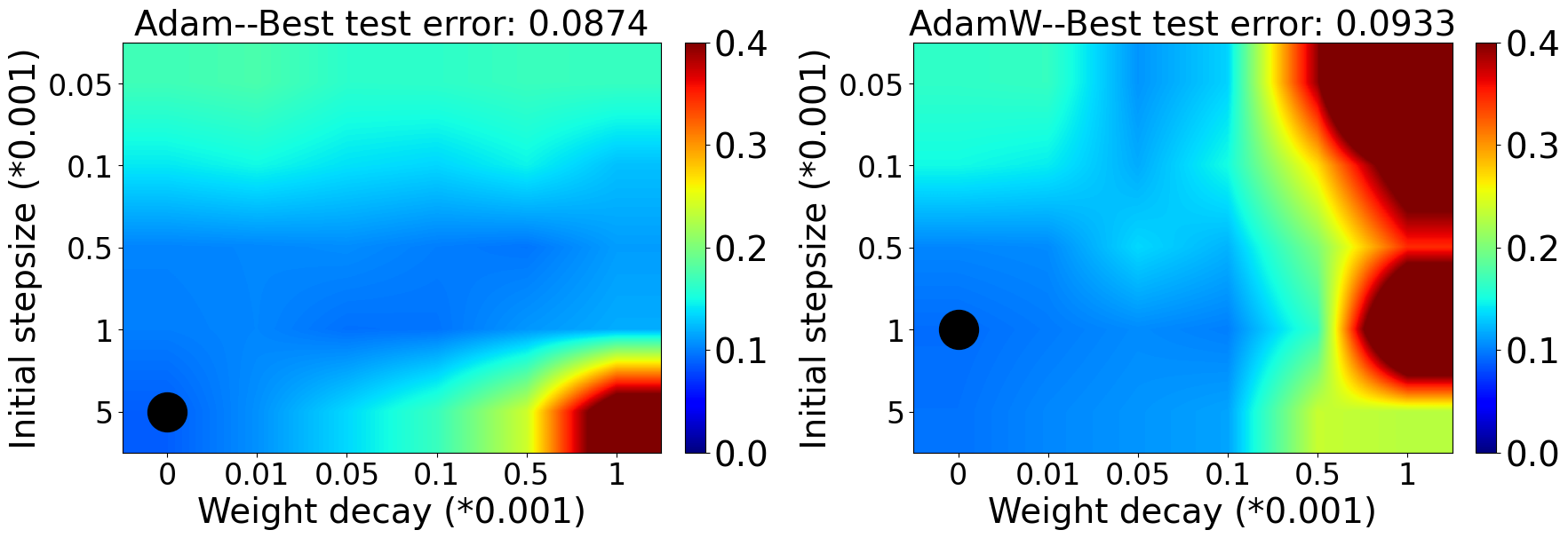}
        \caption{44 Layer Resnet on CIFAR10}
       \label{fig:resnet44} 
    \end{subfigure}
    
    \begin{subfigure}[b]{0.48\textwidth}
       \includegraphics[width=\textwidth]{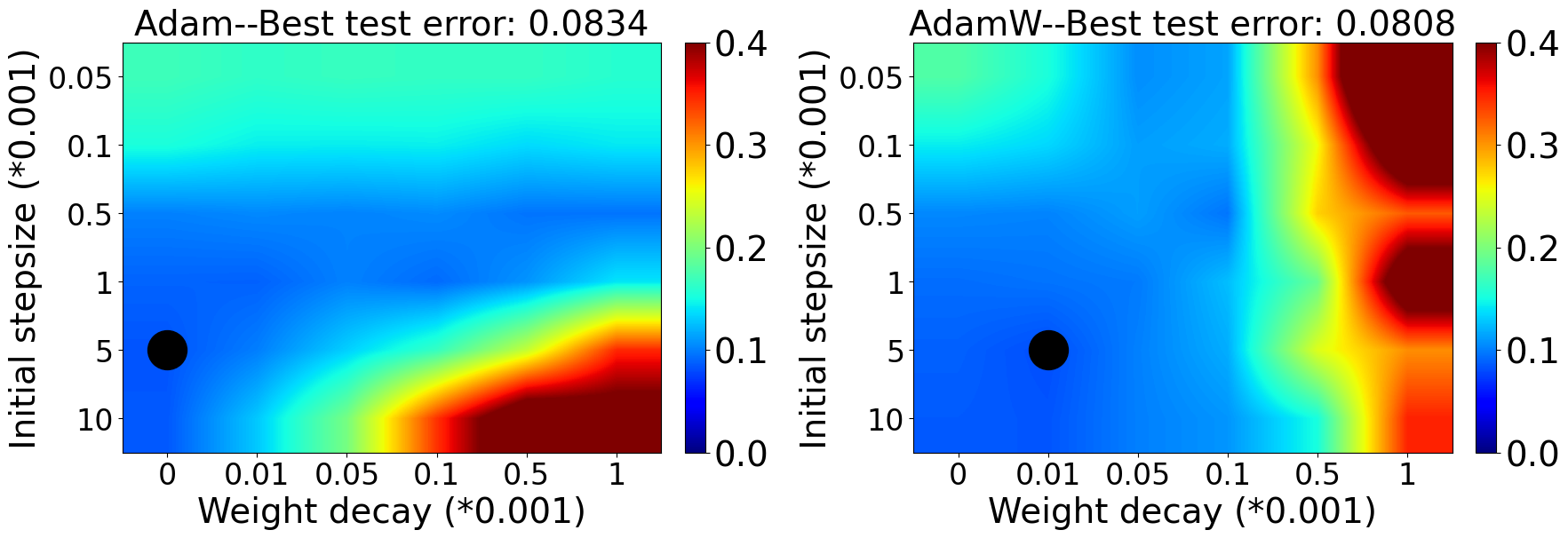}
       \caption{56 Layer Resnet on CIFAR10}
       \label{fig:resnet56} 
    \end{subfigure}
    \hspace{0.02\textwidth}
    \begin{subfigure}[b]{0.48\textwidth}
       \includegraphics[width=\textwidth]{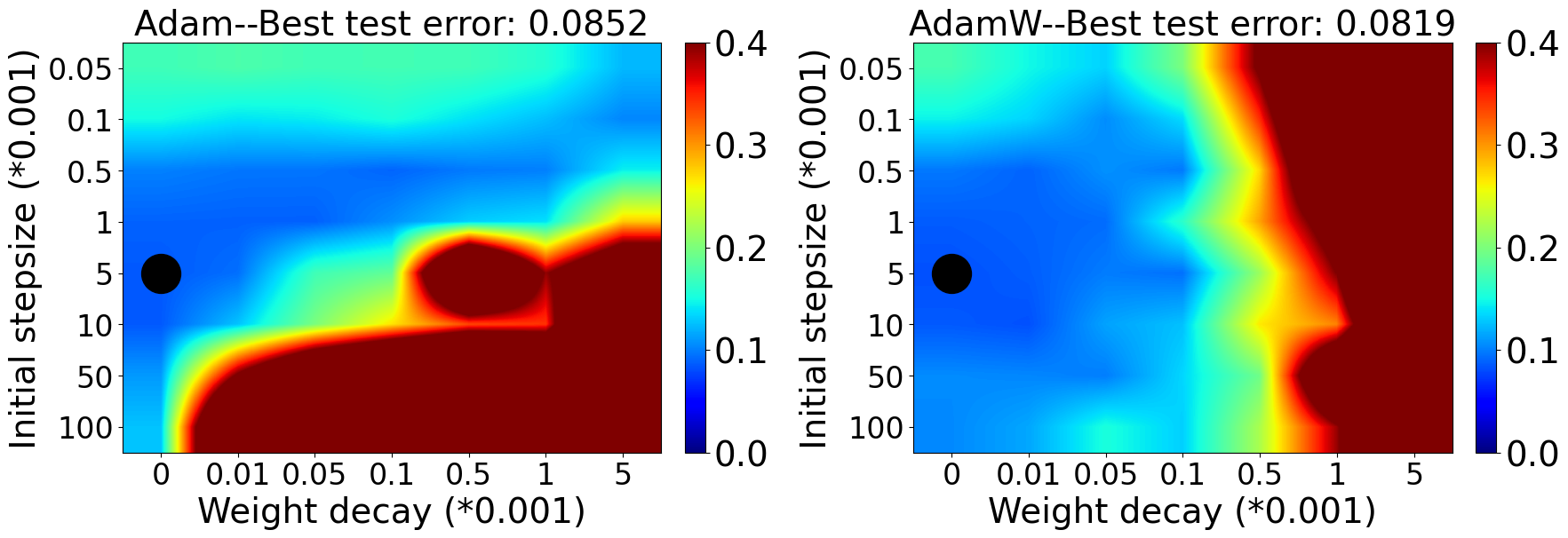}
        \caption{110 Layer Resnet on CIFAR10}
       \label{fig:resnet110} 
    \end{subfigure}
    
    \begin{subfigure}[b]{0.48\textwidth}
       \includegraphics[width=\textwidth]{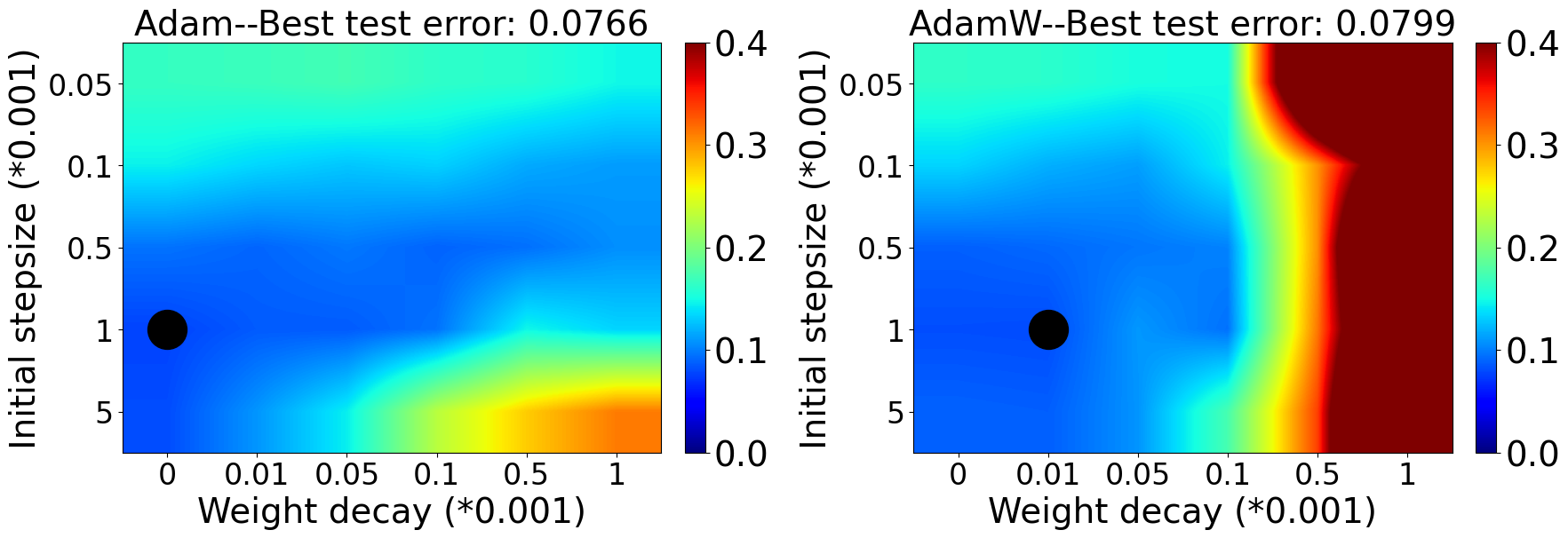}
       \caption{218 Layer Resnet on CIFAR10}
       \label{fig:resnet218} 
    \end{subfigure}
    \hspace{0.02\textwidth}
    \begin{subfigure}[b]{0.48\textwidth}
       \includegraphics[width=\textwidth]{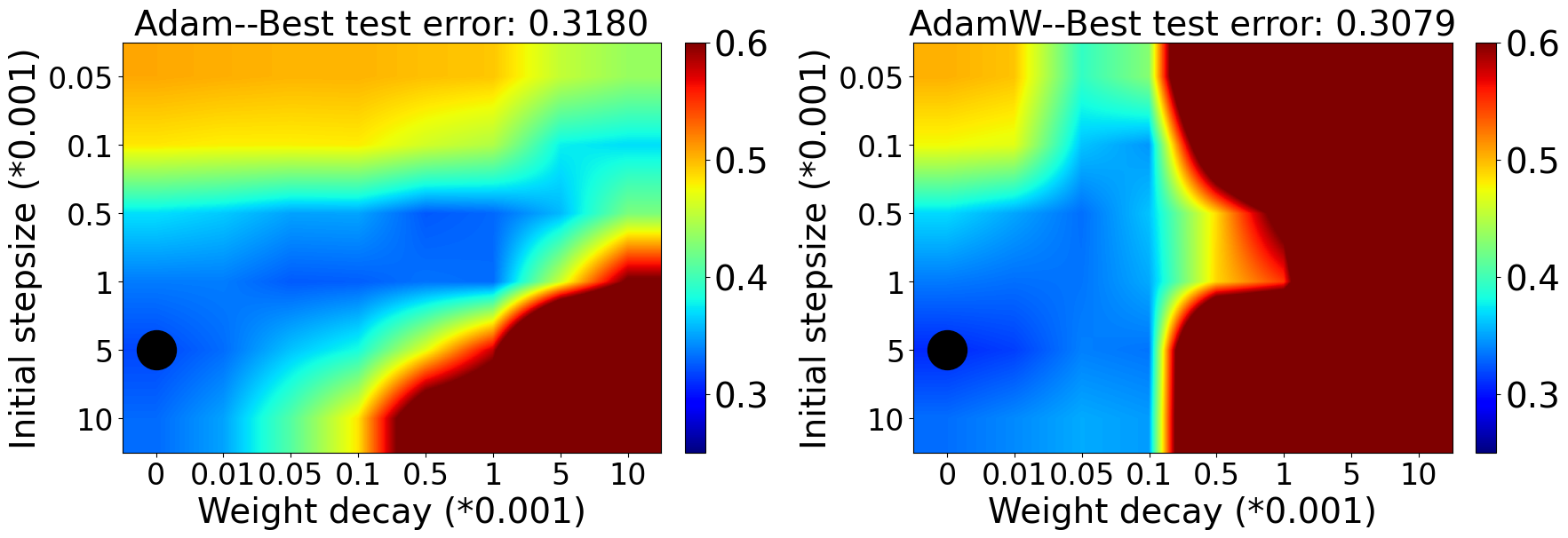}
       \caption{100 layer DenseNet-BC on CIFAR100}
       \label{fig:densenet} 
    \end{subfigure}
    
    \caption[On using AdamW vs.~Adam-$\ell_2$ to train a Resnet/DenseNet with Batch Normalization on CIFAR10/100]{The final Top-1 test error on using AdamW vs.~Adam-$\ell_2$ on training a Resnet/DenseNet with Batch Normalization on CIFAR10/100 (\emph{the black circle denotes the best setting}). Note how close are the best performing hyperparameter combinations and the smallest testing error each optimizer obtains between Adam-$\ell_2$ and AdamW for each setting, suggesting they perform similarly when BN is enabled.}
    \label{fig:bn}
\end{figure}
    
\begin{figure}[p]
    \centering
    \begin{subfigure}[b]{\textwidth}
       \includegraphics[width=0.30\textwidth]{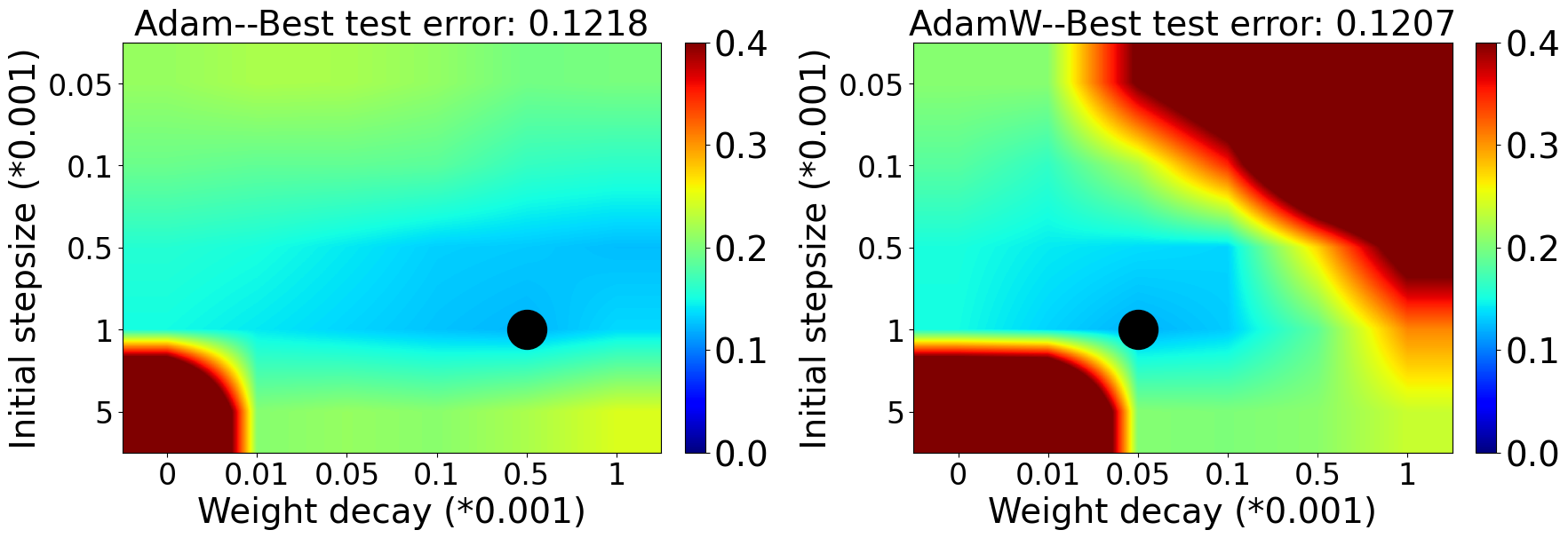}
       \hspace{0.02\textwidth}
       \includegraphics[width=0.30\textwidth]{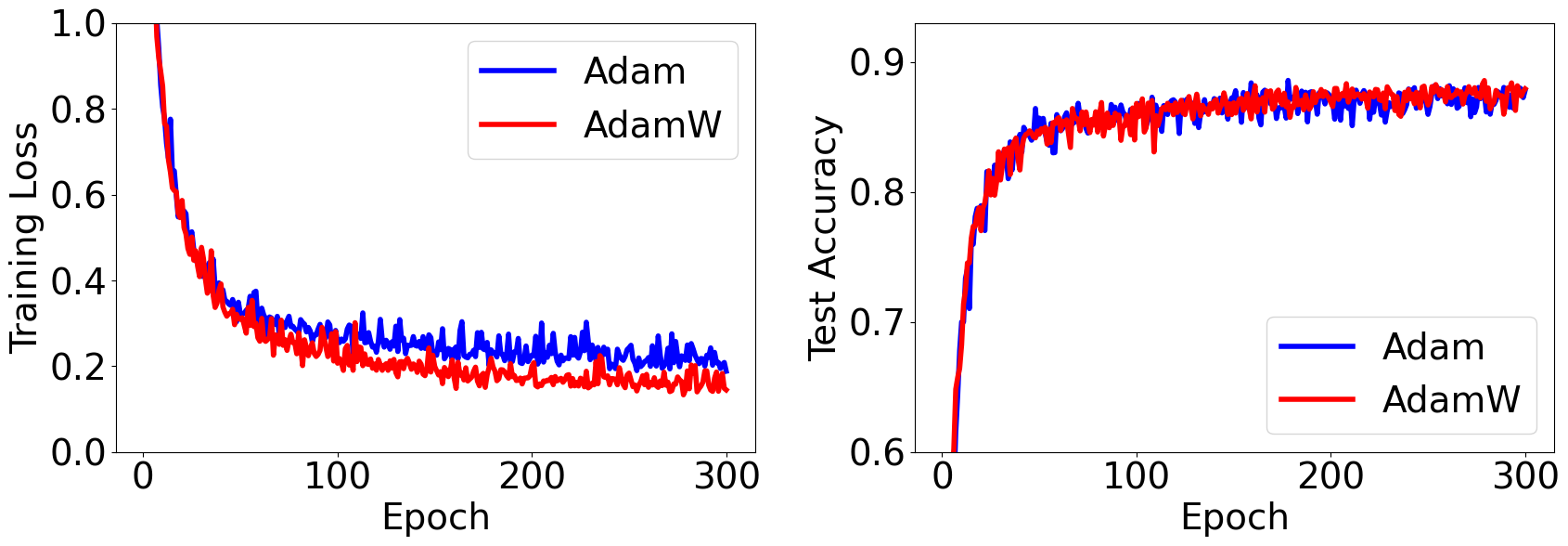}
       \hspace{0.02\textwidth}
       \includegraphics[width=0.15\linewidth]{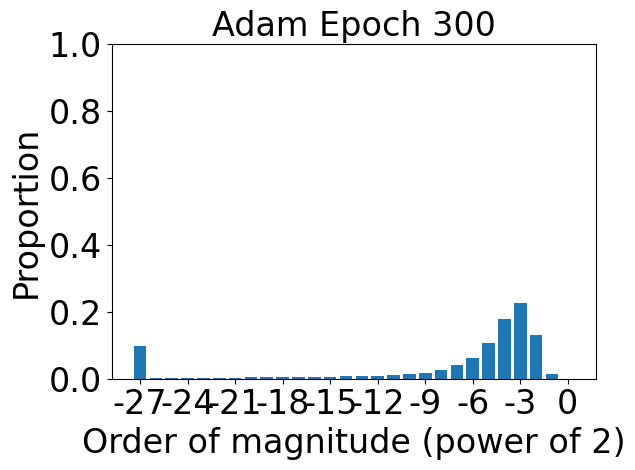}
            \hspace{0pt}
            \includegraphics[width=0.15\linewidth]{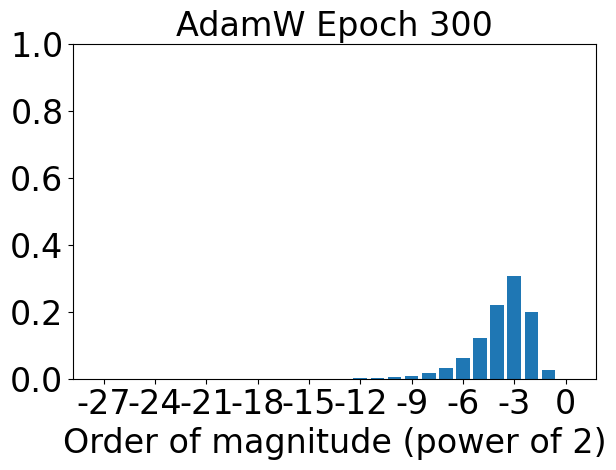}
       \caption{20 Layer Resnet}
       \label{fig:resnet20_nobn} 
    \end{subfigure}
    
    \begin{subfigure}[b]{\textwidth}
       \includegraphics[width=0.30\textwidth]{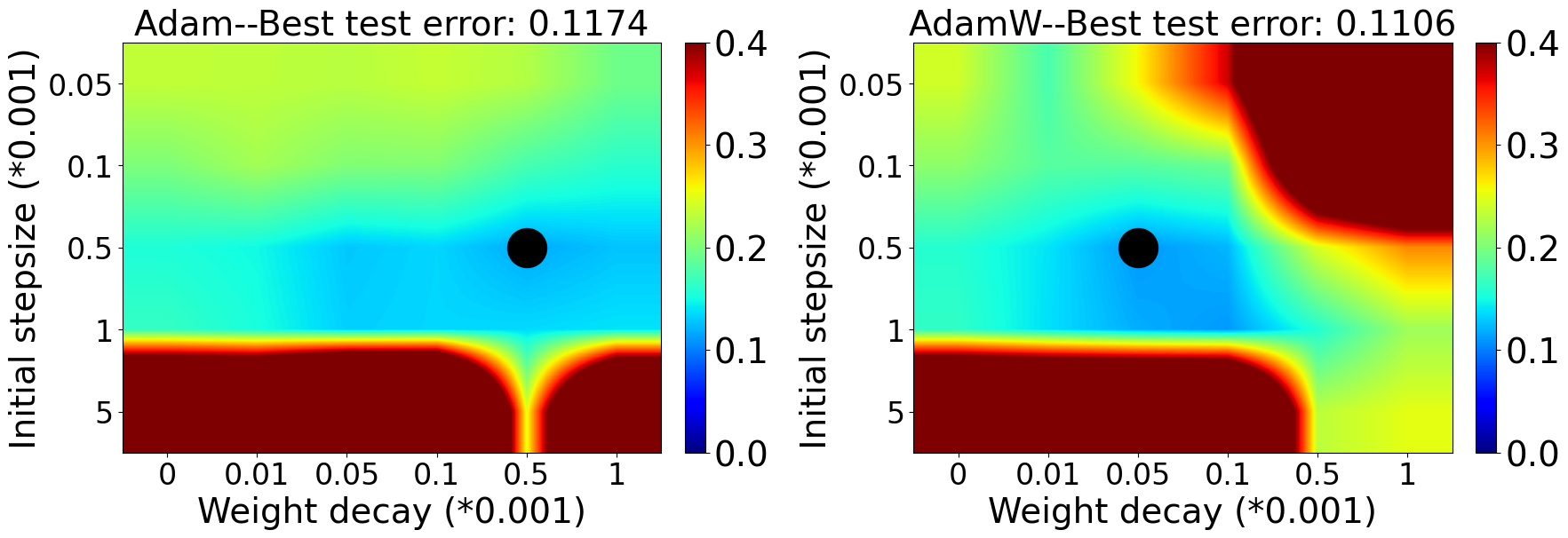}
       \hspace{0.02\textwidth}
       \includegraphics[width=0.30\textwidth]{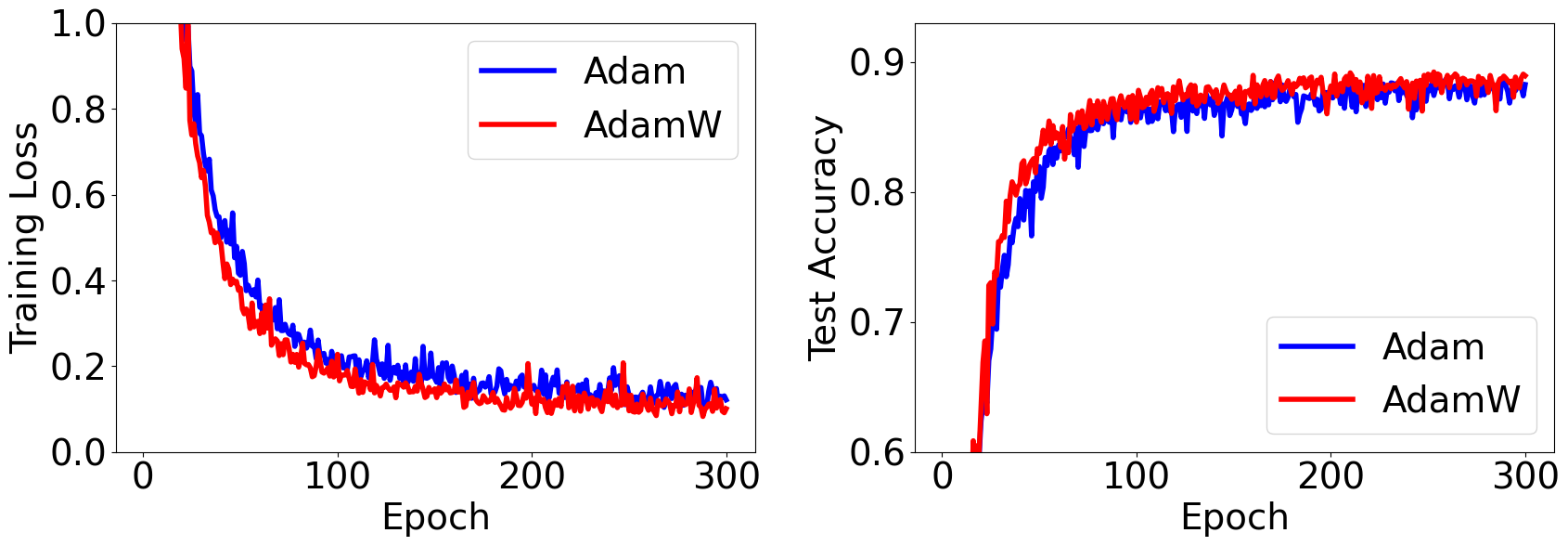}
       \hspace{0.02\textwidth}
       \includegraphics[width=0.15\linewidth]{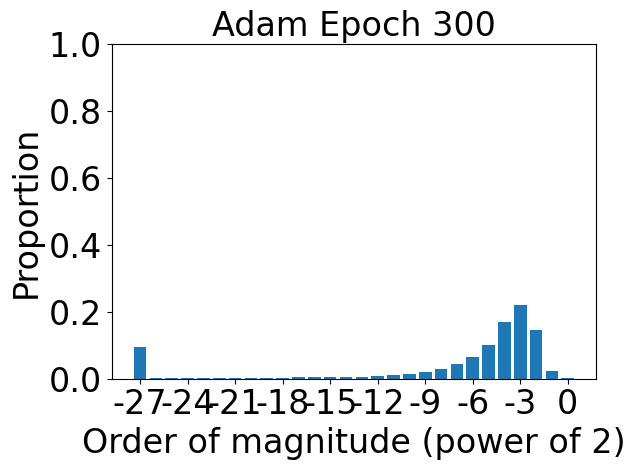}
            \hspace{0pt}
            \includegraphics[width=0.15\linewidth]{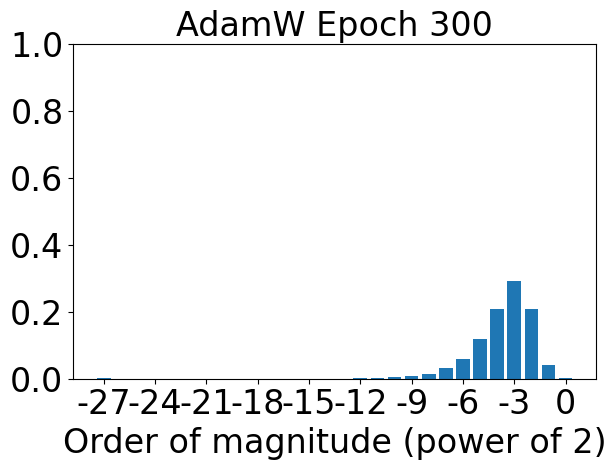}
       \caption{44 Layer Resnet}
       \label{fig:resnet44_nobn} 
    \end{subfigure}
    
    \begin{subfigure}[b]{\textwidth}
       \includegraphics[width=0.30\textwidth]{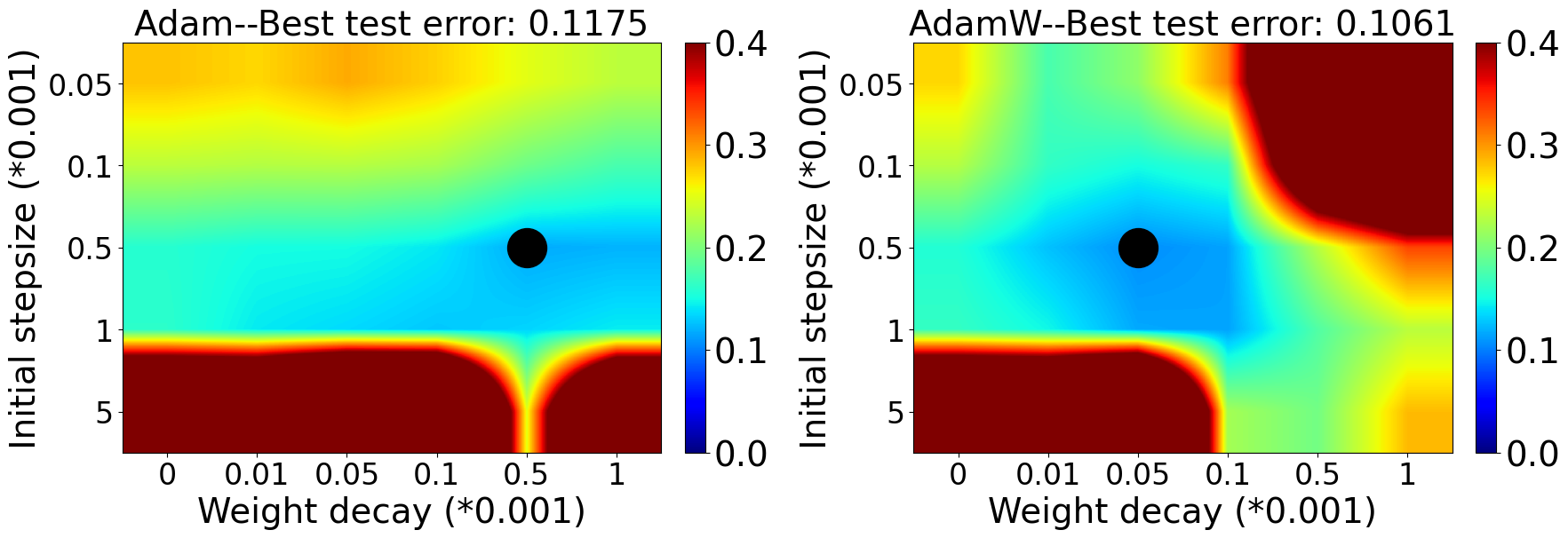}
       \hspace{0.02\textwidth}
       \includegraphics[width=0.30\textwidth]{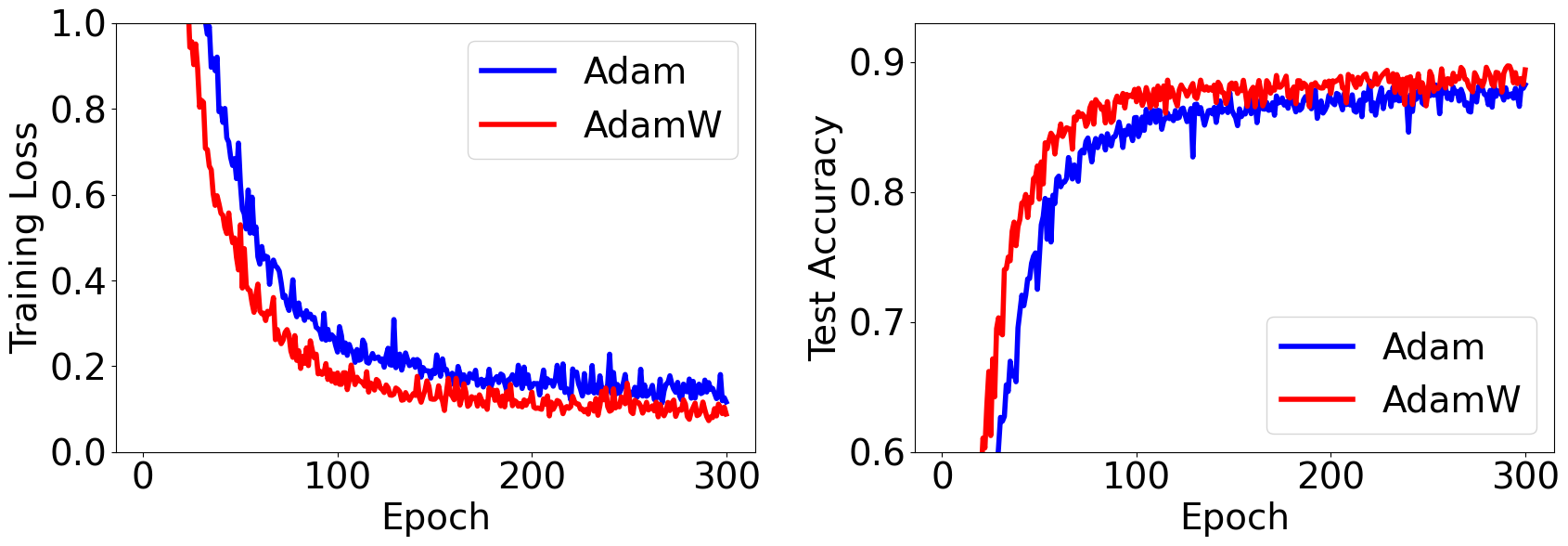}
       \hspace{0.02\textwidth}
       \includegraphics[width=0.15\linewidth]{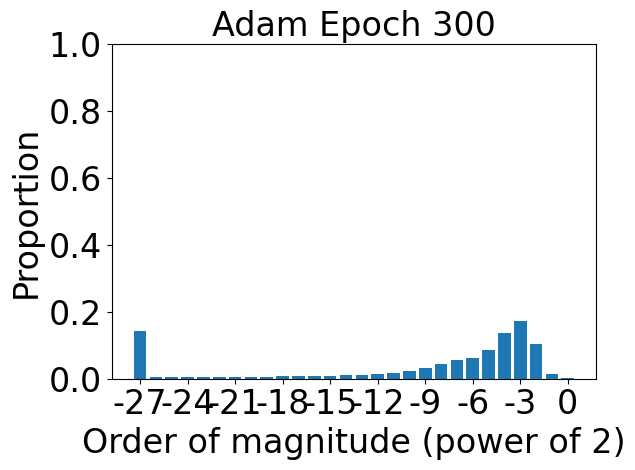}
            \hspace{0pt}
            \includegraphics[width=0.15\linewidth]{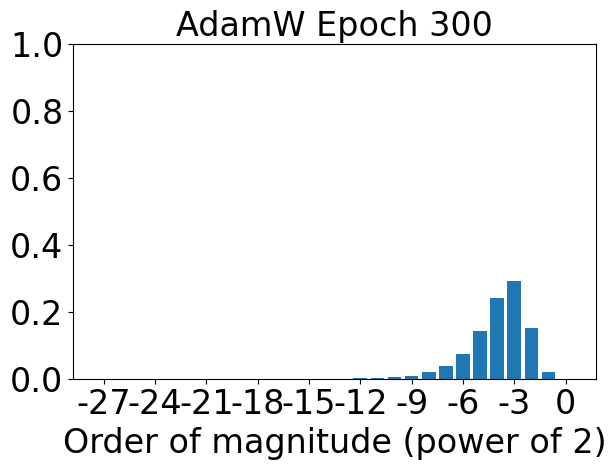}
       \caption{56 Layer Resnet}
       \label{fig:resnet56_nobn} 
    \end{subfigure}
    
    \begin{subfigure}[b]{\textwidth}
       \includegraphics[width=0.30\textwidth]{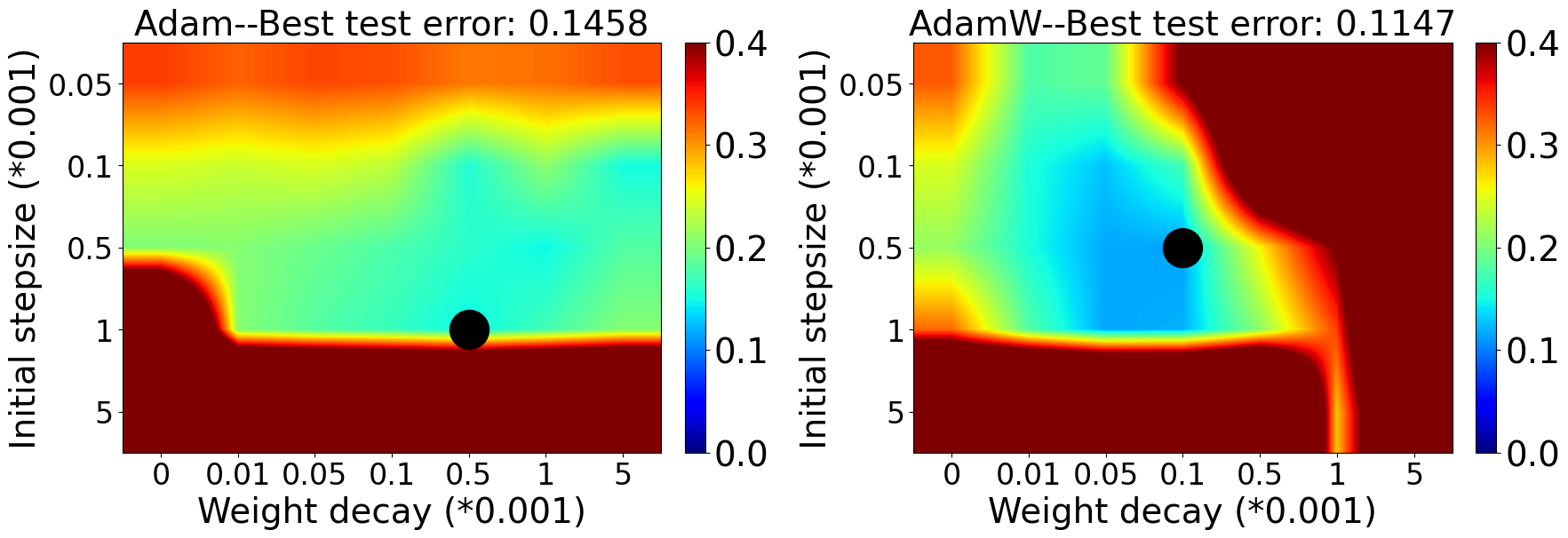}
       \hspace{0.02\textwidth}
       \includegraphics[width=0.30\textwidth]{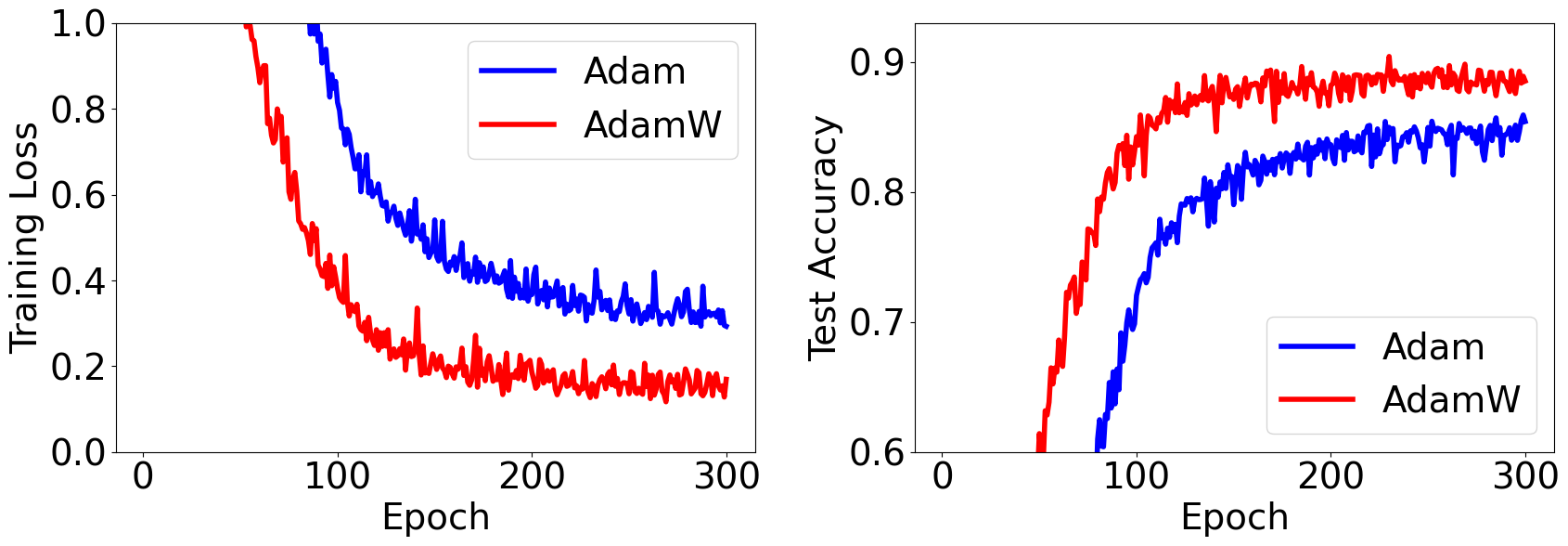}
       \hspace{0.02\textwidth}
       \includegraphics[width=0.15\linewidth]{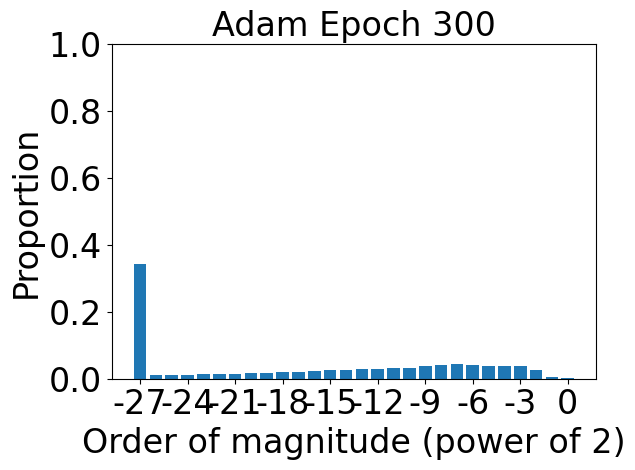}
            \hspace{0pt}
            \includegraphics[width=0.15\linewidth]{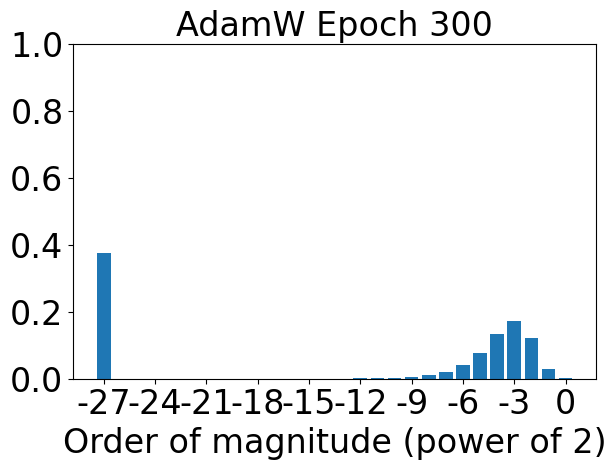}
       \caption{110 Layer Resnet}
       \label{fig:resnet110_nobn} 
    \end{subfigure}
    
    \begin{subfigure}[b]{\textwidth}
       \includegraphics[width=0.30\textwidth]{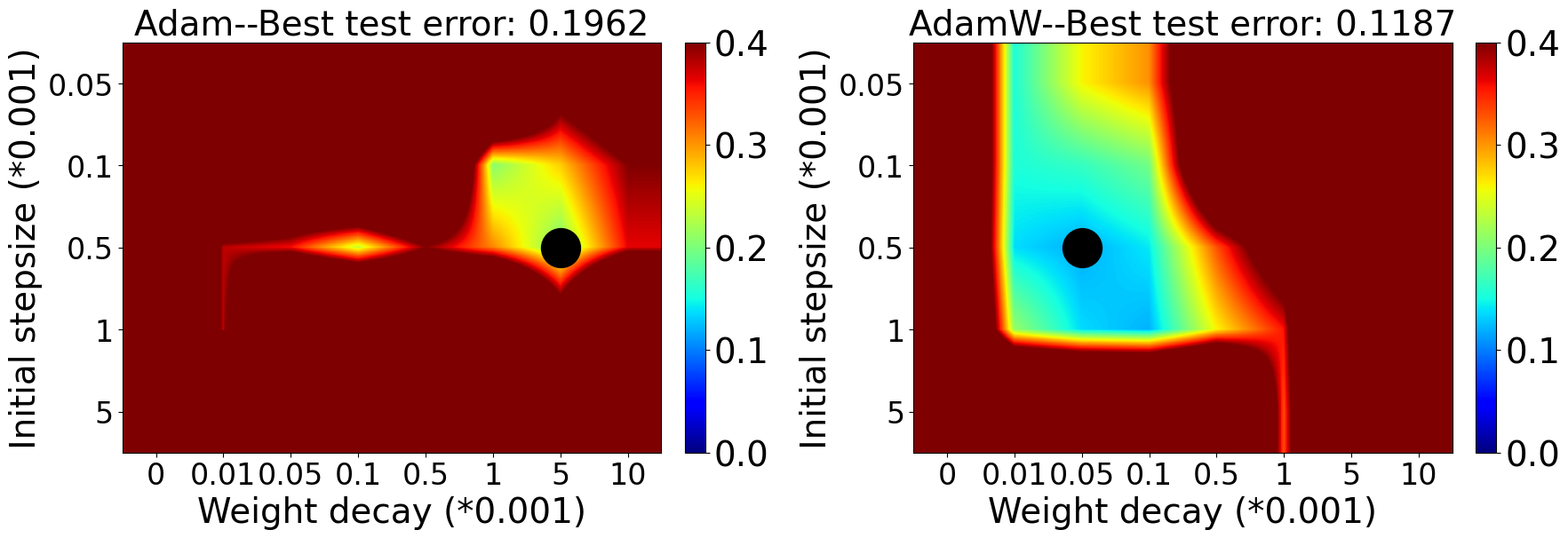}
       \hspace{0.02\textwidth}
       \includegraphics[width=0.30\textwidth]{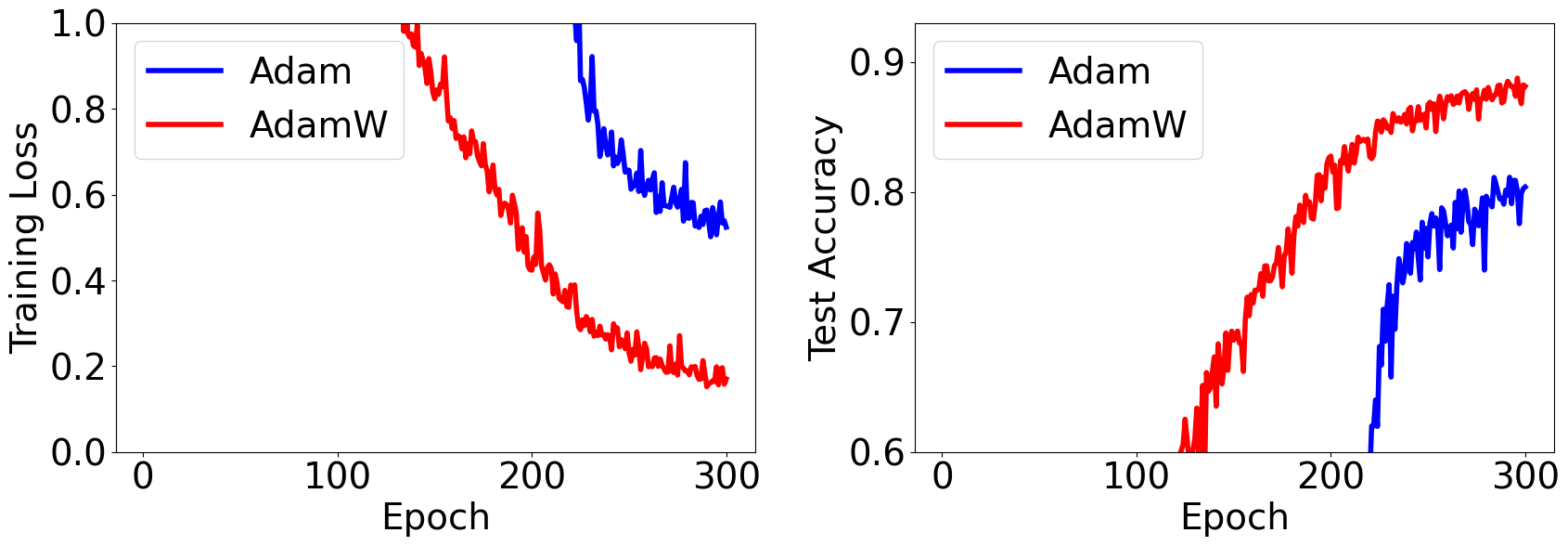}
       \hspace{0.02\textwidth}
       \includegraphics[width=0.15\linewidth]{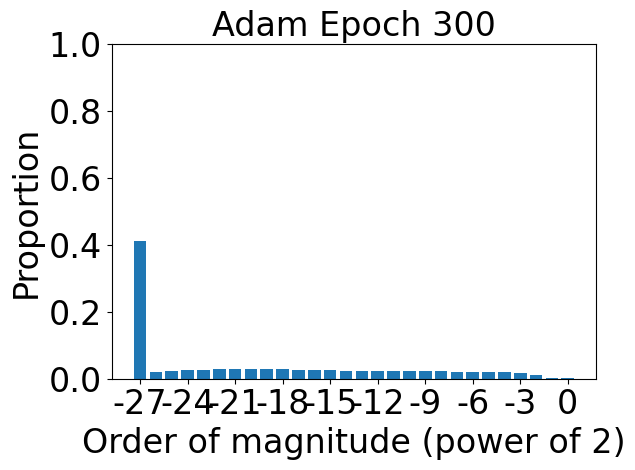}
            \hspace{0pt}
            \includegraphics[width=0.15\linewidth]{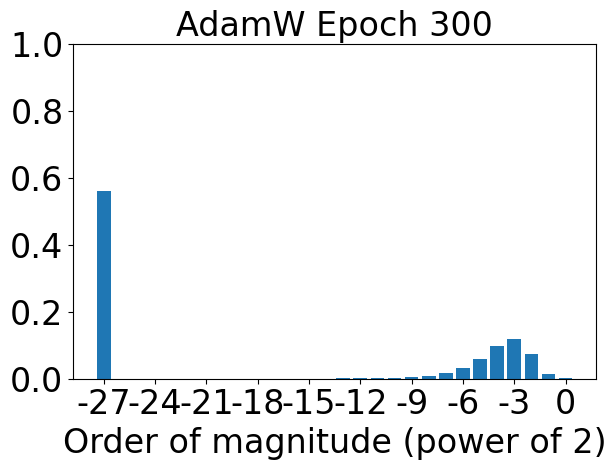}
       \caption{218 Layer Resnet}
       \label{fig:resnet218_nobn} 
    \end{subfigure}
    \caption[On using AdamW vs.~Adam-$\ell_2$ to train a Resnet with Batch Normalizatin disabled on CIFAR10]{On using AdamW vs.~Adam-$\ell_2$ to train a Resnet with Batch Normalizatin disabled on CIFAR10. (Left two) The final Top-1 test error (\emph{the black circle denotes the best setting}). (Middle two) The training loss and test accuracy curves when using the initial step size and the weight decay parameter that gives the smallest test error. (Right two) The histogram of update magnitudes of all coordinates near the end of the training when using the initial step size and the weight decay parameter that gives the smallest test error. Note that as the depth of the neural network increases, Adam-$\ell_2$'s updates scatter more evenly over the entire spectrum while AdamW's updates are still concentrated in a small range, and AdamW's advantage in both training and testing over Adam-$\ell_2$ becomes more significant.}
    \label{fig:nobn_cifar10}
\end{figure}
    
\begin{figure}[p]
    \centering
    \begin{subfigure}[b]{\textwidth}
       \includegraphics[width=0.30\textwidth]{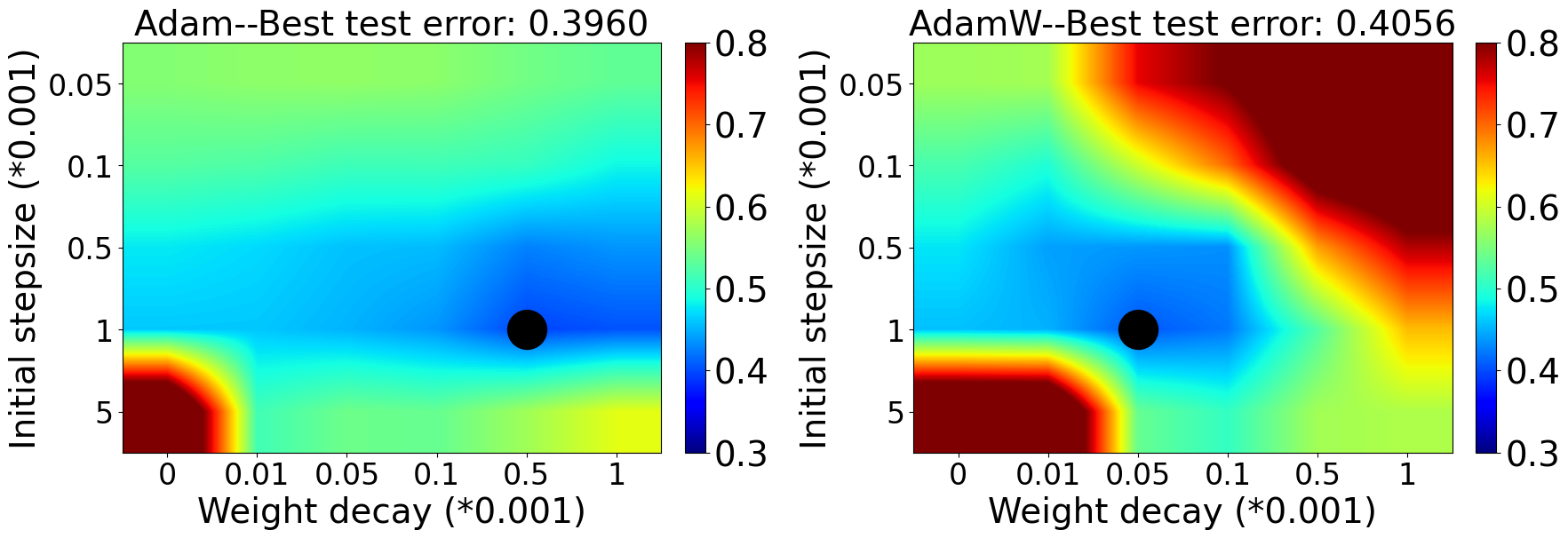}
       \hspace{0.02\textwidth}
       \includegraphics[width=0.30\textwidth]{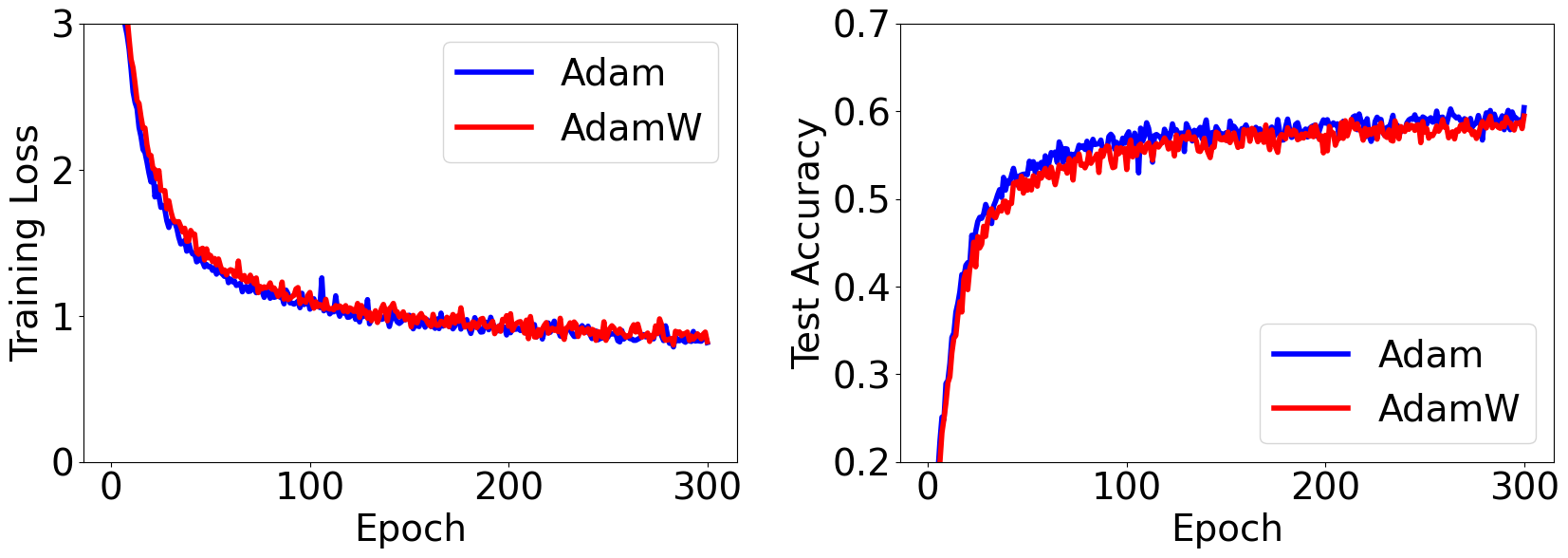}
       \hspace{0.02\textwidth}
       \includegraphics[width=0.15\linewidth]{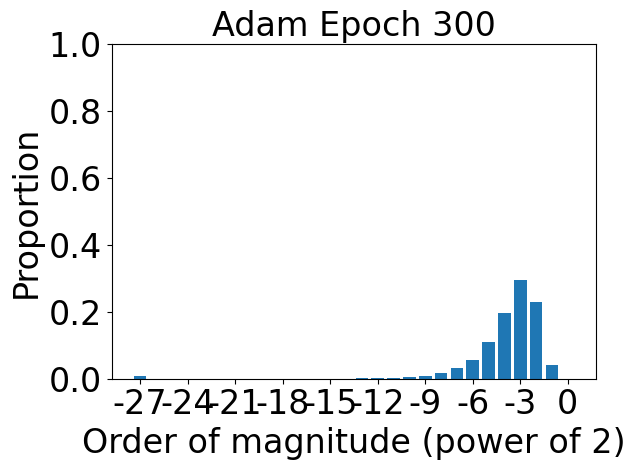}
            \hspace{0pt}
            \includegraphics[width=0.15\linewidth]{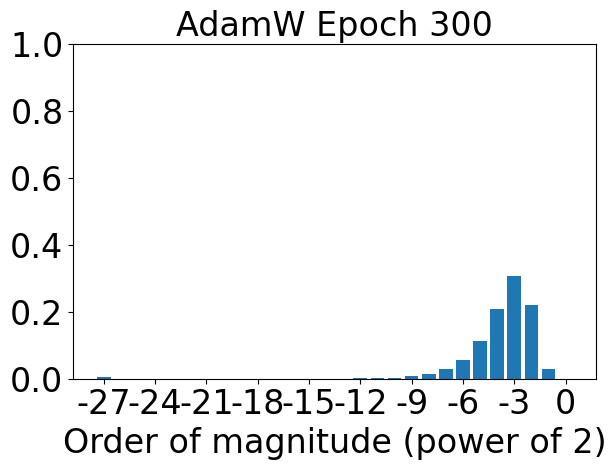}
       \caption{20 Layer Resnet}
       \label{fig:resnet20_c100_nobn} 
    \end{subfigure}
    
    \begin{subfigure}[b]{\textwidth}
       \includegraphics[width=0.30\textwidth]{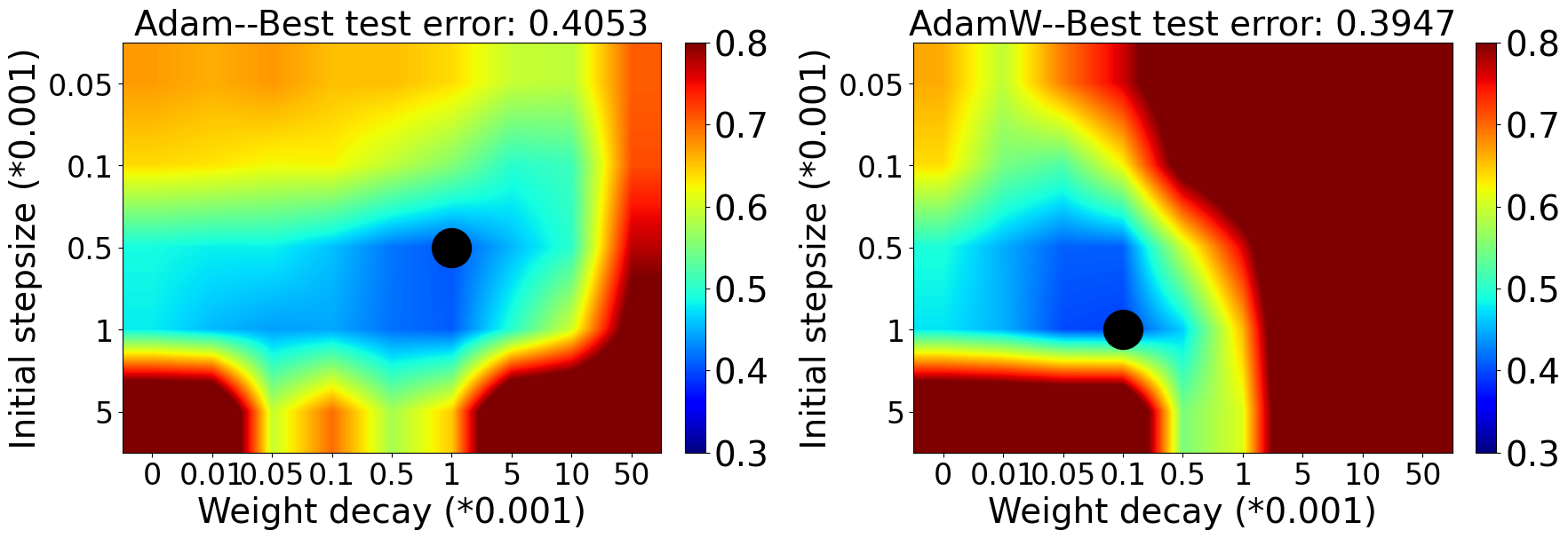}
       \hspace{0.02\textwidth}
       \includegraphics[width=0.30\textwidth]{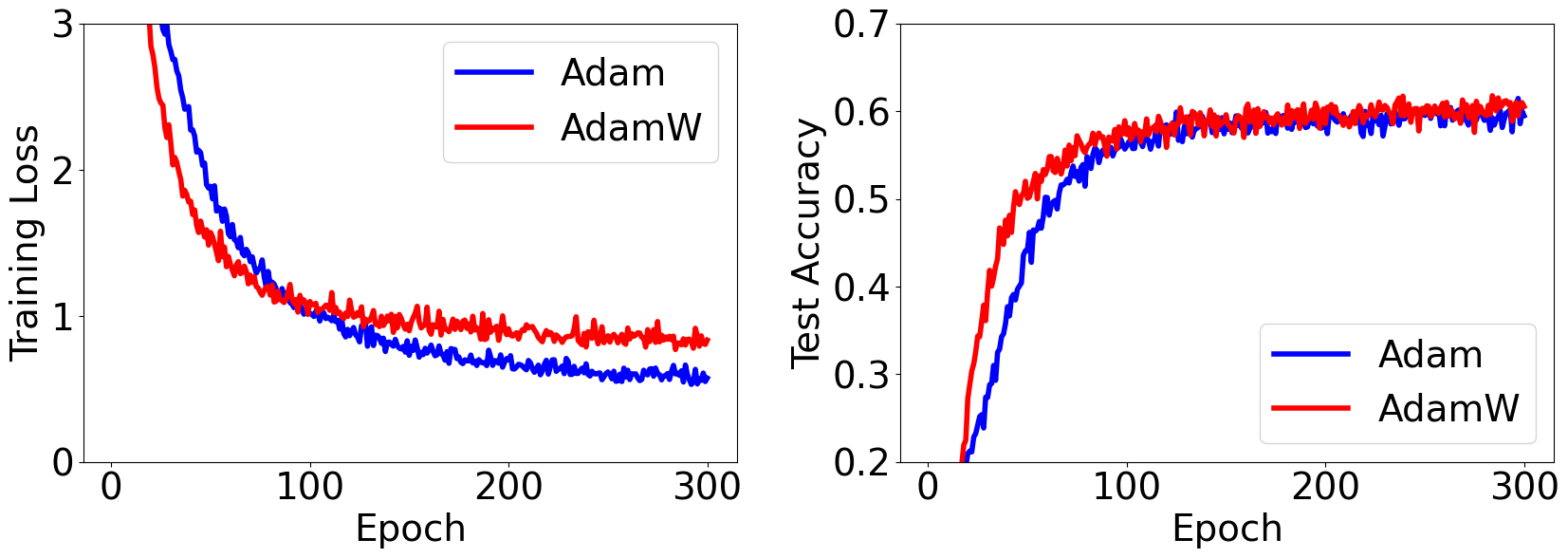}
       \hspace{0.02\textwidth}
       \includegraphics[width=0.15\linewidth]{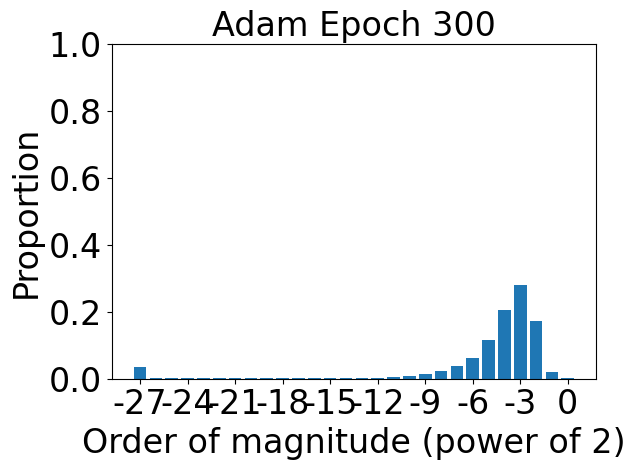}
            \hspace{0pt}
            \includegraphics[width=0.15\linewidth]{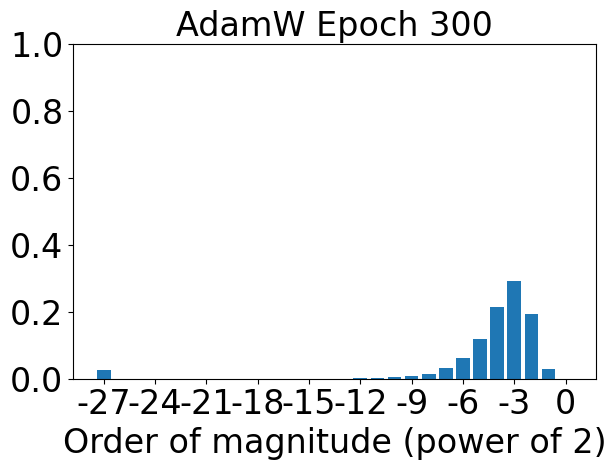}
       \caption{44 Layer Resnet}
       \label{fig:resnet44_c100_nobn} 
    \end{subfigure}
    
    \begin{subfigure}[b]{\textwidth}
       \includegraphics[width=0.30\textwidth]{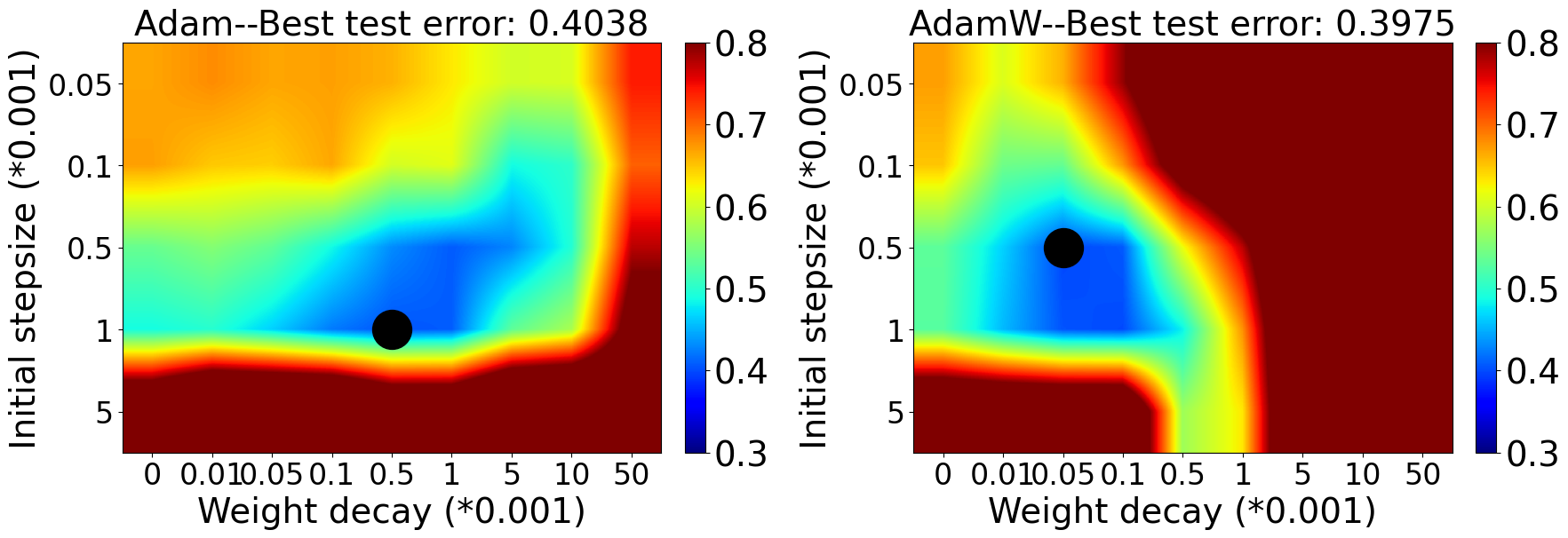}
       \hspace{0.02\textwidth}
       \includegraphics[width=0.30\textwidth]{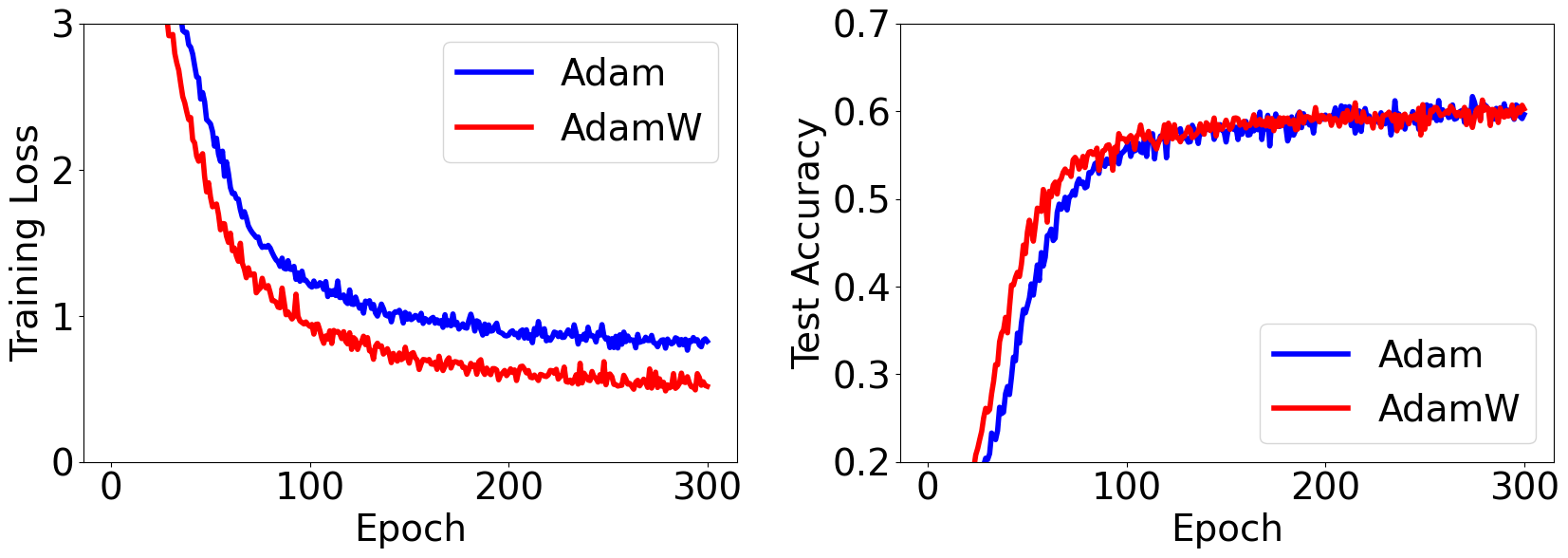}
       \hspace{0.02\textwidth}
       \includegraphics[width=0.15\linewidth]{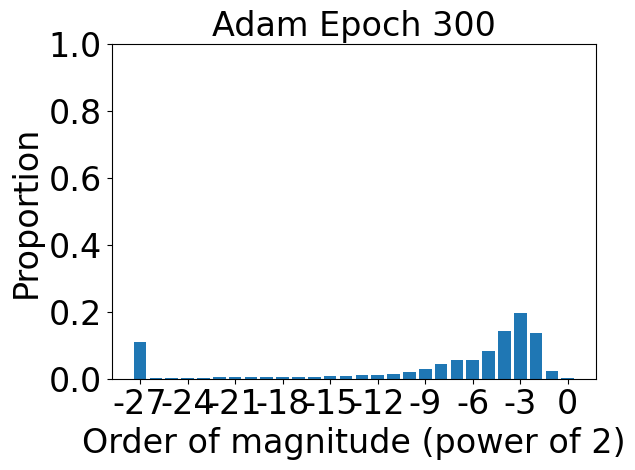}
            \hspace{0pt}
            \includegraphics[width=0.15\linewidth]{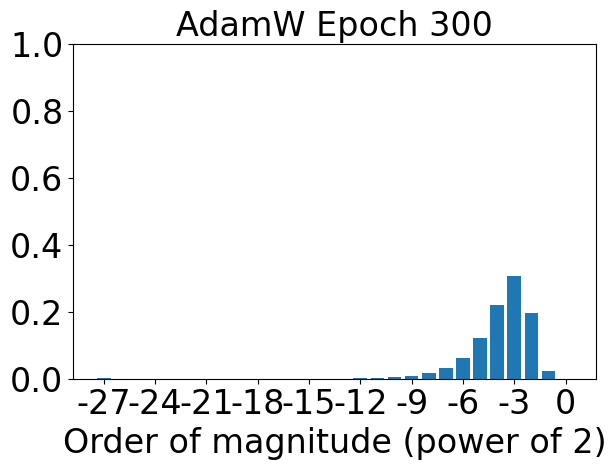}
       \caption{56 Layer Resnet}
       \label{fig:resnet56_c100_nobn} 
    \end{subfigure}
    
    \begin{subfigure}[b]{\textwidth}
       \includegraphics[width=0.30\textwidth]{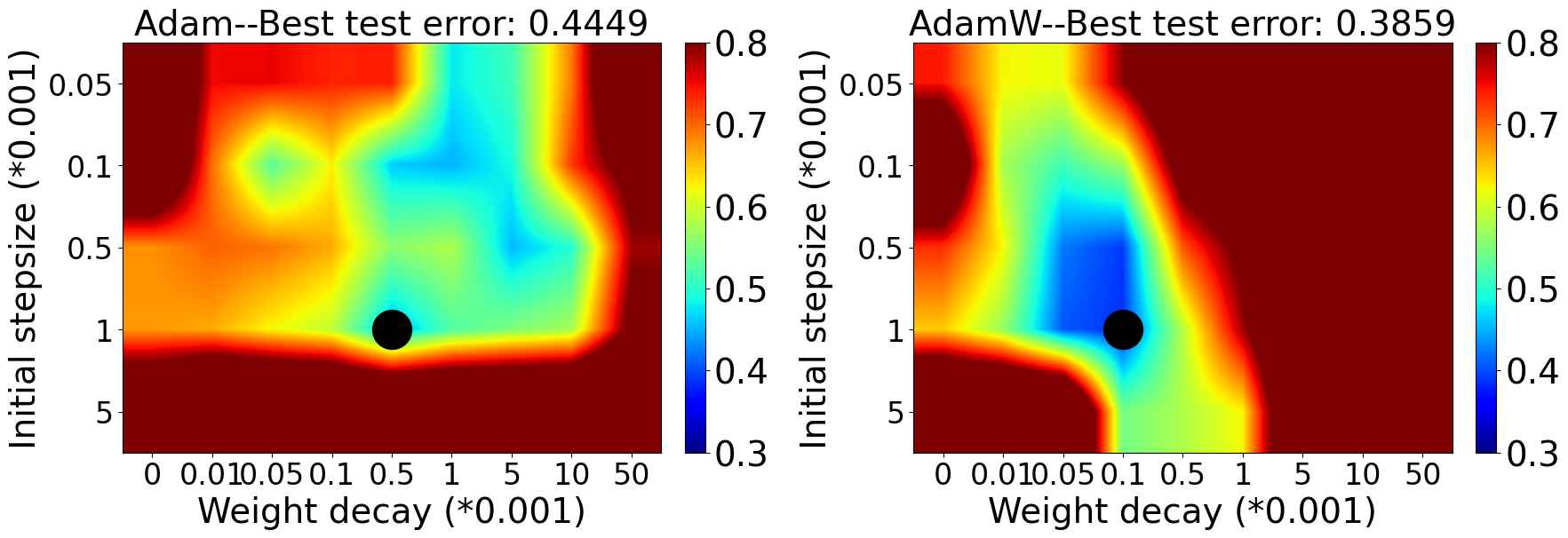}
       \hspace{0.02\textwidth}
       \includegraphics[width=0.30\textwidth]{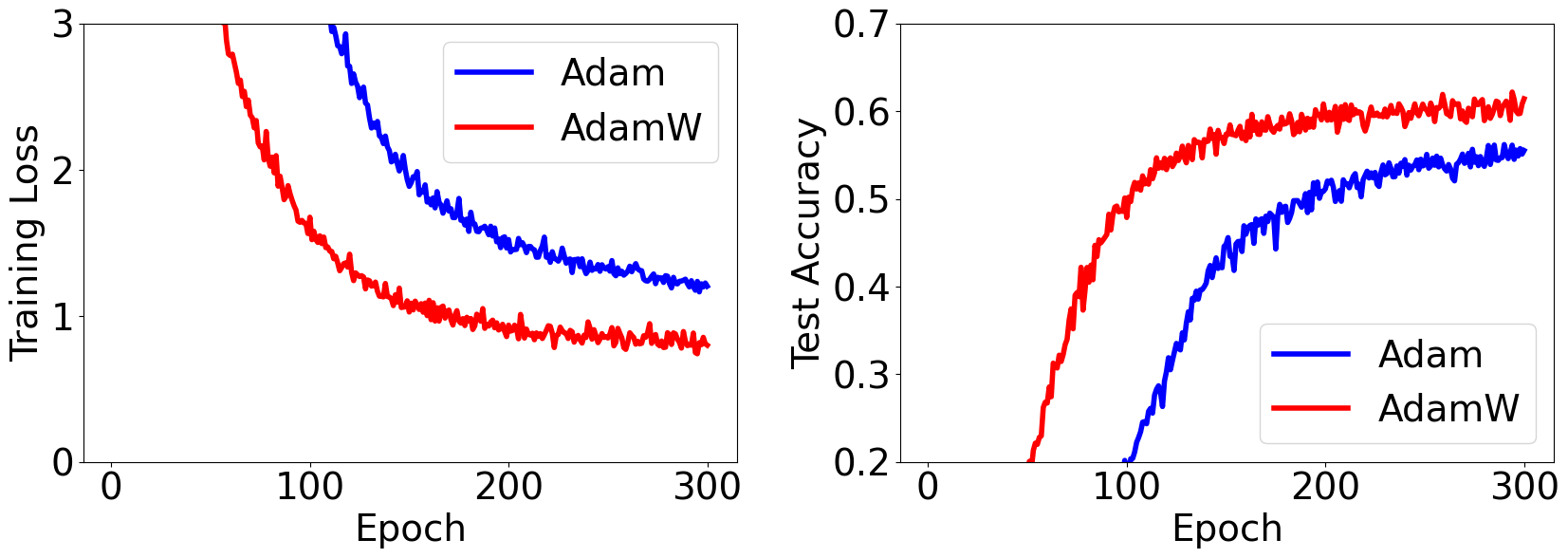}
       \hspace{0.02\textwidth}
       \includegraphics[width=0.15\linewidth]{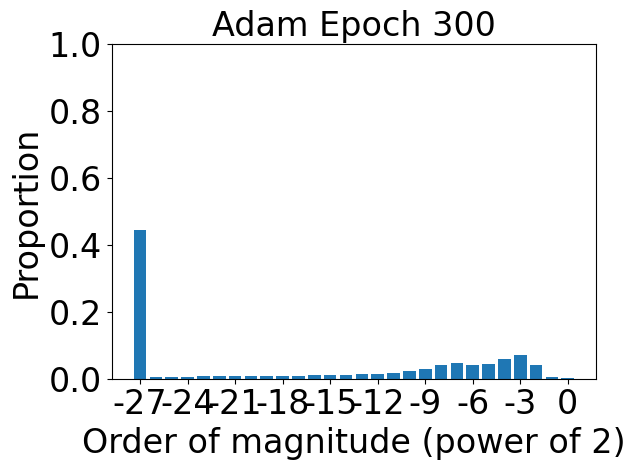}
            \hspace{0pt}
            \includegraphics[width=0.15\linewidth]{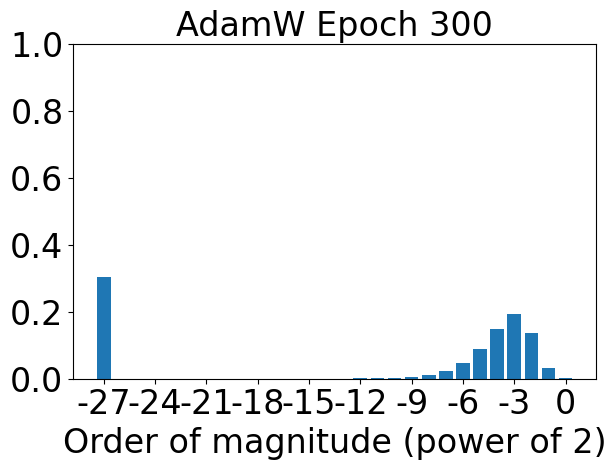}
       \caption{110 Layer Resnet}
       \label{fig:resnet110_c100_nobn} 
    \end{subfigure}
    
    \begin{subfigure}[b]{\textwidth}
       \includegraphics[width=0.30\textwidth]{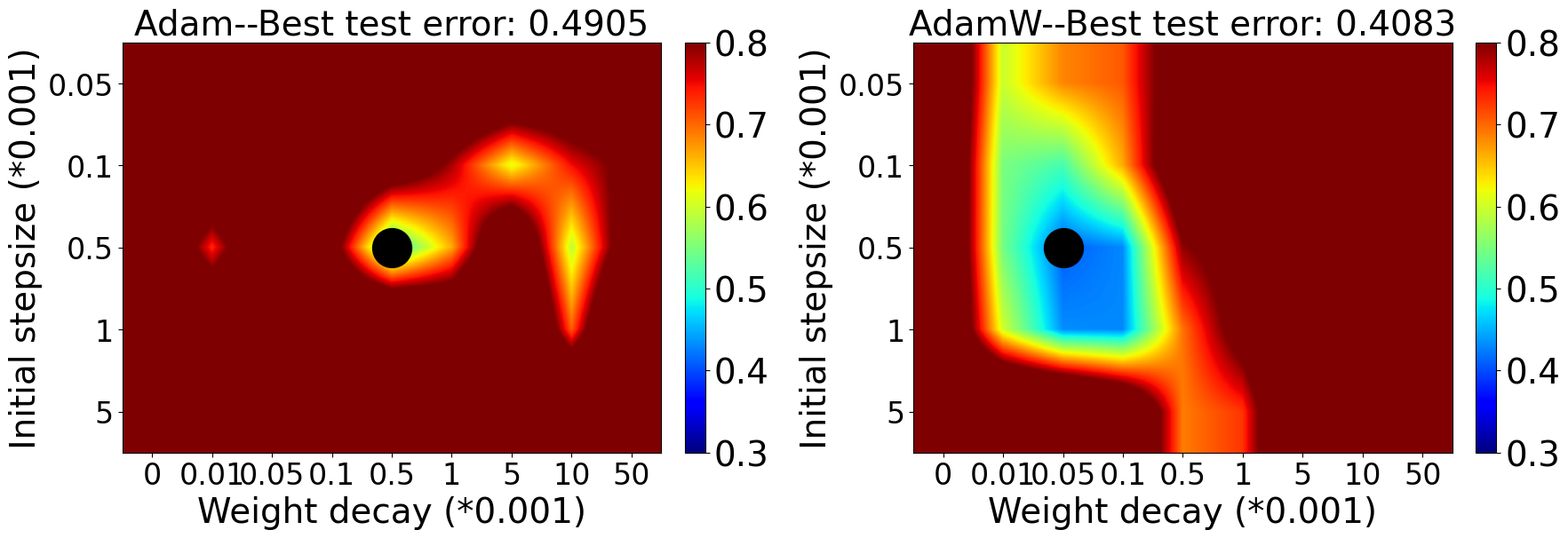}
       \hspace{0.02\textwidth}
       \includegraphics[width=0.30\textwidth]{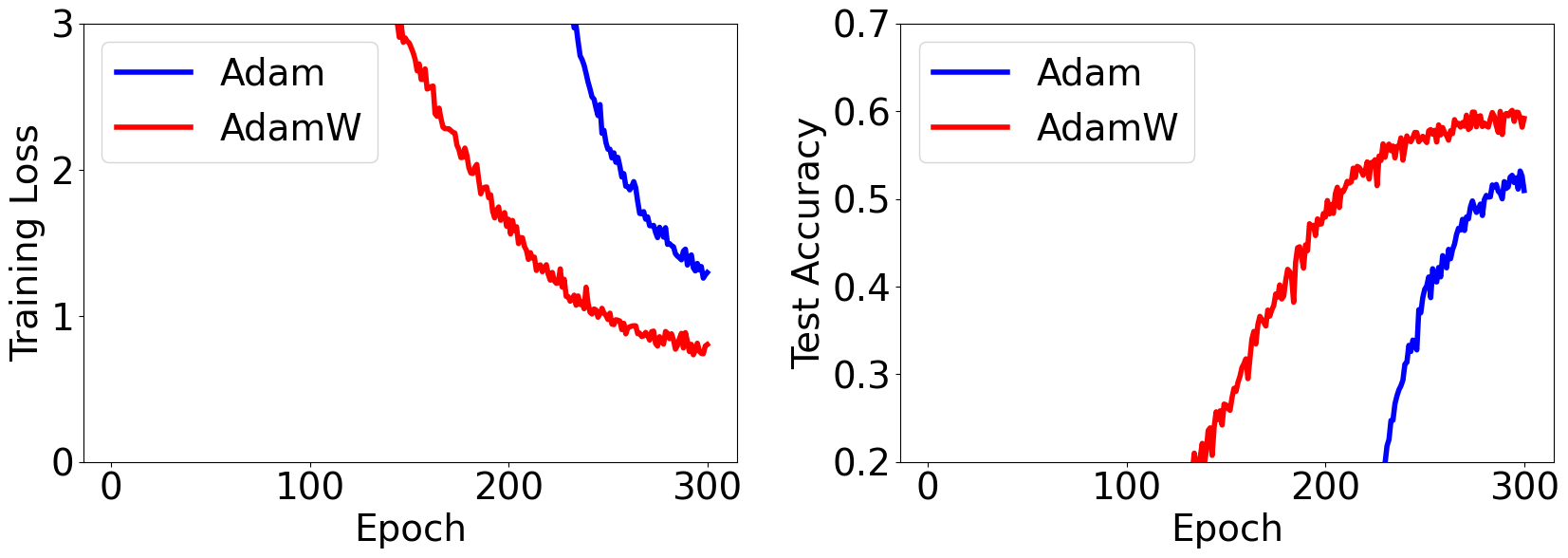}
       \hspace{0.02\textwidth}
       \includegraphics[width=0.15\linewidth]{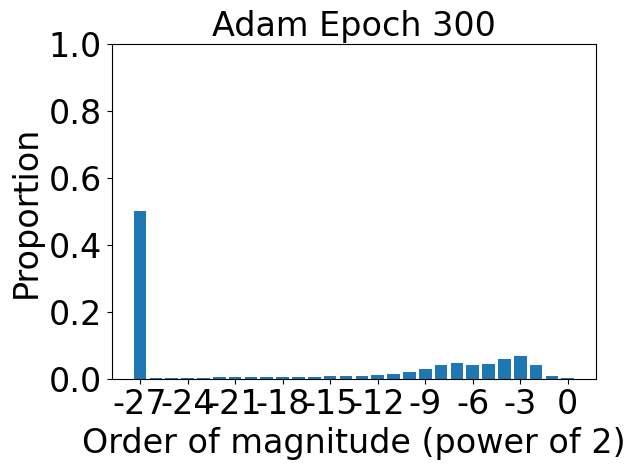}
            \hspace{0pt}
            \includegraphics[width=0.15\linewidth]{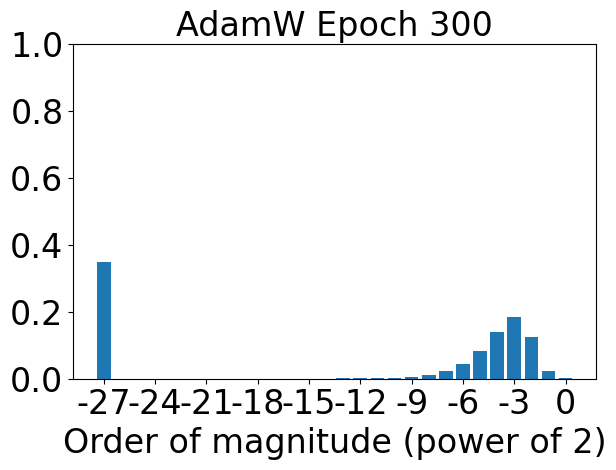}
       \caption{218 Layer Resnet}
       \label{fig:resnet218_c100_nobn} 
    \end{subfigure}
    
    \begin{subfigure}[b]{\textwidth}
       \includegraphics[width=0.30\textwidth]{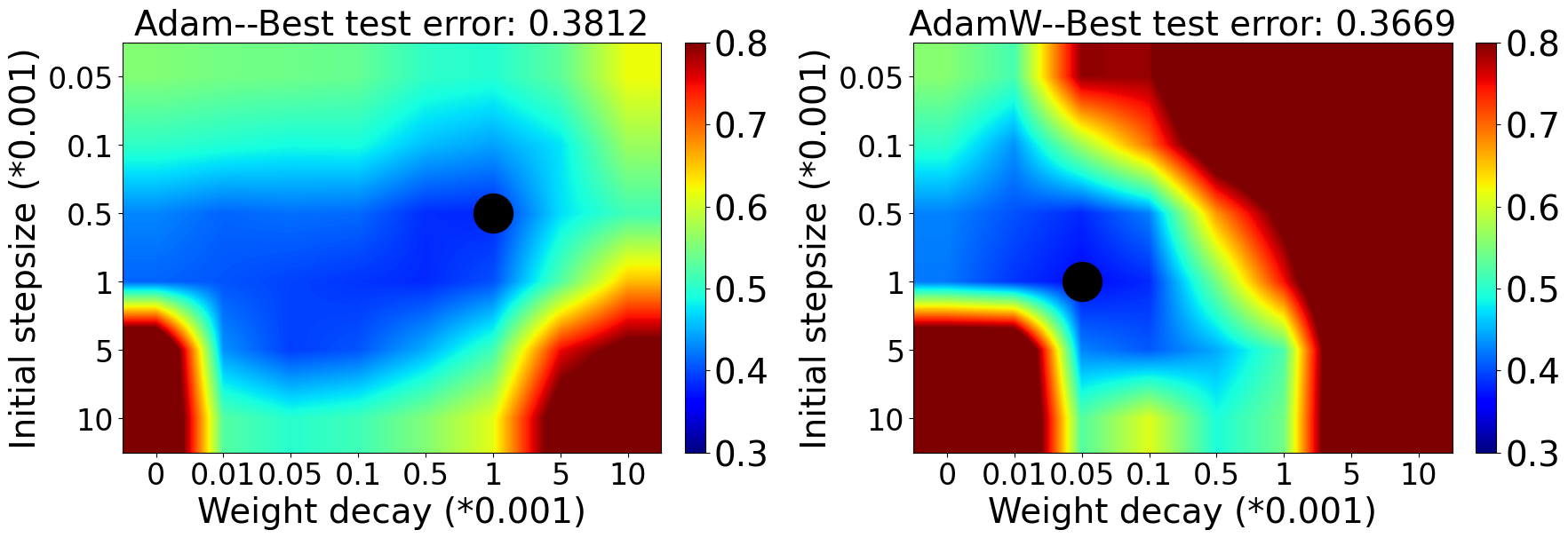}
       \hspace{0.02\textwidth}
       \includegraphics[width=0.30\textwidth]{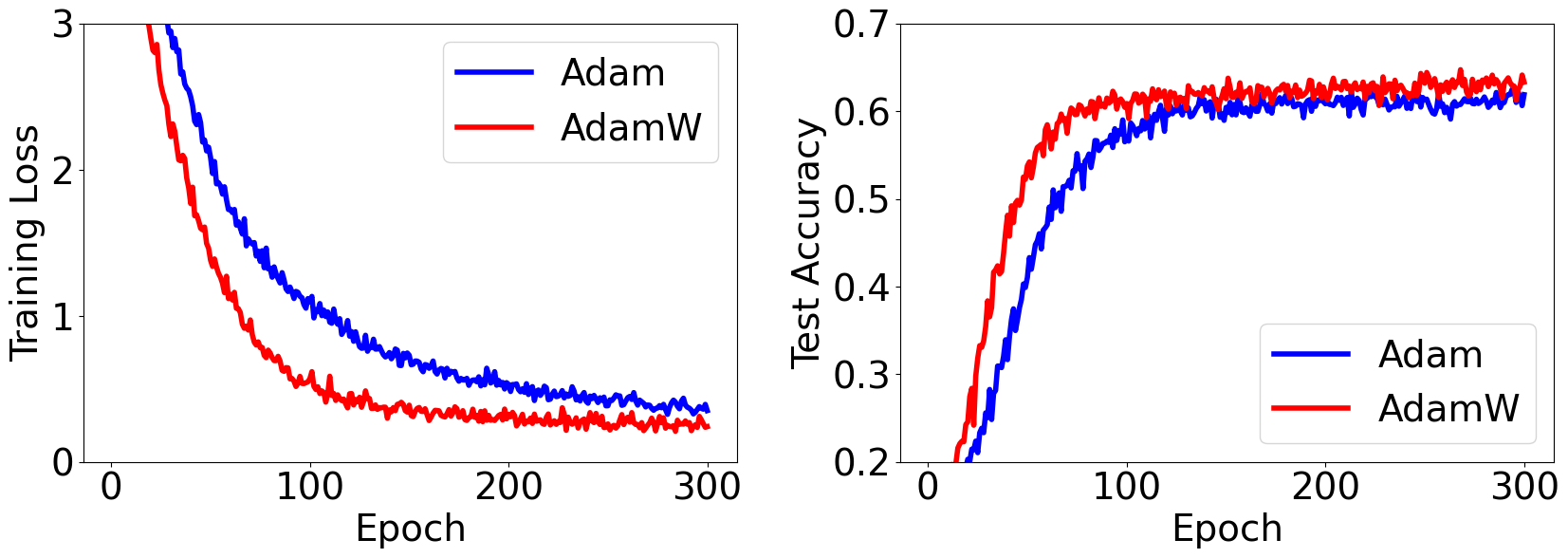}
       \hspace{0.02\textwidth}
       \includegraphics[width=0.15\linewidth]{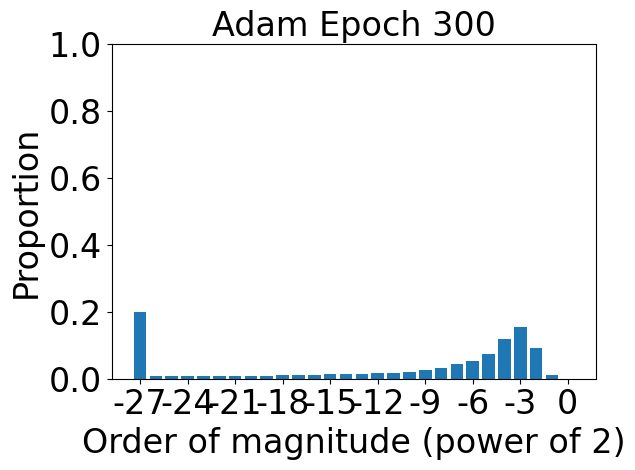}
            \hspace{0pt}
            \includegraphics[width=0.15\linewidth]{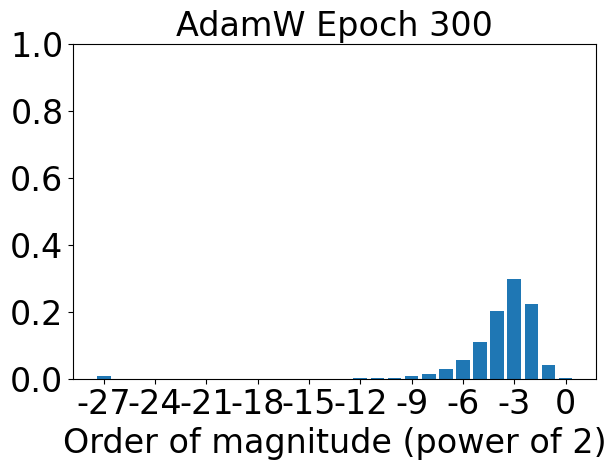}
       \caption{100 layer DenseNet-BC}
       \label{fig:densenet_c100_nobn} 
    \end{subfigure}
    
    \caption[On using AdamW vs.~Adam-$\ell_2$ to train a Resnet/DenseNet with Batch Normalization disabled on CIFAR100]{On using AdamW vs.~Adam-$\ell_2$ to train Resnet/DenseNet with Batch Normalization disabled on CIFAR100. (Left two) The final Top-1 test error (\emph{the black circle denotes the best setting}). (Middle two) The training loss and test accuracy curves when using the initial step size and the weight decay parameter that gives the smallest test error. (Right two) The histogram of update magnitudes of all coordinates near the end of the training when using the initial step size and the weight decay parameter that gives the smallest test error. Note that as the depth of the neural network increases, Adam-$\ell_2$'s updates scatter more evenly over the entire spectrum while AdamW's updates are still concentrated in a small range, and AdamW's advantage in both training and testing over Adam-$\ell_2$ becomes more significant.}
    \label{fig:nobn_cifar100}
\end{figure}

\textbf{With BN, Adam-$\ell_2$ is on par with AdamW} Recently, \citet{BjorckWG20} found that AdamW has no improvement in absolute performance over sufficiently tuned Adam-$\ell_2$ in some reinforcement learning experiments. We also discover the same phenomenon in several image classification tasks, see Figure~\ref{fig:bn}. Indeed, the best weight decay parameter is $0$ for all cases and AdamW coincides with Adam-$\ell_2$  in these cases. Nevertheless, AdamW does decouple the optimal choice of the weight decay parameter from the initial step size much better than Adam-$\ell_2$ in all cases.

\textbf{Removing BN} Notice that the models used in Figure~\ref{fig:bn} all employ BN. Without BN, deep neural networks are known to suffer from gradient explosion and vanishing~\citep{SchoenholzGGS17}. This means each coordinate of the gradient will have very different scales, especially between the first and the last layers. As we will detail in the next section, for Adam-$\ell_2$, the update to the network weights will be affected and each coordinate will proceed at a different pace, whereas AdamW is robust to such issues as the scaling of any single coordinate will not affect the update. Thus, we consider the case where BN is removed as that is where AdamW and Adam-$\ell_2$ will show very different patterns due to scale-freeness.

\textbf{Without BN, AdamW Outperforms Adam-$\ell_2$} In fact, without BN, AdamW outperforms Adam-$\ell_2$ even when both are finely tuned, especially on relatively deep neural networks (see Figure~\ref{fig:nobn_cifar10} and~\ref{fig:nobn_cifar100}). AdamW not only obtains a much better test accuracy but also trains much faster. For example, Figure~\ref{fig:resnet110_nobn} shows that, when training a $110$ layer ResNet~\citep{HeZRS16} with Batch Normalization disabled to do image classification on the CIFAR10 dataset, even when both are finely tuned, AdamW gains a $3\%$ improvement over Adam-$\ell_2$ in test errors as well as converging much faster during training.

In the next section, we propose to understand through the scale-freeness property why this different way of employing regularization leads to AdamW's advantage.

\section{Understanding AdamW through its Scale-freeness}
\label{sec:scale_free}
An optimization algorithm is said to be \emph{scale-free} if its iterates do not change when one multiplies any coordinate of all the gradients by a positive constant~\citep{orabona2015scale}. The scale-free property was first proposed in the online learning field~\citep{Cesa-BianchiMS07,orabona2015scale}. There, they do not need to know a priori the Lipschitz constant of the functions, while still being able to obtain optimal convergence rates. We stress that the scale-freeness is an important but largely overlooked property of an optimization algorithm. It has already been utilized to explain the success of AdaGrad~\citep{orabona2015scale}. Recently, \citet{AgarwalAHKZ20} also provides theoretical and empirical support for setting the $\epsilon$ in the denominator of AdaGrad to be 0, thus making the update exactly scale-free.

It turns out that the update of AdamW is scale-free when $\epsilon = 0$. This is evident as the scaling factor for any coordinate of the gradient is kept in both $\hat{\bm_t}$ and $\sqrt{\hat{\bv_t}}$ and will be canceled out when dividing them. In contrast, for Adam-$\ell_2$, the addition of the gradient of the $\ell_2$ regularization to the gradient (Line 5 of Algorithm~\ref{algo:adamw}) destroys this property.

We want to emphasize the comparison between Adam-$\ell_2$ and AdamW: once Adam-$\ell_2$ adopts a non-zero $\lambda$, it loses the scale-freeness property; in contrast, AdamW enjoys this property for arbitrary $\lambda$. The same applies to any AdaGrad-type and Adam-type algorithm that incorporates the squared $\ell_2$ regularizer by simply adding the gradient of the $\ell_2$ regularizer directly to the gradient of the loss function, as in Adam-$\ell_2$ which is implemented in Tensorflow and Pytorch. Such algorithms are scale-free only when they do not employ regularization.

Nevertheless, one may notice that in practice, the $\epsilon$ factor in the AdamW update is typically small but not 0, in our case $1e$-$8$, thus preventing it from being completely scale-free. Below, we verify that the effect of such an $\epsilon$ on the scale-freeness is negligible.

As a simple empirical verification of the scale-freeness, we consider the scenario where we multiply the loss function by a positive number.
Note that any other method to test scale-freeness would be equally good.
For a feed-forward neural network without BN, this means the gradient would also be scaled up by that factor. In this case, the updates of a scale-free optimization algorithm would remain exactly the same, whereas they would change for an optimization algorithm that is not scale-free. 

Figure~\ref{fig:lossmul} shows the results of the loss function being multiplied by 10 and 100 respectively on optimizing a 110-layer Resnet with BN \emph{disabled}. For results of the original loss see Figure~\ref{fig:resnet110_nobn}.
We can see that AdamW has almost the same performance across the range of initial step sizes and weight decay parameters, and most importantly, the best values of these two hyperparameters remain the same. This verifies that, even when employing a (small) non-zero $\epsilon$, AdamW is still approximately scale-free. In contrast, Adam-$\ell_2$ is not scale-free and we can see that its behavior varies drastically with the best initial step sizes and weight decay parameters in each setting totally different.

\begin{figure}
\centering
\begin{subfigure}{0.48\linewidth}
\centering
\includegraphics[width=\linewidth]{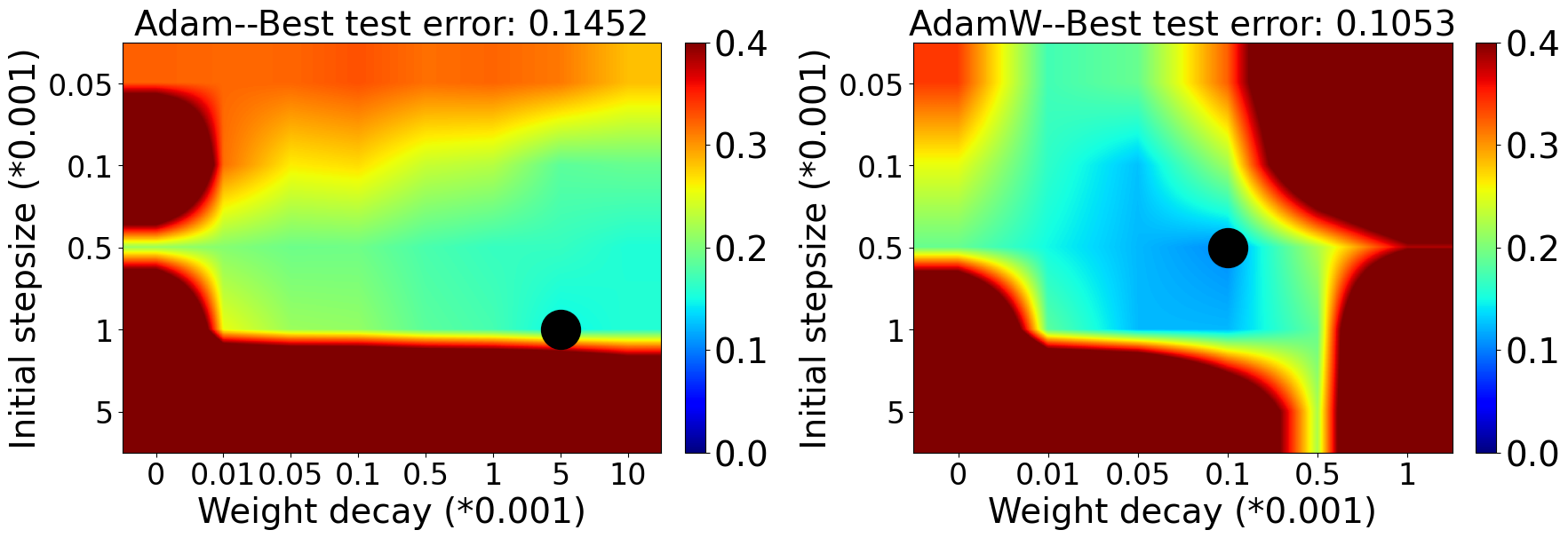}
\caption{Loss multiplied by 10}
\label{fig:lossmul10}
\end{subfigure}
\hspace{0.02\linewidth}
\begin{subfigure}{0.48\linewidth}
\centering
\includegraphics[width=\linewidth]{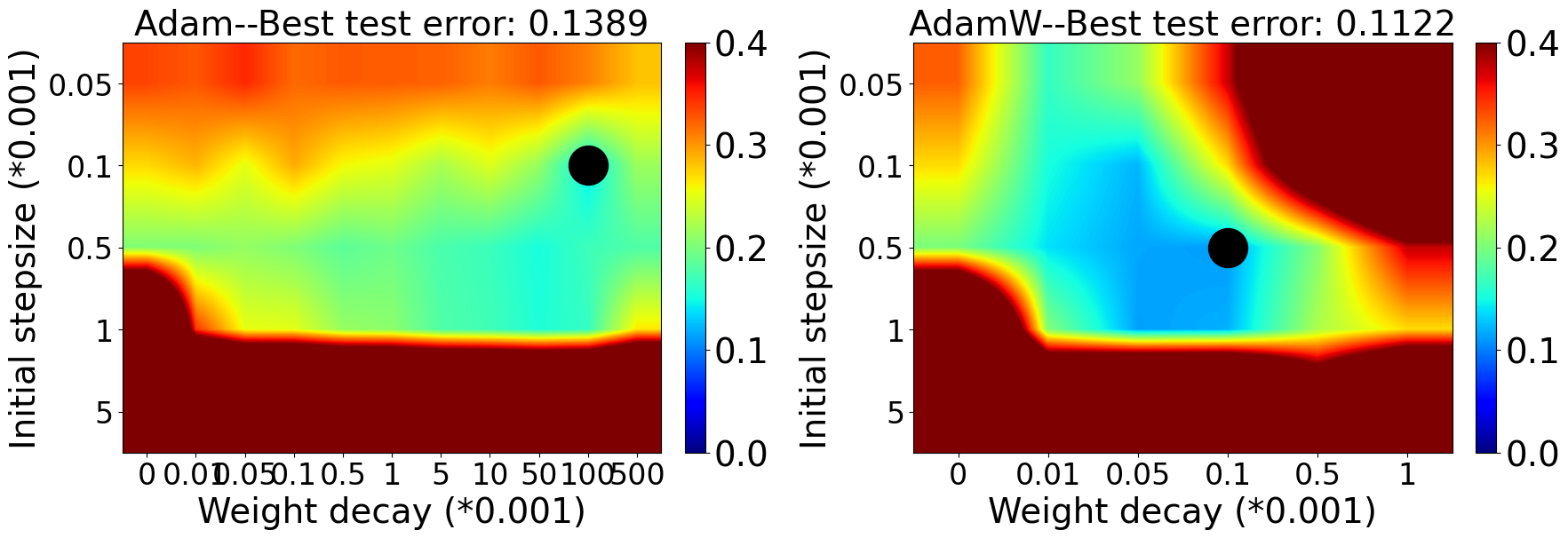}
\caption{Loss multiplied by 100}
\label{fig:lossmul100}
\end{subfigure}
\caption[Empirical verification of the scale-freeness of AdamW]{The final top-1 test error of AdamW vs.~Adam-$\ell_2$ on optimizing a 110-layer Resnet with BN \emph{removed} on CIFAR-10 with the loss function multiplied by 10 (left two figures) and 100 (right two figures). Note how the best performing hyperparameter combinations of AdamW remain the same for different loss multiplication factors as well as the shape of the heatmap being very similar. In contrast, Adam-$\ell_2$'s performance as well as the best performing hyperparameter combinations vary dramatically for different loss multiplication factors.}
\label{fig:lossmul}
\end{figure}

With that said, our main claim is \emph{the lack of scale-freeness seems to harm Adam-$\ell_2$'s performance in certain scenarios in deep learning, while AdamW preserves the scale-freeness even with non-zero regularization}.

This is exactly verified empirically as illustrated in the 5th \& 6th columns of figures in Figure~\ref{fig:nobn_cifar10} and~\ref{fig:nobn_cifar100}. There, we report the histograms of the absolute value of updates of Adam-$\ell_2$ vs.~AdamW of all coordinates near the end of training (for their comparison over the whole training process please refer to the Appendix~\ref{sec:hist_entire_train}).

Indeed, the optimization processes of these two optimizers show the effects of with or without the scale-freeness. As can be seen, the magnitudes of AdamW's updates are much more concentrated than that of Adam-$\ell_2$ throughout the training. This means that a scale-free algorithm like AdamW ensures that each layer is updated at a similar pace; in contrast, for a non-scale-free optimization algorithm like Adam-$\ell_2$, different layers will proceed at very different speeds.

We also observe that the advantage of AdamW becomes more evident as the network becomes deeper. Recall that as the depth grows, without BN, the gradient explosion and vanishing problem becomes more severe. This means that for the non-scale-free Adam-$\ell_2$, the updates of each coordinate will be dispersed on a wider range of scales even when the same weight decay parameter is employed. In contrast, the scales of the updates of AdamW will be much more concentrated in a smaller range.

This correlation between the advantage of AdamW over Adam-$\ell_2$ and the different spread of update scales which is induced by the scale-freeness property of AdamW provides empirical evidence on when AdamW excels over Adam-$\ell_2$.

As a side note, the reader might wonder why SGD is known to provide state-of-the-art performance on many deep learning architectures~\citep[e.g.,][]{HeZRS16, HuangLVW17} \emph{without} being scale-free. At first blush, this seems to contradict our claims that scale-freeness correlates with good performance. In reality, the good performance of SGD in very deep models is linked to the use of BN that normalizes the gradients. Indeed, we verified empirically that SGD fails spectacularly when BN is not used. For example, on training the 110 layer Resnet without BN using SGD with momentum and weight decay of $0.0001$, even a step size of $1e-10$ will lead to divergence.

\section{AdamW and Proximal Updates}
\label{sec:adamw_proximal}
The scale-freeness property of AdamW may seem a natural consequence of the way it constructs its update. Yet, in this section, we reveal the surprising connection between AdamW and proximal updates~\citep{ParikhB14}, suggesting another potential explanation of where AdamW's scale-freeness comes from.

A proximal algorithm is an algorithm for solving a convex optimization
problem that uses the proximal operators of the objective function. The \emph{proximal operator} $\text{prox}_h: \R^d \rightarrow \R^d$ of a convex function $h$ is defined for any $\by\in\R^d$ as $\text{prox}_h(\by) = \arg\min_{\bx\in\R^d}(h(\bx) + \frac12\|\bx-\by\|_2^2)$. The use of proximal updates in the batch optimization literature dates back at least to 1965~\citep{Moreau65,Martinet70,Rockafellar76,ParikhB14} and they are used more recently even in the stochastic setting~\citep{ToulisA17,AsiD19}.

Now consider that we want to minimize the objective function 
\begin{equation}
\label{eq:composite}
F(\bx) = \tfrac{\lambda}{2} \|\bx\|_2^2 + f(\bx),
\end{equation}
where $\lambda>0$ and $f(\bx):\R^d\rightarrow \R$ is a function bounded from below.
We could use a stochastic optimization algorithm that updates in the following fashion
\begin{equation}
\label{eq:stoc_opt}
\bx_{t} = \bx_{t-1} - \eta_t\bp_t,
\end{equation}
where $\eta_t$ is a learning rate schedule, e.g., the constant one or the cosine annealing~\citep{LoshchilovH17} and $\bp_t$ denotes any update direction. This update covers many cases, where $\alpha$ denotes the initial step size:
\begin{enumerate}[topsep=0pt, parsep=1pt]
\item $\bp_t = \alpha\bg_t$ gives us the vanilla SGD;
\item $\bp_t = \alpha\bg_t/(\sqrt{\sum^{t}_{i=1}\bg_i^2 + \epsilon})$ gives the AdaGrad algorithm~\citep{DuchiHS10};
\item $\bp_t = \alpha\hat{\bm}_t/(\sqrt{\hat{\bv}_t} + \epsilon)$ recovers Adam~\citep{KingmaB15}, where $\hat{\bm}_t$ denotes the bias corrected first moment of past gradients and $\hat{\bv}_t$ denotes the bias corrected second moment of past gradients as updated in Line 6-7 in Algorithm~\ref{algo:adamw}.
\end{enumerate}
Note that in the above we use $\bg_t$ to denote the stochastic gradient of the entire objective function: $\bg_t=\nabla f_t(\bx_{t-1})+\lambda\bx_{t-1}$ ($\lambda=0$  if the regularizer is not present), where $\nabla f_t(\bx_{t-1})$ is a stochastic evaluation of the true gradient $\nabla f(\bx_{t-1})$.

This update rule~\eqref{eq:stoc_opt} is given by the following online mirror descent update~\citep{NemirovskyY83, WarmuthJ97, BeckT03}:
\begin{equation}
\label{eq:mirror_descent_1}
\begin{aligned}
\bx_{t} = \argmin_{\bx \in \R^d}\ &\tfrac{\lambda}{2}\|\bx_{t-1}\|^2_2 + f(\bx_{t-1}) + \bp_t^\top (\bx-\bx_{t-1})
+ \tfrac{1}{2\eta_t} \|\bx -\bx_{t-1}\|_2^2~.
\end{aligned}
\end{equation}

This approximates minimizing a first-order Taylor approximation of $F$ centered in $\bx_{t-1}$ plus a term that measures the distance between the $\bx_t$ and $\bx_{t-1}$ according to the $\ell_2$ norm. The approximation becomes exact when $\bp_t = \nabla f(\bx_{t-1}) + \lambda\bx_{t-1}$.

Yet, this is not the only way to construct first-order updates for the objective~\eqref{eq:composite}. An alternative route is to linearize only $f$ and to keep the squared $\ell_2$ norm in its functional form:
\begin{equation}
\label{eq:mirror_descent_2}
\begin{aligned}
\bx_{t}
&= \argmin_{\bx \in \R^d}\ \tfrac{\lambda}{2}\|\bx\|^2_2 + f(\bx_{t-1}) + \bp_t^\top (\bx-\bx_{t-1})
+
\tfrac{1}{2\eta_t} \|\bx -\bx_{t-1}\|_2^2\\
&=
\text{prox}_{\frac{\lambda\eta_t}{2}\|\cdot\|_2^2}(\bx_{t-1} - \eta_t\bp_t),
\end{aligned}
\end{equation}
which uses the proximal operator of the convex function $\frac{\lambda\eta_t}{2}\|\cdot\|_2^2$.

It is intuitive why this would be a better update: \emph{We directly minimize the squared $\ell_2$ norm instead of approximating it.} We also would like to note that, similar to~\eqref{eq:mirror_descent_1}, the proximal updates of~\eqref{eq:mirror_descent_2} can be shown to minimize the objective $F$ under appropriate conditions. However, we do not include the convergence analysis of~\eqref{eq:mirror_descent_2} as this is already well-studied in the literature. For example, when $\bp_t = \nabla f(\bx_{t-1})$ in~\eqref{eq:mirror_descent_2} and $f$ is convex and smooth, the update becomes a version of the (non-accelerated) iterative shrinkage-thresholding algorithm. This algorithm guarantees $F(\bx_t) - F^* \le O(1/t)$, which is in the same order as obtained by gradient descent on minimizing $f$ alone~\citep{BeckT09}.

From the first-order optimality condition, the update is
\begin{equation}
\label{eq:prox_sgd}
\bx_{t} = (1 + \lambda \eta_t)^{-1}(\bx_{t-1} - \eta_t \bp_t)~.
\end{equation}
When $\lambda=0$, the update in \eqref{eq:stoc_opt} and this one coincide. Yet, when $\lambda\neq0$, they are no longer the same.

We now show how the update in~\eqref{eq:prox_sgd} generalizes the one in AdamW. The update of AdamW is
\begin{equation}
\label{eq:adamw}
\bx_{t} = (1-\lambda\eta_t)\bx_{t-1} - \eta_t\alpha\hat{\bm}_t/(\sqrt{\hat{\bv}_t} + \epsilon)~.
\end{equation}
On the other hand, using $\bp_t=\alpha\hat{\bm}_t/(\sqrt{\hat{\bv}_t} + \epsilon)$ in~\eqref{eq:prox_sgd} gives:
\begin{equation}
\label{eq:adamprox}
\bx_{t} = (1 + \lambda\eta_t)^{-1}(\bx_{t-1} - \eta_t\alpha\hat{\bm}_t/(\sqrt{\hat{\bv}_t} + \epsilon)),
\end{equation}
which we will call \emph{AdamProx} hereafter. Its first-order Taylor approximation around $\eta_t=0$ is
\begin{equation*}
\bx_{t} \approx (1 - \lambda\eta_t)\bx_{t-1} - \eta_t\alpha\hat{\bm}_t/(\sqrt{\hat{\bv}_t} + \epsilon),
\end{equation*}
exactly the AdamW update~\eqref{eq:adamw}. Hence, AdamW is a first-order approximation of a proximal update.

\begin{figure}[t]
\centering
\begin{subfigure}{0.48\linewidth}
\centering
\includegraphics[width=\linewidth]{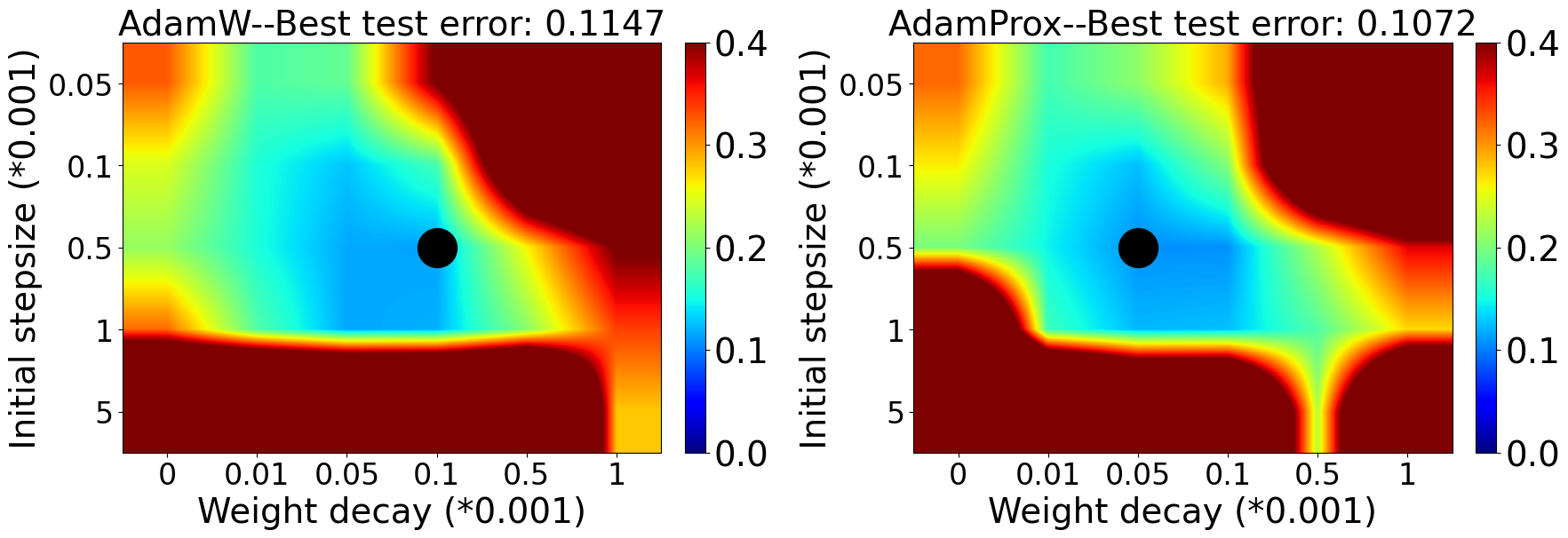}
\caption{ResNet on CIFAR-10}
\label{fig:adamprox1}
\end{subfigure}
\hspace{0.02\linewidth}
\begin{subfigure}{0.48\linewidth}
\centering
\includegraphics[width=\linewidth]{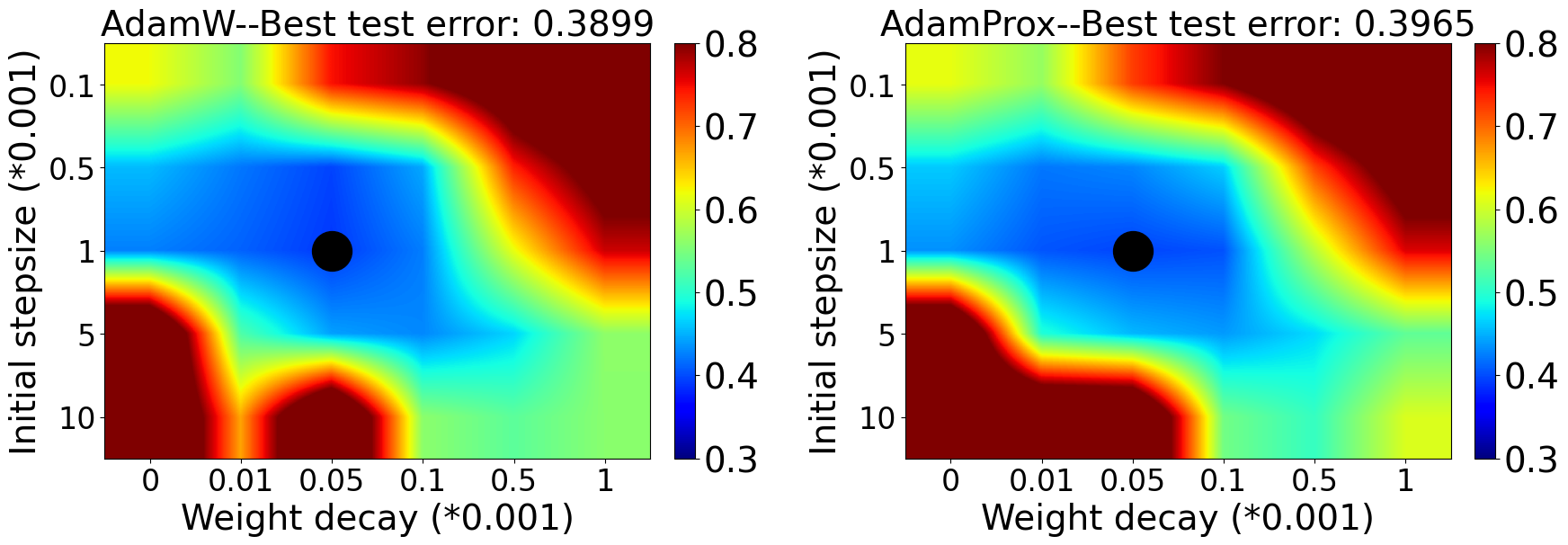}
\caption{DenseNet-BC on CIFAR-100}
\label{fig:adamprox2}
\end{subfigure}
\caption[AdamProx performs very similarly to AdamW.]{The final Top-1 test error of using AdamW vs.~AdamProx on training (\emph{the black circle denotes the best setting}). (Top row) a 110-layer ResNet with BN \emph{removed} on CIFAR-10 (trained for 300 epochs). (Bottom row) a 100-layer DenseNet-BC with BN \emph{removed} on CIFAR-100 (trained for 100 epochs). Note how similar are the shapes of the heatmaps, the best performing hyperparameter combinations, and the test errors between AdamW and AdamProx.}
\label{fig:adamadamprox}
\end{figure}

The careful reader might notice that the approximation from AdamW to the AdamProx update in \eqref{eq:adamprox} becomes less accurate when $\eta_t$ becomes too large, and so be concerned about whether this approximation is practical at all.
Fortunately, in practice, $\eta_t$ is never large enough for this to be an issue.
The remainder term of this approximation is $O(\lambda \eta_t^2)$
which we should always expect to be small as both $\lambda$ and $\eta_t$ are small.
So, we can expect AdamW and AdamProx to perform similarly for step size schedules $\eta_t$ commonly employed in practice.

Indeed, we verified empirically that the approximation is good as reported in Figure~\ref{fig:adamadamprox}, where we consider the case when $\eta_t$ = 1 for all $t$, a relatively large constant step size schedule. In such cases, AdamW and AdamProx still behave very similarly. This suggests that for most step size schedules, e.g., cosine, exponential, polynomial, and step decay, which all monotonously decrease from $\eta_0 = 1$, AdamProx will remain a very good approximation to AdamW. Thus, it is reasonable to use the more classically-linked AdamProx to try to understand AdamW.

Let's now derive the consequences of this connection with proximal updates.

First of all, at least in the convex case, the convergence rate of the proximal updates will depend on $\|\nabla f(\bx_t)\|^2_2$ rather than on $\|\nabla f(\bx_t) + \lambda \bx_t\|^2_2$~\citep{DuchiSST10}. This could be a significant improvement: the regularized loss function is never Lipschitz, so the regularized gradients $\nabla f(\bx_t) + \lambda \bx_t$ could be much larger than $\nabla f(\bx_t)$ when $f$ itself is Lipschitz.

Also, as we wrote above, AdamW can be seen as the first-order Taylor approximation on $\eta_t = 0$ of the AdamProx update in~\eqref{eq:adamprox}; in turn, the scale-freeness of AdamProx directly comes from the proximal updates. Of course, there may be other ways to design scale-free updates solving~\eqref{eq:composite}; yet, for AdamW, its scale-free property derives directly from the proximal update.

More importantly, proximal updates are fundamentally better at keeping the weights small. Let us consider a couple of simple examples to see how this could be. First, suppose the weights are \emph{already zero}. Then, when taking an update according to \eqref{eq:stoc_opt}, we increase the weights to $-\eta_t \bp_t$. In contrast, update \eqref{eq:prox_sgd} clearly leads to a smaller value. This is because it computes an update using the regularizer rather than its gradient. As an even more disturbing, yet actually more realistic example, consider the case that $\bx_{t-1}$ is non-zero, but $\bg_t=\boldsymbol{0}$. In this case, taking an update using \eqref{eq:stoc_opt} may actually \emph{increase} the weights by causing $\bx_{t}$ to \emph{overshoot} the origin. In contrast, the proximal update will never demonstrate such pathological behavior. Notice that this pathological behavior of \eqref{eq:stoc_opt} can be mitigated by properly tuning the step size. However, one of the main attractions of adaptive optimizers is that we should not need to tune the step size as much. Thus, \emph{the proximal update can be viewed as augmenting the adaptive methods with an even greater degree of learning-rate robustness.}

\section{Scale-free Algorithms can Adapt to the Condition Number}
\label{sec:scale_free_cond_num}
Interestingly, scale-freeness comes with another benefit: it can effectively reduce the effects of the condition number in certain scenarios, as detailed below.

For a twice continuously differentiable function $F$, its Hessian matrix is symmetric and its \emph{condition number} $\kappa$ is defined as the ratio of its largest absolute value eigenvalue to its smallest one. It is well-known that the best convergence rate when minimizing such $F$ using a first-order optimization algorithm (e.g., gradient descent) must depend on the condition number~\citep[Theorem 2.1.13,][]{Nesterov04}, formally,
\begin{equation}
\|\bx_t - \bx^*\|_2^2 \ge \left(\frac{\sqrt{\kappa} - 1}{\sqrt{\kappa} + 1}\right)^{2t}\|\bx_0 - \bx^*\|_2^2~.
\end{equation}
In particular, a problem with a small $\kappa$ can be solved more efficiently than one with a big $\kappa$.

One way to reduce the effect of the condition number is to use a \emph{preconditioner}~\citep{NocedalW06}. While originally designed for solving systems of linear equations, preconditioning can be extended to the optimization of non-linear functions and it should depend on the Hessian of the function~\citep{BoydV04,Li18}.
However, it is unclear how to set the preconditioner given that the Hessian might not be constant~\citep[Section 9.4.4][]{BoydV04} and 
in stochastic optimization the Hessian cannot be easily estimated~\citep{Li18}.

In the following theorem, we show that scale-freeness gives similar advantages to the use of an optimal diagonal preconditioner, \emph{for free}. Specifically, a scale-free algorithm can automatically transform solving the original problem into solving a problem with a potentially much smaller condition number and thus could provide substantial improvements over non-scale-free ones.

\begin{thm}
\label{thm:scalefree}
Let $F$ be a twice continuously differentiable function and $\bx^*$ such that $\nabla F(\bx^*)=\boldsymbol{0}$. Next, let $\tilde{F}_\Lambda$ be the family of functions such that $\nabla \tilde{F}_\Lambda(\bx^*) = \boldsymbol{0}$, and $\nabla^2 \tilde{F}_\Lambda(\bx) = \Lambda\nabla^2 F(\bx)$, where $\Lambda = diag(\lambda_1, \ldots, \lambda_d) \succeq 0$. Then, running any scale-free optimization algorithm on $F$ and $\tilde{F}_\Lambda$ will result exactly in the same iterates, assuming the same noise on the gradients. Moreover, any dependency on the condition number of the scale-free algorithm will be reduced to the smallest condition number among all the functions $\tilde{F}_\Lambda$.
\end{thm}
\begin{proof}[Proof of Theorem~\ref{thm:scalefree}]
From the Fundamental Theorem of Calculus we have:
\begin{align}
\label{eq:base}
\nabla F(\bx)
&=
\nabla F(\bx^*) + \int^1_0\nabla^2 F(\bx^* + t (\bx - \bx^*))(\bx - \bx^*)dt\\
&=
\int^1_0\nabla^2 F(\bx^* + t (\bx - \bx^*))(\bx - \bx^*)dt ~.
\end{align}
Thus, for any function $\tilde{F}_\Lambda(\bx)$ whose Hessian is $\Lambda\nabla^2 F(\bx)$ and $\nabla\tilde{F}_\Lambda(\bx^*)=0$, we have $\nabla\tilde{F}_\Lambda(\bx) = \Lambda\nabla F(\bx)$.

Now, from the definition of a scale-free algorithm, the iterates of such an algorithm do not change when one multiplies each coordinate of all the gradients by a positive constant. Thus, a scale-free algorithm optimizing $F$ behaves the same as if it is optimizing $\tilde{F}_\Lambda$.
\end{proof}

To give an example of when this is advantageous, consider when $\nabla^2 F(\bx)$ is a diagonal matrix:
\[
\nabla^2 F(\bx) = diag(g_1(\bx), g_2(\bx), \ldots, g_d(\bx))~.
\]
Assume $0 < \mu \le \mu_i \le g_i(\bx) \le L_i \le L$ for $i \in \{1,\ldots,d\}$. Denote $j = \arg\max_i L_i/\mu_i$. Choose $\lambda_i$ s.t.~$\mu_j\le\lambda_i\mu_i\le\lambda_i g_i(\bx)\le\lambda_i L_i\le L_j$ then $\Lambda\nabla^2 F(\bx)$ has a condition number $\kappa^{\prime} = L_j/\mu_j$. This gives scale-free algorithms a big advantage when $\max_i L_i / \mu_i \ll L/\mu$.

Another example is one of the quadratic functions.

\begin{figure}[t]
\begin{subfigure}{0.48\linewidth}
  \centering
\includegraphics[width=\linewidth]{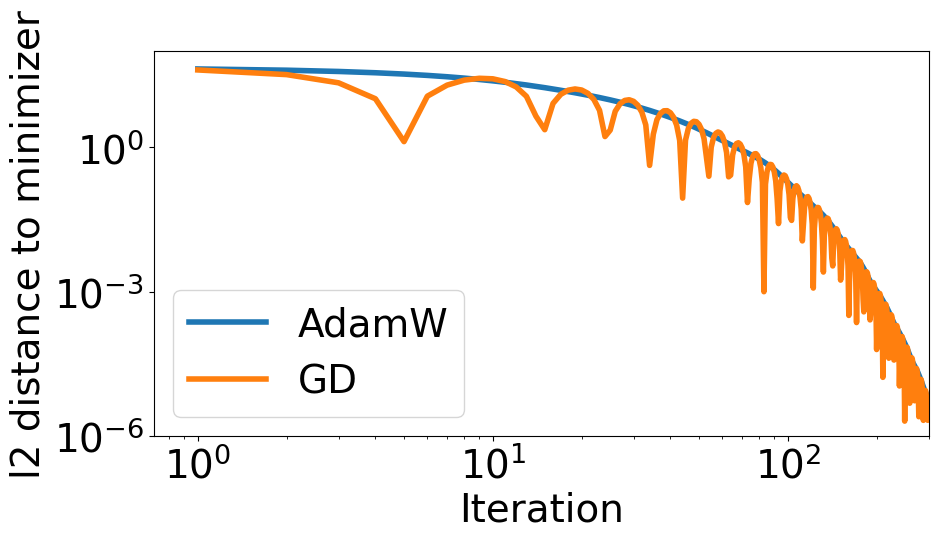}
\subcaption{Condition number $1$.}
\end{subfigure}
\hfill
\begin{subfigure}{0.48\linewidth}
  \centering
\includegraphics[width=\linewidth]{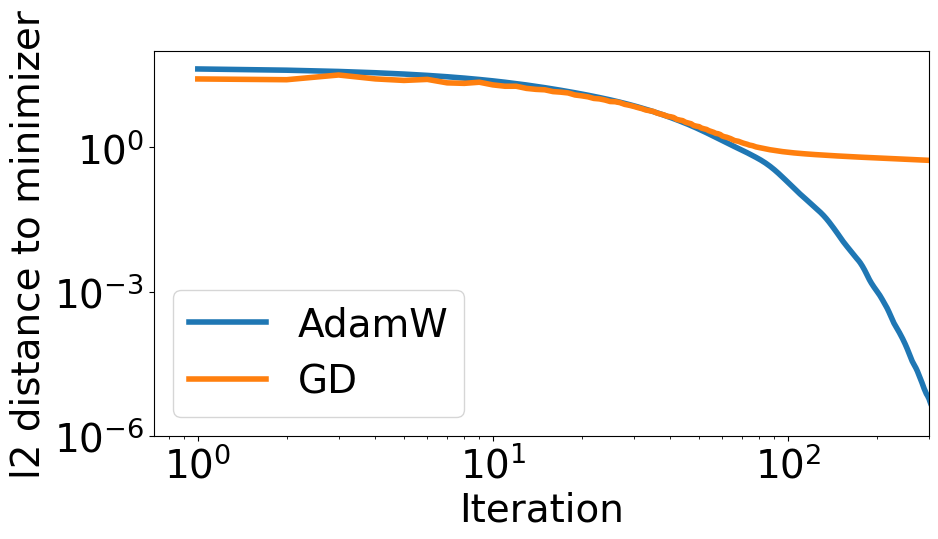}
\subcaption{Condition number $100000$.}
\end{subfigure}
\caption{Non-scale-free GD v.s.~scale-free AdamW on optimizing quadratic functions with different condition numbers.}
\label{fig:quadratic}
\end{figure}

\begin{corollary}
\label{coro:scale_free_quadratic}
For quadratic problems $F(\bx) = \tfrac12 \bx^\top H \bx + \bb^\top \bx +c $, with $H$ diagonal and positive definite, any scale-free algorithm will not differentiate between minimizing $f$ and $\tilde{F}(\bx) = \frac12 \bx^\top \bx + (H^{-1} \bb)^\top \bx +c$. As the condition number of $\tilde{F}$ is 1, the operation, and most importantly, the convergence, of a scale-free algorithm will not be affected by the condition number of $F$ at all.
\end{corollary}

Figure~\ref{fig:quadratic} illustrates Corollary~\ref{coro:scale_free_quadratic}: we compare GD (non-scale-free) with AdamW (scale-free) on optimizing two quadratic functions with the same minimizer, but one's Hessian matrix is a rescaled version of the other's, resulting in different condition numbers. The figure clearly shows that, even after tuning the step sizes, the updates of AdamW (starting from the same point) and thus its convergence to the minimizer, are completely unaffected by the condition number, while GD's updates change drastically and its performance deteriorates significantly when the condition number is large. It is not hard to imagine that such poor training performance would likely also lead to poor testing performance.

This can also explain AdaGrad's improvements over SGD in certain scenarios. Note that the folklore justification for such improvements is that the step size of AdaGrad approximates the inverse of the Hessian matrix, but this is incorrect: AdaGrad does not compute Hessians and there is no reason to believe it approximates them in general. As an additional example, we analyze below a variant of AdaGrad with restarts (Algorithm~\ref{algo:restart_AdaGrad}) and show its convergence rate guarantee on strongly convex functions (ones satisfying~\eqref{eq:strongly_convex}) which exhibits a dependency on the condition number $\kappa$ and thus can benefit from its scale-freeness.

\begin{algorithm}[t]
\caption{AdaGrad~\citep{DuchiHS10, McMahanS10} \emph{(All operations on vectors are element-wise.)}}
\label{algo:AdaGrad}
\begin{algorithmic}
\STATE \textbf{Input}: Number of iterations $T$, a set $\mathcal{K}$, $\bx_1\in \mathcal{K}$, step size $\eta$
\FOR{$t=1 \ldots T$ }
\STATE {Receive:} $\nabla F(\bx_t)$
\STATE {Set:} $\boldeta_t = \frac{\eta}{\sqrt{\sum^t_{i=1}(\nabla F(\bx_i))^2}}$
\STATE {Update:}
$\bx_{t+1}= \Pi_{\mathcal{K}}\left(\bx_{t}-\boldeta_t\nabla F(\bx_t)\right)$ where $\Pi_{\mathcal{K}}$ is the projection onto $\mathcal{K}$
\ENDFOR
\STATE {Output:} $\bar{\bx} = \frac1T\sum^T_{t=1}\bx_t$
\end{algorithmic}
\end{algorithm}

\begin{algorithm}[t]
\caption{AdaGrad with Restart}
\label{algo:restart_AdaGrad}
\begin{algorithmic}
\STATE \textbf{Input}: {{Number of rounds $N$, $\bx_0\in\R^d$, upper bound on $\|\bx_0 - \bx^*\|_{\infty}$ as $D_{\infty}$, strong convexity $\mu$, smoothness $L$}}
\STATE {Set}: $\bar{\bx}_0 = \bx_0$
\FOR{$i=1 \ldots N$ }
\STATE {{{Run Algorithm~\ref{algo:AdaGrad} to get $\bar{\bx}_{i}$ with $T=32d\frac{L}{\mu}$, $\bx_1 = \bar{\bx}_{i-1}$, $\mathcal{K} = \{\bx: \|\bx - \bar{\bx}_{i-1}\|_{\infty}^2\le \frac{D_{\infty}^2}{4^{i-1}}\}$, $\eta = \frac{D_{\infty}/\sqrt{2}}{2^{i-1}}$}}}
\ENDFOR
\STATE {Output:} $\bar{\bx}_{N}$
\end{algorithmic}
\end{algorithm}

\begin{thm}\label{thm:AdaGrad}
Let $\mathcal{K}$ be a hypercube with $\|\bx - \by\|_{\infty} \le D_{\infty}$ for any $\bx, \by\in\mathcal{K}$. For a convex function $F$, set $\eta=\frac{D_{\infty}}{\sqrt{2}}$, then Algorithm~\ref{algo:AdaGrad} guarantees for any $\bx\in\mathcal{K}$:
{{\begin{equation}
\sum_{t=1}^T F(\bx_t) - F(\bx) \le \sqrt{2dD_{\infty}^2\sum_{t=1}^T \|\nabla F(\bx_t)\|^2}~.
\label{eq:adagrad}
\end{equation}}}
\end{thm}
\begin{proof}[Proof of Theorem~\ref{thm:AdaGrad}]
{\allowdisplaybreaks
\begin{align}
&\sum^T_{t=1}F(\bx_t) - F(\bx)\\
&\le
\sum^T_{t=1}\langle \nabla F(\bx_t), \bx_t - \bx \rangle\\
&=
\sum^T_{t=1}\sum^d_{j=1}\frac{\partial F}{\partial x_{j}}(\bx_t)\times(x_{t,j} - x_j)\\
&=
\sum^T_{t=1}\sum^d_{j=1}\frac{(x_{t, j} - x_{j})^2 - \left(x_{t, j} - \eta_{t, j}\frac{\partial F}{\partial x_{j}}(\bx_t) - x_{j}\right)^2}{2\eta_{t,j}}
+
\sum^T_{t=1}\sum^d_{j=1}\frac{\eta_{t,j}}{2}\left(\frac{\partial F}{\partial x_{j}}(\bx_t)\right)^2\\
&\le
\sum^T_{t=1}\sum^d_{j=1}\frac{(x_{t, j} - x_{j})^2 - (x_{t+1, j} - x_{j})^2}{2\eta_{t,j}}
+
\sum^T_{t=1}\sum^d_{j=1}\frac{\eta_{t,j}}{2}\left(\frac{\partial F}{\partial x_{j}}(\bx_t)\right)^2\\
&\le
\sum^d_{j=1}\sum^T_{t=1}\frac{(x_{t, j} - x_j)^2}{2}\left(\frac{1}{\eta_{t,j}} - \frac{1}{\eta_{t-1, j}}\right)
+
\sum^d_{j=1}\sum^T_{t=1}\frac{\eta_{t,j}}{2}\left(\frac{\partial F}{\partial x_{j}}(\bx_t)\right)^2\\
&\le
\frac{D_{\infty}^2}{2\eta}\sum^d_{j=1}\sum^T_{t=1}\left(\sqrt{\sum^t_{i=1}\left(\frac{\partial F}{\partial x_{j}}(\bx_i)\right)^2} - \sqrt{\sum^{t-1}_{i=1}\left(\frac{\partial F}{\partial x_{j}}(\bx_i)\right)^2}\right)\\
&\quad+
\sum^d_{j=1}\sum^T_{t=1}\frac{\eta}{2\sqrt{\sum^t_{i=1}\left(\frac{\partial F}{\partial x_{j}}(\bx_i)\right)^2}}\left(\frac{\partial F}{\partial x_{j}}(\bx_t)\right)^2\\
&\le
\sum^d_{j=1}\left(\frac{D_{\infty}^2}{2\eta}\sqrt{\sum^T_{t=1}\left(\frac{\partial F}{\partial x_{j}}(\bx_t)\right)^2}
+
\eta\sqrt{\sum^T_{t=1}\left(\frac{\partial F}{\partial x_{j}}(\bx_t)\right)^2}\right)\\
&=
\sum^d_{j=1}\sqrt{2D_{\infty}^2\sum^T_{t=1}\left(\frac{\partial F}{\partial x_{j}}(\bx_t)\right)^2}\\
&\le
\sqrt{2dD_{\infty}^2\sum^T_{t=1}\sum^d_{j=1}\left(\frac{\partial F}{\partial x_{j}}(\bx_t)\right)^2}\\
&=
\sqrt{2dD_{\infty}^2\sum^T_{t=1}\|\nabla F(\bx_t))\|^2},
\end{align}}
where the first inequality is by convexity, the second one by the projection lemma as the projection onto a hypercube equals performing the projection independently for each coordinate, the fifth one by Lemma 5 in~\citep{McMahanS10}, and the last one by the concavity of $\sqrt{\cdot}$.
\end{proof}

\begin{thm}
\label{thm:restart_adagrad}
For a $\mu$ strongly convex and $L$ smooth function $F$, denote its unique minimizer as $\bx^*\in\R^d$. Given $\bx_0\in\R^d$, assume that $\|\bx_0 - \bx^*\|_{\infty} \le D_{\infty}$, then Algorithm~\ref{algo:restart_AdaGrad} guarantees
\begin{equation}
\|\bar{\bx}_{N} - \bx^*\|_{\infty}^2 \le \frac{D_{\infty}^2}{4^N}~.
\end{equation}
Thus, to get a $\bx$ s.t.~$\|\bx - \bx^*\|_{\infty}^2\le\epsilon$, we need at most $32d\frac{L}{\mu}\log_4\left({D_{\infty}^2}/{\epsilon}\right)$ gradient calls.
\end{thm}
\begin{proof}[Proof of Theorem~\ref{thm:restart_adagrad}]
Consider round $i$ and assume $\mathcal{K}$ passed to Algorithm~\ref{algo:AdaGrad} is bounded w.r.t.~$\ell_{\infty}$ norm by $D_{\infty_i}$. When $F$ is $\mu$ strongly convex and $L$ smooth, let $\bx = \bx^*$, then~\eqref{eq:adagrad} becomes
{{\begin{equation}
\sum_{t=1}^T F(\bx_t) - F(\bx^*)
\le
\sqrt{2dD_{\infty_{i}}^2\sum_{t=1}^T \|\nabla F(\bx_t)\|^2}
\le
\sqrt{4LdD_{\infty_{i}}^2\sum_{t=1}^T(F(\bx_t) - F(\bx^*))},
\end{equation}}}
where the second inequality is by the $L$ smoothness of $F$. This gives
{{\begin{equation}
\sum_{t=1}^T F(\bx_t) - F(\bx^*)
\le
4LdD_{\infty_{i}}^2~.
\end{equation}}}

Let $\bar{\bx}_i = \frac1T\sum^T_{t=1}\bx_t$ we have by the $\mu$-strong-convexity that
{{\begin{equation}
\|\bar{\bx}_i - \bx^*\|_{\infty}^2
\le
\|\bar{\bx}_i - \bx^*\|^2
\le
\frac{2}{\mu}(F(\bar{\bx}) - F(\bx^*))
\le
\frac{2}{\mu}\frac1T\sum^T_{t=1}(F(\bx_t) - F(\bx^*))
\le
\frac{8LdD_{\infty_{i}}^2}{\mu T}~.
\label{eq:shrink}
\end{equation}}}

Put $T=32d\frac{L}{\mu}$ in the above inequality, we have that $\|\bar{\bx}_i - \bx^*\|_{\infty}^2 \le \frac{D_{\infty_{i}}^2}{4}$. Thus, after each round, the $\ell_{\infty}$ distance between the update $\bar{\bx}_i$ and $\bx^*$ is shrinked by half, which in turn ensures that $\bx^*$ is still inside the $\mathcal{K}$ passed to Algorithm~\ref{algo:AdaGrad} in the next round with $D_{\infty_{i+1}} = \frac{D_{\infty_{i}}}{2}$. This concludes the proof.
\end{proof}

\section{Conclusion}
In this chapter, we discussed the problem of gradient scales varying too much across layers during training deep neural networks without effective normalization techniques. After reporting Adam-$\ell_2$'s inferior performance compared with AdamW on a scenario we identified, we correlated this with the scale-freeness property AdamW enjoys while Adam-$\ell_2$ does not with theoretical and empirical evidence. We then revealed the connection between AdamW and proximal updates for understanding where the scale-freeness of AdamW comes from. Finally, we introduced a theoretical merit of scale-free algorithms, namely that they can reduce the effects of the condition number of the objective function on convergence rates in certain cases.

%% file: 5_Adapt_Smoothness/adapt_smoothness.tex
\chapter{Adaptation to Unbounded Smoothness}
\label{chapter:adapt_smoothness}
\thispagestyle{myheadings}

\newcommand{\VecLOneNorm}[1]{\|\boldsymbol{#1}\|_1}
\newcommand{\VecLInftyNorm}[1]{\|\boldsymbol{#1}\|_{\infty}}

\newcommand{\BoldLZeroLOne}{\ensuremath{(\boldsymbol{L_0} ,\boldsymbol{L_1})}\xspace}

\newtheorem{customcounttheorem}{Theorem}
\newenvironment{addcustomcounttheorem}[1]
  {\renewcommand\thecustomcounttheorem{#1}\customcounttheorem}
  {\endcustomcounttheorem}

\newtheorem{customcountlemma}{Lemma}
\newenvironment{addcustomcountlemma}[1]
  {\renewcommand\thecustomcountlemma{#1}\customcountlemma}
  {\endcustomcountlemma}

[The results in this chapter appeared in~\citet{crawshaw2022robustness}.]

Traditional analyses in non-convex optimization typically rely on the smoothness assumption, namely requiring the gradients to be Lipschitz. However, recent evidence shows that this smoothness condition does not capture the properties of some deep learning objective functions, including the ones involving Recurrent Neural Networks and LSTMs. Instead, they satisfy a much more relaxed condition, with potentially unbounded smoothness. Under this relaxed assumption, it has been theoretically and empirically shown that the gradient-clipped SGD has an advantage over the vanilla one. In this paper, we show that clipping is not indispensable for Adam-type algorithms in tackling such scenarios: we theoretically prove that a generalized SignSGD algorithm can obtain similar convergence rates as SGD with clipping but does not need explicit clipping at all. This family of algorithms on one end recovers SignSGD and on the other end closely resembles the popular Adam algorithm. Our analysis underlines the critical role that momentum plays in analyzing SignSGD-type and Adam-type algorithms: it not only reduces the effects of noise, thus removing the need for large mini-batch in previous analyses of SignSGD-type algorithms, but it also substantially reduces the effects of unbounded smoothness and gradient norms. To the best of our knowledge, this work is the first one showing the benefit of Adam-type algorithms compared with non-adaptive gradient algorithms such as gradient descent in the unbounded smoothness setting. We also compare these algorithms with popular optimizers on a set of deep learning tasks, observing that we can match the performance of Adam while beating others.

The layout of this chapter is as follows: we first discuss why we study a relaxed smoothness condition and report empirical evidence showing Transformers observe such condition in Section~\ref{sec:relaxed_smooth_transformer}. Then, in Section~\ref{sec:coordinate_l0l1} we show that refining the relaxed smoothness condition to the coordinate-wise granularity is necessary. Next, we introduce a generalized SignSGD algorithm to handle this scenario in Section~\ref{sec:generalized_sgd}, with a detailed discussion on the convergence rates it guarantees and the role of momentum. The proofs are contained in Section~\ref{sec:generalized_sgd_proof}. We then report the experimental results in Section~\ref{sec:generalized_sgd_experiments}, comparing our algorithm with some popular competitors in deep learning tasks. Finally, we draw some conclusions and discuss the limitations of our work in Section~\ref{sec:generalized_sgd_conclusion}.

\section{A Relaxed Smoothness Assumption and Transformers}
\label{sec:relaxed_smooth_transformer}
In previous chapters, we consider objective functions that are smooth, namely with Lipschitz gradients. The smoothness assumption is ubiquitous in analyses of optimization algorithms in the non-convex setting and has enabled many important theoretical results. For example, as we already discussed in Section~\ref{sec:noise_harm},~\citet{GhadimiL13} showed that, for a smooth function, Stochastic Gradient Descent can converge, in expectation, to a stationary point with $\|\nabla F(\bx)\|\le\epsilon$ in the order of $\epsilon^{-4}$, which was later proven to be optimal~\citep{ArjevaniCDFSW19}.

Nevertheless, this assumption might be too restrictive. For example, simple polynomials such as $F(x) = x^4$ are not smooth unless the domain is bounded. Moreover,~\citet{ZhangHSJ20} found that the smoothness assumption is not a good characterization of the landscapes for objective functions in training deep neural networks including LSTMs~\citep{hochreiter1997long}. In those scenarios, the gradient does not vary uniformly over the loss landscape and the local gradient Lipschitz constant can cross multiple orders of magnitudes for different points. Instead, the landscapes are observed to be better captured by a relaxed version of smoothness:
\begin{assumption} A second order differentiable function $F$ is $(L_0, L_1)$-smooth if for all $\bx\in\R^d$ we have
\begin{equation}
\label{eq:global_l0l1}
\|\nabla^2 F(\bx)\| \le L_0 + L_1\|\nabla F(\bx)\|~.
\end{equation}
\end{assumption}

Note that when $L_1 = 0$, it reduces to the original smoothness assumption. Also, all univariate
polynomials (which can possibly be non-convex) like $F(x) = x^4$ are $(L_0, L_1)$-smooth.

Under this assumption,~\citet{ZhangHSJ20} proved that the well-known gradient clipping technique can ensure SGD's convergence. Later, their results were improved by~\cite{ZhangJFW20} to show that SGD with clipping can be made unaffected by the $L_1$ in~\eqref{eq:global_l0l1} and is able to recover the optimal convergence rate of SGD under the original smoothness setting.

Recently,~\citet{liu2022communication} also analyzed SGD with a local gradient clipping algorithm in the distributed learning setting under this condition and showed its linear parallel speedup property both theoretically and empirically, which means that the iteration complexity of the algorithm is reduced by a factor of $N$, the number of machines. This corroborated the motivation for studying the above assumption.

\begin{figure}[t]
    \centering
    \begin{minipage}{0.48\linewidth}
        \includegraphics[width=\linewidth]{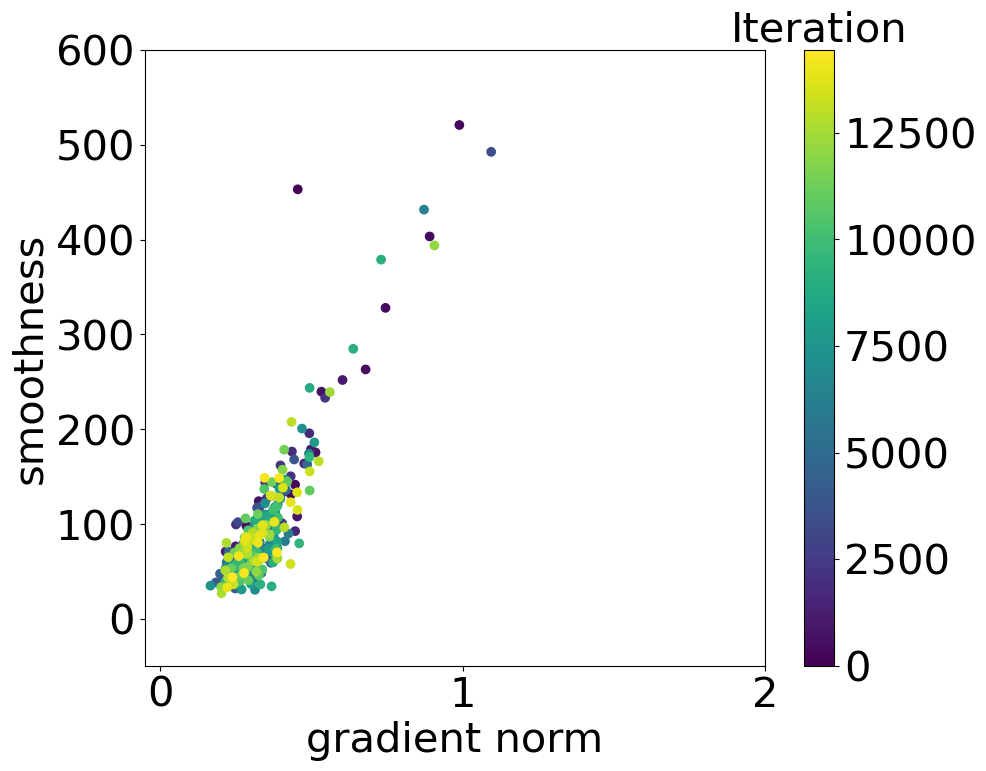}
        {\centerline{(a) Wikitext-2}}
    \end{minipage}
    \hfill
    \begin{minipage}{0.48\linewidth}
        \includegraphics[width=\linewidth]{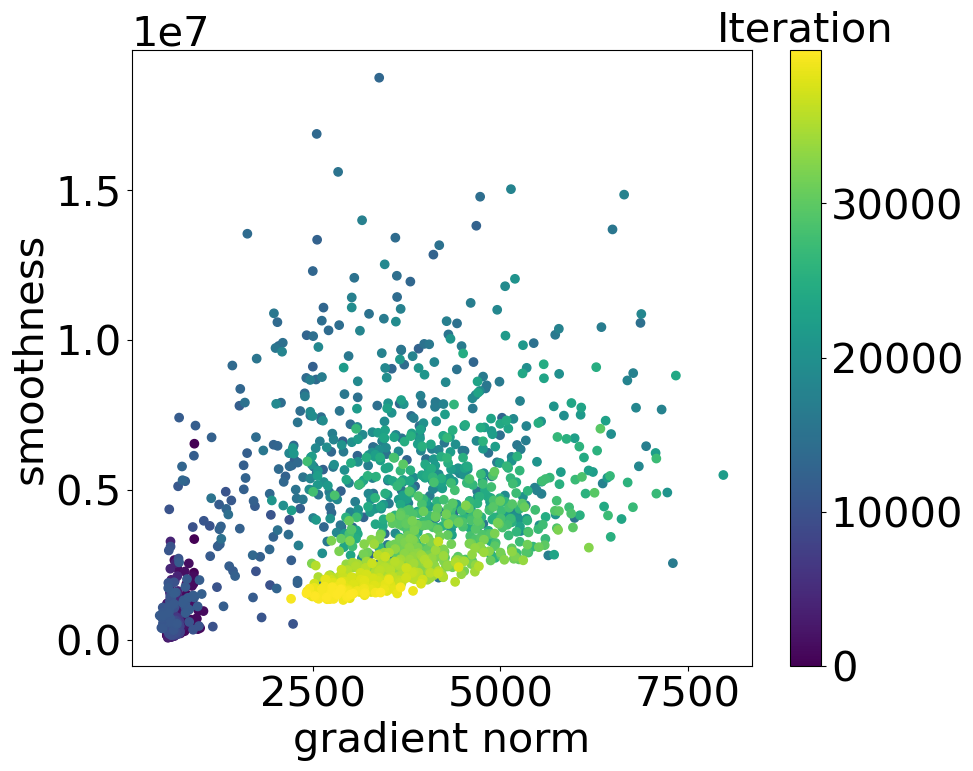}
        {\centerline{(b) WMT'16 de-en}}
    \end{minipage}
    \caption[Transformers observe the relaxed smoothness condition globally]{Local gradient Lipschitz constant vs.~gradient norm on training (a) a $2$-layer transformer encoder model on Wikitext-2 (b) a $6$-layer Transformer on WMT'16 Multimodal Machine Translation de-en dataset. The colorbar indicates the number of iterations during training.}
    \label{fig:global_l0l1}
\end{figure}

\begin{figure}[p]
     \centering
     \begin{subfigure}[b]{0.48\textwidth}
         \centering
         \includegraphics[width=\textwidth]{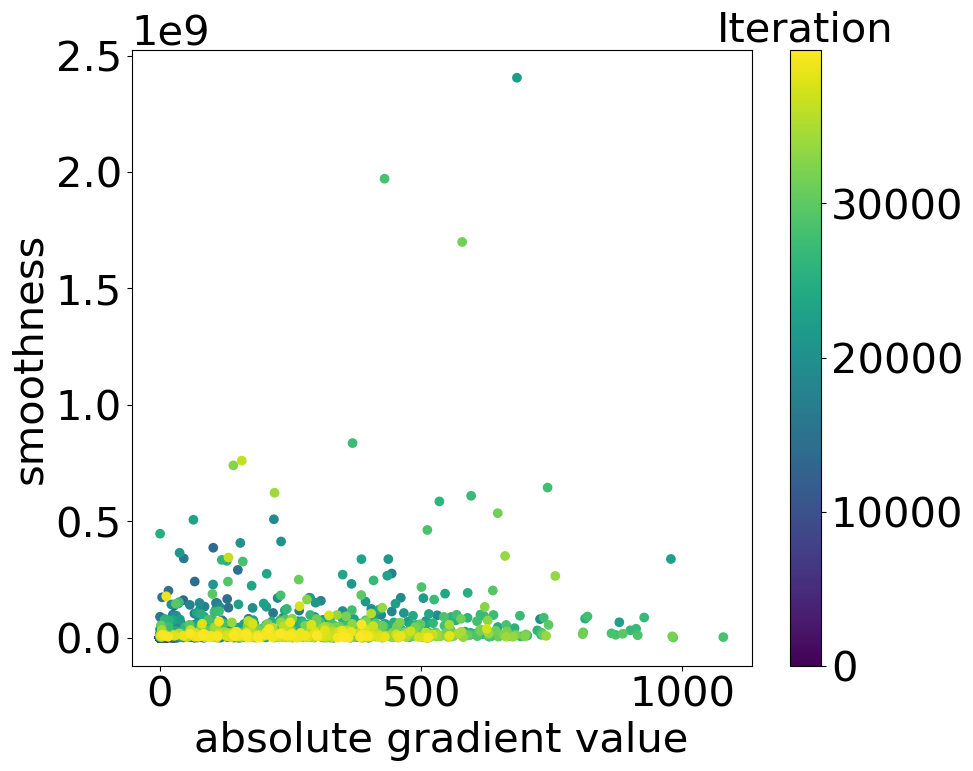}
         \caption{Encoder First Layer}
         \label{fig:enc_first}
     \end{subfigure}
     \hfill
     \begin{subfigure}[b]{0.48\textwidth}
         \centering
         \includegraphics[width=\textwidth]{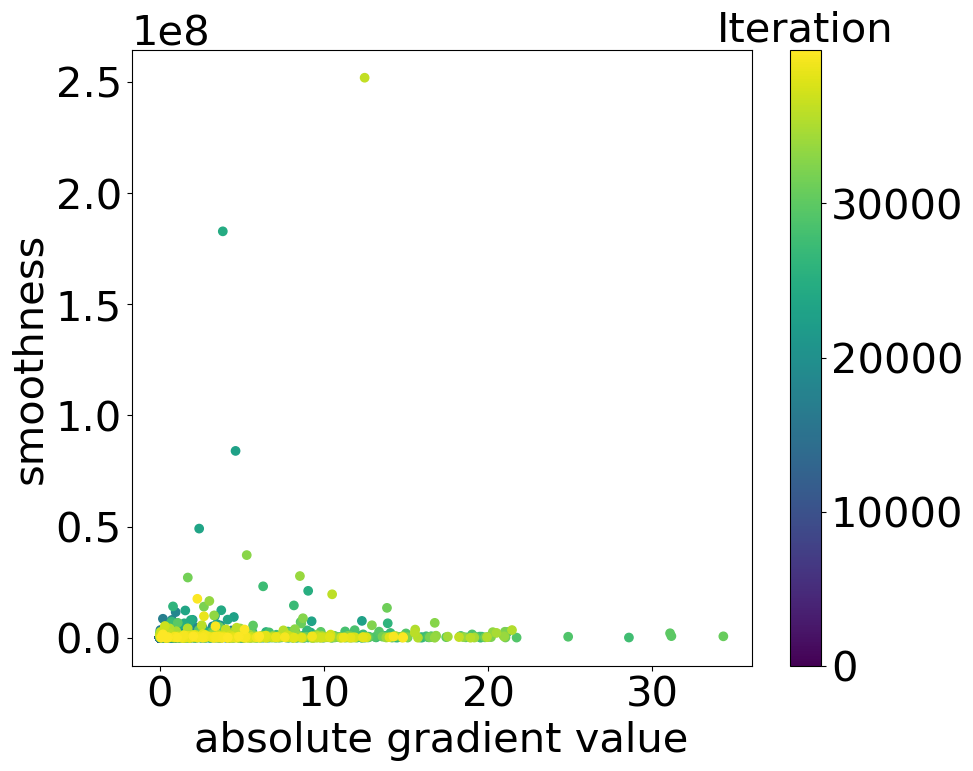}
         \caption{Encoder Last Layer}
         \label{fig:enc_last}
     \end{subfigure}

     \begin{subfigure}[b]{0.48\textwidth}
         \centering
         \includegraphics[width=\textwidth]{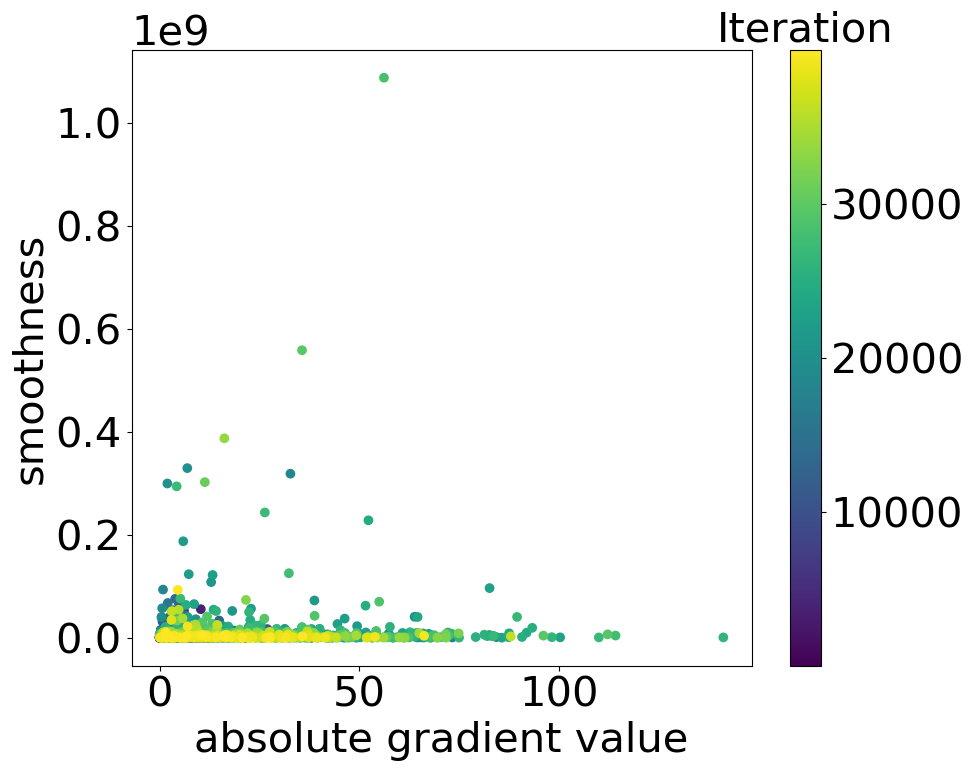}
         \caption{Decoder Second Layer}
         \label{fig:dec_second}
     \end{subfigure}
     \hfill
     \begin{subfigure}[b]{0.48\textwidth}
         \centering
         \includegraphics[width=\textwidth]{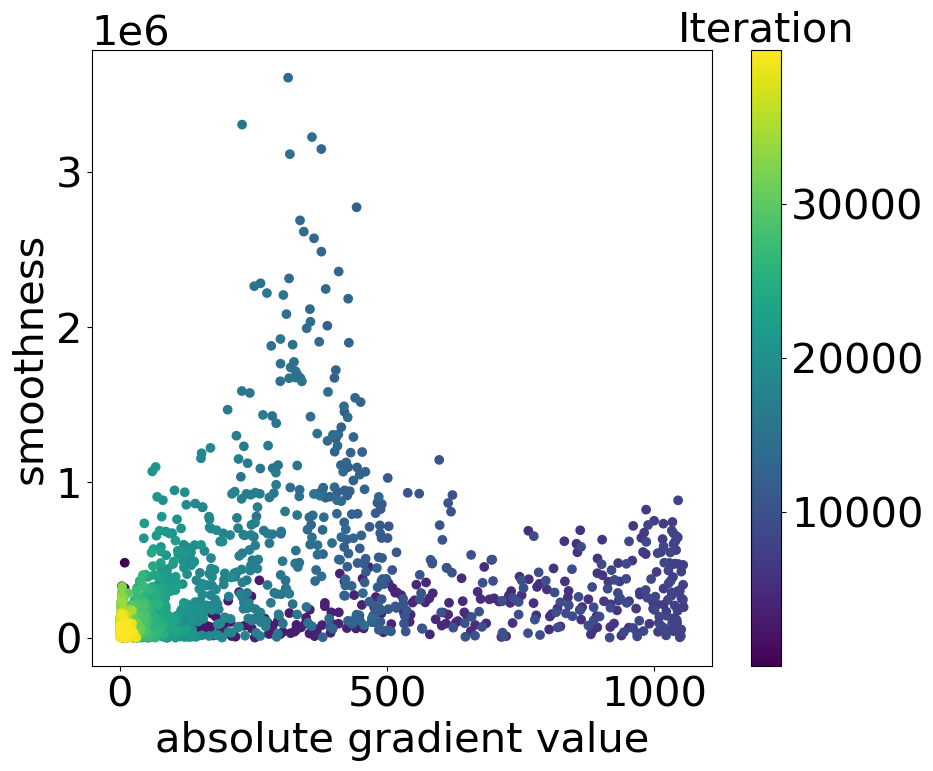}
         \caption{Decoder Last Layer}
         \label{fig:dec_last}
     \end{subfigure}
    \caption[Transformers observe the relaxed smoothness condition coordinate-wisely]{Local gradient Lipschitz constant vs.~absolute gradient value on training a $6$-layer Transformer on WMT'16 Multimodal Translation de-en dataset. Each figure represents a randomly picked coordinate in corresponding layers. The colorbar indicates the number of iterations during training.}
    \label{fig:transformer_l0l1_coordinate_wise}
\end{figure}

As observed by~\citet{ZhangHSJ20}, LSTMs empirically observe the $(L_0, L_1)$ smoothness condition~\eqref{eq:global_l0l1}. LSTMs have long been a powerhouse behind many machine learning tasks, especially the ones of Natural Language Processing (NLP)~~\citep{sundermeyer2012lstm, tai2015improved, zhou2016deep}. Yet, recently, Transformers~\citep{VaswaniSPUJGKP17} are becoming more and more popular and have been consistently achieving SOTA results in the NLP field and beyond~\citep{DevlinCLT19, DosovitskiyB0WZ21}. Motivated by this, we empirically verified that losses employing Transformers also seem to satisfy the $(L_0, L_1)$-smooth assumption~\eqref{eq:global_l0l1}, see Figure~\ref{fig:global_l0l1}.

For plotting Figure~\ref{fig:global_l0l1}, we followed the method in Section H.3 of~\citet{ZhangHSJ20}. Specifically, given $\bx_t$ and $\bx_{t+1}$, denote $\bd := \bx_{t+1} - \bx_t$. We estimate the smoothness at $\bx_t$ by
\begin{equation}
    \hat{L}_t = \max_{\gamma\in\{\delta_1, \delta_2,\ldots,\delta_N\}}\frac{\|\nabla F(\bx_{t}+\gamma\bd) - \nabla F(\bx_t)\|_2}{\|\gamma\bd\|_2},\label{eq:global_smooth_estimation}
\end{equation}
where $\{\delta_1, \delta_2,\ldots,\delta_N\}$ denotes the sample locations and we use $\{\frac16, \frac26, \frac36, \frac46, \frac56\}$.

Figure~\ref{fig:global_l0l1}(a) is on training a $2$-layer Transformer Encoder to do language modeling on the Wikitext-2 dataset. The implementation, settings, and parameter choices follow this.\footnote{\url{https://pytorch.org/tutorials/beginner/transformer_tutorial.html}} We only plot the first $5$ training epochs.
Figure~\ref{fig:global_l0l1}(b) and~\ref{fig:transformer_l0l1_coordinate_wise} are on training a $6$-layer Transformer to do machine translation on the WMT'16 Multimodal Machine Translation Task German-English dataset.
The implementation of the transformer is forked from here\footnote{\url{https://github.com/jadore801120/attention-is-all-you-need-pytorch}} and we also follow their default settings. The mini-batch size is $256$ and we trained for $400$ epochs using Adam and report the whole training trajectory.

Yet, this is not the end of the story, as we will show in the next section: Transformers actually satisfy a more fine-grained relaxed smooth condition.

\section{A Coordinate-wise Relaxed Smoothness Condition}
\label{sec:coordinate_l0l1}
The verification that Transformers observe the $(L_0, L_1)$-smooth condition in the previous section is exciting. However, after further investigation, we noticed that different coordinates, especially when they are in different layers of the model, exhibit very distinct degrees of variation of gradients as shown in Figure~\ref{fig:transformer_l0l1_coordinate_wise}, in which we compared $\frac{\left|\PartialDerivativeGeneral{\bx_{t+1}}{j} - \PartialDerivativeGeneral{\bx_t}{j}\right|}{|x_{t+1,j} - x_{t,j}|}$ vs.~$\min\left(\left|\PartialDerivativeGeneral{\bx_t}{j}\right|, \left|\PartialDerivativeGeneral{\bx_{t+1}}{j}\right|\right)$.

Consequently, we propose to refine the $(L_0, L_1)$ condition in~\eqref{eq:global_l0l1} to a coordinate-wise version to better capture the loss surface when training deep neural networks like Transformers.

\begin{assumption}
\label{asp:l0l1_coordinate}
We say that a differentiable function $F(\bx)$ is $(\boldsymbol{L_0}, \boldsymbol{L_1})$-smooth coordinate-wisely, if for any $\bx,\by \in \R^d$ for which $\PNorm{\bx - \by} \le \frac{1}{\VecLInftyNorm{L_1}}$, we have for any $j\in[d]$ that
\begin{equation}
\left|\frac{\partial F}{\partial x_j}(\by) - \frac{\partial F}{\partial x_j}(\bx) \right|
\leq
\left(\frac{L_{0,j}}{\sqrt{d}} + L_{1,j}\left|\frac{\partial F}{\partial x_j}(\bx)\right|\right)\PNorm{\by - \bx}~.\label{eq:coordinate_wise_l0l1}
\end{equation}
We will denote $\boldsymbol{L_0}:=[L_{0,1}, L_{0,2},\ldots,L_{0,d}]^T$ and $\boldsymbol{L_1}:=[L_{1,1}, L_{1,2},\ldots,L_{1,d}]^T$.
\end{assumption}

The following Lemma shows that our coordinate-wise \BoldLZeroLOne smooth assumption~\ref{asp:l0l1_coordinate} is equivalent to the original $(L_0, L_1)$ smooth assumption~\eqref{eq:global_l0l1} at least in $1$-d case.
\begin{lemma}
\label{lem:equivalence_l0l1_asps}
Let $F: \R\rightarrow\R$ be a twice continuously differentiable function. Then if (1) there exists some $K_0, K_1 \ge 0$ such that it holds for any $x, y\in\R$ with $|y - x| \le \frac{1}{K_1}$ that $|F^{\prime}(y) - F^{\prime}(x)| \le (K_0 + K_1|F^{\prime}(x)|)|y - x|$, then (2) there exists some $L_0, L_1 \ge 0$ such that it holds for any $x\in\R^d$ that $|F^{\prime\prime}(x)| \le L_0+ L_1|F^{\prime}(x)|$, and vice versa.
\end{lemma}
\begin{proof}[Proof of Lemma~\ref{lem:equivalence_l0l1_asps}]
\textbf{(1) $\Rightarrow$ (2)} By definition, for any $x\in\R$, we know that
\begin{align}
    F^{\prime\prime}(x)
    &=
    \underset{h\rightarrow 0}{\text{lim}}\frac{F^{\prime}(x + h) - F^{\prime}(x)}{h}
    \le
    \underset{h\rightarrow 0}{\text{lim}}\frac{|F^{\prime}(x + h) - F^{\prime}(x)|}{|h|}\\
    &\le
    \underset{h\rightarrow 0}{\text{lim}}\frac{(K_0 + K_1|F^{\prime}(\bx)|)|h|}{|h|}
    =
    K_0 + K_1|F^{\prime}(x)|~.
\end{align}

\textbf{(2) $\Rightarrow$ (1)} This is a special case for $1$-d and $c=1$ of Corollary A.4 in~\citet{ZhangJFW20}.
\end{proof}

Another motivation for this assumption comes from~\cite[Remark 2.3,][]{ZhangJFW20} where they noted that~\eqref{eq:global_l0l1} can be relaxed to an assumption on gradient differences: there exists $K_0, K_1 > 0$ s.t.~$\forall \bx, \by \in \R^d$ with $\|\bx - \by\|_2 \le \frac{1}{K_1}$, we have
\begin{equation}
\label{eq:l0l1_global_gradient}
    \|\nabla F(\bx) - \nabla F(\by)\|_2
    \le
    (K_0 + K_1\|\nabla F(\bx)\|_2)\|\bx - \by\|_2.
\end{equation}

Indeed, our Assumption~\ref{asp:l0l1_coordinate} implies~\eqref{eq:l0l1_global_gradient} when $L_{0,j} = L_0$ and $L_{1,j} = L_1$ for all $j\in[d]$, up to constants (See Lemma~\ref{lem:l0l1_recover} below). Note that the $\frac{1}{\sqrt{d}}$ factor in ours is exactly for easy comparison with~\eqref{eq:l0l1_global_gradient}.
\begin{lemma}
\label{lem:l0l1_recover}
When $L_{0,j} = L_0$ and $L_{1,j} = L_1$ for all $j\in[d]$, Assumption~\ref{asp:l0l1_coordinate} implies~\eqref{eq:l0l1_global_gradient} (up to constants).
\end{lemma}

\begin{proof}
Suppose all $L_{0,j}, L_{1,j}$ are the same across $j$, then we have
\begin{align}
    \|\nabla F(\by) - \nabla F(\bx)\|_2
    &=
    \sqrt{\sum^d_{j=1}
    \left|\frac{\partial F}{\partial x_j}(\by) - \frac{\partial F}{\partial x_j}(\bx) \right|^2}\\
    &\leq
    \sqrt{\sum^d_{j=1}\left(\frac{L_{0,j}}{\sqrt{d}} + L_{1,j}\left|\frac{\partial F}{\partial x_j}(\bx)\right|\right)^2\times\|\by - \bx\|_2^2}\\
    &\leq
    \sqrt{\sum^d_{j=1}\left(\frac{2L_{0,j}^2 }{d}+ 2L_{1,j}^2\left|\frac{\partial F}{\partial x_j}(\bx)\right|^2\right)\times\|\by - \bx\|_2^2}\\
    &\leq
    \sqrt{2}L_0\|\by - \bx\|_2 + \sqrt{2}L_1\|\by - \bx\|_2\sqrt{\sum^d_{j=1}\left|\frac{\partial F}{\partial x_j}(\bx)\right|^2}\\
    &=
    \left(\sqrt{2}L_0 + \sqrt{2}L_1\|\nabla F(\bx)\|_2\right)\times \|\by - \bx\|_2~.\qedhere
\end{align}
\end{proof}

The original $(L_0, L_1)$ smoothness assumption~\eqref{eq:global_l0l1} in~\cite{ZhangHSJ20} was proposed as a generalization of the more common smoothness assumption, which says that the gradient should be Lipschitz. Indeed, when $L_1$ is zero, we recover the smoothness assumption. In contrast, when $L_{1,j}$ are non-zero, the smoothness of the function is potentially \emph{unbounded}.
However, \cite{ZhangHSJ20} works with norms and applies to the global scale, while ours is more fine-grained and it applies to each coordinate separately.  We also note in passing that the smoothness assumption has been generalized in orthogonal directions in other work~\citep{RichtrikT14, bernstein2018signsgd, khaled2020better}.

One merit of Assumption~\ref{asp:l0l1_coordinate} is that it gives us the following descent lemma.
\begin{lemma}
\label{lem:desent_ineq_l0l1_coordinatewise}
Let $F$ be \BoldLZeroLOne-smooth coordinate-wisely. For any $\bx, \by\in\R^d$ for which $\PNorm{\bx - \by} \le \frac{1}{\VecLInftyNorm{L_1}}$, we have
\begin{equation}
    F(\by)
    \leq
    F(\bx) + \left\langle\nabla F(\bx), \by-\bx\right\rangle + \sum^d_{j=1}\frac{ \left(\frac{L_{0,j}}{\PNormDimension}+ L_{1,j}\left|\frac{\partial F}{\partial x_j}(\bx)\right|\right)\PNorm{\by - \bx}}{2}|y_j-x_j|~.
\end{equation}
\end{lemma}
\begin{proof}[Proof of Lemma~\ref{lem:desent_ineq_l0l1_coordinatewise}]
\begin{align}
    &F(\by)\\
    &=
    F(\bx) + \int^1_0 \langle \nabla F(\bx + u(\by - \bx)), \by - \bx\rangle du\\
    &=
    F(\bx) + \left\langle\nabla F(\bx), \by-\bx\right\rangle + \int^1_0 \langle \nabla F(\bx + u(\by - \bx)) - \nabla F(\bx), \by - \bx \rangle du\\
    &\le
    F(\bx) + \left\langle\nabla F(\bx), \by-\bx\right\rangle + \left|\int^1_0 \langle \nabla F(\bx + u(\by - \bx)) - \nabla F(\bx), \by - \bx) \rangle du\right|\\
    &\le
    F(\bx) + \left\langle\nabla F(\bx), \by-\bx\right\rangle + \int^1_0 \left|\langle\nabla F(\bx + u(\by - \bx)) - \nabla F(\bx), \by - \bx \rangle \right|du\\
    &\le
    F(\bx) + \left\langle\nabla F(\bx), \by-\bx\right\rangle + \int^1_0\sum^d_{j=1} \left|\left[\frac{\partial F}{\partial x_j}(\bx + u(\by - \bx)) - \frac{\partial F}{\partial x_j}(\bx)\right] (y_j - x_j)\right|du\\
    &\le
    F(\bx) + \left\langle\nabla F(\bx), \by-\bx\right\rangle + \int^1_0\sum^d_{j=1} \left|\frac{\partial F}{\partial x_j}(\bx + u(\by - \bx)) - \frac{\partial F}{\partial x_j}(\bx)\right| |y_j - x_j|du\\
    &=
    F(\bx) + \left\langle\nabla F(\bx), \by-\bx\right\rangle + \sum^d_{j=1}\int^1_0 \left|\frac{\partial F}{\partial x_j}(\bx + u(\by - \bx)) - \frac{\partial F}{\partial x_j}(\bx)\right| |y_j - x_j|du\\
    &\le
    F(\bx) + \left\langle\nabla F(\bx), \by-\bx\right\rangle + \sum^d_{j=1}\int^1_0 u \left(\frac{L_{0,j}}{\PNormDimension} + L_{1,j}\left|\frac{\partial F}{\partial x_j}(\bx)\right|\right)\PNorm{\by - \bx}|y_j - x_j| du\\
    & =
    F(\bx) + \left\langle\nabla F(\bx), \by-\bx\right\rangle + \sum^d_{j=1}\frac{\left(\frac{L_{0,j}}{\PNormDimension}+ L_{1,j}\left|\frac{\partial F}{\partial x_j}(\bx)\right|\right)\PNorm{\by - \bx}}{2}|y_j-x_j|~,
\end{align}
where the second inequality uses the fact that $\left|\int^b_a F(x)dx\right|\le\int^b_a|F(x)|dx$ and the final inequality is due to Assumption~\ref{asp:l0l1_coordinate}.
\end{proof}

Under this refined Assumption~\ref{asp:l0l1_coordinate}, we introduce in the next section an algorithm that can theoretically match the convergence rates of SGD with gradient clipping and empirically match the performance of Adam~\citep{KingmaB15} which is widely used in training LSTMs and Transformers.

\section{A Generalized SignSGD Algorithm}
\label{sec:generalized_sgd}
When training LSTMs and Transformers on NLP tasks, a common practice is to use the Adam optimizer~\citep{KingmaB15}. Moreover, given that these models observe the $(L_0,L_1)$ assumption, it would be natural to use some clipping procedure.

The algorithm and analysis of gradient clipping can be traced back to~\citep{alber1998projected,shor2012minimization,ermoliev1988stochastic} under the assumption that the function is convex and rapidly growing.~\citet{hazan2015beyond} considered gradient clipping in quasi-convex optimization.~\citet{mai2021stability} showed the stability and convergence of stochastic gradient clipping algorithms for convex problems without the smoothness condition. Gradient clipping is a standard technique in training deep neural networks~\citep{PascanuMB13} such as RNNs and LSTMs. The theoretical analysis of gradient clipping for nonconvex models is pioneered by~\citet{ZhangHSJ20}, in which the authors analyzed the convergence of gradient clipping under the relaxed smoothness assumption rather than the standard smoothness assumption.~\citet{ZhangJFW20} further improved the convergence rate bound under the same assumption as in~\citet{ZhangHSJ20}. Gradient clipping is also used when there is a heavy tail noise in the stochastic gradient to establish high probability convergence rates~\citep{CutkoskyM21,gorbunov2020stochastic,zhang2020adaptive}. Also,~\citet{cutkosky2020momentum} proved that normalized momentum improves normalized SGD under a second-order smoothness condition.

\begin{figure}[t]
    \centering
    \includegraphics[width=\textwidth]{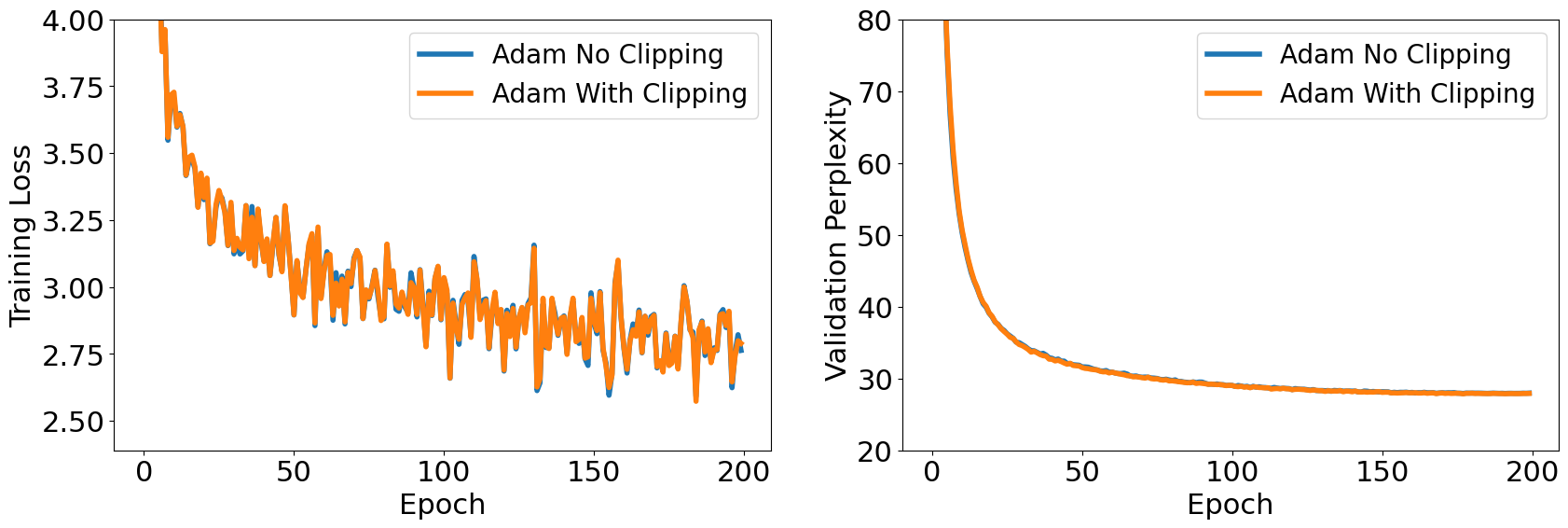}
    \caption{Training GPT-2 on Wikitext-103 using Adam with or without gradient clipping.}
    \label{fig:gpt2_adam_clipping}
\end{figure}

However, we found out that the use of clipping on Adam, while carried out in common practice~\cite[e.g.,][]{gpt2_naacl}, \emph{has no effect on the training and testing performance on optimizing a large transformer model as shown in Figure~\ref{fig:gpt2_adam_clipping}.}

Specifically, we conduct the experiment on the Wikitext-103 (103 million tokens, 180MB) \citep{wikitext103} language modeling task, with a 16-layer GPT-2 transformer model \citep{gpt2}.  This GPT-2 model has an input length of 256 tokens, 410-dimension word embedding, 16 Attention layers with 10 Attention heads and 2100 hidden dimensions. The model size is 201.58 MB. The vocabulary size is 28996. We use the hyper-parameter settings prescribed in \citep{gpt2_naacl}: batch size 256, warm up step size from 0 to $2.5 \times 10^{-4}$  in the first 64000 samples (i.e., 250 iterations) and then cosine-anneal step size to zero, on top of an Adam optimizer. It takes about 40 hours to train 200 epochs on 8 V100 GPUs. We use clipping threshold max\_norm 0.25 for the entire model as prescribed in the literature \citep{gpt2_naacl}. We also count that with this clipping scheme, clipping occurs in every single batch. As we can see from Figure~\ref{fig:gpt2_adam_clipping}, neither training loss (2.79 vs 2.76) nor perplexity score (27.92 vs 27.97) differs much in the clipping and the non-clipping cases, which suggests that Adam naturally achieves gradients clipping effect.

In retrospect, this might not be surprising: It is known that Adam has an implicit clipping behavior due to the normalization by the estimated second moment of the gradients. Indeed, Adam can be interpreted as a variant of SignSGD~\citep{balles2018dissecting}.

\begin{algorithm}[t]
    \caption{Generalized SignSGD \emph{(All operations on vectors are element-wise.)}}\label{alg:generalized_signsgd}
	\begin{algorithmic}[1]
	    \STATE Inputs: $\bx_1$, $\beta_1$, $\beta_2$, $\eta$
	    \STATE $\bm_0 = 0$, $\bv_0 = 0$
	    \FOR {$t = 1, \ldots, T$}
            \STATE Compute an unbiased estimate $\nabla f(\bx_t, \xi_t)$ of $\nabla F(\bx_{t})$, denoted as $\bg_t$
            \STATE $\bm_t = \beta_1 \bm_{t-1} + (1 - \beta_1) \bg_t$
            \STATE $\bv_t = \beta_2 \bv_{t-1} + (1 - \beta_2) \bm_t^2$
            \STATE $\bx_{t+1} = \bx_{t} - \eta \frac{\bm_t}{\sqrt{\bv_t}}$
		\ENDFOR
	\end{algorithmic}
\end{algorithm}

Inspired by this, we present in Algorithm~\ref{alg:generalized_signsgd} a generalized SignSGD algorithm. This algorithm encompasses a variety of optimization algorithms.

At first sight, it seems very similar to Adam. Indeed, if we employ $\bg_t^2$ in computing $\bv_t$ instead of $\bm_t^2$, then it is exactly Adam, except for the bias correction terms. We would like to clarify that the idea of this change has been explored before.~\citet{reddi2021adaptive} adopted this change to prove the convergence of Adam in a federated learning setting; yet, they only consider the smooth setting and require a large $\epsilon$ to obtain convergence in contrast to the original Adam. Later,~\citet{WangKQWXZF21} explored this idea in more detail, but their analyses are still restricted to the smooth setting.
The intuition is that $\bm_t$ represents a better update direction than $\bg_t$ and can thus better capture the second-moment information. Yet, in this work, the motivation for adopting this idea comes from the known effect of momentum on reducing the influence of noises~\citep{cutkosky2020momentum}. Indeed, in our analysis the difference between $\bm_t$ and $\nabla F(\bx_t)$ is much more controllable than between $\bg_t$ and $\nabla F(\bx_t)$. Thus, we consider employing $\bm_t$ in computing $\bv_t$ a better choice.

On the other end, the careful reader might observe that Algorithm~\ref{alg:generalized_signsgd} recovers the SignSGD with Momentum algorithm, also called SIGNUM in~\citet{bernstein2018signsgd}, when setting $\beta_2 = 0$. Sign-based algorithms are naturally suited to distributed learning~\citep{li2014scaling} and the idea dated back to at least Rprop~\citep{riedmiller1993direct}. The convergence to a stationary point (with $\ell_1$ norm) under a coordinate-wise smoothness condition has been established for SignSGD with/without the momentum in~\citet{bernstein2018signsgd} though they necessitate large mini-batches to control the variance of the noise. Yet, we are more interested in their property of the update size being bounded without the need for explicit gradient clipping.

Note that both SignSGD and Adam are good candidates for optimization algorithms whose updates must be bounded on functions that satisfy the \BoldLZeroLOne condition.
Indeed, SignSGD can be seen as an extreme form of gradient clipping. On the other hand, as said above, Adam does not seem to require gradient clipping at all when used to train the large Transformer model as shown in Figure~\ref{fig:gpt2_adam_clipping}. 

Consequently, we expect our algorithm, a generalization of SignSGD and a close resemblance to Adam, can enjoy the merits of both and be robust to the unbounded smoothness in the \BoldLZeroLOne scenario. This claim is formalized in Theorem~\ref{thm:generalized_signsgd} for which we require the following assumption.

\begin{assumption}
\label{asp:noise_as_bounded}
For each $j \in [d]$, there exists $\sigma_j > 0$ such
that for all $\bx \in \R^d$ and $\xi\sim\mathcal{D}$, the noise satisfies $\left|\frac{\partial f}{\partial x_j}(\bx, \xi) - \PartialDerivativeGeneral{\bx}{j}\right| \le \sigma_j$ with probability $1$. We will denote $\boldsymbol{\sigma}:=[\sigma_{1}, \sigma_{2},\ldots,\sigma_{d}]^T$.
\end{assumption}
\begin{thm}
\label{thm:generalized_signsgd}
Under Assumptions~\ref{asp:objective_func},~\ref{asp:l0l1_coordinate}, and~\ref{asp:noise_as_bounded}, assume\\ $M_j := \sup \left\{ \left|\PartialDerivativeGeneral{\bx}{j}\right| : F(\bx) \leq F(\bx_1)\right\}$ is finite for each $j\in[d]$, let $\Delta$ be any upper bound on 
$F(\bx_1) - F^*$,
$\alpha = \min\left(\frac{\sqrt{\VecLOneNorm{L_0}}\sqrt{\Delta}}{\VecLOneNorm{\sigma}\sqrt{T}}, 1\right)$, $\beta_1 = 1 - \alpha$, $\frac{\sqrt{\beta_2}}{\beta_1} < 1$,
$\rho = 1-\frac{\sqrt{\beta_2}}{\beta_1}$,\\
$\eta = \frac{\sqrt{\Delta\alpha}}{\sqrt{\VecLOneNorm{L_0}}\sqrt{T}}$,
for $T\ge\max\left(\frac{100d\Delta\VecLInftyNorm{L_1}^2}{(1-\beta_2)\rho^2\VecLOneNorm{L_0}}, \frac{10000d^2\Delta\VecLOneNorm{\sigma}^2\VecLInftyNorm{L_1}^4}{(1-\beta_2)^2\rho^4\VecLOneNorm{L_0}^3}\right)$, Algorithm~\ref{alg:generalized_signsgd} guarantees, with probability at least $1 - \delta$, that
\begin{align}
    \min_{t\in[T]} \ \|\nabla F(\bx_t)\|_1
    =
    &\mathcal{O}\left(\frac{\sqrt{\log(dT/\delta)}\VecLOneNorm{L_0}^{1/4}\Delta^{1/4}\VecLOneNorm{\sigma}^{1/2}}{\rho\sqrt{1-\beta_2}T^{1/4}} + \frac{\log(dT/\delta)\sqrt{\VecLOneNorm{L_0}\Delta}}{\rho\sqrt{T}}\right)\\
    &+ \mathcal{O}\left(\frac{\VecLOneNorm{M} + \VecLOneNorm{\sigma}}{\rho}\exp\left(-\frac{\sqrt{1-\beta_2}\VecLOneNorm{L_0}^{3/4}}{\PNormDimension\VecLInftyNorm{L_1}\VecLOneNorm{\sigma}^{1/2}\Delta^{1/4}}T^{1/4}\right)\right)\\
    &+\mathcal{O}\left(\frac{\|\nabla F(\bx_1)\|_1}{T} \right)~.
\end{align}

Furthermore, for the case when $\beta_2 = 0$, we have the following refined guarantee:
\begin{align}
\min_{t\in[T]} \ \|\nabla F(\bx_t)\|_1
=
&\mathcal{O}\left(\frac{\sqrt{\log(dT/\delta)}\VecLOneNorm{L_0}^{1/4}\Delta^{1/4}\VecLOneNorm{\sigma}^{1/2}}{T^{1/4}}
+ \frac{\log(dT/\delta)\sqrt{\VecLOneNorm{L_0}\Delta}}{\sqrt{T}}\right)\\
& + \mathcal{O}\left(\frac{\|\nabla F(\bx_1)\|_1}{\sqrt{T}}\left(\frac{1}{\sqrt{T}} + \frac{\VecLOneNorm{\sigma}}{\sqrt{\VecLOneNorm{L_0}\Delta}}\right)
+ \frac{\VecLOneNorm{\sigma}}{T}\right)~.
\end{align}
\end{thm}

Here, $M_j$ denotes the maximum absolute value of the partial derivative of $F$ for coordinate $j$ among the sub-level set of $F(\bx_1)$, namely any point $\bx$ with $F(\bx) \le F(\bx_1)$. In other words, we assume gradients to be bounded in the sub-level set of $F(\bx_1)$; yet, we do not make any restriction on gradients outside of this set. We believe this is not a strong assumption, for example, when the sub-level set of $F(\bx_1)$ is bounded, then by the assumed continuity of gradients it trivially holds. Also, we just require an upper bound and it can even be exponentially large as we have an exponentially decaying coefficient to counteract it: notice how the term $\VecLOneNorm{M}$ is multiplied by a term that decays exponentially with $T$. Better still, when $\beta_2 = 0$, we no longer even need this assumption and the algorithm is entirely free of the influence of $\VecLOneNorm{M}$. To see why this is good, we show a refined lower bound of Gradient Descent under the relaxed smoothness scenario below which is originally in~\citet{ZhangHSJ20}.

\begin{thm}
\label{thm:lower_bound_gd_fixed}
Fix $\epsilon > 0, L_0>0, L_1>0, M\geq \max(\frac{L_0}{L_1},\epsilon)$, and $x_0 \in \mathbb{R}$. Pick any constant step size $\eta$ for GD, with the knowledge of the above constants. Then, there exists a 1-d $(L_0,
L_1)$-smooth function $F$, bounded from below by $F^*$ (finite), and such that $ \sup \{ |
F'(x)| : F(x) \leq F(x_0)\} \leq M$ on which the number of iterations $T$ of GD with step size $\eta$ to guarantee $| F'(x_T)|< \epsilon$ is at least
\begin{equation}
    \frac{ M L_1 (F(x_0) - F^*-\frac{15 \epsilon^2}{16L_0}) }{2\epsilon^2\left(\ln \frac{M L_1}{L_0}+1\right) }~.
\end{equation}
\end{thm}
\begin{proof}[Proof of Theorem~\ref{thm:lower_bound_gd_fixed}]
By Lemma~\ref{lem:equivalence_l0l1_asps}, we know that, in 1-d case, our coordinate-wise \BoldLZeroLOne assumption~\ref{asp:l0l1_coordinate} is equivalent as the original one~\eqref{eq:global_l0l1}. Thus, without loss of generality, we use the original condition~\eqref{eq:global_l0l1} in the proof. We will construct two different $(L_0,L_1)$-smooth functions based on the value of $\eta$.

\textbf{Case $\eta > \frac{2}{M L_1}\left(\ln \frac{M L_1}{L_0}+1\right)$.}
In this case, we can construct a function on which GD does not converge, hence the lower bound is trivially true.
Consider the function
\begin{equation*}
    F(x) = \begin{cases}
        L_0\frac{e^{-L_1 x-1}}{L_1^2} & x < -\frac{1}{L_1} \\
        \\
        L_0 \frac{x^2}{2} + \frac{L_0}{2 L_1^2} & x \in [-\frac{1}{L_1}, \frac{1}{L_1}] \\
        \\
        L_0\frac{e^{L_1 x-1}}{L_1^2} & x > \frac{1}{L_1}
    \end{cases}
\end{equation*}
Note that $F$ is $(L_0,L_1)$-smooth.
Without loss of generality, we can assume $x_0=\frac{1}{L_1}\left(\ln \frac{M L_1}{L_0}+1\right)$, in fact if this is not the case we can translate the function $F$ accordingly. This setting of $x_0$ guarantees that the bound on the gradient is correct.
Moreover, with this choice, we claim the function will diverge. 
To see this, we use mathematical induction to show that $|x_{t+1}| > |x_{t}|$ and $\text{sign}(x_{t+1})\ne \text{sign}(x_{t})$ for any $t\ge0$. First, for the case when $t=0$, we have
\begin{equation}
     x_{1}
     = x_{0} - \eta F^{\prime}(x_0)
     = x_0 - \frac{\eta L_0}{L_1}e^{L_1x_0 - 1}
     = x_0 - \eta M
     < x_0 - 2x_0
     = -x_0~.
\end{equation}
Then suppose the condition holds up until $t$ and we prove for $t+1$. From the formula of $F$, we have that $\text{sign}(F^{\prime}(x)) = \text{sign}(x)$ and that $F$ is monotonically increasing with $|x|$. Thus, from the update of gradient descent which moves along the negative direction of the gradient, if we can show that $|x_{t+1}| > |x_{t}|$, then $\text{sign}(x_{t+1})\ne \text{sign}(x_{t})$. This leads to
\[
|x_{t+1}| = |x_{t} - \eta F^{\prime}(x_t)| > |x_t|
\Leftarrow
\eta |F^{\prime}(x_t)| > 2 |x_t|
\Leftarrow
\eta L_0 > \frac{2|x_t| L_1}{\exp(L_1 |x_t|-1)}~.
\]
Now, note that $\psi(x)=\frac{2|x| L_1}{\exp(L_1 |x|-1)}$ is decreasing for $x>\frac{1}{L_1}$ and increasing for $x<-\frac{1}{L_1}$. Hence, we have that
\[
\eta L_0 
> \frac{2|x_0| L_1}{\exp(L_1 |x_0|-1)}
> \frac{2|x_t| L_1}{\exp(L_1 |x_t|-1)},
\]
where the first inequality is true by the choice of $x_0>\frac{1}{L_1}$ and the condition on $\eta$ and the second one is true by the induction hypothesis.

%

\textbf{Case $\eta\leq \frac{2}{M L_1}\left(\ln \frac{M L_1}{L_0}+1\right)$.}

Now, consider
\begin{equation*}
    F(x) = \begin{cases}
        -\epsilon x , & x < -\frac{3\epsilon}{2L_0} \\
        \frac{L_0}{2} x^2-\frac{L_0^3 x^4}{27 \epsilon^2} +\frac{9\epsilon^2}{16L_0}, & x \in [-\frac{3\epsilon}{2L_0}, \frac{3\epsilon}{2L_0}] \\
        \epsilon x, & x > \frac{3\epsilon}{2L_0}
    \end{cases}
\end{equation*}
We have that $F$ is $(L_0, 0)$-smooth, hence also $(L_0, L_1)$-smooth. Note that the presence of the fourth power makes this function twice differentiable. Moreover, the maximum gradient in this case is $\epsilon\leq M$.

As before, without loss of generality, let the initial point $x_0 = \frac{3\epsilon}{2L_0}+\Delta$, where $\Delta>0$.
We have that $F(x_0)-F^* = \epsilon\left(\Delta + \frac{3\epsilon}{2L_0}\right)-\frac{9\epsilon^2}{16 L_0}$, hence $\Delta=\frac{1}{\epsilon}\left(F(x_0)-F^*\right)-\frac{15\epsilon}{16 L_0}$.
Now, while we stay on the last branch of the function, we have
\[
x_{t+1} 
= x_t - \eta \epsilon
\geq x_t - \epsilon \frac{2}{M L_1}\left(\ln \frac{M L_1}{L_0}+1\right)~.
\]
Hence, we have that, for 
\[
t
\leq \frac{M L_1 \Delta}{2 \epsilon \left(\ln \frac{M L_1}{L_0}+1\right)}
=\frac{M L_1 \left(F(x_0)-F^*-\frac{15 \epsilon^2}{16 L_0}\right)}{2\epsilon^2 \left(\ln \frac{M L_1}{L_0}+1\right)},
\]
we guarantee $|F^{\prime}(x_t)| = \epsilon$.
\end{proof}

On a side note, Theorem~\ref{thm:lower_bound_gd_fixed} is a fixed version of the lower bound in \cite{ZhangHSJ20}. First of all, they have a logarithm of a quantity with units, $M$, which is an undefined mathematical operation. A closer look at the proof reveals that, differently from the statement of their theorem, they construct a function with $L_0=L_1$, which explains why these terms are missing in the logarithm. Moreover, it is also unclear if the second constructed function satisfies the assumptions of the theorem.
We correct all these issues by properly scaling the constructed functions so that they always satisfy the $(L_0, L_1)$ condition and all the units are coherent. This result in the correct term inside the logarithm and the right conditions on $L_0$, $L_1$, $M$, and $\epsilon$.

Theorem~\ref{thm:lower_bound_gd_fixed} shows that in the relaxed smoothness setting, GD with any constant step size will suffer from a linear term depending on $L_1 M$. Compared with GD, our algorithm only has an exponentially decaying dependence on $L_1 M$. We consider this to be a substantial merit of our algorithm. Furthermore, when $\beta_2 = 0$ in which case we recover the SignSGD with Momentum algorithm, we can even show that it completely removes the effects of the unbounded gradient norms. Also notice that in such case we actually no longer need the assumption of $M_j := \sup \left\{ \left|\PartialDerivativeGeneral{\bx}{j}\right| : F(\bx) \leq F(\bx_1)\right\}$ being finite for each $j\in[d]$ anymore, and the $\VecLInftyNorm{L_1}$ term does not appear in the final bound anymore.

We also would like to point out that this bound closely resembles the one achieved by SGD with gradient clipping algorithm~\citep{ZhangJFW20} except that we consider the coordinate-wise setting: take the setting of $\beta_2 = 0$ for example, we need at most $\mathcal{O}\left(\Delta\max\left\{\frac{
\VecLOneNorm{\sigma}^2\VecLOneNorm{L_0}}{\epsilon^4},
\frac{d^2\VecLOneNorm{\sigma}^2\VecLInftyNorm{L_1}^4}{\VecLOneNorm{L_0}^3},
\frac{d\VecLInftyNorm{L_1}^2}{\VecLOneNorm{L_0}}\right\}\right)$ to get a point $\bx$ with $\|\nabla F(\bx)\|_1 \le \epsilon$ with high probability.

Careful readers might be concerned on the relations between $\alpha$, $\beta_1$, $\beta_2$, $\rho$, and $T$ when $\alpha\neq 1$. We would like to note that, when $\beta_2$ is fixed, $\alpha$ is inversely proportional to $\sqrt{T}$. In turn, the definition of $\rho$ means that as $T$ grows, $\rho$ grows and approaches $1 - \sqrt{\beta_2}$. Thus, the two conditions for $T$ decreases when $T$ grows. This means that there must exists a threshold of $T$ above which the two conditions on $T$ always hold. In summary, Theorem~\ref{thm:generalized_signsgd} conveys the same message as~\cite{ZhangJFW20} that as long as the expected $\epsilon$ is sufficiently small, the complexity no longer has a dependency on $\boldsymbol{L_1}$. Also note that we do not need the knowledge of $\boldsymbol{L_1}$ to set the algorithm, and in this case, as long as $\epsilon$ is small enough, the above bound reduces to $\mathcal{O}\left(\Delta\VecLOneNorm{\sigma}^2\VecLOneNorm{L_0}\epsilon^{-4}\right)$ with no dependency on $\boldsymbol{L_1}$ at all. We can thus consider the algorithm to be \emph{adaptive} to $\boldsymbol{L_1}$.

Finally, as a side note, the almost surely bounded assumption~\ref{asp:noise_as_bounded} of the noise can be relaxed to sub-gaussian noise, using standard extensions of the Freedman inequality~\cite[e.g.,][]{harvey2019tight}.

\section{Convergence Analyses of our Generalized SignSGD Algorithm}
\label{sec:generalized_sgd_proof}
The proof of the theorem is highly technical and it uses recent advancements in the analysis of momentum methods~\citep{cutkosky2020momentum}, key techniques to deal with the $(L_0,L_1)$ assumption~\citep{ZhangJFW20}, as well as a novel and essential inductive argument to control the norm of past gradients. 

We first write down some notations that we will use heavily for easier reference:
\begin{align}
    & \bar{\tau} = \frac{\sqrt{1-\beta_2}}{\eta\PNormDimension\VecLInftyNorm{L_1}},\quad
    \alpha = 1 - \beta_1,\quad
    \rho = 1 - \beta_2^{1/2}\beta_1^{-1},\\
    & \boldepsilon_t = \bm_t - \nabla F(\bx_t),\quad
    \tilde{\boldepsilon}_t = \bg_t - \nabla F(\bx_t),\\
    & E_j = 6\sigma_j\max(1, \log(1/\delta))
    + \frac{6}{\sqrt{1-\beta_1^2}}\sqrt{\sigma_j^2\max(1, \log(1/\delta))},\\
    & B_j = \frac{\eta L_{0,j} }{\sqrt{1-\beta_2}(1-\beta_1)}
    +
    \beta_1^{\bar{\tau}}(M_j + \sigma_j)
    + (1-\beta_1)E_j,\\
    & C_j = 1 + \frac{\eta\PNormDimension L_{1,j}}{(1-\beta_1)\sqrt{1-\beta_2}},\quad
    D = 1 - \frac{2\eta\PNormDimension\VecLInftyNorm{L_1}}{\sqrt{1-\beta_2}(1-\beta_1)},\\
    & A = \frac{\rho}{10\sqrt{1-\beta_2}}~.
\end{align}

Also, we would need the following formula many times:
\begin{equation}
    \beta_1^{\bar{\tau}}
    =
    (1 - \alpha)^{\frac{1}{\alpha}\frac{\alpha\sqrt{1-\beta_2}}{\eta\PNormDimension\VecLInftyNorm{L_1}}}
    \le
    e^{-\frac{\alpha\sqrt{1-\beta_2}}{\eta\PNormDimension\VecLInftyNorm{L_1}}},\label{eq:beta1_bar_tau}
\end{equation}
where in the first inequality we used the fact that $(1 - x)^{1/x} \le \frac{1}{e}$ for $0<x<1$.

\begin{lemma}
\label{lem:moment_decomp_generalized_signsgd}
With the notations in Algorithm~\ref{alg:generalized_signsgd}, for each $j\in[d]$ we have
\begin{align}
    m_{t,j} = (1 - \beta_1)\sum^{t}_{\tau=1}\beta_1^{t-\tau} \StocGradient{\tau}{j},\quad
    v_{t,j} =
    (1 - \beta_2)\sum^{t}_{\tau=1}\beta_2^{t-\tau} m_{\tau, j}^2,\quad
    \frac{|m_{t,j}|}{\sqrt{v_{t,j}}} \le \frac{1}{\sqrt{1-\beta_2}}~.
\end{align}
\end{lemma}
\begin{proof}[Proof of Lemma~\ref{lem:moment_decomp_generalized_signsgd}]
For all $t\geq 1$, we have
\begin{align}
    m_{t,j}
    &=
    \beta_1 m_{t-1, j} + (1 - \beta_1) \StocGradient{t}{j}\\
    &=
    \beta_1[\beta_1 m_{t-2, j} + (1 - \beta_1) \StocGradient{t-1}{j}] + (1 - \beta_1) \StocGradient{t}{j}\\
    &=
    \ldots
    =
    (1 - \beta_1)\sum^{t}_{\tau=1}\beta_1^{t-\tau} \StocGradient{\tau}{j}~.
\end{align}
Similarly for $v_{t,j}$. Next,
\begin{align}
    \frac{|m_{t,j}|}{\sqrt{v_{t,j}}}
    &=
    \frac{|m_{t,j}|}{\sqrt{(1 - \beta_2)\sum^{t}_{\tau=1}\beta_2^{t-\tau} m_{\tau,j}^2}}
    \le
    \frac{1}{\sqrt{1-\beta_2}}~. \qedhere
\end{align}
\end{proof}

The following lemma shows when we can apply Assumption~\ref{asp:l0l1_coordinate} and Lemma~\ref{lem:desent_ineq_l0l1_coordinatewise}.
\begin{lemma}
\label{lem:generalized_signsgd_bounded_updates_t}
With notations in Algorithm~\ref{alg:generalized_signsgd}, for $\tau \le \bar{\tau} =  \frac{\sqrt{1-\beta_2}}{\eta\PNormDimension\VecLInftyNorm{L_1}}$, we have $\PNorm{\bx_{t-\tau} - \bx_{t}} \le \frac{1}{\VecLInftyNorm{L_1}}$.
\end{lemma}
\begin{proof}[Proof of Lemma~\ref{lem:generalized_signsgd_bounded_updates_t}]
Using Lemma~\ref{lem:moment_decomp_generalized_signsgd} we have
\begin{align}
    |x_{t-\tau,j} - x_{t,j}|
    &\le 
    \sum^{\tau}_{i=1}|x_{t-i,j} - x_{t-i+1,j}|
    \le
    \frac{\eta\tau}{\sqrt{1-\beta_2}}
    \le
    \frac{1}{\PNormDimension\VecLInftyNorm{L_1}}\\
    &\Rightarrow
    \PNorm{\bx - \by} \le \frac{1}{\VecLInftyNorm{L_1}}
    ~.\qedhere
\end{align}
\end{proof}

The following two lemmas are the major tools we use to analyze the effects of noises in which Lemma~\ref{lem:martingale_diff} is from~\citet[Lemma 12]{CutkoskyM21}.
\begin{lemma}
\label{lem:martingale_diff}
Suppose $X_1,\ldots,X_T$ is a martingale difference sequence in a Hilbert space and $\|X_t\|\le R$ almost surely for some constant $R$. Further, assume $\E_{t}[\|X_t\|^2]\le\sigma_t^2$ with probability $1$ for some constants $\sigma_t$, where $\E_t[\cdot] \triangleq \E[\cdot|\xi_1, \xi_2,\ldots,\xi_{t-1}]$ denotes the expectation conditioned on all past randomnesses. Then, with probability at least $1-3\delta$, for all $k\leq T$ we have
\begin{equation}
\left\|\sum^k_{t=1}X_t\right\|
\le
3R\max(1, \log(1/\delta))
+ 3\sqrt{\sum^k_{t=1}\sigma^2_t\max(1, \log(1/\delta))}~.    
\end{equation}
\end{lemma}

\begin{lemma}
\label{lem:noise_seq_generalized_signsgd}
Assume Assumption~\ref{asp:noise_as_bounded}.
With the notation of Algorithm~\ref{alg:generalized_signsgd}, let $j \in [d]$ and $\beta_1 < 1$. Then, with probability at least $1-3\delta$, for any $t_0\in[t]$, we have
\begin{align}
    \left|\sum^{t_0}_{\tau=1}\beta_1^{t-\tau}\left(\StocGradient{\tau}{j} - \PartialDerivative{\tau}{j}\right)\right|
    \le
    &3\sigma_j\max(1, \log(1/\delta))\\
    &+ \frac{3}{\sqrt{1-\beta_1^2}}\sqrt{\sigma_j^2\max(1, \log(1/\delta))}~.
\end{align}
\end{lemma}
\begin{proof}[Proof of Lemma~\ref{lem:noise_seq_generalized_signsgd}]
Recall Assumption~\ref{asp:noise_as_bounded} and notice that $\beta_1^{t-\tau} \le 1$ for all $\tau\in[1,t]$, we know that $\left|\beta_1^{t-\tau}\left(\StocGradient{\tau}{j} - \PartialDerivative{\tau}{j}\right)\right| \le \sigma_j$ almost surely. It also means\\ $\E_{\tau}\left[\left(\beta_1^{t-\tau}\left(\StocGradient{\tau}{j} - \PartialDerivative{\tau}{j}\right)\right)^2\right] \le \beta_1^{2(t-\tau)}\sigma_j^2$.
Now, in Algorithm~\ref{alg:generalized_signsgd} we noted that $\bg_{\tau}$ is an unbiased estimate of $\nabla F(\bx_{\tau})$ which means $\E_{\tau}\left[\beta_1^{t-\tau}\left(\StocGradient{\tau}{j} - \PartialDerivative{\tau}{j}\right)\right]=0$. Thus, $\left\{\beta_1^{t-\tau}\left(\StocGradient{\tau}{j} - \PartialDerivative{\tau}{j}\right)\right\}_{1,\ldots,t}$ is a martingale difference sequence. Then, using Lemma~\ref{lem:martingale_diff}, with probability at least $1 - 3\delta$, we have for all $t_0\in[t]$ that
\begin{align}
    &\left|\sum^{t_0}_{\tau=1}\beta_1^{t-\tau}\left(\StocGradient{\tau}{j} - \PartialDerivative{\tau}{j}\right)\right|\\
    &\qquad\qquad\le
    3\sigma_j\max(1, \log(1/\delta))
    + 3\sqrt{\sum^{t_0}_{\tau=1}\beta_1^{2(t-\tau)}\sigma_j^2\max(1, \log(1/\delta))}\\
    &\qquad\qquad\le
    3\sigma_j\max(1, \log(1/\delta))
    + \frac{3}{\sqrt{1-\beta_1^2}}\sqrt{\sigma_j^2\max(1, \log(1/\delta))}~.\qedhere
\end{align}
\end{proof}

The following lemma upper bounds the differences between recent true gradients and the current one.
\begin{lemma}
\label{lem:true_grad_seq_generalized_signsgd}
With the notation of Algorithm~\ref{alg:generalized_signsgd} and 
under the assumptions in Theorem~\ref{thm:generalized_signsgd}, for any $j\in [d]$ and any $t_0$ with $t - t_0 \le \bar{\tau} = \frac{\sqrt{1-\beta_2}}{\eta\PNormDimension\VecLInftyNorm{L_1}}$, we have
\begin{equation}
    \sum^{t}_{\tau=t_0}\beta_1^{t-\tau} \left|\PartialDerivative{t}{j}
    - \PartialDerivative{\tau}{j}\right|
    \le
    \left(L_{0,j} + L_{1,j}\PNormDimension \left|\PartialDerivative{t}{j}\right|\right)\frac{\eta}{(1-\beta_1)^2\sqrt{1-\beta_2}}~.
\end{equation}
\end{lemma}
\begin{proof}[Proof of Lemma~\ref{lem:true_grad_seq_generalized_signsgd}]
\begin{align}
    &\sum^{t}_{\tau=t_0}\beta_1^{t-\tau} \left|\PartialDerivative{t}{j}
    - \PartialDerivative{\tau}{j}\right|\\
    \le&
    \sum^{t}_{\tau=t_0}\beta_1^{t-\tau}\left(\frac{L_{0,j}}{\PNormDimension}+ L_{1,j} \left|\PartialDerivative{t}{j}\right|\right)\PNorm{x_{t} - x_{\tau}}\\
    \le&
    \sum^{t}_{\tau=t_0}\beta_1^{t-\tau}\left(L_{0,j} + L_{1,j}\PNormDimension \left|\PartialDerivative{t}{j}\right|\right)\frac{\eta(t-\tau)}{\sqrt{1-\beta_2}}\\
    =&
    \left(L_{0,j} + L_{1,j}\PNormDimension \left|\PartialDerivative{t}{j}\right|\right)\frac{\eta}{\sqrt{1-\beta_2}}\sum^{t}_{\tau=t_0}(t-\tau)\beta_1^{t-\tau}\\
    \le&
    \left(L_{0,j} + L_{1,j}\PNormDimension \left|\PartialDerivative{t}{j}\right|\right)\frac{\eta}{(1-\beta_1)^2\sqrt{1-\beta_2}}~,
\end{align}
where the first inequality is due to Assumption~\ref{asp:l0l1_coordinate} and Lemma~\ref{lem:generalized_signsgd_bounded_updates_t}, the second inequality uses Lemma~\ref{lem:moment_decomp_generalized_signsgd}, and the final inequality uses the fact that $\sum^{N}_{k=1}k a^k\le\frac{1}{(1-a)^2}$ for any $0<a<1$.
\end{proof}

The following lemma upper bounds a past momentum with the current one.
\begin{lemma}
\label{lem:mom_diff_generalized_signsgd}
With the notation of Algorithm~\ref{alg:generalized_signsgd} and under the assumptions of Theorem~\ref{thm:generalized_signsgd}, for any $\tau \le \bar{\tau} = \frac{\sqrt{1-\beta_2}}{\eta\PNormDimension\VecLInftyNorm{L_1}}$, with probability at least $1-3\delta$, it holds that
\begin{align}
    |m_{t-\tau, j}|
    \le
    &\beta_1^{-\tau}\left(|m_{t,j}|
    +
    \left|\PartialDerivative{t}{j}\right|
    +
    (1-\beta_1)E_j\right)\\
    &+\beta_1^{-\tau}\left(
    \left(L_{0,j} + L_{1,j}\PNormDimension \left|\PartialDerivative{t}{j}\right|\right)\frac{\eta}{(1-\beta_1)\sqrt{1-\beta_2}}\right)
    ~.
\end{align}
\end{lemma}
\begin{proof}[Proof of Lemma~\ref{lem:mom_diff_generalized_signsgd}]
Denoting by $\tilde{\boldepsilon}_t = \bg_t - \nabla F(\bx_t)$ and using Lemma~\ref{lem:moment_decomp_generalized_signsgd}, we have
\begin{align}
    &|m_{t-\tau, j} - \beta_1^{-\tau}m_{t,j}|\\
    &=
    \left|(1 - \beta_1)\sum^{t}_{i=1}\beta_1^{t-\tau-i} \StocGradient{i}{j}
    -
    (1 - \beta_1)\sum^{t-\tau}_{i=1}\beta_1^{t-\tau-i} \StocGradient{i}{j}\right|\\
    &=
    (1 - \beta_1)\left|\sum^{t}_{i=t-\tau+1}\beta_1^{t-\tau-i} \StocGradient{i}{j}\right|\\
    &\le
    (1 - \beta_1)\left|\sum^{t}_{i=t-\tau+1}\beta_1^{t-\tau-i} \PartialDerivative{i}{j}\right|
    +
    (1 - \beta_1)\left|\sum^{t}_{i=t-\tau+1}\beta_1^{t-\tau-i} \tilde{\epsilon}_{i, j}\right|\label{eq:momentum_diff_generalized_signsgd}~.
\end{align}
We now upper bound the first term of~\eqref{eq:momentum_diff_generalized_signsgd} using Lemma~\ref{lem:true_grad_seq_generalized_signsgd} by using the fact that $\tau \le \frac{\sqrt{1-\beta_2}}{\eta\VecLOneNorm{L_1}}$:
\begin{align}
    &\left|\sum^{t}_{i=t-\tau+1}\beta_1^{t-\tau-i} \PartialDerivative{i}{j}\right|\\
    &\le
    \left|\sum^{t}_{i=t-\tau+1}\beta_1^{t-\tau-i} \PartialDerivative{t}{j}\right|
    +
    \sum^{t}_{i=t-\tau+1}\beta_1^{t-\tau-i} \left|\PartialDerivative{t}{j}
    - \PartialDerivative{i}{j}\right|\\
    &\le
    \left|\PartialDerivative{t}{j}\right|\frac{\beta_1^{-\tau}}{1-\beta_1}
    +
    \left(L_{0,j} + L_{1,j}\PNormDimension \left|\PartialDerivative{t}{j}\right|\right)\frac{\eta\beta_1^{-\tau}}{(1-\beta_1)^2\sqrt{1-\beta_2}}~.
\end{align}
Finally, the second term of~\eqref{eq:momentum_diff_generalized_signsgd} can be bounded using Lemma~\ref{lem:noise_seq_generalized_signsgd} by noticing that
\begin{align}
    \left|\sum^{t}_{i=t-\tau+1}\beta_1^{t-\tau-i} \tilde{\epsilon}_{i, j}\right|
    &\le
    \beta_1^{-\tau}\left(\left|\sum^{t}_{i=1}\beta_1^{t-i} \tilde{\epsilon}_{i, j}\right|
    +
    \left|\sum^{t-\tau}_{i=1}\beta_1^{t-i} \tilde{\epsilon}_{i, j}\right|\right)~.\qedhere
\end{align}
\end{proof}

The following Lemma is adapted from~\citep{ZouCLG21}. Yet, they only considered Adam under the $L$-smooth setting and when there is no noise. The existence of noise and the relaxed smoothness assumption makes the proofs substantially more challenging. With this lemma, we know that either the true gradient is small or that the update of our Algorithm~\ref{alg:generalized_signsgd} can be lower bounded.
\begin{lemma}
\label{lem:generalized_signsgd_update_lower_bound}
With the notation of Algorithm~\ref{alg:generalized_signsgd} and 
under the assumptions of Theorem~\ref{thm:generalized_signsgd}, if $\left|\PartialDerivativeGeneral{\bx_{\tau}}{j}\right|\le M_j$ holds for all $\tau\le t$ and $j\in[d]$, and $D > 0$, then, with probability at least $1-3t\delta$ we have that,
\[
\text{either $\left|\PartialDerivative{t}{j}\right| < \frac{5B_j}{D}$ or $\frac{|m_{t,j}|}{\sqrt{v_{t,j}}} \ge \frac{\rho D}{5\sqrt{1-\beta_2}}$}~.
\]
\end{lemma}
\begin{proof}[Proof of Lemma~\ref{lem:generalized_signsgd_update_lower_bound}]
Given that $\left|\PartialDerivativeGeneral{\bx_{\tau}}{j}\right|\le M_j$ for any $\tau\le t$ and $j\in[d]$, using Lemma~\ref{lem:moment_decomp_generalized_signsgd} and Assumption~\ref{asp:noise_as_bounded}, it is immediate to show that $|m_{t,j}|\le M_j + \sigma_j$. Then, denote $\hat{\tau} = \lfloor\bar{\tau}\rfloor = \left\lfloor\frac{\sqrt{1-\beta_2}}{\eta\PNormDimension\VecLInftyNorm{L_1}}\right\rfloor$ namely the largest integer that is no greater than $\bar{\tau}$, from Lemma~\ref{lem:moment_decomp_generalized_signsgd}, we have
\begin{align}
    \frac{|m_{t,j}|}{\sqrt{v_{t,j}}}
    &=
    \frac{|m_{t,j}|}{\sqrt{(1 - \beta_2)\sum^{t-1}_{\tau=0}\beta_2^{\tau} m_{t-\tau,j}^2}}\\
    &=
    \frac{|m_{t,j}|}{\sqrt{1 - \beta_2}\sqrt{\sum^{t-1}_{\tau=\hat{\tau}+1}\beta_2^{\tau}m_{t-\tau, j}^2 + \sum^{\hat{\tau}}_{\tau=0}\beta_2^{\tau}m_{t-\tau, j}^2}}\\
    &\ge    \frac{|m_{t,j}|}{\sqrt{1 - \beta_2}\sqrt{(M_j+\sigma_j)^2\frac{\beta_2^{\hat{\tau}+1}}{1-\beta_2} + \sum^{\hat{\tau}}_{\tau=0}\beta_2^{\tau}m_{t-\tau, j}^2}}\\
    &\ge
    \frac{|m_{t,j}|}{(M_j+\sigma_j)\beta_2^{\bar{\tau}/2} + \sqrt{1 - \beta_2}\sum^{\hat{\tau}}_{\tau=0}\beta_2^{\tau/2}| m_{t-\tau,j}|}\\
    &>
    \frac{|m_{t,j}|}{(M_j+\sigma_j)\beta_1^{\bar{\tau}} + \sqrt{1 - \beta_2}\sum^{\hat{\tau}}_{\tau=0}\beta_2^{\tau/2}| m_{t-\tau,j}|},
\end{align}
where the final inequality uses the assumption that $\sqrt{\beta_2} < \beta_1$.
Using Lemma~\ref{lem:mom_diff_generalized_signsgd} and the definition of $\rho = 1-\beta_2^{1/2}\beta_1^{-1} \in (0,1]$, with probability at least $1-3t\delta$, as we need to invoke Lemma~\ref{lem:noise_seq_generalized_signsgd} for at most $t$ times, we have
\begin{align}
    &\frac{\sqrt{v_{t,j}}}{\sqrt{1-\beta_2}}\\
    &\le
    \left(|m_{t,j}|
    +
    (1-\beta_1)E_j\right)\sum^{\hat{\tau}}_{\tau=0}\beta_2^{\tau/2}\beta_1^{-\tau}
    +
    \frac{\beta_1^{\bar{\tau}}(M_j+\sigma_j)}{\sqrt{1-\beta_2}}\\
    &\quad+\left(\left|\PartialDerivative{t}{j}\right|
    +
    \left(L_{0,j} + L_{1,j}\PNormDimension \left|\PartialDerivative{t}{j}\right|\right)\frac{\eta}{(1-\beta_1)\sqrt{1-\beta_2}}\right)\sum^{\hat{\tau}}_{\tau=0}\beta_2^{\tau/2}\beta_1^{-\tau}\\
    &\le
    \left(|m_{t,j}| + C_j\left|\PartialDerivative{t}{j}\right| + B_j\right)\frac{1}{\rho},
\end{align}
where in the last inequality we used the fact that $\frac{1}{\rho}\geq \frac{1}{\sqrt{1-\beta_2}}$.

Thus, we consider following two cases depending on the relative size of $|m_{t,j}|$ vs.~$C_j\left|\PartialDerivative{t}{j}\right|+ B_j$.

\textbf{Case 1}: $|m_{t,j}| > C_j\left|\PartialDerivative{t}{j}\right|+ B_j$, then
\begin{equation}
    \frac{|m_{t,j}|}{\sqrt{v_{t,j}}} > \frac{\rho}{2\sqrt{1-\beta_2}}~.\label{eq:generalized_signsgd_lower_bound_m_large}
\end{equation}

\textbf{Case 2}: $|m_{t,j}| \le C_j\left|\PartialDerivative{t}{j}\right|
+ B_j$, then we have
\begin{equation}
    \sqrt{v_{t,j}}
    \le
    \frac{2\sqrt{1-\beta_2}}{\rho}\left(C_j\left|\PartialDerivative{t}{j}\right| + B_j\right)~.
\end{equation}

Also, for $|m_{t,j}|$ we have from Lemma~\ref{lem:moment_decomp_generalized_signsgd} that
\begin{align}
    |m_{t,j}|
    =&
    (1 - \beta_1)\left|\sum^{t}_{\tau=1}\beta_1^{t-\tau} \StocGradient{\tau}{j}\right|
    \ge
    (1 - \beta_1)\left|\sum^{t}_{\tau=t-\hat{\tau}}\beta_1^{t-\tau} \StocGradient{\tau}{j}\right|
    -
    (1 - \beta_1)\left|\sum^{t-\hat{\tau}-1}_{\tau=1}\beta_1^{t-\tau} \StocGradient{\tau}{j}\right|\\
    \ge&
    \underbrace{(1-\beta_1)\left|\sum^{t}_{\tau=t-\hat{\tau}}\beta_1^{t-\tau} \PartialDerivative{\tau}{j}\right|}_{R_1}
    - \underbrace{(1-\beta_1)\left|\sum^{t}_{\tau=t-\hat{\tau}}\beta_1^{t-\tau}\left(\PartialDerivative{\tau}{j} - \StocGradient{\tau}{j}\right)\right|}_{R_2}\\
    &-
    \underbrace{(1 - \beta_1)\left|\sum^{t-\hat{\tau}-1}_{\tau=1}\beta_1^{t-\tau} \StocGradient{\tau}{j}\right|}_{R_3}~.
\end{align}
The first term can be bounded below by using Lemma~\ref{lem:true_grad_seq_generalized_signsgd} and that $\hat{\tau} + 1 \ge \bar{\tau}$:
\begin{align}
    R_1
    &\ge
    (1-\beta_1)\left|\sum^{t}_{\tau=t-\hat{\tau}}\beta_1^{t-\tau} \PartialDerivative{t}{j}\right|
    -
    (1-\beta_1)\left|\sum^{t}_{\tau=t-\hat{\tau}}\beta_1^{t-\tau}\left(\PartialDerivative{\tau}{j}-  \PartialDerivative{t}{j}\right)\right|\\
    &\ge
    \left(1-\beta_1^{\bar{\tau}}-\frac{\PNormDimension L_{1,j}\eta}{(1-\beta_1)\sqrt{1-\beta_2}}\right)\left|\PartialDerivative{t}{j}\right|
    -
    \frac{L_{0,j}\eta}{(1-\beta_1)\sqrt{1-\beta_2}}~.
\end{align}
The second term can be bounded using Lemma~\ref{lem:noise_seq_generalized_signsgd}.
Thus,
\begin{align}
    |m_{t,j}|
    \ge&
    \left(1-\beta_1^{\bar{\tau}}-\frac{\PNormDimension L_{1,j}\eta}{(1-\beta_1)\sqrt{1-\beta_2}}\right)\left|\PartialDerivative{t}{j}\right|
    -
    \frac{L_{0,j}\eta}{(1-\beta_1)\sqrt{1-\beta_2}}\\
    &- \beta_1^{\bar{\tau}}(M_j+\sigma_j)
    - (1-\beta_1)E_j\\
    \ge&
    D\left|\PartialDerivative{t}{j}\right| - B_j~,
\end{align}
where we used~\eqref{eq:beta1_bar_tau} and that $e^{-x} \le \frac{1}{x}$ for $x > 0$.

Therefore, with probability at least $1 - 3t\delta$ we have
\begin{equation}
\label{eq:lower_bound_m_over_sqrt_v}
    \frac{|m_{t,j}|}{\sqrt{v_{t,j}}}
    \ge
    \frac{\rho\left(D\left|\PartialDerivative{t}{j}\right| - B_j\right)}{2\sqrt{1-\beta_2}\left(C_j\left|\PartialDerivative{t}{j}\right| + B_j\right)}~.
\end{equation}

Given that $D>0$, depending on the relative size of $\left|\PartialDerivative{t}{j}\right|$ vs.~$B_j$, we have following two cases.

\textbf{Case 2.1}:
$\left|\PartialDerivative{t}{j}\right| < \frac{5B_j}{D}$.

\textbf{Case 2.2}: $\left|\PartialDerivative{t}{j}\right| \ge \frac{5B_j}{D}$, using the fact that the r.h.s of \eqref{eq:lower_bound_m_over_sqrt_v} is decreasing in $B_j$, we have
\begin{align}
    \frac{|m_{t,j}|}{\sqrt{v_{t,j}}}
    &\ge \frac{\frac{4D}{5}\rho\left|\PartialDerivative{t}{j}\right|}{2\sqrt{1-\beta_2}\left(C_j+\frac{D}{5}\right)\left|\PartialDerivative{t}{j}\right|}
    \ge \frac{2\rho D}{5\sqrt{1-\beta_2}(C_j + D )}
    \ge \frac{\rho D}{5\sqrt{1-\beta_2}},
\end{align}
where in the last inequality we used the fact that $C_j+D\leq 2$. Note that $D < 1$ so the above lower bound is smaller than~\eqref{eq:generalized_signsgd_lower_bound_m_large}.
\end{proof}

The following two lemmas are for the special case of $\beta_2 = 0$.
\begin{lemma}
\label{lem:signsgdm_bounded_updates}
With choices of parameters in Theorem~\ref{thm:generalized_signsgd}, when $\beta_2 = 0$, we have $\PNorm{\bx_{t+1} - \bx_{t}} = \eta\PNormDimension \le \frac{1}{\VecLInftyNorm{L_{1}}}$.
\end{lemma}
\begin{proof}[Proof of Lemma~\ref{lem:signsgdm_bounded_updates}]
Using the fact that $\alpha \le 1$ and the condition on $T$, we have
\begin{align}
\eta &= \frac{\sqrt{\Delta\alpha}}{\sqrt{\VecLOneNorm{L_0}}\sqrt{T}}
\leq
\frac{\sqrt{\Delta}}{\sqrt{\VecLOneNorm{L_0}}\sqrt{T}}
\leq
\frac{\sqrt{\Delta}}{\sqrt{\VecLOneNorm{L_0}}}\frac{\sqrt{\VecLOneNorm{L_0}}}{10\PNormDimension\sqrt{\Delta}\VecLInftyNorm{L_1}}
\le\frac{1}{\PNormDimension\VecLInftyNorm{L_{1}}}~.
\qedhere
\end{align}
\end{proof}

The following lemma is adapted from the proof of Theorem 2 in~\citep{CutkoskyM21}.
\begin{lemma}
\label{lem:err_unravel_signsgdmom}
Under Assumptions~\ref{asp:objective_func}, \ref{asp:l0l1_coordinate}, and \ref{asp:noise_as_bounded}, using the settings of the hyperparameters in Theorem~\ref{thm:generalized_signsgd}, denoting $\alpha = 1 - \beta_1$ and $\boldepsilon_t = \bm_t - \nabla F(\bx_t)$, for all $t\geq 1$ and $j\in[d]$ we have, with probability at least $1-3\delta$,
\begin{align}
    |\epsilon_{t+1, j}|
    \le
    &(1 - \alpha)^t\left(\alpha\sigma_j + (1-\alpha)\left|\PartialDerivative{1}{j}\right|\right) +
    \frac{\eta L_{0,j}}{\alpha}
    + \alpha E_j\\
    &
    + (1 - \alpha)\eta \PNormDimension L_{1,j}\sum^{t-1}_{\tau=0}(1-\alpha)^{\tau}\left|\PartialDerivative{t-\tau}{j}\right|~.
\end{align}
\end{lemma}
\begin{proof}[Proof of Lemma~\ref{lem:err_unravel_signsgdmom}]
Denote $\tilde{\boldepsilon}_t = \bg_t - \nabla F(\bx_t)$ and $S_j(\ba, \bb) = \frac{\partial F}{\partial x_j}(\ba) - \frac{\partial F}{\partial x_j}(\bb)$. Then, from Assumption~\ref{asp:l0l1_coordinate} and Lemma~\ref{lem:signsgdm_bounded_updates}, for all $t \ge 1$ and all $j\in[d]$ we have
\begin{align}
&\boldepsilon_1 = \alpha\tilde{\boldepsilon}_1-(1-\alpha)\nabla F(\bx_1),\label{eq:init_mom_grad_diff}\\
&\PNorm{\bx_{t+1} - \bx_{t}}\le\frac{1}{\VecLInftyNorm{L_{1}}}\\
&\qquad\Rightarrow |S_j(\bx_{t+1}, \bx_t)| \le \left(\frac{L_{0,j}}{\sqrt{d}} + L_{1,j}\left|\frac{\partial F}{\partial x_j}(\bx_t)\right|\right)\PNorm{\bx_{t+1} - \bx_{t}}~.\label{eq:adjacent_step_diff}
\end{align}

We can derive the following recursive formulation for any $t \ge 1$:
\begin{align*}
    m_{t+1, j}
    & =
    (1 - \alpha)m_{t, j} + \alpha \StocGradient{t+1}{j}\\
    & =
    (1 - \alpha)\PartialDerivative{t}{j} + (1 - \alpha)\epsilon_{t, j} + \alpha\PartialDerivative{t+1}{j} + \alpha\tilde{\epsilon}_{t+1,j}\\
    & =
    \PartialDerivative{t+1}{j} + (1 - \alpha)S_j(\bx_t, \bx_{t+1}) + (1 - \alpha)\epsilon_{t, j} + \alpha\tilde{\epsilon}_{t+1,j},
\end{align*}
which implies
\begin{align}
\label{eq:recur_err}
    \epsilon_{t+1, j}
    &=
    (1 - \alpha)\epsilon_{t, j} + (1 - \alpha)S_j(\bx_t, \bx_{t+1}) + \alpha\tilde{\epsilon}_{t+1,j}~.
\end{align}
Unravel~\eqref{eq:recur_err} from $1$ to $t$ gives us
\begin{equation}
    \epsilon_{t+1, j}
    =
    (1 - \alpha)^t\epsilon_{1, j} + 
    (1 - \alpha)\sum^{t-1}_{\tau=0}(1-\alpha)^{\tau} S_j(\bx_{t-\tau}, \bx_{t+1-\tau}) + \alpha\sum^{t-1}_{\tau=0}(1-\alpha)^{\tau}\tilde{\epsilon}_{t+1-\tau,j}~.
\end{equation}
Take the absolute value of both sides, to obtain
\begin{align}
    \label{eq:err_unravel}
    |\epsilon_{t+1, j}|
    \le &
    (1 - \alpha)^t|\epsilon_{1, j}| + 
    (1 - \alpha)\sum^{t-1}_{\tau=0}(1-\alpha)^{\tau} |S_j(\bx_{t-\tau}, \bx_{t+1-\tau})|\\ 
    &+ \alpha\left|\sum^{t-1}_{\tau=0}(1-\alpha)^{\tau}\tilde{\epsilon}_{t+1-\tau,j}\right|\\
    \le &
    (1 - \alpha)^t|\epsilon_{1, j}|
    + \alpha E_j\\
    &+ 
    (1 - \alpha)\sum^{t-1}_{\tau=0}(1-\alpha)^{\tau} \left(\frac{L_{0,j}}{\sqrt{d}} + L_{1,j}\left|\PartialDerivative{t-\tau}{j}\right|\right)\PNorm{\bx_{t+1-\tau} - \bx_{t-\tau}}\\
    \le &
    (1 - \alpha)^t|\epsilon_{1, j}| +
    (1 - \alpha)\eta L_{0,j}\sum^{t-1}_{\tau=0}(1-\alpha)^{\tau}\\
    &+ (1 - \alpha)\eta\PNormDimension L_{1,j}\sum^{t-1}_{\tau=0}(1-\alpha)^{\tau}\left|\PartialDerivative{t-\tau}{j}\right| + \alpha E_j\\
    \le &
    (1 - \alpha)^t\left(\alpha\sigma_j + (1-\alpha)\left|\PartialDerivative{1}{j}\right|\right) +
    \frac{\eta L_{0,j}}{\alpha}\\
    &+ (1 - \alpha)\eta\PNormDimension L_{1,j}\sum^{t-1}_{\tau=0}(1-\alpha)^{\tau}\left|\PartialDerivative{t-\tau}{j}\right|
    + \alpha E_j~,
\end{align}
where the second inequality uses~\eqref{eq:adjacent_step_diff} and Lemma~\ref{lem:noise_seq_generalized_signsgd}, the fourth and fifth inequalities use~\eqref{eq:adjacent_step_diff}, and the final one is due to~\eqref{eq:init_mom_grad_diff}.
\end{proof}

\begin{proof}[Proof of Theorem~\ref{thm:generalized_signsgd} for $\beta_2 = 0$]
From Lemma~\ref{lem:signsgdm_bounded_updates} we know that $\PNorm{\bx_{t+1} - \bx_{t}} \le \frac{1}{\VecLInftyNorm{L_{1}}}$ for all $t\in[T]$. Thus, we can apply Lemma~\ref{lem:desent_ineq_l0l1_coordinatewise} to have
\begin{align}
    &F(\bx_{t+1}) - F(\bx_t)\\
    \le&
    \langle \nabla F(\bx_t), \bx_{t+1} - \bx_t \rangle
    +
    \sum^d_{j=1}\frac{\left(\frac{L_{0,j}}{\sqrt{d}} +  L_{1,j}\left|\PartialDerivative{t}{j}\right|\right)\PNorm{\bx_{t+1}-\bx_{t}}}{2}|x_{t+1,j}-x_{t,j}|\\
    =&
    \langle \nabla F(\bx_t), -\eta\text{sign}(\bm_t) \rangle
    +
    \sum^d_{j=1}\frac{ L_{0,j}+ L_{1,j}\PNormDimension\left|\PartialDerivative{t}{j}\right|}{2}\eta^2\\
    =&
    -\eta\|\nabla F(\bx_t)\|_1
    + \eta\langle \nabla F(\bx_t), \text{sign}(\nabla F(\bx_t)) - \text{sign}(\bm_t) \rangle\\
    &+
    \sum^d_{j=1}\frac{ L_{0,j}+ L_{1,j}\PNormDimension\left|\PartialDerivative{t}{j}\right|}{2}\eta^2\\
    =&
    -\eta\|\nabla F(\bx_t)\|_1
    + 2\eta\sum^d_{j=1}\left|\PartialDerivative{t}{j}\right|\mathbb{I}\left[ \text{sign}\left(\PartialDerivative{t}{j}\right) \neq \text{sign}(m_{t,j})\right]\\
    &+
    \sum^d_{j=1}\frac{ L_{0,j}+ L_{1,j}\PNormDimension\left|\PartialDerivative{t}{j}\right|}{2}\eta^2\label{eq:descent_signsgdm}~,
\end{align}
where $\mathbb{I}(\cdot)$ is the indicator function and the first inequality uses Lemma~\ref{lem:desent_ineq_l0l1_coordinatewise},

Now, note that
\begin{align}
    \mathbb{I}\left[\text{sign}\left(\PartialDerivative{t}{j}\right) \neq \text{sign}(m_{t,j})\right]
    & \le
    \mathbb{I}\left[\left|\PartialDerivative{t}{j} - m_{t,j}\right| \ge \left|\PartialDerivative{t}{j}\right|\right]\\
    & \le
    \frac{\left|\PartialDerivative{t}{j} - m_{t,j}\right|}{\left|\PartialDerivative{t}{j}\right|}~.
\end{align}

Thus, denoting by $\boldepsilon_t = \bm_t - \nabla F(\bx_t)$ gives
\begin{align}
    &F(\bx_{t+1}) - F(\bx_t)\\
    &\le
    -\eta\|\nabla F(\bx_t)\|_1
    + 2\eta\|\boldepsilon_t\|_1
    +
    \sum^d_{j=1}\frac{L_{0,j} + L_{1,j}\PNormDimension\left|\PartialDerivative{t}{j}\right|}{2}\eta^2\\
    &=
    -\eta\|\nabla F(\bx_t)\|_1
    + 2\eta\|\boldepsilon_t\|_1
    +
    \frac{\VecLOneNorm{L_0}\eta^2}{2}
    + \frac{\eta^2\PNormDimension}{2}\sum^d_{j=1}L_{1,j}\left|\PartialDerivative{t}{j}\right|~.
\end{align}
Sum both sides over $t=1, \dots,T$, to have
\begin{align}
    F^* - F(\bx_1)
    \le
    &-\eta\sum^T_{t=1}\|\nabla F(\bx_t)\|_1
    + 2\eta\sum^T_{t=1}\|\boldepsilon_t\|_1
    +
    \frac{\VecLOneNorm{L_0}\eta^2 T}{2}\\
    &+ \frac{\eta^2\PNormDimension}{2}\sum^T_{t=1}\sum^d_{j=1}L_{1,j}\left|\PartialDerivative{t}{j}\right|~. \label{eq:decomp_ssgdm}
\end{align}
Use Lemma~\ref{lem:err_unravel_signsgdmom} to bound each coordinate of $\sum^T_{t=1}\|\boldepsilon_t\|_1$:
\begin{align}
    \sum^{T-1}_{t=0}|\epsilon_{t+1, j}|
    \le&
    \sum^{T-1}_{t=0}\left[(1 - \alpha)^t\left(\alpha\sigma_j + (1-\alpha)\left|\PartialDerivative{1}{j}\right|\right) +
    \frac{\eta L_{0,j}}{\alpha}
    + \alpha E_j\right]\\
    &+ (1 - \alpha)\eta\PNormDimension L_{1,j}\sum^{T-1}_{t=0}\sum^{t-1}_{\tau=0}(1-\alpha)^{\tau}\left|\PartialDerivative{t-\tau}{j}\right|\\
    =&
    \sigma_j +
    \frac{1}{\alpha}\left|\PartialDerivative{1}{j}\right|
    + \frac{\eta L_{0,j} T}{\alpha}
    + \alpha E_j T\\
    &+ (1 - \alpha)\eta \PNormDimension L_{1,j}\sum^{T-1}_{t=1}\sum^{t}_{\tau^{\prime}=1}(1-\alpha)^{t-\tau^{\prime}}\left|\PartialDerivative{\tau^{\prime}}{j}\right|\\
    =&
    \sigma_j +
    \frac{1}{\alpha}\left|\PartialDerivative{1}{j}\right|
    + \frac{\eta L_{0,j} T}{\alpha}
    + \alpha E_j T\\
    &+ (1 - \alpha)\eta\PNormDimension L_{1,j}\sum^{T-1}_{\tau^{\prime}=1}\left(\sum^{T-1}_{t=\tau^{\prime}}(1-\alpha)^{t}\right)(1-\alpha)^{-\tau^{\prime}}\left|\PartialDerivative{\tau^{\prime}}{j}\right|\\
    \le&
    \sigma_j +
    \frac{1}{\alpha}\left|\PartialDerivative{1}{j}\right|
    + \frac{\eta L_{0,j} T}{\alpha}
    + \alpha E_j T\\
    &+ \frac{(1 - \alpha)\eta\PNormDimension L_{1,j}}{\alpha}\sum^{T-1}_{t=1}\left|\PartialDerivative{t}{j}\right|~.
\end{align}

The above one holds with probability at least $1-3T\delta$ as we invoked Lemma~\ref{lem:err_unravel_signsgdmom} for $T$ times which in turns means invoking Lemma~\ref{lem:noise_seq_generalized_signsgd} for $T$ times. Yet, the above inequality would need to sum from $j=1$ to $d$, meaning in total we would invoke Lemma~\ref{lem:noise_seq_generalized_signsgd} for $dT$ times. Thus, following results hold with probability at least $1-3dT\delta$.

Now, put the above inequality back into \eqref{eq:decomp_ssgdm} to have
\begin{align}
    &F^* - F(\bx_1)\\
    \le
    & -\eta\sum^T_{t=1}\|\nabla F(\bx_t)\|_1
    +
    \frac{\VecLOneNorm{L_0}\eta^2T}{2}
    + \frac{\eta^2\PNormDimension}{2}\sum^T_{t=1}\sum^d_{j=1}L_{1,j}\left|\PartialDerivative{t}{j}\right|
    + 2\eta\sum^d_{j=1}\sigma_j\\
    & + 2\eta\sum^d_{j=1}\left(\frac{1}{\alpha}\left|\PartialDerivative{1}{j}\right|
    + \frac{\eta L_{0,j} T}{\alpha}
    + \alpha E_j T
    + \frac{(1 - \alpha)\eta\PNormDimension L_{1,j}}{\alpha}\sum^{T}_{t=1}\left|\PartialDerivative{t}{j}\right|\right)\\
    =
    & -\eta\sum^T_{t=1}\|\nabla F(\bx_t)\|_1
    +
    \frac{\VecLOneNorm{L_0}\eta^2T}{2}
    + \frac{\eta^2\PNormDimension}{2}\sum^T_{t=1}\sum^d_{j=1}L_{1,j}\left|\PartialDerivative{t}{j}\right|
    + 2\eta\VecLOneNorm{\sigma}\\
    &
    + \frac{2\eta}{\alpha}\|\nabla F(\bx_1)\|_1
    + \frac{2\eta^2 \VecLOneNorm{L_0} T}{\alpha}
    + 2\eta\alpha T\sum^{d}_{j=1}E_j
    + \frac{2\eta^2\PNormDimension}{\alpha}\sum^T_{t=1}\sum^d_{j=1}L_{1,j}\left|\PartialDerivative{t}{j}\right|\\
    =
    & -\eta\sum^T_{t=1}\|\nabla F(\bx_t)\|_1
    + 2\eta\VecLOneNorm{\sigma}
    + \frac{2\eta}{\alpha}\|\nabla F(\bx_1)\|_1
    + 2\eta\alpha T\sum^{d}_{j=1}E_j\\
    &
    + \left(\frac12 + \frac{2}{\alpha}\right)\VecLOneNorm{L_0}\eta^2T + \left(\frac12 + \frac{2}{\alpha}\right)\eta^2\PNormDimension\sum^T_{t=1}\sum^d_{j=1}L_{1,j}\left|\PartialDerivative{t}{j}\right|~.
\end{align}

Now, using the definitions of $\eta$ and $\alpha$, and the conditions on $T$, we have
\begin{align}
    &\left(\frac12 + \frac{2}{\alpha}\right)\eta^2\PNormDimension L_{1,j}\\
    &\leq
    \frac{\eta}{2}\left(1 + \frac{4}{\alpha}\right)\frac{\PNormDimension\sqrt{\Delta\alpha}}{\sqrt{T}}\frac{\VecLInftyNorm{L_1}}{\sqrt{\VecLOneNorm{L_0}}}
    \le
    \frac{\eta}{2}\frac{5\PNormDimension\sqrt{\Delta}}{\sqrt{\alpha T}}\frac{\VecLInftyNorm{L_1}}{\sqrt{\VecLOneNorm{L_0}}}\\
    &=
    \frac{\eta}{2}\frac{5\PNormDimension\sqrt{\Delta}}{\sqrt{ T}}\frac{\VecLInftyNorm{L_1}}{\sqrt{\VecLOneNorm{L_0}}}\cdot\max\left(\frac{\sqrt{\VecLOneNorm{\sigma}}T^{1/4}}{\VecLOneNorm{L_0}^{1/4}\Delta^{1/4}}, 1\right)\\
    &=
    \frac{\eta}{2}\max\left(\frac{5\PNormDimension\sqrt{\VecLOneNorm{\sigma}}\VecLInftyNorm{L_1}\Delta^{1/4}}{\VecLOneNorm{L_0}^{3/4}T^{1/4}},\frac{5\PNormDimension\sqrt{\Delta}}{\sqrt{ T}}\frac{\VecLInftyNorm{L_1}}{\sqrt{\VecLOneNorm{L_0}}}\right)
    \le
    \frac{\eta}{2}~.
\end{align}

Thus, we have
\begin{align}
    F^* - F(\bx_1)
    \le
    &-\frac{\eta}{2}\sum^T_{t=1}\|\nabla F(\bx_t)\|_1
    +
    2\eta\VecLOneNorm{\sigma}
    + \frac{2\eta}{\alpha}\|\nabla F(\bx_1)\|_1\\
    &+ \left(\frac12 + \frac{2}{\alpha}\right)\VecLOneNorm{L_0}\eta^2T
    + 2\eta\alpha T\sum^{d}_{j=1}E_j~.
\end{align}
Divide both sides by $T$ and rearrange terms to give
\begin{align}
    &\frac{1}{T}\sum^T_{t=1}\|\nabla F(\bx_t)\|_1\\
    \le
    &\frac{2}{\eta T}[F(\bx_1) - F^*]
    + \frac{4}{T}\VecLOneNorm{\sigma}
    + \frac{4}{\alpha T}\|\nabla F(\bx_1)\|_1
    +
    \frac{5}{\alpha}\VecLOneNorm{L_0}\eta\\
    &+
    24\VecLOneNorm{\sigma}(\alpha\max(1, \log(1/\delta))
    + \sqrt{\alpha}\sqrt{\max(1, \log(1/\delta))})~.
\end{align}

Now, we need to consider the following two cases:
\begin{enumerate}
    \item $\VecLOneNorm{\sigma} < \frac{\sqrt{\VecLOneNorm{L_0}}\sqrt{\Delta}}{\sqrt{T}}$: then $\alpha = 1$ and $\eta = \frac{\sqrt{\Delta}}{\sqrt{\VecLOneNorm{L_0}}\sqrt{T}}$
    \begin{align}
        &\frac{1}{T}\sum^T_{t=1}\|\nabla F(\bx_t)\|_1\\
        &\le
        \frac{2\sqrt{\VecLOneNorm{L_0}}}{\sqrt{\Delta}\sqrt{T}}[F(\bx_1) - F^*]
        + \frac{5\VecLOneNorm{L_0}\sqrt{\Delta}}{\sqrt{\VecLOneNorm{L_0}}\sqrt{T}}\\
        &+ \frac{4\sqrt{\VecLOneNorm{L_0}}\sqrt{\Delta}}{T^{3/2}}
        + \frac{4}{T} \|\nabla F(\bx_1)\|_1\\
        &\quad
        + \frac{24\sqrt{\VecLOneNorm{L_0}}\sqrt{\Delta}(\max(1, \log(1/\delta)) + \sqrt{\max(1, \log(1/\delta))})}{\sqrt{T}}\\
        &\le
        \frac{59\max(1, \log(1/\delta))\sqrt{\VecLOneNorm{L_0}\Delta}}{\sqrt{T}} + \frac{4}{T} \|\nabla F(\bx_1)\|_1\label{eq:small_noise_bound}~.
    \end{align}
    \item $\VecLOneNorm{\sigma}\ge\frac{\sqrt{\VecLOneNorm{L_0}\Delta}}{\sqrt{T}}$: then $\alpha = \frac{\sqrt{\VecLOneNorm{L_0}\Delta}}{\VecLOneNorm{\sigma}\sqrt{T}} \le 1$ and $\eta = \frac{\Delta^{3/4}}{\VecLOneNorm{L_0}^{1/4}\sqrt{\VecLOneNorm{\sigma}}T^{3/4}}$
    \begin{align}
        \frac{1}{T}\sum^T_{t=1}\|\nabla F(\bx_t)\|_1
        \le& \frac{2\VecLOneNorm{L_0}^{1/4}\sqrt{\VecLOneNorm{\sigma}}}{\Delta^{3/4}T^{1/4}}[F(\bx_1) - F^*]
        + \frac{4\VecLOneNorm{\sigma}}{T}\\
        &+
        \frac{4\VecLOneNorm{\sigma}}{\sqrt{\VecLOneNorm{L_0}\Delta T}}\|\nabla F(\bx_1)\|_1\\
        &+
        \frac{5\VecLOneNorm{L_0}^{1/4}\sqrt{\VecLOneNorm{\sigma}}\Delta^{1/4}}{T^{1/4}}\\
        &+ \frac{24\max(1, \log(1/\delta))\sqrt{\VecLOneNorm{L_0}\Delta}}{\sqrt{T}}\\
        &+\frac{24\sqrt{\max(1, \log(1/\delta))}\VecLOneNorm{L_0}^{1/4}\sqrt{\VecLOneNorm{\sigma}}\Delta^{1/4}}{T^{1/4}}\\
        \le&
        \frac{31\sqrt{\max(1, \log(1/\delta))}\VecLOneNorm{L_0}^{1/4}\Delta^{1/4}\sqrt{\VecLOneNorm{\sigma}}}{T^{1/4}}\\
        &+ \frac{24\max(1, \log(1/\delta))\sqrt{\VecLOneNorm{L_0}\Delta}}{\sqrt{T}}\\
        &+ \frac{4\VecLOneNorm{\sigma}\|\nabla F(\bx_1)\|_1}{\sqrt{\VecLOneNorm{L_0}\Delta T}}
        + \frac{4\VecLOneNorm{\sigma}}{T}~.
    \end{align}
\end{enumerate}

Put $\delta^{\prime} = \frac{\delta}{3dT}$ concludes the proof.
\end{proof}

\begin{proof}[Proof of Theorem~\ref{thm:generalized_signsgd} for general $\beta_2$]
Note that when $\VecLOneNorm{\sigma}\le\frac{\sqrt{\VecLOneNorm{L_0}\Delta}}{\sqrt{T}}$, $\alpha = 1$, $\beta_2 < \beta_1^2 = 0$. Then our Generalized SignSGD algorithm~\ref{alg:generalized_signsgd} reduces to the SignSGD algorithm, and thus has the same guarantee of~\eqref{eq:small_noise_bound}. Therefore, we only consider the other case when $\VecLOneNorm{\sigma}\ge\frac{\sqrt{\VecLOneNorm{L_0}\Delta}}{\sqrt{T}}$.

Note that the only randomness comes from evaluating stochastic gradients. In the following proof, we will need to invoke Lemma~\ref{lem:noise_seq_generalized_signsgd} for $T$ times for each coordinate $j\in[d]$. Therefore, the following results hold with probability at least $1 - 3dT\delta$. For simplicity, we use the term "with high probability" later in the proof to denote this.

We derive the following quantity which will be used multiple times later:
\begin{equation}
    \frac{\eta}{\alpha}
    =
    \frac{1}{\sqrt{\VecLOneNorm{L_0}}}\cdot\frac{\sqrt{\Delta}}{\sqrt{T}} \cdot \frac{1}{\sqrt{\alpha}}
    =
    \frac{1}{\sqrt{\VecLOneNorm{L_0}}} \cdot\frac{\sqrt{\Delta}}{\sqrt{T}} \cdot \frac{\VecLOneNorm{\sigma}^{1/2}T^{1/4}}{\VecLOneNorm{L_0}^{1/4}\Delta^{1/4}}
    =
    \frac{\VecLOneNorm{\sigma}^{1/2}\Delta^{1/4}}{\VecLOneNorm{L_0}^{3/4}T^{1/4}}~.\label{eq:eta_over_alpha}
\end{equation}

First, from Lemma~\ref{lem:generalized_signsgd_update_lower_bound} we have $D = 1-\frac{2\PNormDimension\VecLInftyNorm{L_{1}}\eta}{(1-\beta_1)\sqrt{1-\beta_2}}$. Then, from the choice of the hyperparameters, we have
\begin{equation}
    \frac{\PNormDimension\VecLInftyNorm{L_{1}}\eta}{(1-\beta_1)\sqrt{1-\beta_2}}
    =
    \frac{\VecLInftyNorm{L_{1}}}{\sqrt{1-\beta_2}} \cdot \frac{\PNormDimension\VecLOneNorm{\sigma}^{1/2}\Delta^{1/4}}{\VecLOneNorm{L_0}^{3/4}T^{1/4}}
    \le
    \frac{\rho}{10}
    \le
    \frac{1}{10}~.\label{eq:L1_eta_over_alpha}
\end{equation}
Thus, we have $D \ge 1 - \frac{1}{5} \ge \frac12$ and, as $\sqrt{\beta_2}/\beta_1 < 1$,
\begin{equation}
    \frac{\rho D}{5\sqrt{1-\beta_2}}
    \ge
    \frac{\rho}{10\sqrt{1-\beta_2}}
    =
    A~.\label{eq:update_lower_bound_generalized_signsgd}
\end{equation}

Also, for those coordinates with small gradients $\left|\PartialDerivative{t}{j}\right| < \frac{5B_j}{D} \le 10B_j$, we have
\begin{align}
    &\PartialDerivative{t}{j}\cdot(x_{t+1, j} - x_{t,j})\\
    =
    &-\eta\PartialDerivative{t}{j}\cdot\frac{m_{t,j}}{\sqrt{v_{t,j}}}\\
    =
    &-A\eta\left|\PartialDerivative{t}{j}\right| + \eta\left|\PartialDerivative{t}{j}\right|\cdot\left(A - \text{sign}\left(\PartialDerivative{t}{j}\right)\frac{m_{t,j}}{\sqrt{v_{t,j}}}\right)\\
    \le
    &-A\eta\left|\PartialDerivative{t}{j}\right| + 10B_j\eta\left(\frac{1}{\sqrt{1-\beta_2}} + A\right)~. \label{eq:inner_prod_upper_bound_grad_small}
\end{align}

We are now ready to prove the theorem. We will need to use Lemma~\ref{lem:generalized_signsgd_update_lower_bound}, hence we first need to show that all past true gradients are bounded by $M_j$, namely that, for any $t$, $\left|\PartialDerivativeGeneral{\bx_{\tau}}{j}\right| \le M_j$ holds for all $\tau \le t$ and all $j\in[d]$. From the definition of $M_j$ stated in the theorem, in order to guarantee this, we only need to prove that $F(\bx_{\tau}) \le F(\bx_1)$ for all $\tau\le t$. We will prove this by induction.

For $t = 1$ the condition is trivially true.

We then assume that the condition holds for $t$ and prove based on this that it still holds for $t+1$.

For those coordinates with $\left|\PartialDerivative{t}{j}\right| \ge \frac{5B_j}{D}$,
denote $\hat{\tau} = \lfloor\bar{\tau}\rfloor = \left\lfloor\frac{\sqrt{1-\beta_2}}{\eta\PNormDimension\VecLInftyNorm{L_1}}\right\rfloor$, with high probability, we have
\begin{align}
    \left|m_{t,j} - \PartialDerivative{t}{j}\right|
    &=
    \left|(1 - \beta_1)\sum^{t}_{\tau=1}\beta_1^{t-\tau} \StocGradient{\tau}{j} - \PartialDerivative{t}{j} \right|\\
    &\le
    \left|(1 - \beta_1)\sum^{t-\hat{\tau}-1}_{\tau=1}\beta_1^{t-\tau} \StocGradient{\tau}{j}\right| + \left|(1 - \beta_1)\sum^{t}_{\tau=t-\hat{\tau}}\beta_1^{t-\tau} \StocGradient{\tau}{j} - \PartialDerivative{t}{j} \right|\\
    &\le
    \beta_1^{\bar{\tau}}(M_j + \sigma_j)
    +
    \beta_1^{\bar{\tau}}\left|\PartialDerivative{t}{j}\right|\\
    &\quad+(1-\beta_1)\left|\sum^{t}_{\tau=t-\hat{\tau}}\beta_1^{t-\tau} \left(\PartialDerivative{\tau}{j}- \PartialDerivative{t}{j}\right)\right|\\
    &\quad+
    (1-\beta_1)\left|\sum^{t}_{\tau=t-\hat{\tau}}\beta_1^{t-\tau}\left(\StocGradient{\tau}{j} - \PartialDerivative{\tau}{j}\right)\right|\\
    &\le
    \left(1 - D\right)\left|\PartialDerivative{t}{j}\right| + B_j\\
    &\le
    \left(1 - \frac{4D}{5}\right)\left|\PartialDerivative{t}{j}\right|
    \le
    \left|\PartialDerivative{t}{j}\right|, \label{eq:equal_sign_moment_grad}
\end{align}
where the first equality comes from Lemma~\ref{lem:moment_decomp_generalized_signsgd}; for the second inequality, the third term can be bounded using Lemma~\ref{lem:true_grad_seq_generalized_signsgd}, and the final term can be bounded using Lemma~\ref{lem:noise_seq_generalized_signsgd}; for the third inequality, we used~\eqref{eq:beta1_bar_tau} and that $e^{-x}\le\frac{1}{x}$ for $x > 0$.
This inequality implies that $\text{sign}(m_{t,j}) = \text{sign}\left(\PartialDerivative{t}{j}\right)$ with high probability.

Denote $\mathcal{U}_t = \left\{j\in[d]: \left|\PartialDerivative{t}{j}\right| \ge \frac{5B_j}{D}\right\}$. From the choices of hyperparameters we can show that $1\le\bar{\tau}$ which means $\PNorm{\bx_{t+1} - x_{t}}\le\frac{1}{\VecLInftyNorm{L_{1}}}$ (Lemma~\ref{lem:generalized_signsgd_bounded_updates_t}). Thus, using Lemma~\ref{lem:desent_ineq_l0l1_coordinatewise}, with high probability we have
\begin{align}
    & F(\bx_{t+1}) - F(\bx_t)\\
    \le&
    \langle \nabla F(\bx_t), \bx_{t+1} - \bx_t \rangle
    +
    \sum^d_{j=1}\frac{\left(\frac{L_{0,j}}{\sqrt{d}} +  L_{1,j}\left|\PartialDerivative{t}{j}\right|\right)\PNorm{\bx_{t+1}-\bx_{t}}}{2}|x_{t+1,j}-x_{t,j}|\\
    =&
    \sum^d_{j=1}\left(-\PartialDerivative{t}{j} \cdot \eta\text{sign}(m_{t,j})\frac{|m_{t,j}|}{\sqrt{v_{t,j}}}\right)\\
    &+    \sum^d_{j=1}\left(\frac{\left( \frac{L_{0,j}}{\sqrt{d}}+ L_{1,j}\left|\PartialDerivative{t}{j}\right|\right)\PNorm{\bx_{t+1}-\bx_{t}}}{2}|x_{t+1,j}-x_{t,j}|\right)\\
    \le&
    \sum_{j\in\mathcal{U}_t}-\eta\left|\PartialDerivative{t}{j}\right|\cdot \frac{|m_{t,j}|}{\sqrt{v_{t,j}}}
    +
    \sum_{j\not\in\mathcal{U}_t}\left(-A\eta\left|\PartialDerivative{t}{j}\right|+
    10\eta B_j\left(\frac{1}{\sqrt{1-\beta_2}} + A\right)\right)\\
    &+
    \sum^d_{j=1}\frac{ L_{0,j}+ L_{1,j}\PNormDimension\left|\PartialDerivative{t}{j}\right|}{2(1-\beta_2)}\eta^2\\
    &\le
    -A\eta\|\nabla F(x_t)\|_1
    +
    \sum^d_{j=1}\frac{ L_{0,j}+ L_{1,j}\PNormDimension\left|\PartialDerivative{t}{j}\right|}{2(1-\beta_2)}\eta^2\\
    &+ 10\eta\left(\frac{1}{\sqrt{1-\beta_2}} + A\right)\sum^d_{j=1}B_j, \label{eq:descent_generalized_signsgd}
\end{align}
where the second inequality uses~\eqref{eq:inner_prod_upper_bound_grad_small},~\eqref{eq:equal_sign_moment_grad}, and Lemma~\ref{lem:moment_decomp_generalized_signsgd}, and the third inequality uses Lemma~\ref{lem:generalized_signsgd_update_lower_bound} and~\eqref{eq:update_lower_bound_generalized_signsgd}.

Now, noticing the conditions on $\eta$, $\alpha$, $\beta_2 < \beta_1^2 < \beta_1$, and $T$, use~\eqref{eq:eta_over_alpha} to have
\begin{align}
    \frac{\eta^2\PNormDimension L_{1,j}}{2(1-\beta_2)}
    \le
    \frac{\eta\PNormDimension\VecLInftyNorm{L_{1}}}{2}\cdot\frac{\eta}{\alpha}
    =
    \frac{\eta}{2}\frac{\PNormDimension\VecLInftyNorm{L_{1}}\VecLOneNorm{\sigma}^{1/2}\Delta^{1/4}}{\VecLOneNorm{L_0}^{3/4}T^{1/4}}
    \le
    \frac{\eta}{2}\frac{\rho\sqrt{1-\beta_2}}{10}
    \le
    \frac{\eta}{2}A~.
\end{align}
Thus,~\eqref{eq:descent_generalized_signsgd} becomes
\begin{align}
    F(\bx_{t+1}) - F(\bx_t)
    \le
    &-\frac{A\eta}{2}\|\nabla F(x_t)\|_1
    +
    \frac{\eta^2\VecLOneNorm{L_{0}}}{2(1-\beta_2)}\\
    &+ 10\eta\left(\frac{1}{\sqrt{1-\beta_2}} + A\right)\sum^d_{j=1}B_j~. \label{eq:main_descent_generalized_signsgd}
\end{align}
Therefore, either $F(\bx_{t+1}) - F(\bx_t) \le 0$ or 
\begin{equation}
\|\nabla F(\bx_t)\|_1\le \frac{\eta\VecLOneNorm{L_0}}{A(1-\beta_2)} + 20\left(\frac{1}{A\sqrt{1-\beta_2}} + 1\right)\sum^d_{j=1}B_j~.\label{eq:generalized_signsgd_stop_condition}
\end{equation}

This concludes the mathematical induction up until~\eqref{eq:generalized_signsgd_stop_condition} is met for the first time which we denote as $T_0$. 
In the following, we will explain that if \eqref{eq:generalized_signsgd_stop_condition} holds then the algorithm has found an approximate stationary point.

Now, suppose $T \le T_0$, then~\eqref{eq:main_descent_generalized_signsgd} holds all the time and we sum both sides of it from $1$ to $T$ to have, with high probability,
\begin{align}
    F^* - F(\bx_1)
    \le
    &-\frac{A\eta}{2}\sum^T_{t=1}\|\nabla F(\bx_t)\|_1 +
    \frac{\VecLOneNorm{L_0}\eta^2T}{2(1-\beta_2)}
    + 10\eta T\left(\frac{1}{\sqrt{1-\beta_2}} + A\right)\sum^d_{j=1}B_j~.
\end{align}
Rearrange terms to obtain
\begin{align}
    \min_{t\in[T]}\|\nabla F(\bx_t)\|_1
    \le
    &\frac{1}{T}\sum^T_{t=1}\|\nabla F(\bx_t)\|_1
    \le
    \frac{2}{A\eta T}[F(\bx_1) - F^*] +
    \frac{\eta\VecLOneNorm{L_0}}{A(1-\beta_2)}\\
    &+ 20\left(\frac{1}{A\sqrt{1-\beta_2}} + 1\right)\sum^d_{j=1}B_j\label{eq:generalized_signsgd_min_grad_norm_bound}~.
\end{align}

Note that RHS of~\eqref{eq:generalized_signsgd_stop_condition} is less than RHS of~\eqref{eq:generalized_signsgd_min_grad_norm_bound}. Thus, for the other case of $T > T_0$,~\eqref{eq:generalized_signsgd_min_grad_norm_bound} still holds.

Recalling that $\rho = 1-\frac{\sqrt{\beta_2}}{\beta_1}$, $A = \frac{\rho}{10\sqrt{1-\beta_2}}$, $\beta_2 < \beta_1^2 < \beta_1$, and $B_j \triangleq
\frac{\eta L_{0,j}}{(1-\beta_1)\sqrt{1-\beta_2}} + \beta_1^{\bar{\tau}}(M_j + \sigma_j)
+ 6(1-\beta_1) \sigma_j\max(1, \log(1/\delta))
+ \frac{6(1-\beta_1)}{\sqrt{1-\beta_1^2}}\sqrt{\sigma_j^2\max(1, \log(1/\delta))}$, we have
\begin{align}
    \min_{t\in[T]}\|\nabla F(\bx_t)\|_1
    \le
    &\frac{20}{\rho\eta T}[F(\bx_1) - F^*] +
    \frac{10\eta\VecLOneNorm{L_0}}{\rho\sqrt{1-\beta_1}}\\
    &+ 20\left(\frac{10}{\rho} + 1\right)\left(\frac{\VecLOneNorm{L_{0}}\eta}{(1-\beta_1)\sqrt{1-\beta_2}} + \beta_1^{\bar{\tau}}(\VecLOneNorm{M} + \VecLOneNorm{\sigma})\right)\\
    &
    +120\VecLOneNorm{\sigma}(1-\beta_1)\left(\frac{10}{\rho} + 1\right)\max(1, \log(1/\delta))\\
    &
    +120\VecLOneNorm{\sigma}(1-\beta_1)\left(\frac{10}{\rho} + 1\right)\frac{1}{\sqrt{1-\beta_1^2}}\sqrt{\max(1, \log(1/\delta))}~.
\end{align}

When $\VecLOneNorm{\sigma}\ge\frac{\sqrt{\VecLOneNorm{L_0}\Delta}}{\sqrt{T}}$, then $\alpha = \frac{\sqrt{\VecLOneNorm{L_0}\Delta}}{\VecLOneNorm{\sigma}\sqrt{T}}$ and $\eta = \frac{\Delta^{3/4}}{\VecLOneNorm{L_0}^{1/4}\sqrt{\VecLOneNorm{\sigma}}T^{3/4}}$. Hence,
\begin{align}
    &\min_{t\in[T]}\|\nabla F(\bx_t)\|_1\\
    \le&
    \frac{20\Delta^{1/4}\VecLOneNorm{L_0}^{1/4}\VecLOneNorm{\sigma}^{1/2}T^{3/4}}{\rho T}
    + \frac{10\sqrt{\VecLOneNorm{L_0}\Delta}}{\rho\sqrt{T}}\\
    &+ 20\left(\frac{10}{\rho} + 1\right)\left(\frac{\Delta^{1/4}\VecLOneNorm{\sigma}^{1/2}\VecLOneNorm{L_{0}} }{\sqrt{1-\beta_2}\VecLOneNorm{L_0}^{3/4}T^{1/4}} + \beta_1^{\bar{\tau}}(\VecLOneNorm{M} + \VecLOneNorm{\sigma})\right)\\
    &+120\left(\frac{10}{\rho} + 1\right)\frac{\sqrt{\VecLOneNorm{L_0}\Delta}}{\sqrt{T}}\max(1, \log(1/\delta))\\
    &+ 120\left(\frac{10}{\rho} + 1\right)\frac{\Delta^{1/4}\VecLOneNorm{L_0}^{1/4}\VecLOneNorm{\sigma}^{1/2}}{T^{1/4}}\sqrt{\max(1, \log(1/\delta))}\\
    \le&
    \left(\frac{20}{\rho T^{1/4}} + \left(\frac{10}{\rho} + 1\right)\frac{20+120\sqrt{\max(1, \log(1/\delta))}}{\sqrt{1-\beta_2}T^{1/4}}\right)\Delta^{1/4}\VecLOneNorm{L_0}^{1/4}\VecLOneNorm{\sigma}^{1/2}\\
    &+\left(\frac{10}{\rho}
    + 120\left(\frac{10}{\rho} + 1\right)\max(1, \log(1/\delta))\right)\frac{\sqrt{\VecLOneNorm{L_0}\Delta}}{\sqrt{T}}\\
    &+ 20\left(\frac{10}{\rho} + 1\right)\beta_1^{\bar{\tau}}(\VecLOneNorm{M} + \VecLOneNorm{\sigma})\\
    \le&
    \frac{1560\Delta^{1/4}\VecLOneNorm{L_0}^{1/4}\VecLOneNorm{\sigma}^{1/2}\sqrt{\max(1, \log(1/\delta))}}{\rho\sqrt{1-\beta_2}T^{1/4}}\\ &+ \frac{1330\max(1, \log(1/\delta))\sqrt{\VecLOneNorm{L_0}\Delta}}{\rho\sqrt{T}}\\
    &+\frac{220}{\rho}(\VecLOneNorm{M} + \VecLOneNorm{\sigma})\exp\left(-\frac{\sqrt{1-\beta_2}\VecLOneNorm{L_0}^{3/4}}{\PNormDimension\VecLInftyNorm{L_1}\VecLOneNorm{\sigma}^{1/2}\Delta^{1/4}}T^{1/4}\right)~.
\end{align}

Finally, taking $\delta^{\prime} = \frac{\delta}{3dT}$, we obtain the stated result.
\end{proof}

\section{Experiments Comparing our Generalized SignSGD with Others}
\label{sec:generalized_sgd_experiments}
To validate the efficacy of our Algorithm~\ref{alg:generalized_signsgd}, we compare it with Adam~\citep{KingmaB15}, SGD~\citep{robbins1951stochastic}, SGDClipGrad, and SGDClipMomentum. The latter two are from Algorithm 1 in~\citep{ZhangJFW20} where SGDClipGrad corresponds to the case where $\nu = 0$ and SGDClipMomentum corresponds to the case when $\nu = 1$.

\textbf{Training} Unless otherwise specified, we use grid-search to fine-tune the initial step size for all optimizers, as well as the clipping threshold for SGDClipGrad and SGDClipMomentum, and $\beta_2$ for Adam and our algorithm, to select the one giving the best validation performance on a separated validation set. We then employ the best performing hyperparameters to train the model over all training data and report the testing performance. The testing is repeated with random seeds 5 times to eliminate the influence of stochasticity.

\begin{table}[t]
    \centering
    \caption[Hyperparameter grid search range and final choices for training a 20-layer ResNet on CIFAR-10]{Hyperparameter grid search ranges and choices yielding the highest validation accuracy for each optimizer for training a 20-layer Resnet to do image classification on CIFAR-10. ({"lr" denotes the initial step size, "clip" denotes the clipping parameter $\gamma$ in Algorithm 1 of~\citet{ZhangJFW20}, and "$\beta_2$" is defined in~\citet{KingmaB15} for Adam and in Algorithm~\ref{alg:generalized_signsgd} for ours.})}
    \label{tab:cifar10_hp}
    {\scriptsize
    \begin{tabular}{|c|c|c|}
        \hline
        Optimizer & Grid Search Range & Best Choice \\
        \hline
        SGD Momentum & lr \{1e-5, 0.0001, 0.001, 0.01, 0.05, 0.07, 0.1, 0.2, 0.3, 1, 10\} & lr=0.07\\
        \hline
        SGDClipGrad & \makecell{lr \{0.001, 0.01, 0.05, 0.1, 0.5, 1, 10\}\\ clip \{0.1, 1, 10\}} & \makecell{lr=0.5\\ clip=1} \\
        \hline
        SGDClipMomentum & \makecell{lr \{0.001, 0.01, 0.1, 1, 5, 10, 20, 50\}\\ clip \{0.01, 0.1, 1, 10\}} & \makecell{lr=10\\ clip=0.1} \\
        \hline
        Adam & \makecell{lr \{1e-5, 0.0001, 0.0007, 0.0009, 0.001, 0.002, 0.003, 0.01, 0.1\} \\ $\beta_2$ \{0.4, 0.8, 0.999\}} & \makecell{lr=0.0009\\ $\beta_2$=0.999} \\
        \hline
        Our Algorithm~\ref{alg:generalized_signsgd} & \makecell{lr \{5e-5, 8e-5, 0.0001, 0.0002, 0.0005, 0.001, 0.01\} \\ $\beta_2$ \{0.4, 0.8, 0.999\}} & \makecell{lr=0.0002\\ $\beta_2$ = 0.999} \\
        \hline
    \end{tabular}
    }
\end{table}

\textbf{Hyperparameter Tuning} During the validation stage, we used grid-search to fine-tune respective hyperparameters and choose the ones that yield the best validation results.
We tuned the hyperparameters using the following two-stage grid searching strategy: First, search over a coarse grid, and select the one yielding the best validation result. Next, continue searching in a fine grid centering at the best-performing hyperparameters found in the coarse stage, and in turn, take the best one as the final choice. Also, whenever the best-performing hyperparameters lie in the boundary of the searching grid, we always extend the grid to make the final best-performing hyperparameters fall into the interior of the grid, if possible.

\begin{table}[t]
    \centering
    \caption[Hyperparameter grid search range and final choices for training an AWD-LSTM on Penn Treebank]{Hyperparameter grid search ranges and choices yielding the lowest validation perplexity for each optimizer on training an AWD-LSTM to do language modeling on Penn Treebank. ({"wd" denotes the weight decay value, "lr" denotes the initial step size, "clip" denotes the clipping parameter $\gamma$ in Algorithm 1 of~\citet{ZhangJFW20}, and "$\beta_2$" is defined in~\citet{KingmaB15} for Adam and in Algorithm~\ref{alg:generalized_signsgd} for ours.})}
    \label{tab:ptb_hp}
    {\small{
    \begin{tabular}{|c|c|c|}
        \hline
        Optimizer & Grid Search Range & Best Choice \\
        \hline
        SGD Momentum & \makecell{wd \{1e-7, 1.2e-6, 5e-6, 1e-5, 1e-4, 1e-3\}\\ lr \{0.001, 0.01, 0.1, 0.5, 0.8, 1, 2, 4, 5\}} & \makecell{wd=1e-5 \\ lr=1}\\
        \hline
        SGDClipGrad & \makecell{wd \{1e-7, 1.2e-6, 5e-6, 1e-5\}\\ lr \{0.1, 0.5, 1, 5, 10, 20, 30, 40, 50, 60, 70\}\\ clip \{1, 2.5, 7.5, 10, 15, 20\}} & \makecell{wd=1.2e-6 \\ lr=50 \\ clip=10}\\
        \hline
        SGDClipMomentum & \makecell{wd \{1e-7, 1.2e-6, 5e-6, 1e-5\}\\ lr \{5, 10, 20, 30, 50, 100\}\\ clip \{1, 2.5, 7.5\}} & \makecell{wd=1.2e-6 \\ lr=20 \\ clip=2.5} \\
        \hline
        Adam & \makecell{wd \{1e-7, 1.2e-6, 5e-6, 1e-5\}\\ lr \{0.0001, 0.001, 0.002, 0.003, 0.01, 0.1\}\\ $\beta_2$ \{0.4, 0.8, 0.999\}} & \makecell{wd=5e-6 \\ lr=0.002 \\ $\beta_2$=0.999} \\
        \hline
        Our Algorithm~\ref{alg:generalized_signsgd} & \makecell{wd \{1e-7, 1.2e-6, 5e-6, 1e-5\}\\ lr \{0.0001, 0.001, 0.002, 0.003, 0.01, 0.1\}\\ $\beta_2$ \{0.4, 0.8, 0.999\}} & \makecell{wd=1.2e-6 \\ lr=0.001 \\ $\beta_2$=0.999}  \\
        \hline
    \end{tabular}
    }}
\end{table}

\begin{figure}[t]
    \centering
    \includegraphics[width=\textwidth]{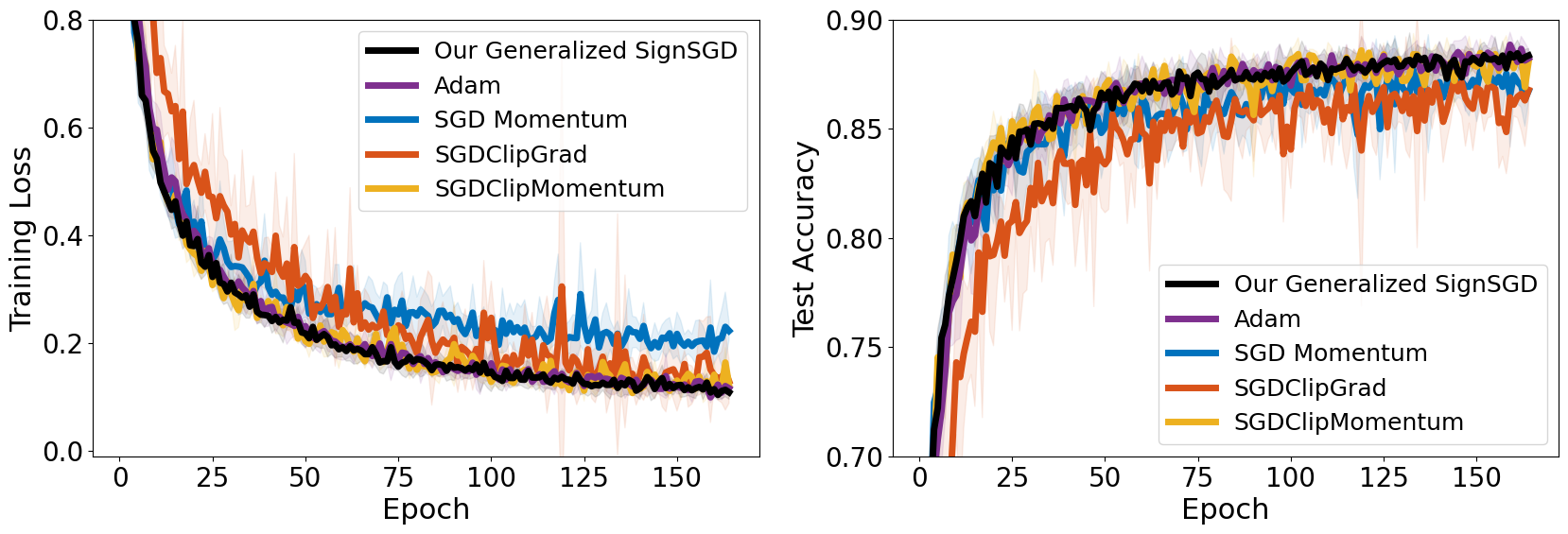}
    \caption[Comparison of our Generalized SignSGD with other optimizers on training a 20-layer Resnet on CIFAR10.]{Training a 20-layer Resnet on CIFAR10. The shading of each curve represents the 95\% confidence interval computed across $5$ independent runs from different random seeds.}
    \label{fig:cifar10}
\end{figure}

\begin{figure}[t]
    \centering
    \includegraphics[width=\textwidth]{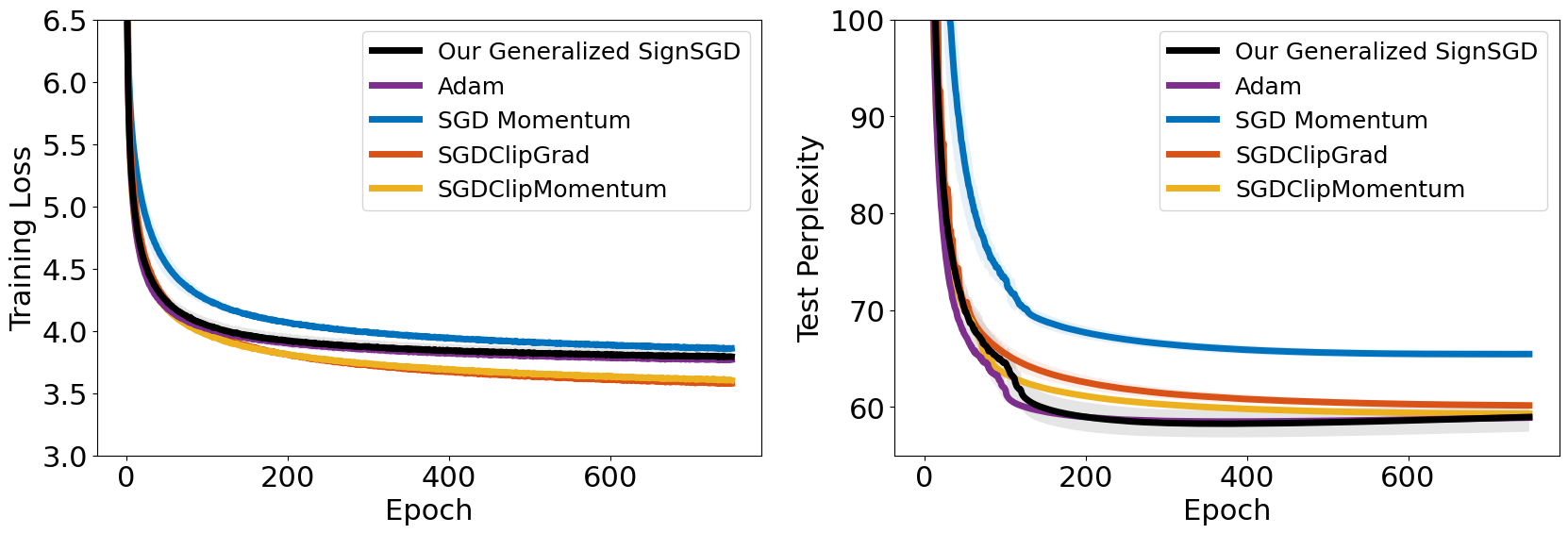}
    \caption[Comparison of our Generalized SignSGD with other optimizers on taining an AWD-LSTM on Penn Treebank.]{Training an AWD-LSTM model to do language modeling (word level) on Penn Treebank. The shading of each curve represents the 95\% confidence interval computed across five independent runs from different random seeds.}
    \label{fig:penntreebank}
\end{figure}

\textbf{Resnet for Image Classification on CIFAR-10}
We employ the 20-layer Residual Network model~\citep{HeZRS16} to do image classification on the CIFAR-10 dataset. Images are normalized per channel using the means and standard deviations
computed from all training images. We adopt the data augmentation technique following~\citep{LeeXGZT15} (for training only): 4 pixels are padded on each side of an image and a 32 × 32 crop is randomly sampled from the padded image or its horizontal flip. The mini-batch size is $128$ and we train all algorithms for $164$ epochs. We do not employ any step size decay schedule in order to focus on the comparison of the optimizers themselves. We fixed the weight decay value to be $0.0001$ and the momentum parameter ($\beta_1)$ to be $0.9$.
We randomly selected $10\%$ images from the training dataset for validation. Yet, during testing, we trained on the whole training dataset.
The detailed search ranges and the hyperparameter choices yielding the highest validation accuracy for each optimizer are listed in Table~\ref{tab:cifar10_hp}.
Figure~\ref{fig:cifar10} and Table~\ref{tab:generalized_signsgd_results} report the training and testing performance for each algorithm, showing that our algorithm closely matches Adam and is among the best of all.

\textbf{LSTM for Language Modeling on Penn Treebank}
We adopt a 3-layer AWD-LSTM \citep{merity2018regularizing} to do language modeling on the Penn Treebank (PTB) dataset \citep{marcus1993building}(word level). The mini-batch size is $40$ and we trained each algorithm for $750$ epochs.
We used the original train-validation-test split that comes with the dataset.
Apart from the hyperparameters we stated above, we fixed The momentum parameter ($\beta_1)$ to be $0.9$ except for SGDClipGrad which does not use momentum, and further fine-tuned the weight decay value for all algorithms as we noticed its significant influence on the performance. We choose the set of hyperparameters that give the smallest final validation perplexity. The detailed search ranges and the hyperparameter choices yielding the lowest validation perplexity for each optimizer are listed in Table~\ref{tab:ptb_hp}.
We report the results in Figure~\ref{fig:penntreebank} and Table~\ref{tab:generalized_signsgd_results}. It can be seen that we can match the performance of Adam while beating the others.

\begin{table*}[t]
\caption[Final results of experiments comparing our generalized SignSGD algorithm with other optimizers]{Average final training loss and test accuracy achieved by each method when optimizing respective models on each dataset. The $\pm$ shows $95\%$ confidence intervals of the mean loss/accuracy/perplexity value over 5 runs starting from different random seeds.}
\label{tab:generalized_signsgd_results}
\centering
{\scriptsize
\begin{tabular}{|c|c|c|c|c|}
\hline
\multirow{2}{*}{Methods} & \multicolumn{2}{c|}{CIFAR10} & \multicolumn{2}{c|}{Penn Treebank}\\
\cline{2-5}
& Training loss & Test accuracy & Training loss & Test perplexity \\
\hline
SGD Momentum & $0.2226 \pm 0.0169$ & $0.8674 \pm 0.0048$ & $3.8587 \pm 0.0058$ & $65.4622 \pm 0.3842$\\
\hline
SGDClipGrad & $0.1288 \pm 0.0403$ & $0.8677 \pm 0.0106$ & $\mathbf{3.5774 \pm 0.0081}$ & $60.1604 \pm 0.2797$\\
\hline
SGDClipMomentum & $0.1220 \pm 0.0162$ & $0.8809 \pm 0.0022$ & $3.6038 \pm 0.0102$ & $59.3052 \pm 0.2798$\\
\hline
Adam & $0.1161 \pm 0.0111$ & $0.8823 \pm 0.0041$ & $3.7692 \pm 0.0062$ & $\mathbf{58.9005 \pm 0.3058}$\\
\hline
Our Algorithm~\ref{alg:generalized_signsgd} & $\mathbf{0.1086 \pm 0.0129}$ & $\mathbf{0.8835 \pm 0.0032}$ & $3.7928 \pm 0.0425$ & $58.9661 \pm 1.5218$\\
\hline
\end{tabular}}
\end{table*}

\section{Conclusion}
\label{sec:generalized_sgd_conclusion}
Smoothness has been a widely adopted condition for proving convergence rates of algorithms in the non-convex optimization scenario. Yet, it has been found that this assumption does not capture losses when employing some deep learning models including RNNs and LSTMs. In light of this, a relaxed smoothness assumption was proposed that aligns well with the practice. We observed that the loss surface of training using Transformers also exhibits this relaxed smoothness. Under this assumption, SGD with clipped gradient has been proven to work well. However, we found that clipping is not necessary for achieving convergence in such a setting. Indeed, we showed that a generalized SignSGD algorithm does not require explicit clipping but can almost guarantee the same bound as SGD with clipping. In the analyses, we identified the key effect of using momentum in analyzing Adam-type algorithms, that it reduces both the noise and the unbounded gradient norms. Finally, we conducted a variety of deep learning tasks showing that our algorithm can match Adam's performance while exceeding others.

\textbf{Limitations} The current work is in no way a perfect one and there are many directions worth exploring beyond it. First of all, though our algorithm could be seen as a close resemblance to the original Adam algorithm, they are still not equal. Considering the huge popularity of Adam and its established effectivity in practice, it is worth studying whether Adam in its original form can converge in the relaxed smooth setting. Second, while our Theorem~\ref{thm:generalized_signsgd} are upper bounds and cannot be directly compared between the two cases of $\beta_2$, it does suggests that $\beta_2 = 0$ minimizes the worst-case convergence rate. However, it still does not fully explain the phenomenon that a choice of $\beta_2$ close to $1$ yields better performance in using our Algorithm~\ref{alg:generalized_signsgd} as well as Adam in practice. Third, despite there are lower bounds showing that, for example, GD with a constant step size can be arbitrarily worse than GD with clipping, it would be more meaningful to study whether the relaxed smooth condition is inherently more difficult, possibly by establishing a lower bound for all first-order optimization algorithms. Fourth, we did show that Transformers observe the relaxed smoothness condition, but we consider it more beneficial to research in-depth what properties or structures make a model satisfy such conditions. Finally, when conducting our experiments, we observed that the weight decay value plays a prominent role in each optimizer's performance, and that the best weight decay value varies for different optimizers. Thus, one potential direction would be to explore different ways of incorporating the regularization in a way to preserve the scale-freeness~\citep{orabona2015scale, OrabonaP18} of Algorithm~\ref{alg:generalized_signsgd}, just as AdamW does~\citep{ZhuangLCO21}.

%% file: 6_Conclusions/conclusions.tex
\chapter{Conclusions}
\label{chapter:Conclusions}
\thispagestyle{myheadings}

Non-convex optimization problems have been attracting much attention in recent years, especially with the successes of deep neural networks. A dominant optimization algorithm for such a scenario is gradient descent/stochastic gradient descent which necessitates specifying a parameter called the step size. This parameter plays a critical role in the performance of GD/SGD and often requires very careful tuning for each problem individually. The tuning is notoriously tedious and time and resources consuming. To ease this burden, adaptive algorithms are proposed which can automatically guarantee near-optimal convergence even without knowledge of certain properties of the objective problem.

In this dissertation, we discussed our work on studying adaptive strategies in non-convex optimization in three scenarios. We first designed and analyzed algorithms that can adapt to the level of noise in evaluating stochastic gradients in the general smooth non-convex setting and the setting with an additional PL condition. We then addressed the scenario when gradient scales can vary drastically across layers/coordinates, identified scenes when Adam performs worse than AdamW, and unearthed the correlation between AdamW's advantage and its scale-freeness property which makes it adaptive to the gradient scales. Finally, we attacked the relaxed smoothness setting by reporting empirical evidence showing that Transformers observe such a condition and introducing a generalized SignSGD algorithm that incorporates the empirical excellence of Adam and the theoretical merits of SGD with gradient clipping including the adaptivity to the smoothness parameter.

Nevertheless, there are still abundant directions for studying adaptivity in non-convex optimization that is yet to be explored. We believe the progress in this topic of adaptivity would be really beneficial to the field of non-convex optimization as adaptivity means less tuning needed which leads to faster exploration and iteration on making novel findings and on applying existing methods to new tasks.

%% file: Appendix/Appendix.tex
\chapter{Supporting Materials}
\label{chapter:support_materials}
\thispagestyle{myheadings}

\section{Omitted Proofs of Lemmas and Theorems}
\label{sec:omit_proofs}

\begin{proof}[Proof of Lemma~\ref{lem:no_local_minimum_convex}]
We will prove by contradiction. Suppose there exists a global minimum point $\bx^*$ and a local minimum point $\tilde{\bx}$ with $F(\tilde{\bx}) > F(\bx^*)$.

From the definition of a local minimum point, there exists a neighborhood $\mathcal{N}\subset\mathcal{X}$ around $\tilde{\bx}$ such that $F(\bx) \ge F(\tilde{\bx})$ for any $\bx\in\mathcal{N}$.

Consider a point $\by = \theta\tilde{\bx} + (1-\theta)\bx^*$ for some $\theta$ with $0\le\theta\le1$. From the definition of a convex function, $\mathcal{X}$ is a convex set thus $\by\in\mathcal{X}$. Also, we have
\begin{equation}
    F(\by) \le \theta F(\tilde{\bx}) + (1-\theta)F(\bx^*) < \theta F(\tilde{\bx}) + (1-\theta)F(\tilde{\bx}) = F(\tilde{\bx})~.
\end{equation}

Pick $\theta$ to be sufficiently close to $1$ such that $\by\in\mathcal{N}$, we 
get a contradiction.
\end{proof}

\begin{proof}[Proof of Theorem~\ref{thm:pgd_convex}]
\begin{align}
    F(\bar{\bx}) - F(\bx^*)
    &=
    F\left(\frac{1}{T}\sum^T_{t=1}\bx_t\right) - F(\bx^*)\\
    &\le
    \frac{1}{T}\sum^T_{t=1}F\left(\bx_t\right) - F(\bx^*)\\
    &\le
    \frac{1}{T}\sum^T_{t=1}\langle\nabla F(\bx_t), \bx_t - \bx^*\rangle\\
    &=
    \sum^T_{t=1}\frac{\|\bx_t - \bx^*\|_2^2 - \|\bx_t - \eta\nabla F(\bx_t) - \bx^*\|_2^2 + \eta^2\|\nabla F(\bx_t)\|_2^2}{2\eta T}\\
    &\le
    \sum^T_{t=1}\frac{\|\bx_t - \bx^*\|_2^2 - \|\bx_{t+1} - \bx^*\|_2^2 + \eta^2\|\nabla F(\bx_t)\|_2^2}{2\eta T}\\
    &=
    \frac{\|\bx_1 - \bx^*\|_2^2 - \|\bx_{T+1} - \bx^*\|_2^2}{2\eta T} + \frac{\eta}{2T}\sum^T_{t=1}\|\nabla F(\bx_t)\|_2^2\\
    &\le
    \frac{D^2}{2\eta T} + \frac{\eta}{2T}\sum^T_{t=1}\|\nabla F(\bx_t)\|_2^2,
\end{align}
where the first inequality uses the Jensen's inequality for convex functions~\citep[Section 3.1.8]{BoydV04}, the second inequality uses the convexity of $F$, the third one uses the projection lemma~\citep[Lemma 14.9]{shalev2014understanding}, and the last one uses the bounded domain assumption.
\end{proof}

\section{The Histograms of Update Scales of each Coordinate during the Entire Training Phase of Adam vs.~AdamW}
\label{sec:hist_entire_train}
In this section, we report the histograms of the absolute value of updates of Adam-$\ell_2$ vs.~AdamW of all coordinates divided by the initial step size $\alpha$ during the whole training process. From the figures shown below, we can clearly see that AdamW's updates remain in a much more concentrated scale range than Adam-$\ell_2$ during the entire training. Moreover, as the depth of the network grows, Adam-$\ell_2$'s updates become more and more dispersed, while AdamW's updates are still concentrated. \emph{(Note that the leftmost bin contains all values equal to or less than $2^{-27}\approx10^{-8.1}$ and the rightmost bin contains all values equal to or larger than $1$.)}

\begin{figure}[t]
\centering
\begin{subfigure}{0.4\textwidth}
\includegraphics[width=\linewidth]{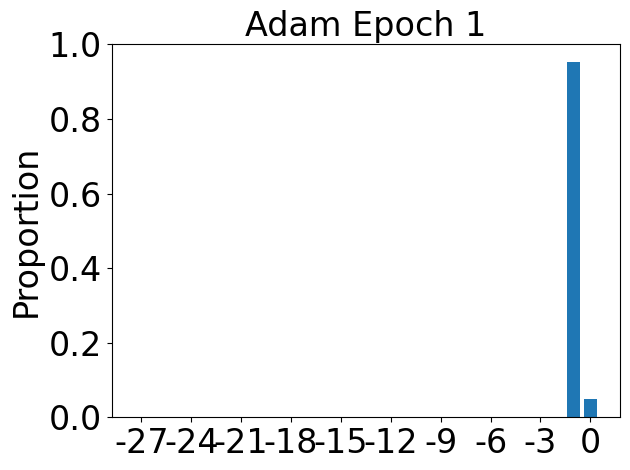}
\includegraphics[width=\linewidth]{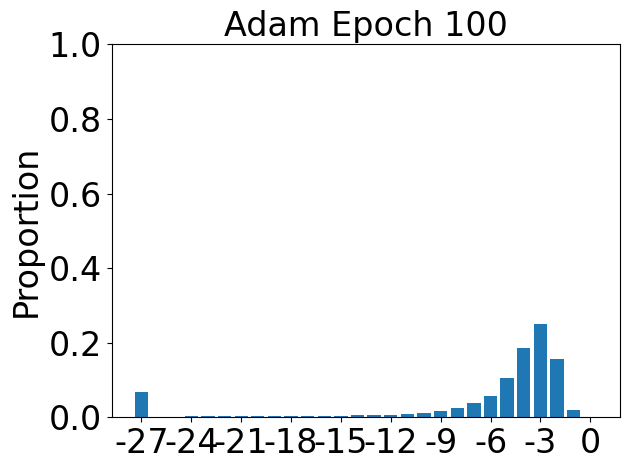}
\includegraphics[width=\linewidth]{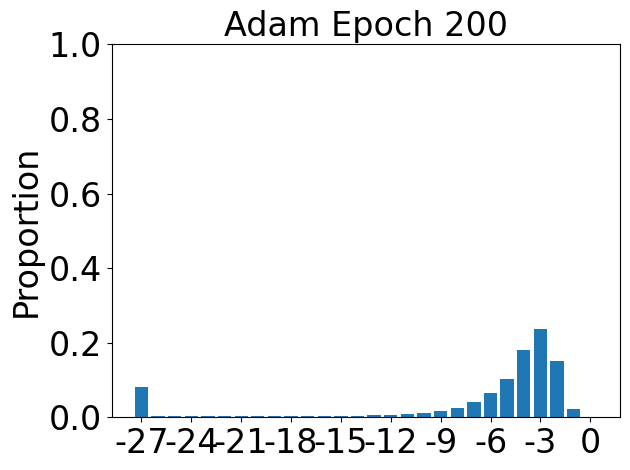}
\includegraphics[width=\linewidth]{figs/adapt_scale/histograms_updates/CIFAR10_resnet20_nobn/magnitude_histogram_CIFAR10_resnet20_nobn_Adam_update_no_alpha_Epoch_299.png}
\caption{Adam-$\ell_2$}
\label{fig:cifar10_resnet20_nobn_adam}
\end{subfigure}
\hspace{0.02\textwidth}
\begin{subfigure}{0.4\textwidth}
\includegraphics[width=\linewidth]{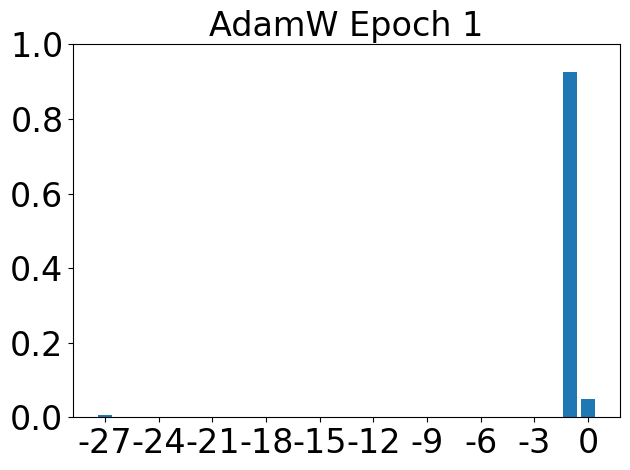}
\includegraphics[width=\linewidth]{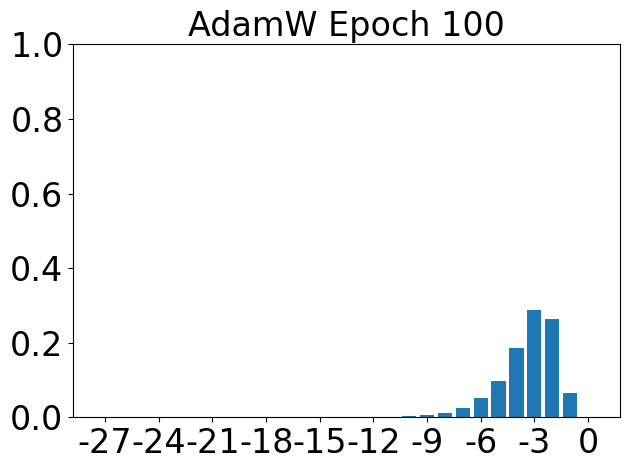}
\includegraphics[width=\linewidth]{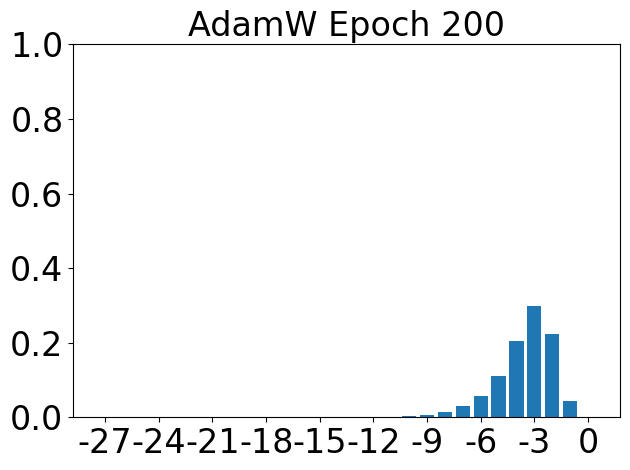}
\includegraphics[width=\linewidth]{figs/adapt_scale/histograms_updates/CIFAR10_resnet20_nobn/magnitude_histogram_CIFAR10_resnet20_nobn_AdamW_update_no_alpha_Epoch_299.png}
\caption{AdamW}
\label{fig:cifar10_resnet20_nobn_adamw}
\end{subfigure}
\caption{The histograms of the magnitudes of all updates of a 20-layer Resnet with BN disabled trained by AdamW or Adam-$\ell_2$  on CIFAR10.}
\label{fig:cifar10_resnet20_nobn_histo}
\end{figure}

\begin{figure}[t]
\centering
\begin{subfigure}{0.4\textwidth}
\includegraphics[width=\linewidth]{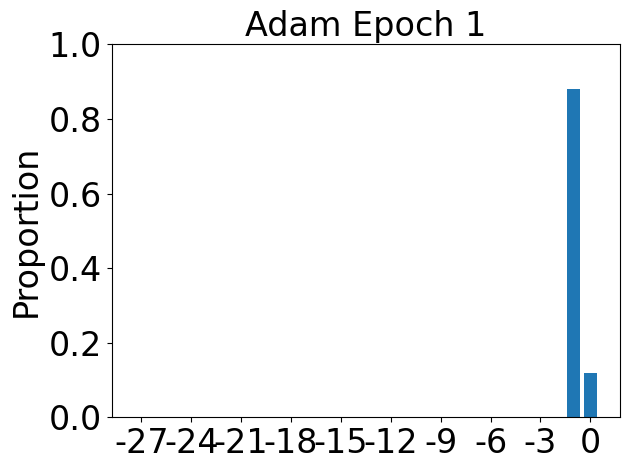}
\includegraphics[width=\linewidth]{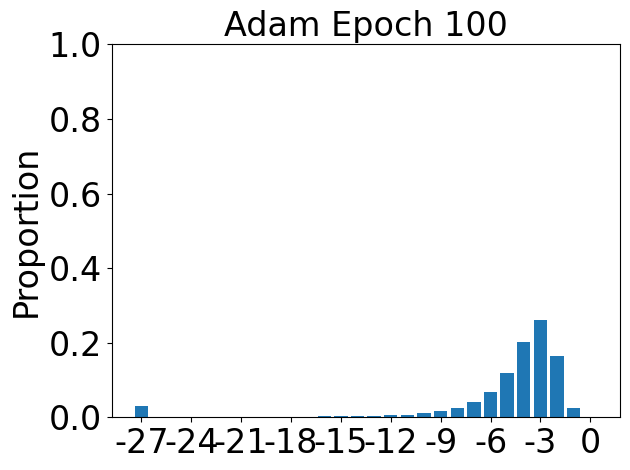}
\includegraphics[width=\linewidth]{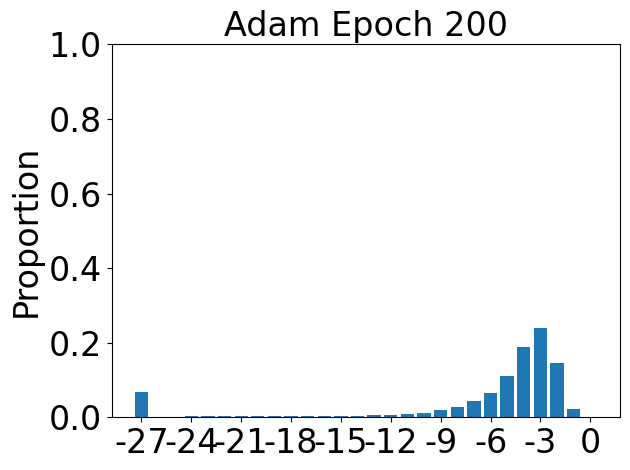}
\includegraphics[width=\linewidth]{figs/adapt_scale/histograms_updates/CIFAR10_resnet44_nobn/magnitude_histogram_CIFAR10_resnet44_nobn_Adam_update_no_alpha_Epoch_299.png}
\caption{Adam-$\ell_2$}
\label{fig:cifar10_resnet44_nobn_adam}
\end{subfigure}
\hspace{0.02\textwidth}
\begin{subfigure}{0.4\textwidth}
\includegraphics[width=\linewidth]{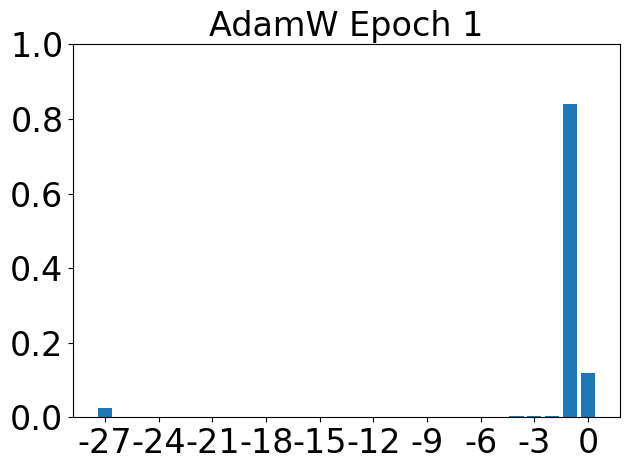}
\includegraphics[width=\linewidth]{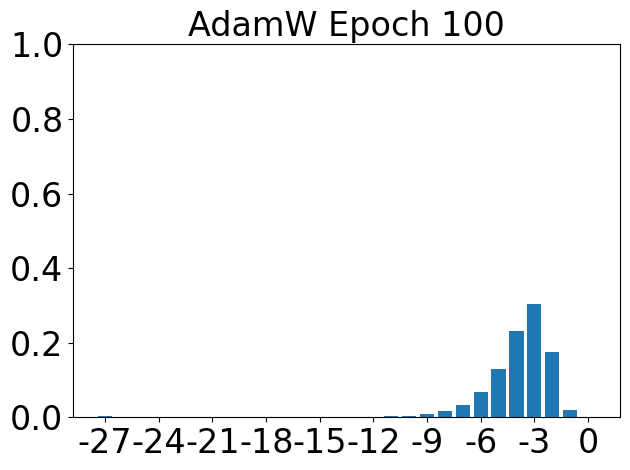}
\includegraphics[width=\linewidth]{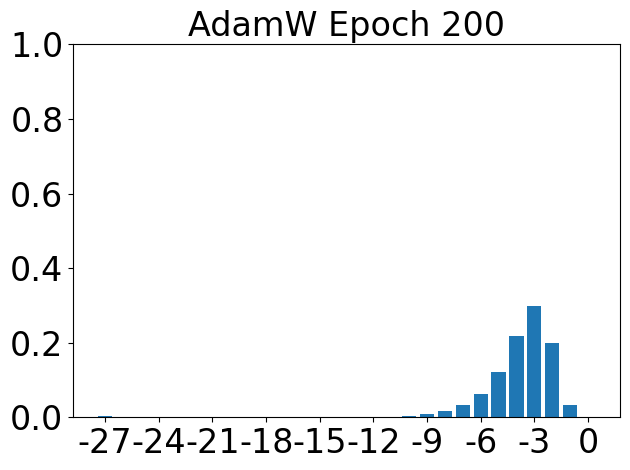}
\includegraphics[width=\linewidth]{figs/adapt_scale/histograms_updates/CIFAR10_resnet44_nobn/magnitude_histogram_CIFAR10_resnet44_nobn_AdamW_update_no_alpha_Epoch_299.png}
\caption{AdamW}
\label{fig:cifar10_resnet44_nobn_adamw}
\end{subfigure}
\caption{The histograms of the magnitudes of all updates of a 44-layer Resnet with BN disabled trained by AdamW or Adam-$\ell_2$  on CIFAR10.}
\label{fig:cifar10_resnet44_nobn_histo}
\end{figure}

\begin{figure}[t]
\centering
\begin{subfigure}{0.4\textwidth}
\includegraphics[width=\linewidth]{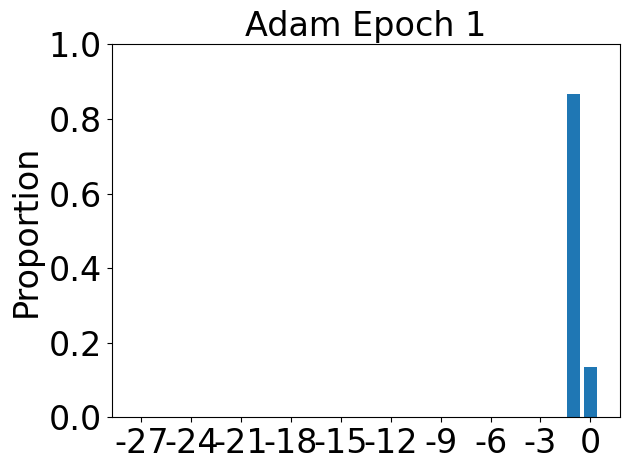}
\includegraphics[width=\linewidth]{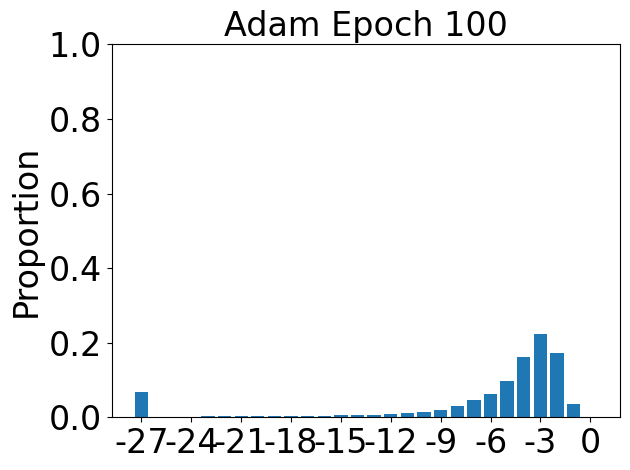}
\includegraphics[width=\linewidth]{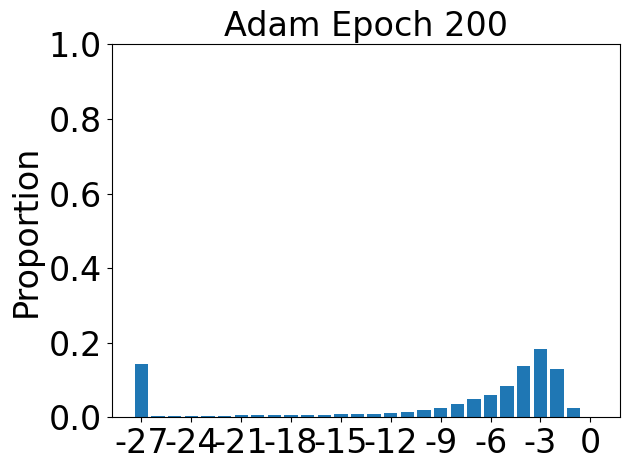}
\includegraphics[width=\linewidth]{figs/adapt_scale/histograms_updates/CIFAR10_resnet56_nobn/magnitude_histogram_CIFAR10_resnet56_nobn_Adam_update_no_alpha_Epoch_299.png}
\caption{Adam-$\ell_2$}
\label{fig:cifar10_resnet56_nobn_adam}
\end{subfigure}
\hspace{0.02\textwidth}
\begin{subfigure}{0.4\textwidth}
\includegraphics[width=\linewidth]{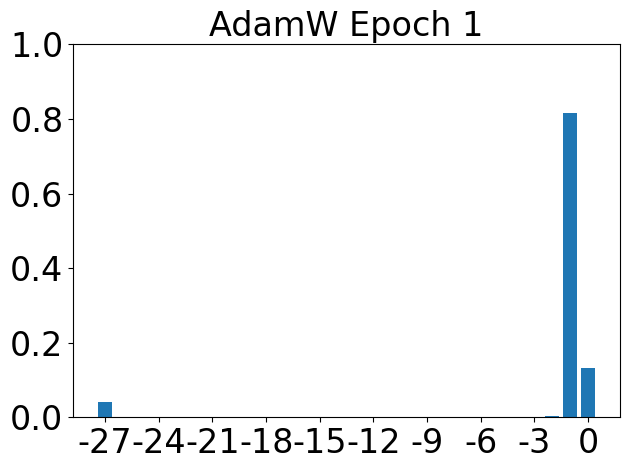}
\includegraphics[width=\linewidth]{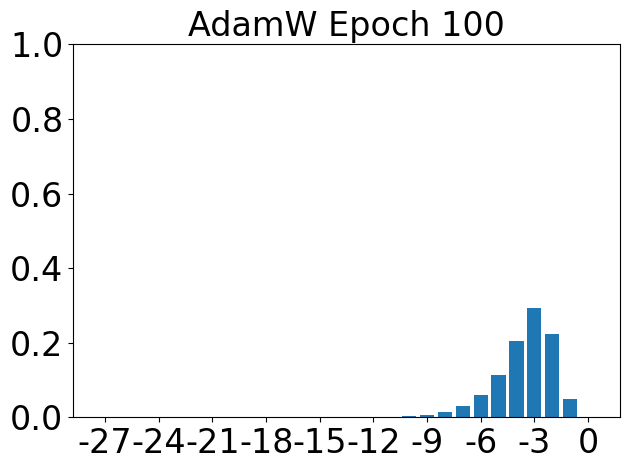}
\includegraphics[width=\linewidth]{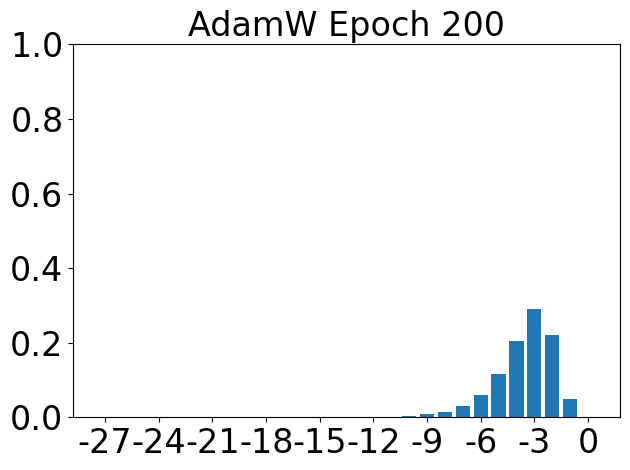}
\includegraphics[width=\linewidth]{figs/adapt_scale/histograms_updates/CIFAR10_resnet56_nobn/magnitude_histogram_CIFAR10_resnet56_nobn_AdamW_update_no_alpha_Epoch_299.png}
\caption{AdamW}
\label{fig:cifar10_resnet56_nobn_adamw}
\end{subfigure}
\caption{The histograms of the magnitudes of all updates of a 56-layer Resnet with BN disabled trained by AdamW or Adam-$\ell_2$  on CIFAR10.}
\label{fig:cifar10_resnet56_nobn_histo}
\end{figure}

\begin{figure}[t]
\centering
\begin{subfigure}{0.4\textwidth}
\includegraphics[width=\linewidth]{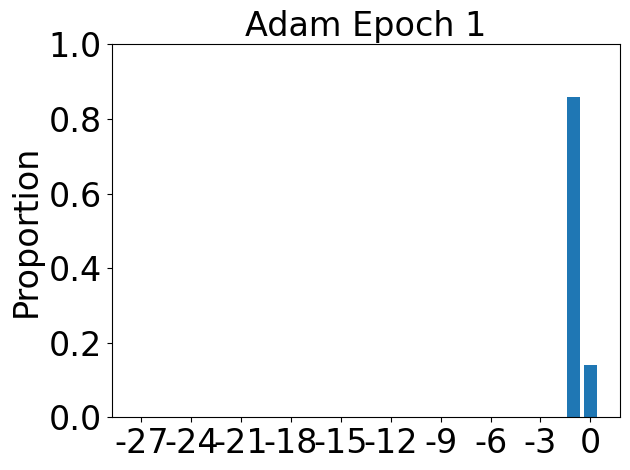}
\includegraphics[width=\linewidth]{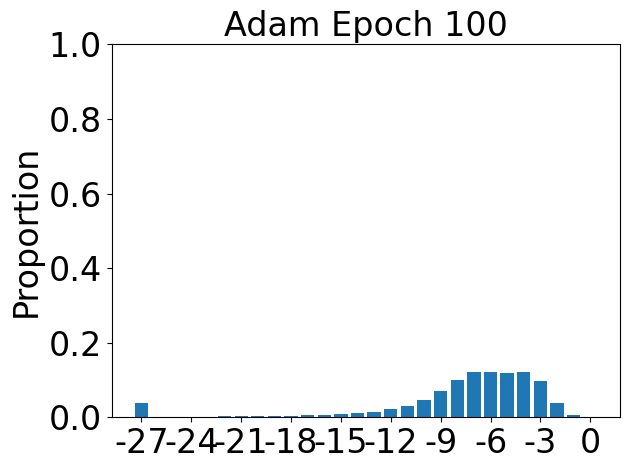}
\includegraphics[width=\linewidth]{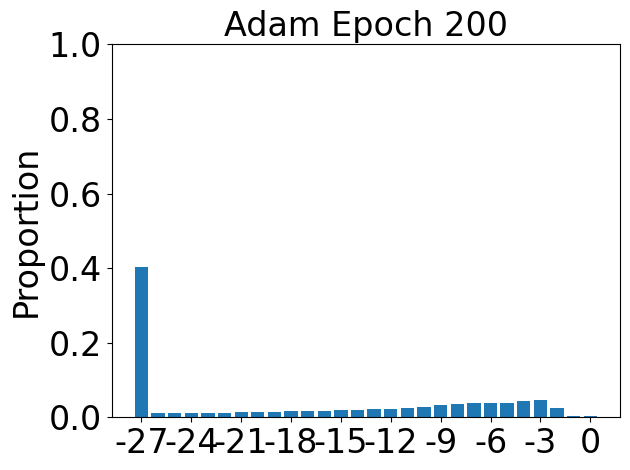}
\includegraphics[width=\linewidth]{figs/adapt_scale/histograms_updates/CIFAR10_resnet110_nobn/magnitude_histogram_CIFAR10_resnet110_nobn_Adam_update_no_alpha_Epoch_299.png}
\caption{Adam-$\ell_2$}
\label{fig:cifar10_resnet110_nobn_adam}
\end{subfigure}
\hspace{0.02\textwidth}
\begin{subfigure}{0.4\textwidth}
\includegraphics[width=\linewidth]{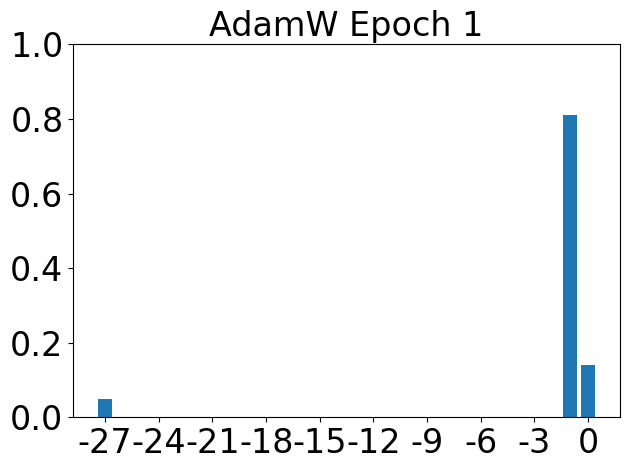}
\includegraphics[width=\linewidth]{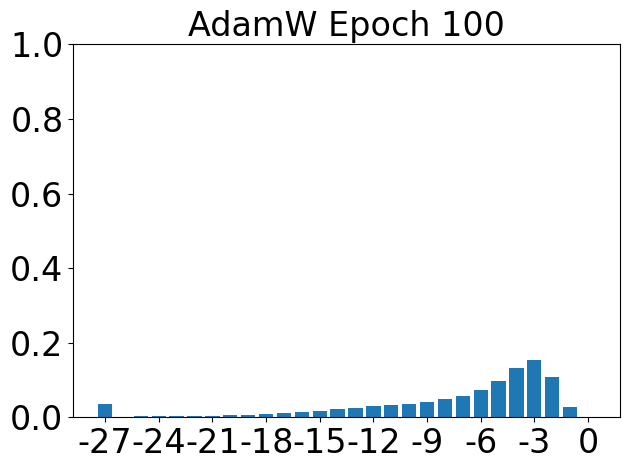}
\includegraphics[width=\linewidth]{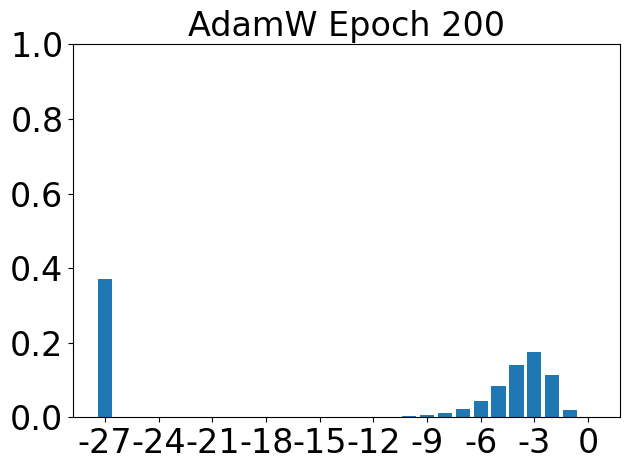}
\includegraphics[width=\linewidth]{figs/adapt_scale/histograms_updates/CIFAR10_resnet110_nobn/magnitude_histogram_CIFAR10_resnet110_nobn_AdamW_update_no_alpha_Epoch_299.png}
\caption{AdamW}
\label{fig:cifar10_resnet110_nobn_adamw}
\end{subfigure}
\caption{The histograms of the magnitudes of all updates of a 110-layer Resnet with BN disabled trained by AdamW or Adam-$\ell_2$  on CIFAR10.}
\label{fig:cifar10_resnet110_nobn_histo}
\end{figure}

\begin{figure}[t]
\centering
\begin{subfigure}{0.4\textwidth}
\includegraphics[width=\linewidth]{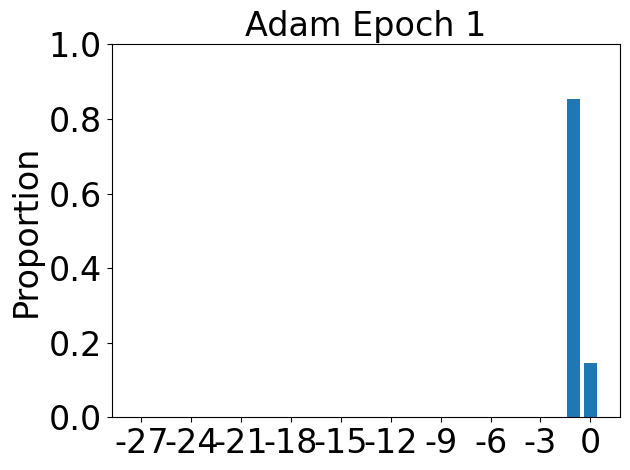}
\includegraphics[width=\linewidth]{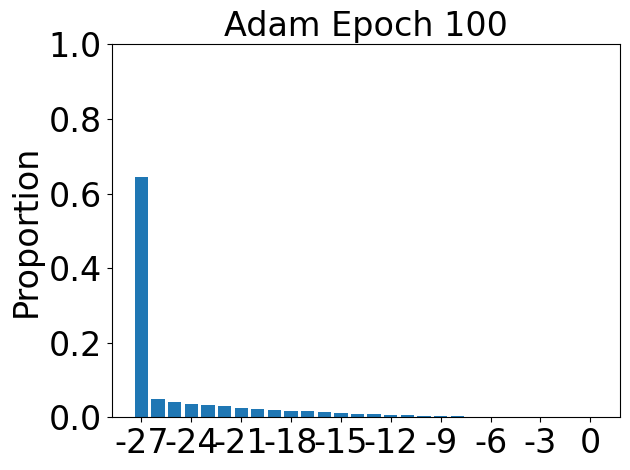}
\includegraphics[width=\linewidth]{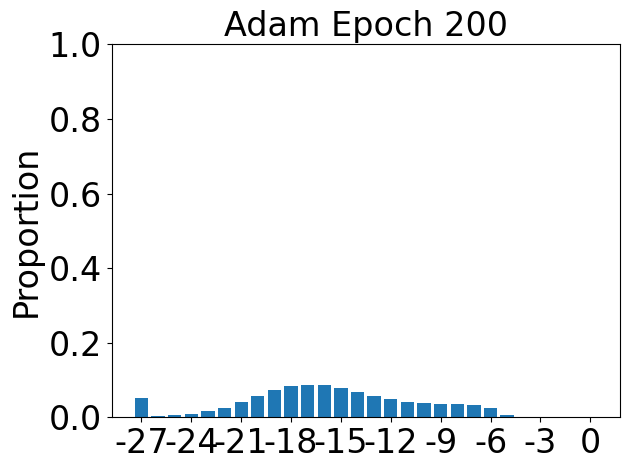}
\includegraphics[width=\linewidth]{figs/adapt_scale/histograms_updates/CIFAR10_resnet218_nobn/magnitude_histogram_CIFAR10_resnet218_nobn_Adam_update_no_alpha_Epoch_299.png}
\caption{Adam-$\ell_2$}
\label{fig:cifar10_resnet218_nobn_adam}
\end{subfigure}
\hspace{0.02\textwidth}
\begin{subfigure}{0.4\textwidth}
\includegraphics[width=\linewidth]{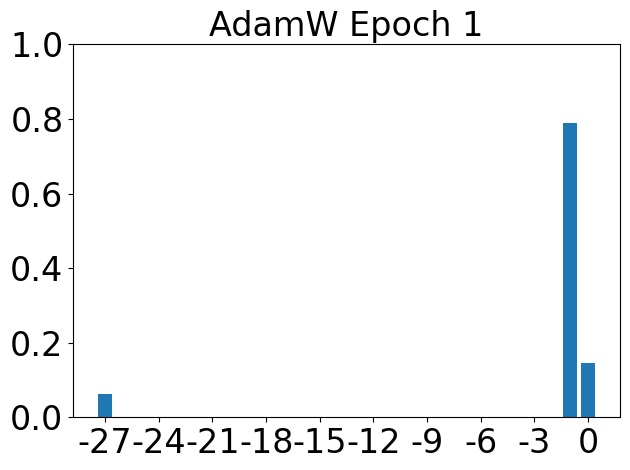}
\includegraphics[width=\linewidth]{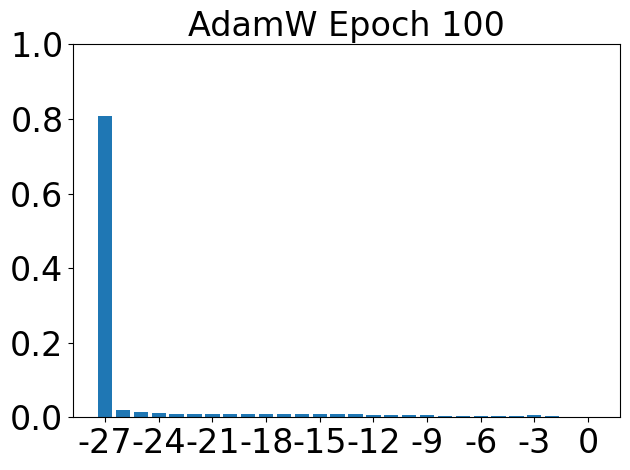}
\includegraphics[width=\linewidth]{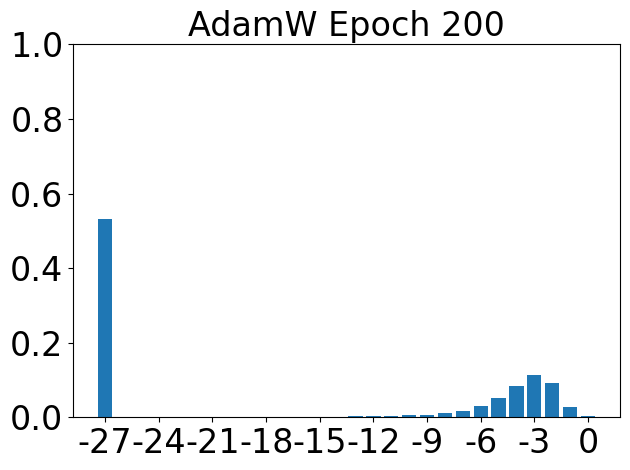}
\includegraphics[width=\linewidth]{figs/adapt_scale/histograms_updates/CIFAR10_resnet218_nobn/magnitude_histogram_CIFAR10_resnet218_nobn_AdamW_update_no_alpha_Epoch_299.png}
\caption{AdamW}
\label{fig:cifar10_resnet218_nobn_adamw}
\end{subfigure}
\caption{The histograms of the magnitudes of all updates of a 218-layer Resnet with BN disabled trained by AdamW or Adam-$\ell_2$  on CIFAR10.}
\label{fig:cifar10_resnet218_nobn_histo}
\end{figure}

\begin{figure}[t]
\centering
\begin{subfigure}{0.4\textwidth}
\includegraphics[width=\linewidth]{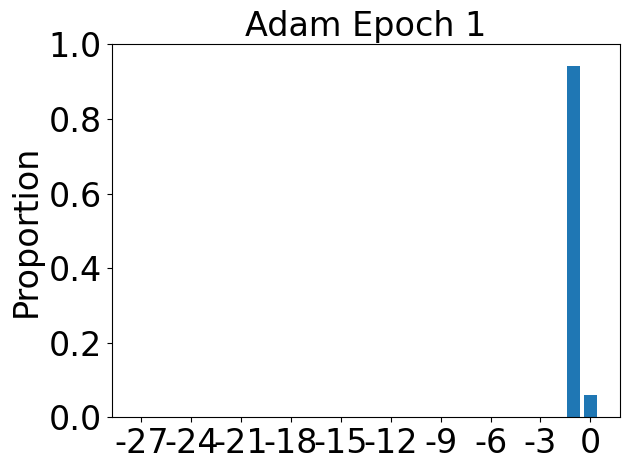}
\includegraphics[width=\linewidth]{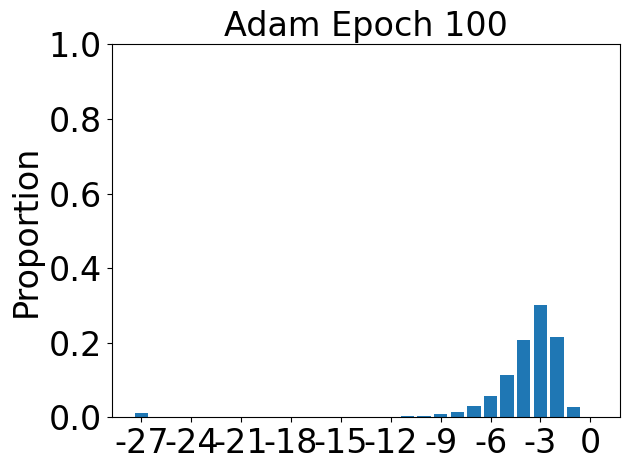}
\includegraphics[width=\linewidth]{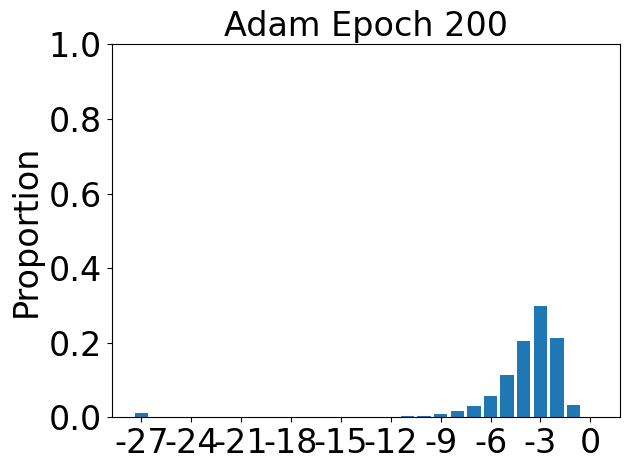}
\includegraphics[width=\linewidth]{figs/adapt_scale/histograms_updates/CIFAR100_resnet20_nobn/magnitude_histogram_CIFAR100_resnet20_nobn_Adam_update_no_alpha_Epoch_299.png}
\caption{Adam-$\ell_2$}
\label{fig:cifar100_resnet20_nobn_adam}
\end{subfigure}
\hspace{0.02\textwidth}
\begin{subfigure}{0.4\textwidth}
\includegraphics[width=\linewidth]{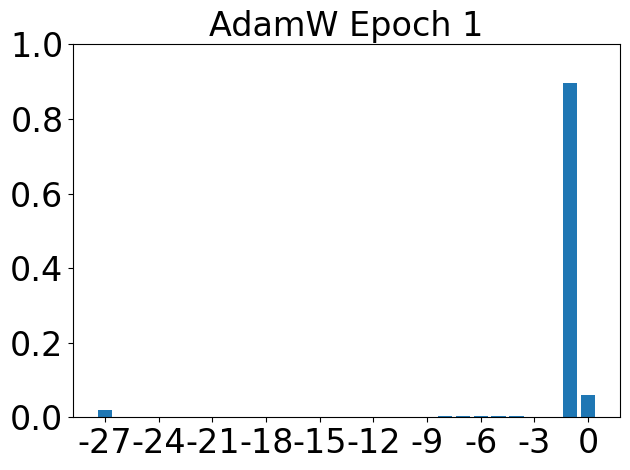}
\includegraphics[width=\linewidth]{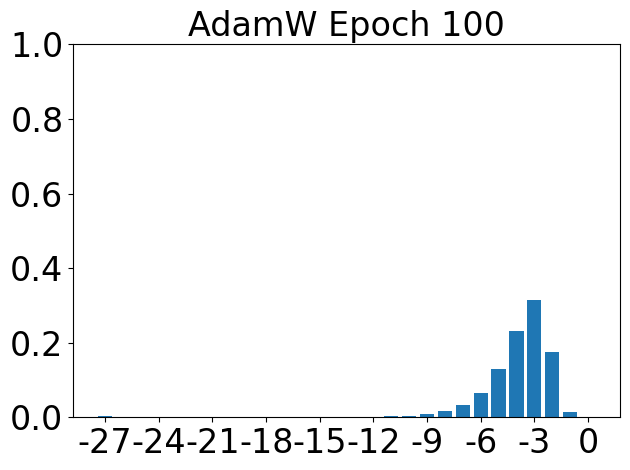}
\includegraphics[width=\linewidth]{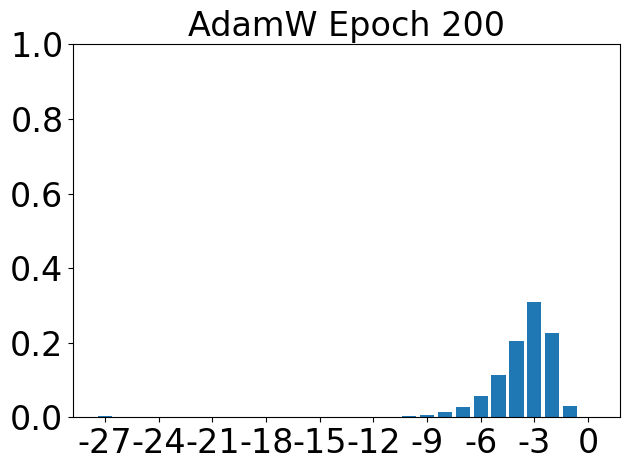}
\includegraphics[width=\linewidth]{figs/adapt_scale/histograms_updates/CIFAR100_resnet20_nobn/magnitude_histogram_CIFAR100_resnet20_nobn_AdamW_update_no_alpha_Epoch_299.png}
\caption{AdamW}
\label{fig:cifar100_resnet20_nobn_adamw}
\end{subfigure}
\caption{The histograms of the magnitudes of all updates of a 20-layer Resnet with BN disabled trained by AdamW or Adam-$\ell_2$  on CIFAR100.}
\label{fig:cifar100_resnet20_nobn_histo}
\end{figure}

\begin{figure}[t]
\centering
\begin{subfigure}{0.4\textwidth}
\includegraphics[width=\linewidth]{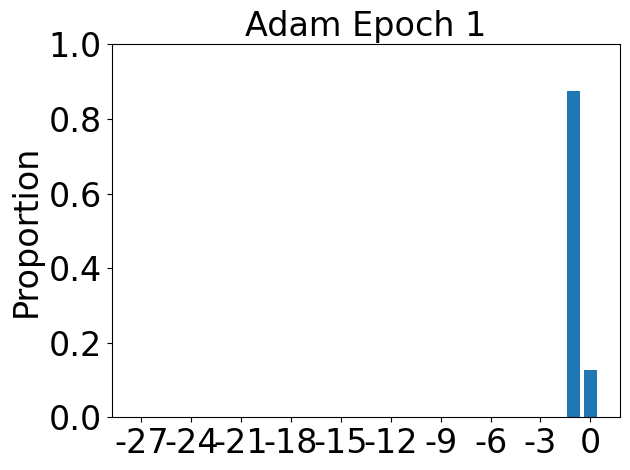}
\includegraphics[width=\linewidth]{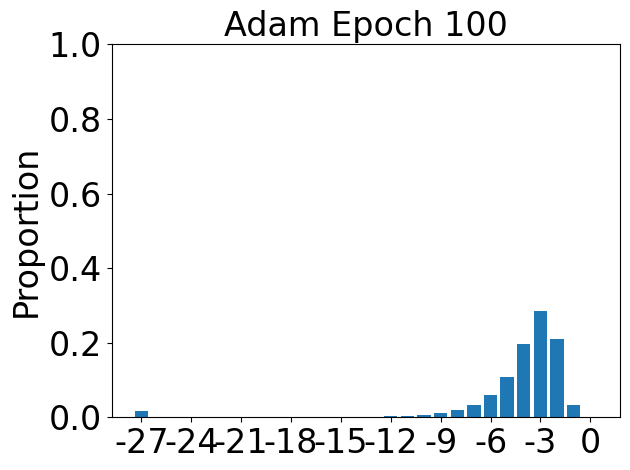}
\includegraphics[width=\linewidth]{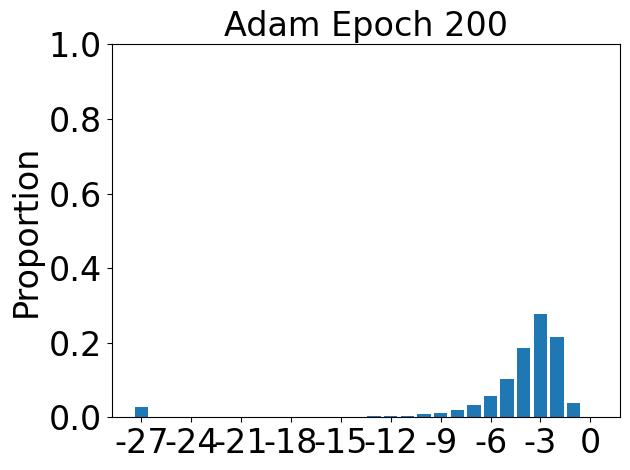}
\includegraphics[width=\linewidth]{figs/adapt_scale/histograms_updates/CIFAR100_resnet44_nobn/magnitude_histogram_CIFAR100_resnet44_nobn_Adam_update_no_alpha_Epoch_299.png}
\caption{Adam-$\ell_2$}
\label{fig:cifar100_resnet44_nobn_adam}
\end{subfigure}
\hspace{0.02\textwidth}
\begin{subfigure}{0.4\textwidth}
\includegraphics[width=\linewidth]{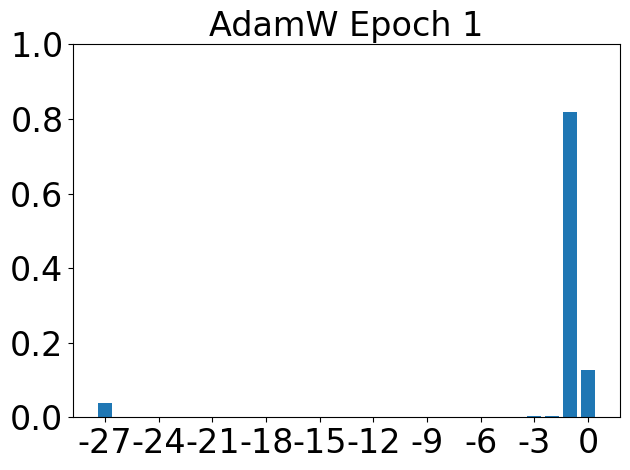}
\includegraphics[width=\linewidth]{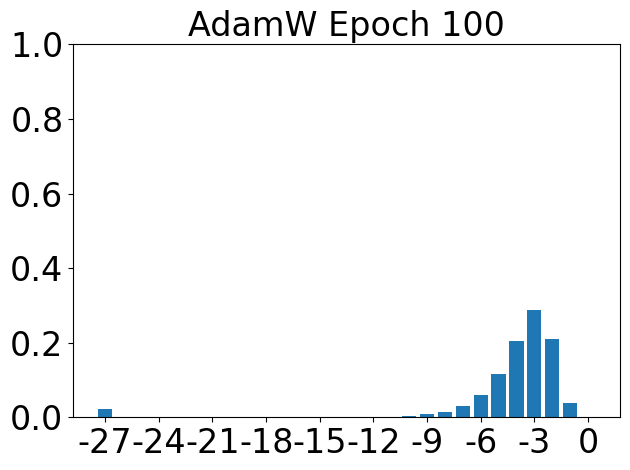}
\includegraphics[width=\linewidth]{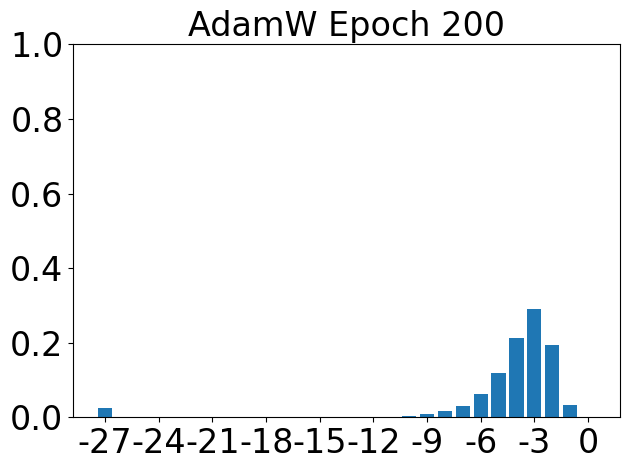}
\includegraphics[width=\linewidth]{figs/adapt_scale/histograms_updates/CIFAR100_resnet44_nobn/magnitude_histogram_CIFAR100_resnet44_nobn_AdamW_update_no_alpha_Epoch_299.png}
\caption{AdamW}
\label{fig:cifar100_resnet44_nobn_adamw}
\end{subfigure}
\caption{The histograms of the magnitudes of all updates of a 44-layer Resnet with BN disabled trained by AdamW or Adam-$\ell_2$  on CIFAR100.}
\label{fig:cifar100_resnet44_nobn_histo}
\end{figure}

\begin{figure}[t]
\centering
\begin{subfigure}{0.4\textwidth}
\includegraphics[width=\linewidth]{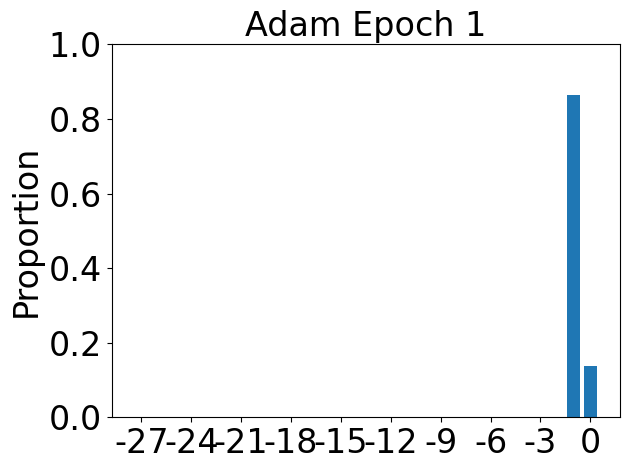}
\includegraphics[width=\linewidth]{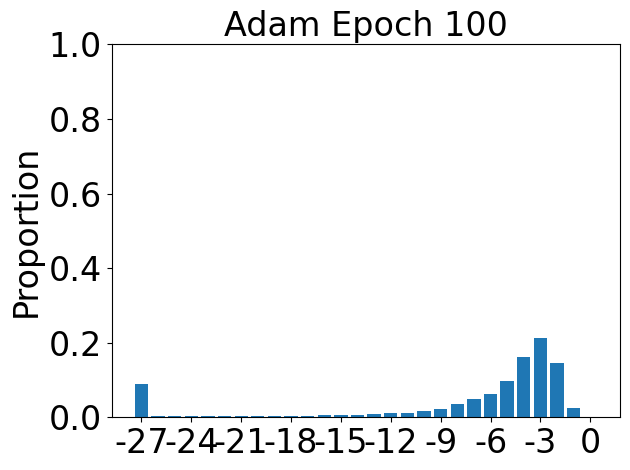}
\includegraphics[width=\linewidth]{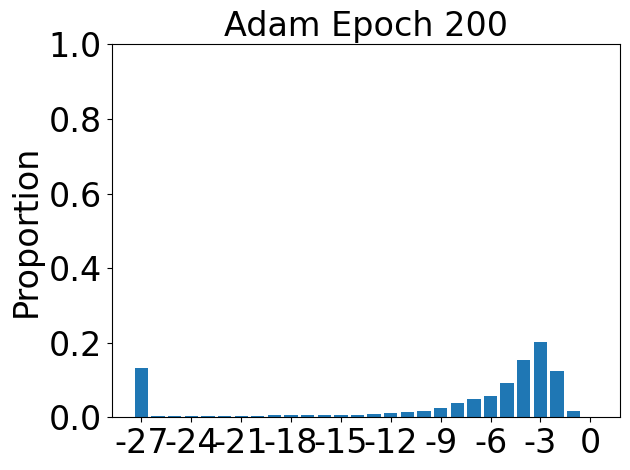}
\includegraphics[width=\linewidth]{figs/adapt_scale/histograms_updates/CIFAR100_resnet56_nobn/magnitude_histogram_CIFAR100_resnet56_nobn_Adam_update_no_alpha_Epoch_299.png}
\caption{Adam-$\ell_2$}
\label{fig:cifar100_resnet56_nobn_adam}
\end{subfigure}
\hspace{0.02\textwidth}
\begin{subfigure}{0.4\textwidth}
\includegraphics[width=\linewidth]{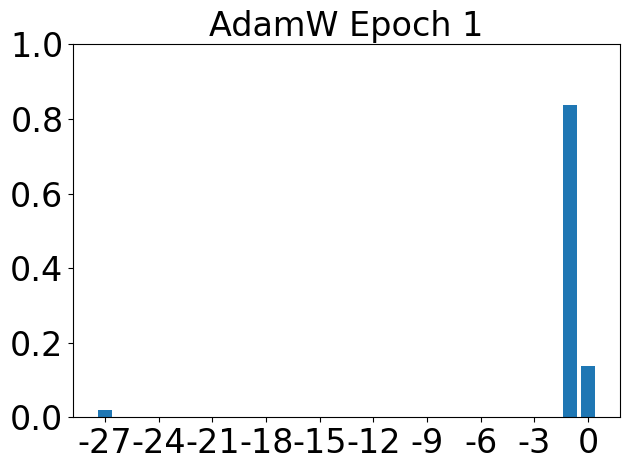}
\includegraphics[width=\linewidth]{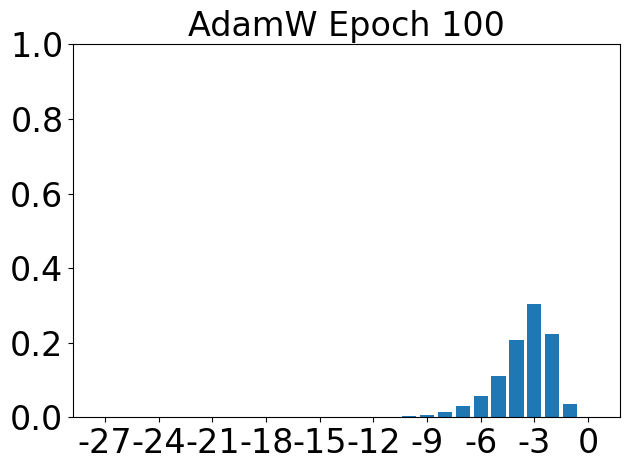}
\includegraphics[width=\linewidth]{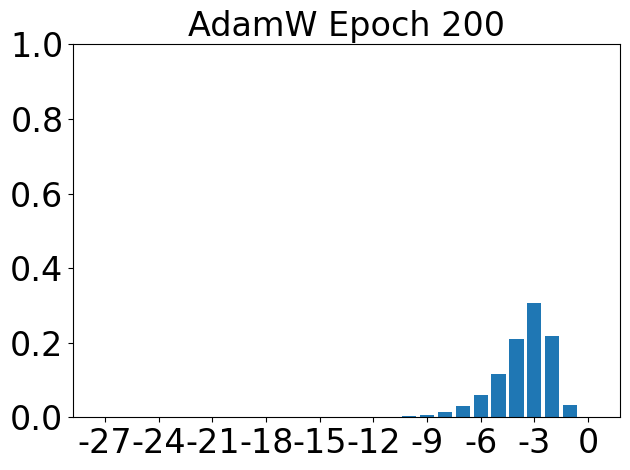}
\includegraphics[width=\linewidth]{figs/adapt_scale/histograms_updates/CIFAR100_resnet56_nobn/magnitude_histogram_CIFAR100_resnet56_nobn_AdamW_update_no_alpha_Epoch_299.png}
\caption{AdamW}
\label{fig:cifar100_resnet56_nobn_adamw}
\end{subfigure}
\caption{The histograms of the magnitudes of all updates of a 56-layer Resnet with BN disabled trained by AdamW or Adam-$\ell_2$  on CIFAR100.}
\label{fig:cifar100_resnet56_nobn_histo}
\end{figure}

\begin{figure}[t]
\centering
\begin{subfigure}{0.4\textwidth}
\includegraphics[width=\linewidth]{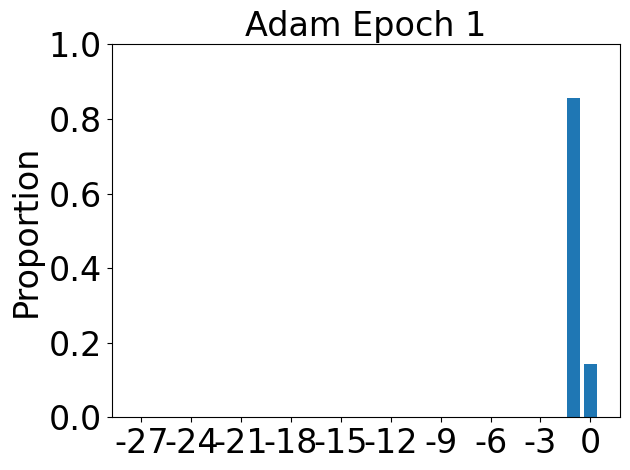}
\includegraphics[width=\linewidth]{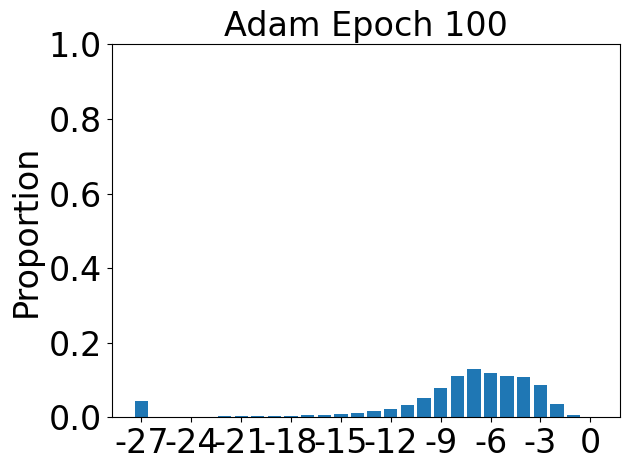}
\includegraphics[width=\linewidth]{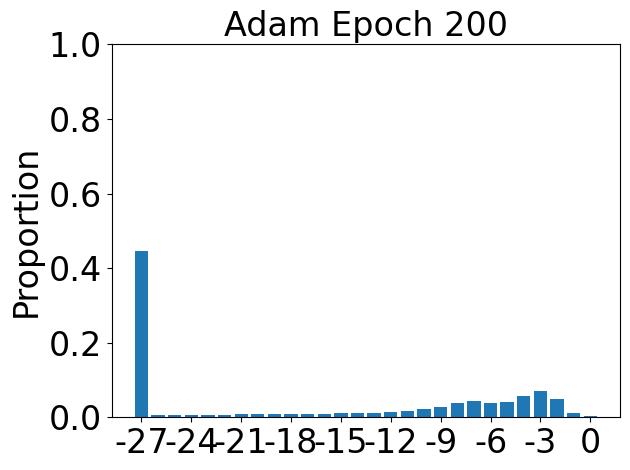}
\includegraphics[width=\linewidth]{figs/adapt_scale/histograms_updates/CIFAR100_resnet110_nobn/magnitude_histogram_CIFAR100_resnet110_nobn_Adam_update_no_alpha_Epoch_299.png}
\caption{Adam-$\ell_2$}
\label{fig:cifar100_resnet110_nobn_adam}
\end{subfigure}
\hspace{0.02\textwidth}
\begin{subfigure}{0.4\textwidth}
\includegraphics[width=\linewidth]{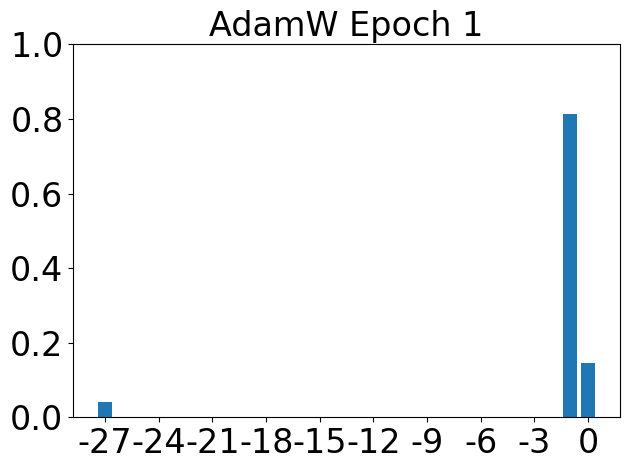}
\includegraphics[width=\linewidth]{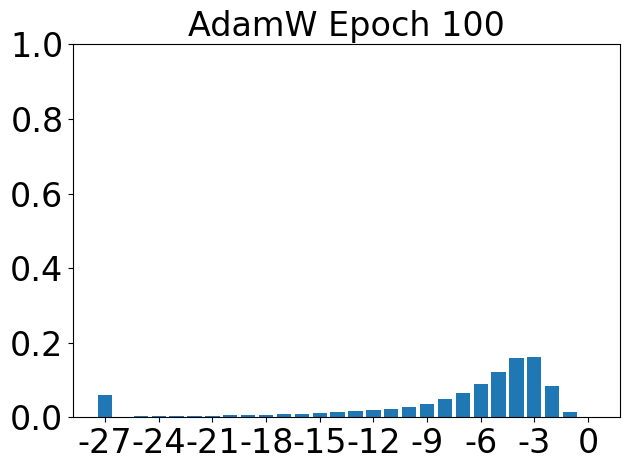}
\includegraphics[width=\linewidth]{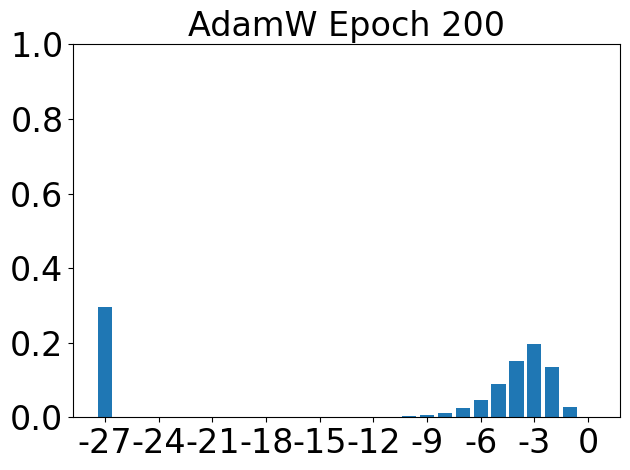}
\includegraphics[width=\linewidth]{figs/adapt_scale/histograms_updates/CIFAR100_resnet110_nobn/magnitude_histogram_CIFAR100_resnet110_nobn_AdamW_update_no_alpha_Epoch_299.png}
\caption{AdamW}
\label{fig:cifar100_resnet110_nobn_adamw}
\end{subfigure}
\caption{The histograms of the magnitudes of all updates of a 110-layer Resnet with BN disabled trained by AdamW or Adam-$\ell_2$  on CIFAR100.}
\label{fig:cifar100_resnet110_nobn_histo}
\end{figure}

\begin{figure}[t]
\centering
\begin{subfigure}{0.4\textwidth}
\includegraphics[width=\linewidth]{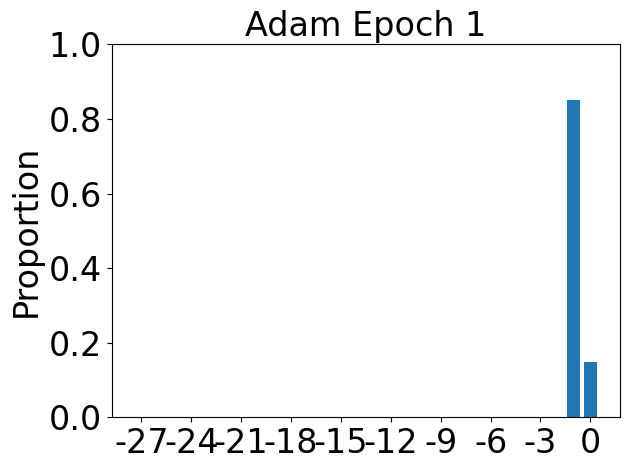}
\includegraphics[width=\linewidth]{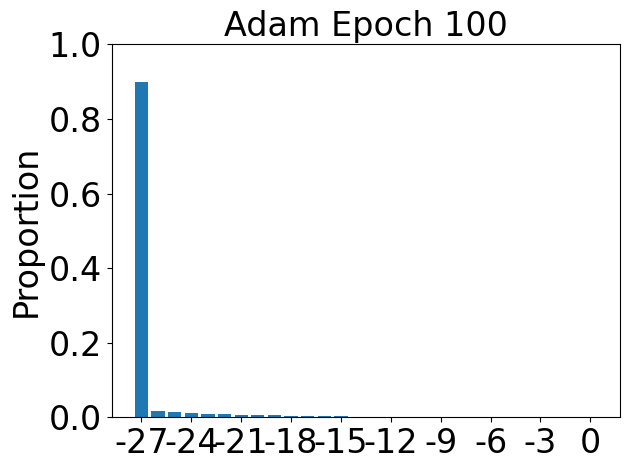}
\includegraphics[width=\linewidth]{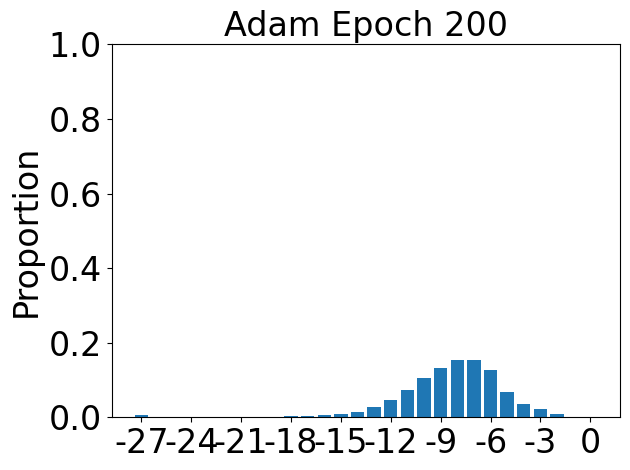}
\includegraphics[width=\linewidth]{figs/adapt_scale/histograms_updates/CIFAR100_resnet218_nobn/magnitude_histogram_CIFAR100_resnet218_nobn_Adam_update_no_alpha_Epoch_299.png}
\caption{Adam-$\ell_2$}
\label{fig:cifar100_resnet218_nobn_adam}
\end{subfigure}
\hspace{0.02\textwidth}
\begin{subfigure}{0.4\textwidth}
\includegraphics[width=\linewidth]{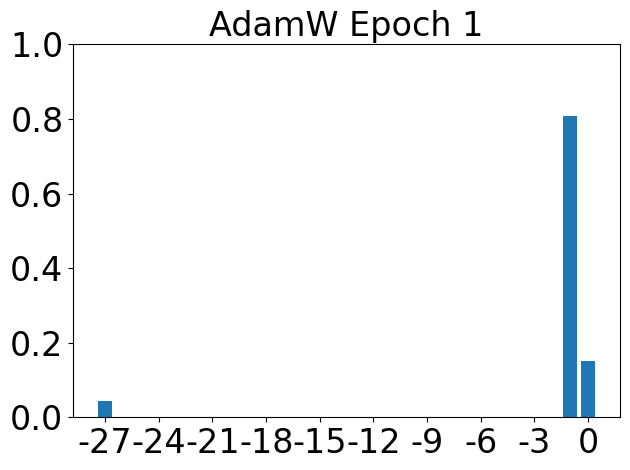}
\includegraphics[width=\linewidth]{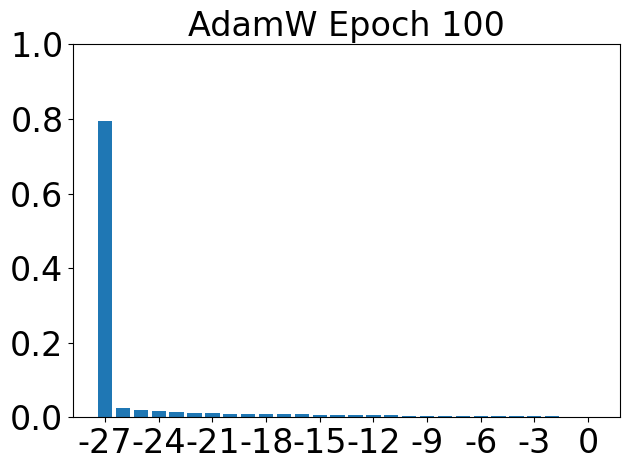}
\includegraphics[width=\linewidth]{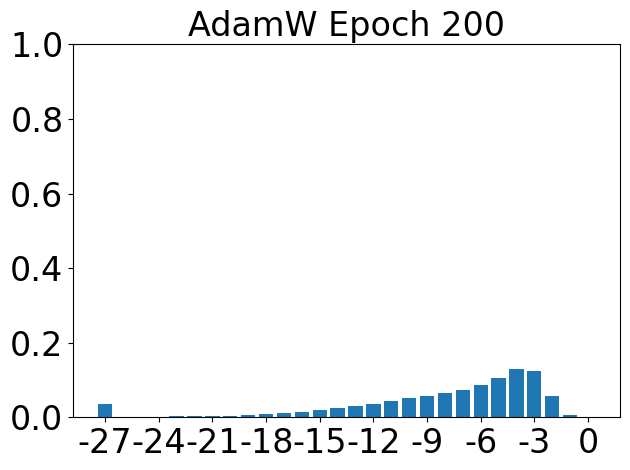}
\includegraphics[width=\linewidth]{figs/adapt_scale/histograms_updates/CIFAR100_resnet218_nobn/magnitude_histogram_CIFAR100_resnet218_nobn_AdamW_update_no_alpha_Epoch_299.png}
\caption{AdamW}
\label{fig:cifar100_resnet218_nobn_adamw}
\end{subfigure}
\caption{The histograms of the magnitudes of all updates of a 218-layer Resnet with BN disabled trained by AdamW or Adam-$\ell_2$  on CIFAR100.}
\label{fig:cifar100_resnet218_nobn_histo}
\end{figure}

\begin{figure}[t]
\centering

\begin{subfigure}{0.4\textwidth}
\includegraphics[width=\linewidth]{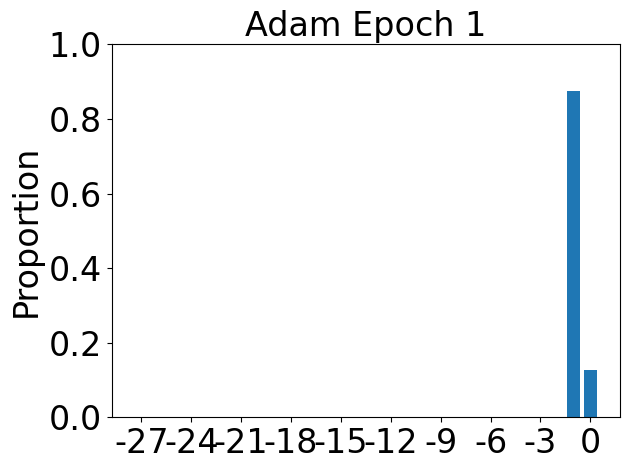}
\includegraphics[width=\linewidth]{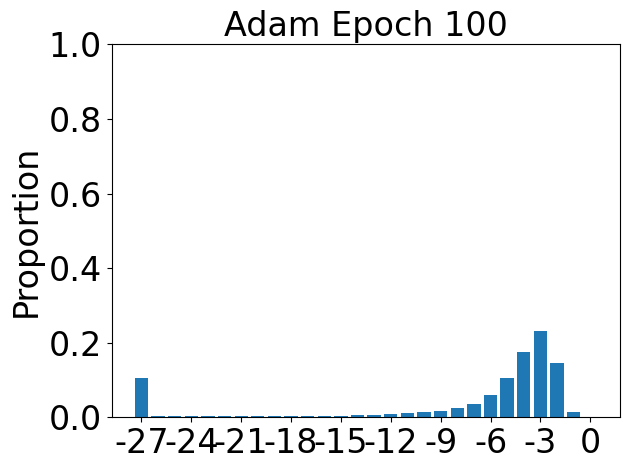}
\includegraphics[width=\linewidth]{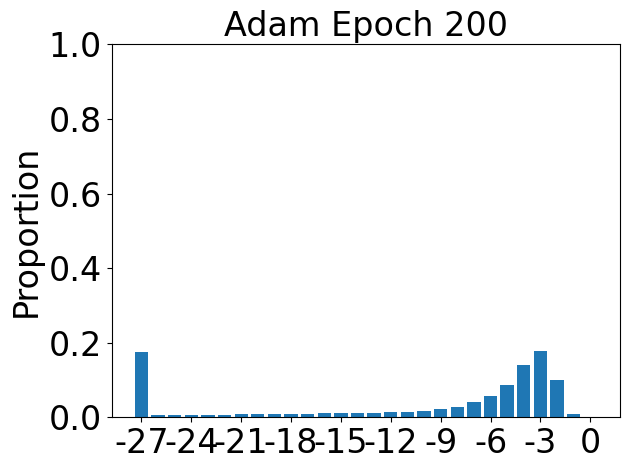}
\includegraphics[width=\linewidth]{figs/adapt_scale/histograms_updates/CIFAR100_DenseNet_NoBN/magnitude_histogram_CIFAR100_DenseNet_NoBN_Adam_update_no_alpha_Epoch_299.png}
\caption{Adam-$\ell_2$}
\label{fig:densenet_nobn_adam}
\end{subfigure}
\hspace{0.02\textwidth}
\begin{subfigure}{0.4\textwidth}
\includegraphics[width=\linewidth]{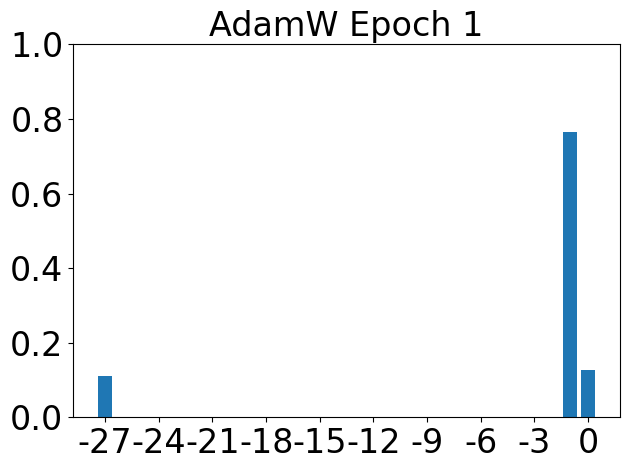}
\includegraphics[width=\linewidth]{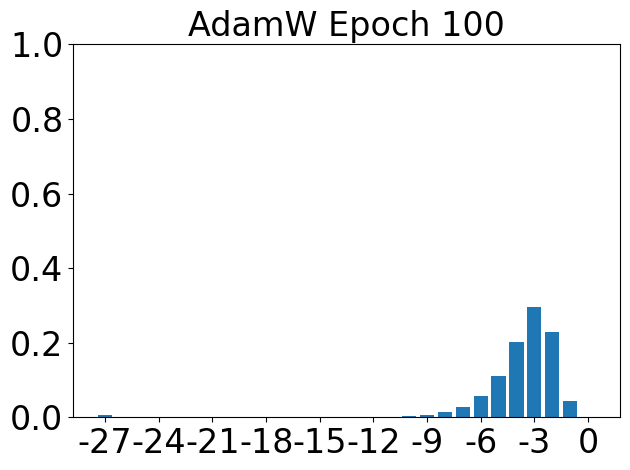}
\includegraphics[width=\linewidth]{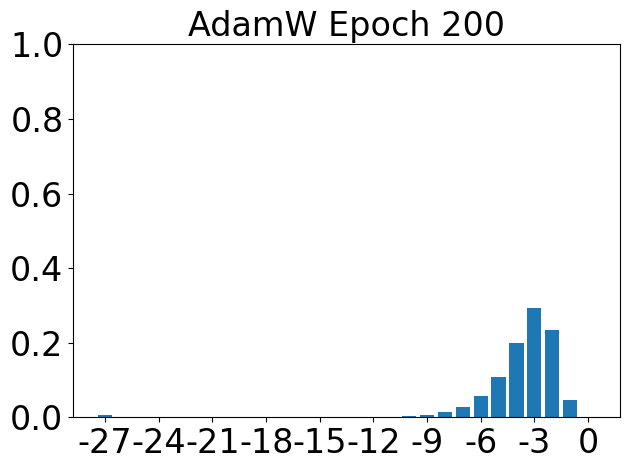}
\includegraphics[width=\linewidth]{figs/adapt_scale/histograms_updates/CIFAR100_DenseNet_NoBN/magnitude_histogram_CIFAR100_DenseNet_NoBN_AdamW_update_no_alpha_Epoch_299.png}
\caption{AdamW}
\label{fig:densenet_nobn_adamw}
\end{subfigure}

\caption{The histograms of the magnitudes of all updates of a 100-layer DenseNet-BC with BN disabled trained by AdamW or Adam-$\ell_2$ on CIFAR100.}
\label{fig:densenet_nobn_histo}
\end{figure}

%% file: 0_Prelim/cv.tex
\addcontentsline{toc}{chapter}{Curriculum Vitae}

\begin{center}
{\LARGE {\bf CURRICULUM VITAE}}\\
\vspace{0.5in}
{\large {\bf Zhenxun Zhuang}}
\end{center}

\centerline{Email: zxzhuang@bu.edu}

\centerline{Department of Computer Science, 111 Cummington Mall,
Boston, MA, 02215}

\vspace{0.2in}
\centerline{\Large \bf{Education}}

\vspace{1em}
\noindent{\bf{Boston University}} \hfill Aug. 2018 - Present\\
Ph.D. in Computer Science \hfill Boston, MA\\
Adviser: Francesco Orabona

\vspace{1em}
\noindent{\bf{Stony Brook University}} \hfill Aug. 2016 - Aug. 2018\\
Ph.D. in Computer Science \hfill Stony Brook, NY\\
Adviser: Francesco Orabona

\vspace{1em}
\noindent{\bf{University of Science and Technology of China}} \hfill Sep. 2012 - Jun. 2016\\
B.Eng. in Electronic Information Engineering \hfill Hefei, Anhui, China\\
Adviser: Feng Wu

\vspace{0.2in}
\centerline{\Large \bf{Internships}}

\vspace{1em}
\noindent{\bf{Facebook}} \hfill Jun. 2021 - Aug. 2021\\
Machine Learning Engineer Intern \hfill Seattle, WA\\

\vspace{1em}
\noindent{\bf{Meemo}} \hfill Jun. 2020 - Aug. 2020\\
Data Science Intern \hfill San Francisco, CA\\

\vspace{1em}
\noindent{\bf{IQVIA}} \hfill May 2019 - Aug. 2019\\
Machine Learning Intern \hfill Plymouth Meeting, PA\\

\vspace{0.2in}
\centerline{\Large \bf{Publications}}

\begin{enumerate}
\item Michael Crawshaw\textsuperscript{\dag}, Mingrui Liu\textsuperscript{\dag}, Francesco Orabona\textsuperscript{\dag}, Wei Zhang\textsuperscript{\dag}, \underline{Zhenxun} \underline{Zhuang}\textsuperscript{\dag}. \textit{Robustness to unbounded smoothness of generalized SignSGD.}
Conference on Neural Information Processing Systems, 2022.
\item Mingrui Liu, \underline{Zhenxun Zhuang}, Yunwei Lei, Chunyang Liao. \textit{A communication-efficient distributed gradient clipping algorithm for training deep neural networks.} Conference on Neural Information Processing Systems, 2022.
\item \underline{Zhenxun Zhuang}, Mingrui Liu, Ashok Cutkosky, Francesco Orabona. \textit{Understanding AdamW through proximal methods and scale-freeness.}
Transactions on Machine Learning Research, 2022.
\item Xiaoyu Li*, \underline{Zhenxun Zhuang}*, and Francesco Orabona. \textit{A second look at exponential and cosine step sizes: Simplicity, adaptivity, and performance.} In Proceedings of the 38th International Conference on Machine Learning, PMLR 139:6553-6564, 2021.
\item \underline{Zhenxun Zhuang}, Yunlong Wang, Kezi Yu, Songtao Lu. \textit{No-regret non-convex online meta-learning.} In Proceedings of IEEE International Conference on Acoustics, Speech and Signal Processing (ICASSP), pages 3942-3946, 2020.
\item \underline{Zhenxun Zhuang}, Kezi Yu, Songtao Lu, Lucas Glass, Yunlong Wang. \textit{Online meta-learning on non-convex setting.}
NeurIPS Workshop on Meta-Learning, 2019
\item \underline{Zhenxun Zhuang}, Ashok Cutkosky, Francesco Orabona. \textit{Surrogate losses for online learning of stepsizes in stochastic non-convex optimization.} In Proceedings of the 36th International Conference on Machine Learning, PMLR 97:7664-7672, 2019.
\end{enumerate}
(\dag\ denotes alphabetical order, * denotes equal contribution.)

\vspace{0.2in}
\centerline{\Large \bf{Teaching Experience}}

\vspace{1em}
\noindent\textbf{CSE 303 Introduction to Theory of Computation} \hfill Stony Brook University\\
Teaching Assistant\hfill Fall 2017

\vspace{1em}
\noindent\textbf{CSE 101 Introduction To Computers} \hfill Stony Brook University\\
Teaching Assistant\hfill Fall 2016

\vspace{0.2in}
\centerline{\Large \bf{Academic Services}}

\vspace{1em}
\noindent Reviewer for AISTATS 2020, ICML 2020-2022, NeurIPS 2019-2022

%% file: thesis.bbl
\begin{thebibliography}{}

\bibitem[Abadi et~al., 2015]{Tensorflow15}
Abadi, M., Agarwal, A., Barham, P., Brevdo, E., Chen, Z., Citro, C., Corrado,
  G.~S., Davis, A., Dean, J., Devin, M., Ghemawat, S., Goodfellow, I., Harp,
  A., Irving, G., Isard, M., Jia, Y., Jozefowicz, R., Kaiser, L., Kudlur, M.,
  Levenberg, J., Man\'{e}, D., Monga, R., Moore, S., Murray, D., Olah, C.,
  Schuster, M., Shlens, J., Steiner, B., Sutskever, I., Talwar, K., Tucker, P.,
  Vanhoucke, V., Vasudevan, V., Vi\'{e}gas, F., Vinyals, O., Warden, P.,
  Wattenberg, M., Wicke, M., Yu, Y., and Zheng, X. (2015).
\newblock {TensorFlow}: Large-scale machine learning on heterogeneous systems.
\newblock Software available from tensorflow.org.

\bibitem[Abernethy et~al., 2008]{AbernethyHR08}
Abernethy, J.~D., Hazan, E., and Rakhlin, A. (2008).
\newblock Competing in the dark: An efficient algorithm for bandit linear
  optimization.
\newblock In {\em Proceedings of the 21st Annual Conference on Learning
  Theory}, pages 263--274. Omnipress.

\bibitem[Abernethy et~al., 2012]{AbernethyHR12}
Abernethy, J.~D., Hazan, E., and Rakhlin, A. (2012).
\newblock Interior-point methods for full-information and bandit online
  learning.
\newblock {\em {IEEE} Transactions on Information Theory}, 58(7):4164--4175.

\bibitem[Agarwal et~al., 2020]{AgarwalAHKZ20}
Agarwal, N., Anil, R., Hazan, E., Koren, T., and Zhang, C. (2020).
\newblock Disentangling adaptive gradient methods from learning rates.
\newblock {\em arXiv preprint arXiv:2002.11803}.

\bibitem[Alber et~al., 1998]{alber1998projected}
Alber, Y.~I., Iusem, A.~N., and Solodov, M.~V. (1998).
\newblock On the projected subgradient method for nonsmooth convex optimization
  in a {Hilbert} space.
\newblock {\em Mathematical Programming}, 81(1):23--35.

\bibitem[Allen-Zhu et~al., 2019]{Allen-ZhuLS19}
Allen-Zhu, Z., Li, Y., and Song, Z. (2019).
\newblock A convergence theory for deep learning via over-parameterization.
\newblock In {\em Proceedings of the 36th International Conference on Machine
  Learning}, volume~97, pages 242--252. PMLR.

\bibitem[Arjevani et~al., 2022]{ArjevaniCDFSW19}
Arjevani, Y., Carmon, Y., Duchi, J., Foster, D.~J., Srebro, N., and Woodworth,
  B.~E. (2022).
\newblock Lower bounds for non-convex stochastic optimization.
\newblock {\em Mathematical Programming}, pages 1--50.

\bibitem[Arora et~al., 2015]{AroraGMM15}
Arora, S., Ge, R., Ma, T., and Moitra, A. (2015).
\newblock Simple, efficient, and neural algorithms for sparse coding.
\newblock In {\em Proceedings of The 28th Conference on Learning Theory},
  volume~40, pages 113--149. PMLR.

\bibitem[Asi and Duchi, 2019]{AsiD19}
Asi, H. and Duchi, J.~C. (2019).
\newblock Stochastic (approximate) proximal point methods: Convergence,
  optimality, and adaptivity.
\newblock {\em SIAM Journal on Optimization}, 29(3):2257--2290.

\bibitem[Auer et~al., 2002]{AuerCG02}
Auer, P., Cesa-Bianchi, N., and Gentile, C. (2002).
\newblock Adaptive and self-confident on-line learning algorithms.
\newblock {\em Journal of Computer and System Sciences}, 64(1):48--75.

\bibitem[Aybat et~al., 2019]{AybatFGO19}
Aybat, N.~S., Fallah, A., Gurbuzbalaban, M., and Ozdaglar, A. (2019).
\newblock A universally optimal multistage accelerated stochastic gradient
  method.
\newblock In {\em Advances in Neural Information Processing Systems}, pages
  8523--8534.

\bibitem[Ba et~al., 2016]{BaKH16}
Ba, J.~L., Kiros, J.~R., and Hinton, G.~E. (2016).
\newblock Layer normalization.
\newblock {\em arXiv preprint arXiv:1607.06450}.

\bibitem[Bach and Moulines, 2011]{MoulinesB11}
Bach, F. and Moulines, E. (2011).
\newblock Non-asymptotic analysis of stochastic approximation algorithms for
  machine learning.
\newblock In {\em Advances in Neural Information Processing Systems 24}, pages
  451--459. Curran Associates, Inc.

\bibitem[Balles and Hennig, 2018]{balles2018dissecting}
Balles, L. and Hennig, P. (2018).
\newblock Dissecting {Adam}: The sign, magnitude and variance of stochastic
  gradients.
\newblock In {\em Proceedings of the 35th International Conference on Machine
  Learning}, volume~80, pages 404--413. PMLR.

\bibitem[Beck and Teboulle, 2003]{BeckT03}
Beck, A. and Teboulle, M. (2003).
\newblock Mirror descent and nonlinear projected subgradient methods for convex
  optimization.
\newblock {\em Operations Research Letters}, 31(3):167--175.

\bibitem[Beck and Teboulle, 2009]{BeckT09}
Beck, A. and Teboulle, M. (2009).
\newblock A fast iterative shrinkage-thresholding algorithm for linear inverse
  problems.
\newblock {\em SIAM journal on imaging sciences}, 2(1):183--202.

\bibitem[Bengio et~al., 1994]{BengioSF94}
Bengio, Y., Simard, P., and Frasconi, P. (1994).
\newblock Learning long-term dependencies with gradient descent is difficult.
\newblock {\em IEEE transactions on neural networks}, 5(2):157--166.

\bibitem[Bernstein et~al., 2018]{bernstein2018signsgd}
Bernstein, J., Wang, Y., Azizzadenesheli, K., and Anandkumar, A. (2018).
\newblock {SignSGD}: Compressed optimisation for non-convex problems.
\newblock In {\em Proceedings of the 35th International Conference on Machine
  Learning}, volume~80, pages 560--569. PMLR.

\bibitem[Bertsekas and Tsitsiklis, 1996]{BertsekasT96}
Bertsekas, D.~P. and Tsitsiklis, J.~N. (1996).
\newblock {\em Neuro-Dynamic Programming}.
\newblock Athena Scientific.

\bibitem[Bjorck et~al., 2021]{BjorckWG20}
Bjorck, J., Weinberger, K.~Q., and Gomes, C. (2021).
\newblock Understanding decoupled and early weight decay.
\newblock {\em Proceedings of the AAAI Conference on Artificial Intelligence},
  35(8):6777--6785.

\bibitem[Bos and Chug, 1996]{BosC96}
Bos, S. and Chug, E. (1996).
\newblock Using weight decay to optimize the generalization ability of a
  perceptron.
\newblock {\em Proceedings of International Conference on Neural Networks
  ({ICNN})}, 1:241--246.

\bibitem[Boyd and Vandenberghe, 2004]{BoydV04}
Boyd, S. and Vandenberghe, L. (2004).
\newblock {\em Convex Optimization}.
\newblock Cambridge University Press.

\bibitem[Bul{\`o} et~al., 2018]{BuloPK18}
Bul{\`o}, S.~R., Porzi, L., and Kontschieder, P. (2018).
\newblock In-place activated {BatchNorm} for memory-optimized training of
  {DNNs}.
\newblock In {\em Proceedings of the IEEE Conference on Computer Vision and
  Pattern Recognition}, pages 5639--5647.

\bibitem[Carion et~al., 2020]{CarionMSUKZ20}
Carion, N., Massa, F., Synnaeve, G., Usunier, N., Kirillov, A., and Zagoruyko,
  S. (2020).
\newblock End-to-end object detection with {Transformers}.
\newblock In {\em Proceedings of European Conference on Computer Vision}, pages
  213--229. Springer.

\bibitem[Carmon et~al., 2021]{CarmonDHS21}
Carmon, Y., Duchi, J., Hinder, O., and Sidford, A. (2021).
\newblock Lower bounds for finding stationary points ii: first-order methods.
\newblock {\em Mathematical Programming}, 185:315--355.

\bibitem[Cauchy, 1847]{cauchy1847methode}
Cauchy, A. (1847).
\newblock M{\'e}thode g{\'e}n{\'e}rale pour la r{\'e}solution des systemes
  d’{\'e}quations simultan{\'e}es.
\newblock {\em Comptes Rendus de l'Academie des Science}, 25:536--538.

\bibitem[Cesa-Bianchi et~al., 2007]{Cesa-BianchiMS07}
Cesa-Bianchi, N., Mansour, Y., and Stoltz, G. (2007).
\newblock Improved second-order bounds for prediction with expert advice.
\newblock {\em Machine Learning}, 66(2):321--352.

\bibitem[Chang and Lin, 2001]{ChangL01}
Chang, C.-C. and Lin, C.-J. (2001).
\newblock {\em {LIBSVM}: a library for support vector machines}.
\newblock Software available at \url{http://www.csie.ntu.edu.tw/~cjlin/libsvm}.

\bibitem[Chen et~al., 2020a]{ChenZTYG18}
Chen, J., Zhou, D., Tang, Y., Yang, Z., and Gu, Q. (2020a).
\newblock Closing the generalization gap of adaptive gradient methods in
  training deep neural networks.
\newblock In {\em Proceedings of the Twenty-Ninth International Joint
  Conference on Artificial Intelligence}, pages 3267--3275.

\bibitem[Chen et~al., 2020b]{ChenKNH20}
Chen, T., Kornblith, S., Norouzi, M., and Hinton, G. (2020b).
\newblock A simple framework for contrastive learning of visual
  representations.
\newblock In {\em Proceedings of the 37th International Conference on Machine
  Learning}, volume 119, pages 1597--1607. PMLR.

\bibitem[Chen and Candes, 2015]{ChenC15}
Chen, Y. and Candes, E. (2015).
\newblock Solving random quadratic systems of equations is nearly as easy as
  solving linear systems.
\newblock In {\em Advances in Neural Information Processing Systems 28}, pages
  739--747. Curran Associates, Inc.

\bibitem[Cortes and Vapnik, 1995]{cortes1995support}
Cortes, C. and Vapnik, V. (1995).
\newblock Support-vector networks.
\newblock {\em Machine learning}, 20(3):273--297.

\bibitem[Cramer, 2002]{cramer2002origin}
Cramer, J.~S. (2002).
\newblock The origins of logistic regression.
\newblock {\em SSRN Electronic Journal}.

\bibitem[Craven and Glover, 1985]{craven1985invex}
Craven, B.~D. and Glover, B.~M. (1985).
\newblock Invex functions and duality.
\newblock {\em Journal of the Australian Mathematical Society}, 39(1):1--20.

\bibitem[Crawshaw et~al., 2022]{crawshaw2022robustness}
Crawshaw, M., Liu, M., Orabona, F., Zhang, W., and Zhuang, Z. (2022).
\newblock Robustness to unbounded smoothness of generalized sign{SGD}.
\newblock In {\em Advances in Neural Information Processing Systems}.

\bibitem[Cubuk et~al., 2019]{CubukZMVL19}
Cubuk, E.~D., Zoph, B., Mane, D., Vasudevan, V., and Le, Q.~V. (2019).
\newblock Autoaugment: Learning augmentation strategies from data.
\newblock In {\em Proceedings of the IEEE/CVF Conference on Computer Vision and
  Pattern Recognition}, pages 113--123.

\bibitem[Cutkosky and Mehta, 2020]{cutkosky2020momentum}
Cutkosky, A. and Mehta, H. (2020).
\newblock Momentum improves normalized {SGD}.
\newblock In {\em Proceedings of the 37th International Conference on Machine
  Learning}, volume 119, pages 2260--2268. PMLR.

\bibitem[Cutkosky and Mehta, 2021]{CutkoskyM21}
Cutkosky, A. and Mehta, H. (2021).
\newblock High-probability bounds for non-convex stochastic optimization with
  heavy tails.
\newblock In {\em Advances in Neural Information Processing Systems},
  volume~34, pages 4883--4895. Curran Associates, Inc.

\bibitem[Davis et~al., 2019]{DavisDC19}
Davis, D., Drusvyatskiy, D., and Charisopoulos, V. (2019).
\newblock Stochastic algorithms with geometric step decay converge linearly on
  sharp functions.
\newblock {\em arXiv preprint arXiv:1907.09547}.

\bibitem[Davis et~al., 2021]{DavisDXZ19}
Davis, D., Drusvyatskiy, D., Xiao, L., and Zhang, J. (2021).
\newblock From low probability to high confidence in stochastic convex
  optimization.
\newblock {\em Journal of Machine Learning Research}, 22(49):1--38.

\bibitem[De and Smith, 2020]{DeS20}
De, S. and Smith, S. (2020).
\newblock Batch normalization biases residual blocks towards the identity
  function in deep networks.
\newblock In {\em Advances in Neural Information Processing Systems},
  volume~33, pages 19964--19975. Curran Associates, Inc.

\bibitem[Devlin et~al., 2019]{DevlinCLT19}
Devlin, J., Chang, M.-W., Lee, K., and Toutanova, K. (2019).
\newblock {BERT}: Pre-training of deep bidirectional {Transformers} for
  language understanding.
\newblock In {\em Proceedings of the 2019 Conference of the North American
  Chapter of the Association for Computational Linguistics: Human Language
  Technologies, {NAACL-HLT}}, pages 4171--4186. Association for Computational
  Linguistics.

\bibitem[Dosovitskiy et~al., 2021]{DosovitskiyB0WZ21}
Dosovitskiy, A., Beyer, L., Kolesnikov, A., Weissenborn, D., Zhai, X.,
  Unterthiner, T., Dehghani, M., Minderer, M., Heigold, G., Gelly, S.,
  Uszkoreit, J., and Houlsby, N. (2021).
\newblock An image is worth 16x16 words: {Transformers} for image recognition
  at scale.
\newblock In {\em 9th International Conference on Learning Representations}.

\bibitem[Duchi et~al., 2010a]{DuchiHS10}
Duchi, J., Hazan, E., and Singer, Y. (2010a).
\newblock Adaptive subgradient methods for online learning and stochastic
  optimization.
\newblock In {\em Proceedings of the 23rd Annual Conference on Learning
  Theory}, pages 257--269. Omnipress.

\bibitem[Duchi et~al., 2010b]{DuchiSST10}
Duchi, J., Shalev-Shwartz, S., Singer, Y., and Tewari, A. (2010b).
\newblock Composite objective mirror descent.
\newblock In {\em Proceedings of the 23rd Annual Conference on Learning
  Theory}, pages 14--26. Omnipress.

\bibitem[Ermoliev, 1988]{ermoliev1988stochastic}
Ermoliev, Y. (1988).
\newblock Stochastic quasigradient methods.
\newblock In {\em Numerical techniques for stochastic optimization}, number~10
  in Springer Series in Computational Mathematics, pages 141--185. Springer.

\bibitem[Gauss, 1820]{gauss1995theory}
Gauss, C.~F. (1820).
\newblock {\em Theory of the combination of observations least subject to
  errors, Part One, Part Two, Supplement}.
\newblock Society for Industrial and Applied Mathematics,.
\newblock Translated from original manuscript by G. W. Stewart in 1995.

\bibitem[Ge et~al., 2019]{GeKKN19}
Ge, R., Kakade, S.~M., Kidambi, R., and Netrapalli, P. (2019).
\newblock The step decay schedule: A near optimal, geometrically decaying
  learning rate procedure for least squares.
\newblock In {\em Advances in Neural Information Processing Systems},
  volume~32, pages 14951--14962. Curran Associates, Inc.

\bibitem[Ghadimi and Lan, 2013]{GhadimiL13}
Ghadimi, S. and Lan, G. (2013).
\newblock Stochastic first- and zeroth-order methods for nonconvex stochastic
  programming.
\newblock {\em SIAM Journal on Optimization}, 23(4):2341--2368.

\bibitem[Gitman and Ginsburg, 2017]{GitmanG17}
Gitman, I. and Ginsburg, B. (2017).
\newblock Comparison of batch normalization and weight normalization algorithms
  for the large-scale image classification.
\newblock {\em arXiv preprint arXiv:1709.08145}.

\bibitem[Glorot and Bengio, 2010]{GlorotB10}
Glorot, X. and Bengio, Y. (2010).
\newblock Understanding the difficulty of training deep feedforward neural
  networks.
\newblock In {\em Proceedings of the thirteenth international conference on
  artificial intelligence and statistics}, volume~9, pages 249--256. PMLR.

\bibitem[Goffin, 1977]{Goffin77}
Goffin, J.-L. (1977).
\newblock On convergence rates of subgradient optimization methods.
\newblock {\em Mathematical programming}, 13(1):329--347.

\bibitem[Gorbunov et~al., 2020]{gorbunov2020stochastic}
Gorbunov, E., Danilova, M., and Gasnikov, A. (2020).
\newblock Stochastic optimization with heavy-tailed noise via accelerated
  gradient clipping.
\newblock In {\em Advances in Neural Information Processing Systems},
  volume~33, pages 15042--15053. Curran Associates, Inc.

\bibitem[Goyal et~al., 2017]{GoyalDGNWKTJH17}
Goyal, P., Doll{\'a}r, P., Girshick, R.~B., Noordhuis, P., Wesolowski, L.,
  Kyrola, A., Tulloch, A., Jia, Y., and He, K. (2017).
\newblock Accurate, large minibatch {SGD}: Training {ImageNet} in 1 hour.
\newblock {\em arXiv preprint arXiv:1706.02677}.

\bibitem[Grill et~al., 2020]{GrillSATRBDABGGPKMV20}
Grill, J., Strub, F., Altch\'{e}, F., Tallec, C., Richemond, P., Buchatskaya,
  E., Doersch, C., Avila~P., B., Guo, Z., Gheshlaghi~Azar, M., Piot, B.,
  kavukcuoglu, K., Munos, R., and Valko, M. (2020).
\newblock Bootstrap your own latent - a new approach to self-supervised
  learning.
\newblock In {\em Advances in Neural Information Processing Systems},
  volume~33, pages 21271--21284. Curran Associates, Inc.

\bibitem[Hanson, 1981]{hanson1981sufficiency}
Hanson, M.~A. (1981).
\newblock On sufficiency of the {Kuhn-Tucker} conditions.
\newblock {\em Journal of Mathematical Analysis and Applications},
  80(2):545--550.

\bibitem[Harvey et~al., 2019]{harvey2019tight}
Harvey, N. J.~A., Liaw, C., Plan, Y., and Randhawa, S. (2019).
\newblock Tight analyses for non-smooth stochastic gradient descent.
\newblock In {\em Proceedings of the Thirty-Second Conference on Learning
  Theory}, volume~99, pages 1579--1613. PMLR.

\bibitem[Hazan and Kale, 2014]{HazanK11}
Hazan, E. and Kale, S. (2014).
\newblock Beyond the regret minimization barrier: Optimal algorithms for
  stochastic strongly-convex optimization.
\newblock {\em Journal of Machine Learning Research}, 15(71):2489--2512.

\bibitem[Hazan et~al., 2015]{hazan2015beyond}
Hazan, E., Levy, K.~Y., and Shalev-Shwartz, S. (2015).
\newblock Beyond convexity: Stochastic quasi-convex optimization.
\newblock In {\em Advances in Neural Information Processing Systems},
  volume~28, pages 1594--1602. Curran Associates, Inc.

\bibitem[He et~al., 2016]{HeZRS16}
He, K., Zhang, X., Ren, S., and Sun, J. (2016).
\newblock Deep residual learning for image recognition.
\newblock In {\em Proceedings of the IEEE Conference on Computer Vision and
  Pattern Recognition}, pages 770--778.

\bibitem[He et~al., 2019]{HeZZZXL19}
He, T., Zhang, Z., Zhang, H., Zhang, Z., Xie, J., and Li, M. (2019).
\newblock Bag of tricks for image classification with convolutional neural
  networks.
\newblock In {\em Proceedings of the IEEE/CVF Conference on Computer Vision and
  Pattern Recognition}, pages 558--567.

\bibitem[Hochreiter and Schmidhuber, 1997]{hochreiter1997long}
Hochreiter, S. and Schmidhuber, J. (1997).
\newblock Long short-term memory.
\newblock {\em Neural computation}, 9(8):1735--1780.

\bibitem[Huang et~al., 2017]{HuangLVW17}
Huang, G., Liu, Z., Van Der~Maaten, L., and Weinberger, K.~Q. (2017).
\newblock Densely connected convolutional networks.
\newblock In {\em Proceedings of the IEEE conference on Computer Vision and
  Pattern Recognition}, pages 4700--4708.

\bibitem[Ioffe and Szegedy, 2015]{IoffeS15}
Ioffe, S. and Szegedy, C. (2015).
\newblock Batch normalization: Accelerating deep network training by reducing
  internal covariate shift.
\newblock In {\em Proceedings of International conference on machine learning},
  volume~37, pages 448--456. PMLR.

\bibitem[Jordan and Mitchell, 2015]{jordan2015machine}
Jordan, M.~I. and Mitchell, T.~M. (2015).
\newblock Machine learning: Trends, perspectives, and prospects.
\newblock {\em Science}, 349(6245):255--260.

\bibitem[Karimi et~al., 2016]{KarimiNS16}
Karimi, H., Nutini, J., and Schmidt, M. (2016).
\newblock {Linear convergence of gradient and proximal-gradient methods under
  the {P}olyak-\L{}ojasiewicz condition}.
\newblock In {\em Joint European Conference on Machine Learning and Knowledge
  Discovery in Databases}, pages 795--811. Springer.

\bibitem[Khaled and Richt{\'a}rik, 2020]{khaled2020better}
Khaled, A. and Richt{\'a}rik, P. (2020).
\newblock Better theory for {SGD} in the nonconvex world.
\newblock {\em arXiv preprint arXiv:2002.03329}.

\bibitem[Kingma and Ba, 2015]{KingmaB15}
Kingma, D.~P. and Ba, J. (2015).
\newblock {Adam}: A method for stochastic optimization.
\newblock In {\em 3rd International Conference on Learning Representations}.

\bibitem[Kleinberg et~al., 2018]{KleinbergLY18}
Kleinberg, B., Li, Y., and Yuan, Y. (2018).
\newblock An alternative view: When does {SGD} escape local minima?
\newblock In {\em Proceedings of the 35th International Conference on Machine
  Learning}, volume~80, pages 2698--2707. PMLR.

\bibitem[Koolen et~al., 2014]{KoolenvEG14}
Koolen, W.~M., van Erven, T., and Gr\"{u}nwald, P. (2014).
\newblock Learning the learning rate for prediction with expert advice.
\newblock In {\em Advances in Neural Information Processing Systems},
  volume~27, pages 2294--2302. Curran Associates, Inc.

\bibitem[Krizhevsky et~al., 2012]{KrizhevskySH12}
Krizhevsky, A., Sutskever, I., and Hinton, G.~E. (2012).
\newblock {ImageNet} classification with deep convolutional neural networks.
\newblock In {\em Advances in Neural Information Processing Systems},
  volume~25, pages 1106--1114. Curran Associates, Inc.

\bibitem[Krogh and Hertz, 1992]{KroghH91}
Krogh, A. and Hertz, J. (1992).
\newblock A simple weight decay can improve generalization.
\newblock In {\em Advances in Neural Information Processing Systems}, volume~4,
  pages 950--957. Morgan-Kaufmann.

\bibitem[Kuen et~al., 2019]{KuenPLZT19}
Kuen, J., Perazzi, F., Lin, Z., Zhang, J., and Tan, Y.-P. (2019).
\newblock Scaling object detection by transferring classification weights.
\newblock In {\em Proceedings of the IEEE/CVF International Conference on
  Computer Vision}, pages 6044--6053.

\bibitem[Kulunchakov and Mairal, 2019]{KulunchakovM19}
Kulunchakov, A. and Mairal, J. (2019).
\newblock A generic acceleration framework for stochastic composite
  optimization.
\newblock In {\em Advances in Neural Information Processing Systems},
  volume~32, pages 12556--12567. Curran Associates, Inc.

\bibitem[Lee et~al., 2015]{LeeXGZT15}
Lee, C., Xie, S., Gallagher, P., Zhang, Z., and Tu, Z. (2015).
\newblock Deeply-supervised nets.
\newblock In {\em Proceedings of the Eighteenth International Conference on
  Artificial Intelligence and Statistics}, volume~38, pages 562--570. PMLR.

\bibitem[Levy et~al., 2018]{LevyYC18}
Levy, K.~Y., Yurtsever, A., and Cevher, V. (2018).
\newblock Online adaptive methods, universality and acceleration.
\newblock In {\em Advances in Neural Information Processing Systems},
  volume~31, pages 6500--6509. Curran Associates, Inc.

\bibitem[Li et~al., 2014]{li2014scaling}
Li, M., Andersen, D.~G., Park, J.~W., Smola, A.~J., Ahmed, A., Josifovski, V.,
  Long, J., Shekita, E.~J., and Su, B.-Y. (2014).
\newblock Scaling distributed machine learning with the parameter server.
\newblock In {\em Proceedings of the 11th USENIX Conference on Operating
  Systems Design and Implementation}, volume~14, pages 583--598. USENIX
  Association.

\bibitem[Li and Orabona, 2019]{LiO19}
Li, X. and Orabona, F. (2019).
\newblock On the convergence of stochastic gradient descent with adaptive
  stepsizes.
\newblock In {\em Proceedings of the Twenty-Second International Conference on
  Artificial Intelligence and Statistics}, volume~89, pages 983--992. PMLR.

\bibitem[Li and Orabona, 2020]{LiO20}
Li, X. and Orabona, F. (2020).
\newblock A high probability analysis of adaptive {SGD} with momentum.
\newblock In {\em ICML 2020 Workshop on Beyond First Order Methods in ML
  Systems}.

\bibitem[Li et~al., 2021]{LiZO21}
Li, X., Zhuang, Z., and Orabona, F. (2021).
\newblock A second look at exponential and cosine step sizes: Simplicity,
  adaptivity, and performance.
\newblock In {\em Proceedings of International Conference on Machine Learning},
  volume 139, pages 6553--6564. PMLR.

\bibitem[Li, 2018]{Li18}
Li, X.-L. (2018).
\newblock Preconditioned stochastic gradient descent.
\newblock {\em IEEE Transactions on Neural Networks and Learning Systems},
  29(5):1454--1466.

\bibitem[Lifchitz et~al., 2019]{LifchitzAPB19}
Lifchitz, Y., Avrithis, Y., Picard, S., and Bursuc, A. (2019).
\newblock Dense classification and implanting for few-shot learning.
\newblock In {\em Proceedings of the IEEE/CVF Conference on Computer Vision and
  Pattern Recognition}, pages 9258--9267.

\bibitem[Liu et~al., 2019]{LiuSY18}
Liu, H., Simonyan, K., and Yang, Y. (2019).
\newblock {DARTS}: Differentiable architecture search.
\newblock In {\em Seventh International Conference on Learning
  Representations}.

\bibitem[Liu et~al., 2022]{liu2022communication}
Liu, M., Zhuang, Z., Lei, Y., and Liao, C. (2022).
\newblock A communication-efficient distributed gradient clipping algorithm for
  training deep neural networks.
\newblock In {\em Advances in Neural Information Processing Systems}.

\bibitem[\L{}ojasiewicz, 1963]{Lojasiewicz63}
\L{}ojasiewicz, S. (1963).
\newblock A topological property of real analytic subsets (in french).
\newblock {\em Coll. du CNRS, Les \'equations aux d\'eriv\'ees partielles},
  pages 87--89.

\bibitem[Loshchilov and Hutter, 2017]{LoshchilovH17}
Loshchilov, I. and Hutter, F. (2017).
\newblock {SGDR}: Stochastic gradient descent with warm restarts.
\newblock In {\em Fifth International Conference on Learning Representations}.

\bibitem[Loshchilov and Hutter, 2019]{LoshchilovH18}
Loshchilov, I. and Hutter, F. (2019).
\newblock Decoupled weight decay regularization.
\newblock In {\em Seventh International Conference on Learning
  Representations}.

\bibitem[Mai and Johansson, 2021]{mai2021stability}
Mai, V.~V. and Johansson, M. (2021).
\newblock Stability and convergence of stochastic gradient clipping: Beyond
  lipschitz continuity and smoothness.
\newblock In {\em International Conference on Machine Learning}, volume 139,
  pages 7325--7335. PMLR.

\bibitem[Marcus et~al., 1993]{marcus1993building}
Marcus, M.~P., Marcinkiewicz, M.~A., and Santorini, B. (1993).
\newblock Building a large annotated corpus of {English}: The {Penn}
  {Treebank}.
\newblock {\em Computational Linguistics}, 19(2):313–330.

\bibitem[Martinet, 1970]{Martinet70}
Martinet, B. (1970).
\newblock Br\`eve communication. {R{\'e}gularisation} d’in{\'e}quations
  variationnelles par approximations successives.
\newblock {\em ESAIM: Mathematical Modelling and Numerical Analysis -
  Mod\'elisation Math\'ematique et Analyse Num\'erique}, 4:154--158.

\bibitem[McMahan, 2017]{McMahan17}
McMahan, H.~B. (2017).
\newblock A survey of algorithms and analysis for adaptive online learning.
\newblock {\em Journal of Machine Learning Research}, 18(90):1--50.

\bibitem[McMahan and Streeter, 2010]{McMahanS10}
McMahan, H.~B. and Streeter, M.~J. (2010).
\newblock Adaptive bound optimization for online convex optimization.
\newblock In {\em Proceedings of the 23rd Annual Conference on Learning
  Theory}, pages 244--256. Omnipress.

\bibitem[Meka et~al., 2008]{MekaJCD08}
Meka, R., Jain, P., Caramanis, C., and Dhillon, I.~S. (2008).
\newblock Rank minimization via online learning.
\newblock In {\em Proceedings of the 25th International Conference on Machine
  learning}, pages 656--663. Omnipress.

\bibitem[Merity et~al., 2018]{merity2018regularizing}
Merity, S., Keskar, N.~S., and Socher, R. (2018).
\newblock Regularizing and optimizing {LSTM} language models.
\newblock In {\em Sixth International Conference on Learning Representations}.

\bibitem[Merity et~al., 2017]{wikitext103}
Merity, S., Xiong, C., Bradbury, J., and Socher, R. (2017).
\newblock Pointer sentinel mixture models.
\newblock In {\em Fifth International Conference on Learning Representations}.

\bibitem[Mohri and Yang, 2016]{mohri2016accelerating}
Mohri, M. and Yang, S. (2016).
\newblock Accelerating online convex optimization via adaptive prediction.
\newblock In {\em Proceedings of the 19th International Conference on
  Artificial Intelligence and Statistics, AISTATS}, volume~51, pages 848--856.
  PMLR.

\bibitem[Moreau, 1965]{Moreau65}
Moreau, J.-J. (1965).
\newblock Proximit{\'e} et dualit{\'e} dans un espace hilbertien.
\newblock {\em Bulletin de la Soci{\'e}t{\'e} Math{\'e}matique de France},
  93:273--299.

\bibitem[Nemirovsky and Yudin, 1983]{NemirovskyY83}
Nemirovsky, A.~S. and Yudin, D. (1983).
\newblock {\em Problem complexity and method efficiency in optimization}.
\newblock John Wiley \& Sons.

\bibitem[Nesterov, 1983]{Nesterov83}
Nesterov, Y. (1983).
\newblock A method for unconstrained convex minimization problem with the rate
  of convergence {$O(1/k^2)$}.
\newblock In {\em Doklady Akademii Nauk SSSR}, volume 269, pages 543--547.

\bibitem[Nesterov, 2004]{Nesterov04}
Nesterov, Y. (2004).
\newblock {\em Introductory lectures on convex optimization: A basic course}.
\newblock Springer.

\bibitem[Nesterov, 2015]{Nesterov15b}
Nesterov, Y. (2015).
\newblock Universal gradient methods for convex optimization problems.
\newblock {\em Mathematical Programming}, 152(1):381--404.

\bibitem[Nocedal and Wright, 2006]{NocedalW06}
Nocedal, J. and Wright, S.~J. (2006).
\newblock {\em Numerical Optimization}.
\newblock Springer Series in Operations Research and Financial Engineering.
  Springer, New York.

\bibitem[Orabona and P{\'a}l, 2015]{orabona2015scale}
Orabona, F. and P{\'a}l, D. (2015).
\newblock Scale-free algorithms for online linear optimization.
\newblock In {\em Proceedings of the 26th International Conference on
  Algorithmic Learning Theory}, pages 287--301. Springer.

\bibitem[Orabona and P{\'a}l, 2018]{OrabonaP18}
Orabona, F. and P{\'a}l, D. (2018).
\newblock Scale-free online learning.
\newblock {\em Theoretical Computer Science}, 716:50--69.
\newblock Special Issue on {ALT} 2015.

\bibitem[Parikh and Boyd, 2014]{ParikhB14}
Parikh, N. and Boyd, S. (2014).
\newblock Proximal algorithms.
\newblock {\em Foundations and Trends in optimization}, 1(3):127--239.

\bibitem[Pascanu et~al., 2013]{PascanuMB13}
Pascanu, R., Mikolov, T., and Bengio, Y. (2013).
\newblock On the difficulty of training recurrent neural networks.
\newblock In {\em Proceedings of International conference on machine learning},
  volume~28, pages 1310--1318. PMLR.

\bibitem[Paszke et~al., 2019]{Pytorch19}
Paszke, A., Gross, S., Massa, F., Lerer, A., Bradbury, J., Chanan, G., Killeen,
  T., Lin, Z., Gimelshein, N., Antiga, L., Desmaison, A., Kopf, A., Yang, E.,
  DeVito, Z., Raison, M., Tejani, A., Chilamkurthy, S., Steiner, B., Fang, L.,
  Bai, J., and Chintala, S. (2019).
\newblock {PyTorch}: An imperative style, high-performance deep learning
  library.
\newblock In {\em Advances in Neural Information Processing Systems},
  volume~32, pages 8024--8035. Curran Associates, Inc.

\bibitem[Polyak, 1963]{Polyak63}
Polyak, B.~T. (1963).
\newblock Gradient methods for minimizing functionals.
\newblock {\em Zhurnal Vychislitel'noi Matematiki i Matematicheskoi Fiziki},
  3(4):643--653.

\bibitem[Radford et~al., 2019]{gpt2}
Radford, A., Wu, J., Child, R., Luan, D., Amodei, D., and Sutskever, I. (2019).
\newblock Language models are unsupervised multitask learners.
\newblock {\em OpenAI blog}.

\bibitem[Reddi et~al., 2021]{reddi2021adaptive}
Reddi, S.~J., Charles, Z., Zaheer, M., Garrett, Z., Rush, K.,
  Kone{\v{c}}n{\'y}, J., Kumar, S., and McMahan, H.~B. (2021).
\newblock Adaptive federated optimization.
\newblock In {\em Ninth International Conference on Learning Representations}.

\bibitem[Reddi et~al., 2018]{ReddiKK18}
Reddi, S.~J., Kale, S., and Kumar, S. (2018).
\newblock On the convergence of {Adam} and beyond.
\newblock In {\em Sixth International Conference on Learning Representations}.

\bibitem[Richt{\'a}rik and Tak{\'a}c, 2014]{RichtrikT14}
Richt{\'a}rik, P. and Tak{\'a}c, M. (2014).
\newblock Iteration complexity of randomized block-coordinate descent methods
  for minimizing a composite function.
\newblock {\em Mathematical Programming}, 144:1--38.

\bibitem[Riedmiller and Braun, 1993]{riedmiller1993direct}
Riedmiller, M. and Braun, H. (1993).
\newblock A direct adaptive method for faster backpropagation learning: The
  {RPROP} algorithm.
\newblock In {\em IEEE international conference on neural networks}, pages
  586--591. IEEE.

\bibitem[Robbins and Monro, 1951]{robbins1951stochastic}
Robbins, H. and Monro, S. (1951).
\newblock A stochastic approximation method.
\newblock {\em The Annals of Mathematical Statistics}, 22(3):400--407.

\bibitem[Rockafellar, 1976]{Rockafellar76}
Rockafellar, R.~T. (1976).
\newblock Monotone operators and the proximal point algorithm.
\newblock {\em SIAM Journal on Control and Optimization}, 14(5):877--898.

\bibitem[Rockafellar, 1993]{Rockafellar93Lagrange}
Rockafellar, R.~T. (1993).
\newblock Lagrange multipliers and optimality.
\newblock {\em SIAM Review}, 35(2):183--238.

\bibitem[Schoenholz et~al., 2017]{SchoenholzGGS17}
Schoenholz, S., Gilmer, J., Ganguli, S., and Sohl-Dickstein, J. (2017).
\newblock Deep information propagation.
\newblock In {\em Fifth International Conference on Learning Representations}.

\bibitem[Shalev-Shwartz, 2007]{Shalev-Shwartz07}
Shalev-Shwartz, S. (2007).
\newblock {\em Online Learning: Theory, Algorithms, and Applications}.
\newblock PhD thesis, The Hebrew University.

\bibitem[Shalev-Shwartz and Ben-David, 2014]{shalev2014understanding}
Shalev-Shwartz, S. and Ben-David, S. (2014).
\newblock {\em Understanding machine learning: From theory to algorithms}.
\newblock Cambridge university press.

\bibitem[Shor, 2012]{shor2012minimization}
Shor, N.~Z. (2012).
\newblock {\em Minimization methods for non-differentiable functions}.
\newblock Springer Science \& Business Media.

\bibitem[Singh and Shrivastava, 2019]{SinghS19}
Singh, S. and Shrivastava, A. (2019).
\newblock {EvalNorm}: Estimating batch normalization statistics for evaluation.
\newblock In {\em Proceedings of the IEEE/CVF International Conference on
  Computer Vision}, pages 3633--3641.

\bibitem[Sun and Luo, 2016]{SunL16}
Sun, R. and Luo, Z. (2016).
\newblock Guaranteed matrix completion via non-convex factorization.
\newblock {\em IEEE Transactions on Information Theory}, 62(11):6535--6579.

\bibitem[Sundermeyer et~al., 2012]{sundermeyer2012lstm}
Sundermeyer, M., Schl{\"u}ter, R., and Ney, H. (2012).
\newblock {LSTM} neural networks for language modeling.
\newblock In {\em Proceedings of the 13th Annual Conference of the
  International Speech Communication Association}, pages 194--197.

\bibitem[Tai et~al., 2015]{tai2015improved}
Tai, K.~S., Socher, R., and Manning, C.~D. (2015).
\newblock Improved semantic representations from tree-structured long
  short-term memory networks.
\newblock In {\em Proceedings of the 53rd Annual Meeting of the Association for
  Computational Linguistics and the 7th International Joint Conference on
  Natural Language Processing (Volume 1: Long Papers)}, pages 1556--1566.
  Association for Computational Linguistics.

\bibitem[Tan and Le, 2019]{TanL19}
Tan, M. and Le, Q. (2019).
\newblock {E}fficient{N}et: Rethinking model scaling for convolutional neural
  networks.
\newblock In {\em Proceedings of the 36th International Conference on Machine
  Learning}, volume~97, pages 6105--6114. PMLR.

\bibitem[Toulis and Airoldi, 2017]{ToulisA17}
Toulis, P. and Airoldi, E.~M. (2017).
\newblock Asymptotic and finite-sample properties of estimators based on
  stochastic gradients.
\newblock {\em The Annals of Statistics}, 45(4):1694--1727.

\bibitem[van Erven and Koolen, 2016]{vanErvenK16}
van Erven, T. and Koolen, W.~M. (2016).
\newblock {MetaGrad}: Multiple learning rates in online learning.
\newblock In {\em Advances in Neural Information Processing Systems},
  volume~29, pages 3666--3674. Curran Associates, Inc.

\bibitem[Vaswani et~al., 2017]{VaswaniSPUJGKP17}
Vaswani, A., Shazeer, N., Parmar, N., Uszkoreit, J., Jones, L., Gomez, A.~N.,
  Kaiser, {\L}., and Polosukhin, I. (2017).
\newblock Attention is all you need.
\newblock In {\em Advances in neural information processing systems},
  volume~30, pages 6000--6010. Curran Associates, Inc.

\bibitem[Vaswani et~al., 2019]{VaswaniMLSGLJ19}
Vaswani, S., Mishkin, A., Laradji, I., Schmidt, M., Gidel, G., and
  Lacoste-Julien, S. (2019).
\newblock Painless stochastic gradient: Interpolation, line-search, and
  convergence rates.
\newblock In {\em Advances in Neural Information Processing Systems},
  volume~32, pages 3732--3745. Curran Associates, Inc.

\bibitem[Verhulst, 1845]{verhulst1845resherches}
Verhulst, P.~F. (1845).
\newblock Resherches math{\'e}matiques sur la loi d'accroissement de la
  population.
\newblock {\em Nouveaux m{\'e}moires de l'Acad{\'e}mie Royale des Sciences et
  Belles-Lettres de Bruxelles}, 18:1--41.

\bibitem[Wang et~al., 2021]{WangKQWXZF21}
Wang, Y., Kang, Y., Qin, C., Wang, H., Xu, Y., Zhang, Y., and Fu, Y.~R. (2021).
\newblock Rethinking {Adam}: A twofold exponential moving average approach.
\newblock {\em arXiv preprint arXiv:2106.11514}.

\bibitem[Ward et~al., 2019]{WardWB19}
Ward, R., Wu, X., and Bottou, L. (2019).
\newblock {AdaGrad} stepsizes: Sharp convergence over nonconvex landscapes,
  from any initialization.
\newblock In {\em Proceedings of the 36th International Conference on Machine
  Learning}, volume~97, pages 6677--6686. PMLR.

\bibitem[Warmuth and Jagota, 1997]{WarmuthJ97}
Warmuth, M.~K. and Jagota, A.~K. (1997).
\newblock Continuous and discrete-time nonlinear gradient descent: Relative
  loss bounds and convergence.
\newblock In {\em Electronic Proceedings of the 5th International Symposium on
  Artificial Intelligence and Mathematics}.

\bibitem[Wolf, 2019]{gpt2_naacl}
Wolf, T. (2019).
\newblock {\em Transfer learning in natural language processing}.
\newblock Available at
  \url{https://github.com/huggingface/naacl_transfer_learning_tutorial}.

\bibitem[Wu and He, 2018]{WuH18}
Wu, Y. and He, K. (2018).
\newblock Group normalization.
\newblock In {\em Proceedings of the European conference on computer vision
  (ECCV)}, pages 3--19.

\bibitem[You et~al., 2020]{YouLHX20}
You, J., Leskovec, J., He, K., and Xie, S. (2020).
\newblock Graph structure of neural networks.
\newblock In {\em Proceedings of the 37th International Conference on Machine
  Learning}, volume 119, pages 10881--10891. PMLR.

\bibitem[Yuan et~al., 2019]{YuanYJY19}
Yuan, Z., Yan, Y., Jin, R., and Yang, T. (2019).
\newblock Stagewise training accelerates convergence of testing error over
  {SGD}.
\newblock In {\em Advances in Neural Information Processing Systems},
  volume~32, pages 2604--2614. Curran Associates, Inc.

\bibitem[Zhang et~al., 2020a]{ZhangJFW20}
Zhang, B., Jin, J., Fang, C., and Wang, L. (2020a).
\newblock Improved analysis of clipping algorithms for non-convex optimization.
\newblock In {\em Advances in Neural Information Processing Systems},
  volume~33, pages 15511--15521. Curran Associates, Inc.

\bibitem[Zhang et~al., 2019a]{ZhangDM19}
Zhang, H., Dauphin, Y.~N., and Ma, T. (2019a).
\newblock Fixup initialization: Residual learning without normalization.
\newblock In {\em Seventh International Conference on Learning
  Representations}.

\bibitem[Zhang et~al., 2020b]{ZhangHSJ20}
Zhang, J., He, T., Sra, S., and Jadbabaie, A. (2020b).
\newblock Why gradient clipping accelerates training: A theoretical
  justification for adaptivity.
\newblock In {\em Eighth International Conference on Learning Representations}.

\bibitem[Zhang et~al., 2020c]{zhang2020adaptive}
Zhang, J., Karimireddy, S.~P., Veit, A., Kim, S., Reddi, S., Kumar, S., and
  Sra, S. (2020c).
\newblock Why are adaptive methods good for attention models?
\newblock In {\em Advances in Neural Information Processing Systems},
  volume~33, pages 15383--15393. Curran Associates, Inc.

\bibitem[Zhang et~al., 2019b]{ZhangWZZ19}
Zhang, X., Wang, Q., Zhang, J., and Zhong, Z. (2019b).
\newblock Adversarial autoaugment.
\newblock In {\em Seventh International Conference on Learning
  Representations}.

\bibitem[Zhang et~al., 2017]{ZhangCJ17Speech}
Zhang, Y., Chan, W., and Jaitly, N. (2017).
\newblock Very deep convolutional networks for end-to-end speech recognition.
\newblock In {\em 2017 IEEE International Conference on Acoustics, Speech and
  Signal Processing (ICASSP)}, pages 4845--4849.

\bibitem[Zhao et~al., 2020]{ZhaoJK20}
Zhao, H., Jia, J., and Koltun, V. (2020).
\newblock Exploring self-attention for image recognition.
\newblock In {\em Proceedings of the IEEE/CVF Conference on Computer Vision and
  Pattern Recognition}, pages 10076--10085.

\bibitem[Zheng et~al., 2017]{ZhengNY17Recommend}
Zheng, L., Noroozi, V., and Yu, P.~S. (2017).
\newblock Joint deep modeling of users and items using reviews for
  recommendation.
\newblock In {\em Proceedings of the Tenth ACM International Conference on Web
  Search and Data Mining}, page 425–434. Association for Computing Machinery.

\bibitem[Zhou et~al., 2016]{zhou2016deep}
Zhou, J., Cao, Y., Wang, X., Li, P., and Xu, W. (2016).
\newblock Deep recurrent models with fast-forward connections for neural
  machine translation.
\newblock {\em Transactions of the Association for Computational Linguistics},
  4:371--383.

\bibitem[Zhu, 2018]{Zhu18b}
Zhu, Z. (2018).
\newblock Natasha 2: Faster non-convex optimization than {SGD}.
\newblock In {\em Advances in Neural Information Processing Systems},
  volume~31, pages 2675--2686. Curran Associates, Inc.

\bibitem[Zhuang et~al., 2019]{ZhuangCO19}
Zhuang, Z., Cutkosky, A., and Orabona, F. (2019).
\newblock Surrogate losses for online learning of stepsizes in stochastic
  non-convex optimization.
\newblock In {\em Proceedings of International Conference on Machine Learning},
  volume~97, pages 7664--7672. PMLR.

\bibitem[Zhuang et~al., 2022]{ZhuangLCO21}
Zhuang, Z., Liu, M., Cutkosky, A., and Orabona, F. (2022).
\newblock Understanding {AdamW} through proximal methods and scale-freeness.
\newblock {\em Transactions on Machine Learning Research}.

\bibitem[Zou et~al., 2021]{ZouCLG21}
Zou, D., Cao, Y., Li, Y., and Gu, Q. (2021).
\newblock Understanding the generalization of {Adam} in learning neural
  networks with proper regularization.
\newblock {\em arXiv preprint arXiv:2108.11371}.

\end{thebibliography}
